\documentclass{article}

\usepackage{PRIMEarxiv}

\makeatletter
\renewcommand{\normalsize}{\@setfontsize\normalsize{11.5pt}{13.2pt}}
\normalsize
\makeatother

\usepackage[utf8]{inputenc} 
\usepackage[T1]{fontenc}    
\usepackage{url}            
\usepackage{booktabs}       
\RequirePackage{amsthm,amsmath,amsfonts,amssymb}     
\usepackage{nicefrac}       
\usepackage{microtype}      
\usepackage{lipsum}
\usepackage{fancyhdr}       

\usepackage{physics,mathtools,mathrsfs}

\RequirePackage[authoryear]{natbib}

\RequirePackage[colorlinks,citecolor=blue,urlcolor=blue]{hyperref}
\RequirePackage{graphicx}
\usepackage{subfigure}
\usepackage{caption}
\usepackage{subcaption}  
\usepackage{makecell}
\usepackage{multirow}

\theoremstyle{plain}

\newtheorem{theorem}{Theorem}[section]
\newtheorem{lemma}[theorem]{Lemma}
\newtheorem{corollary}[theorem]{Corollary}
\newtheorem{proposition}[theorem]{Proposition}
\newtheorem{example}{Example}

\theoremstyle{definition}
\newtheorem{assumption}{Assumption} 

\theoremstyle{remark}
\newtheorem{definition}[theorem]{Definition}

\newtheorem{remark}[theorem]{Remark}

\newcommand*{\inner}{\mathtt{in}}

\newcommand{\bbC}{\mathbb{C}}

\newcommand{\bbP}{\mathbb{P}}

\newcommand{\bbR}{\mathbb{R}}

\newcommand{\calE}{\mathcal{E}}

\newcommand{\calH}{\mathcal{H}}

\newcommand{\calN}{\mathcal{N}}

\newcommand{\calR}{\mathcal{R}}

\newcommand{\calX}{\mathcal{X}}

\newcommand{\reg}{\varphi_\lambda}
\newcommand{\rem}{\psi_\lambda}

\newcommand{\R}{\mathbb{R}}
\newcommand{\E}{\mathbb{E}}

\newcommand{\mr}{\mathrm}

\providecommand{\ang}[1]{\left\langle{#1}\right\rangle}

\newcommand*{\gf}{\mathtt{GF}}
\newcommand*{\krr}{\mathtt{KRR}}
\newcommand*{\Risk}{\calR_{d}}
\newcommand*{\TP}{\Theta_{d, \mathbb{P}}}
\newcommand*{\HH}{\calH}

\newcommand{\green}{\color{green}}

\usepackage{xcolor}
\definecolor{myblue}{HTML}{1f77b4}
\definecolor{mygreen}{HTML}{2ca02c}
\definecolor{myorange}{HTML}{ff7f0e}

\usepackage{comment}
\newcommand{\dhnote}[1]{{\color{red}[#1]} } 

\renewcommand{\op}{\mathrm{op}}
\newcommand{\Pp}{\mathbb P}
\newcommand{\PP}{\mathbb P}
\renewcommand{\inner}[2]{\left\langle #1,#2\right\rangle}
\newcommand{\asympp}{\asymp_{d,\Pp}}

\title{
    Learning Curves and Prediction Persistence to Under-Regularization of Spectral Algorithms in Large Dimensions
}

\author{
  Weihao Lu\thanks{Department of Statistics and Data Science, National University of Singapore. Email: \texttt{weihaolu@nus.edu.sg}} 
  \And
  Qian Lin\thanks{Department of Statistics and Data Science, Tsinghua University.  Email: \texttt{qianlin@tsinghua.edu.cn}} 
  \And
  Yingcun Xia\thanks{Department of Statistics and Data Science, National University of Singapore. Email: \texttt{yingcun.xia@nus.edu.sg}}
  \And
  Dongming Huang\thanks{Department of Statistics and Data Science, National University of Singapore. (Corresponding author). \\ Email: \texttt{huang.dongming@nus.edu.sg}} 
}

\begin{document}
\maketitle

\begin{abstract}
    We study how prediction risk changes as regularization is reduced for spectral algorithms when the sample size $n$ and dimension $d$ satisfy $n\asymp d^\gamma$ for fixed $\gamma>0$. For inner-product kernels on the sphere, under a source condition with smoothness $s\geq0$ and a signal nondegeneracy condition, we characterize the conditional excess risk up to constant factors at each fixed admissible scale $\lambda\asymp d^{-u}$, $u>0$, and at interpolation. The resulting risk orders describe the effects of source smoothness, filter qualification, and transitions between spectral blocks. They identify when the conditional excess risk retains the minimax risk order as regularization is weakened from an optimal scale to interpolation, and when an interval before interpolation exhibits a polynomially larger risk. We also compare the kernel risk with a population sequence risk for the same spectral filter, using a deterministic regularization scale bounded below by order $n^{-1}$. Finally, for noninteger $\gamma$, we extend the risk characterization for kernel ridge regression to kernels on general domains under spectral scaling, hyper-contractivity, and tail diagonal concentration conditions, together with the source and signal conditions.
\end{abstract}



\section{Introduction}\label{sec_intro}

Nonparametric regression studies the estimation of an unknown function $f_{\star}:\mathbb{R}^d\to\mathbb{R}$ from $n$ i.i.d. samples $\{(x_i,y_i)\}_{i=1}^n$, where $y_i = f_{\star}(x_i) + \epsilon_i$ and $\epsilon_i$ represents noise. 
The performance of an estimator $\hat{f}$ is typically measured by the \textit{excess risk} 
\[
\|\hat f-f_{\star}\|_{L^2}^2:=\E_x[(\hat f(x)-f_{\star}(x))^2].
\]
Among the most widely studied estimators for this problem are spectral algorithms, including kernel ridge regression (KRR) and kernel gradient flow (KGF) \citep{Caponnetto2006OptimalRF,caponnetto2007optimal}. 
Their statistical properties are commonly studied under source conditions $f_{\star}\in[\mathcal H]^s$, where $[\mathcal H]^s$ denotes the interpolation space associated with {the reproducing kernel Hilbert space (RKHS) $\mathcal H$} and $s>0$ measures the smoothness of the regression function; see Section~\ref{subsec:interpolation_space}. 

In the fixed-dimensional regime, the performance of spectral algorithms is well understood. Under polynomial eigenvalue decay of the kernel operator, properly regularized spectral methods such as KRR and KGF can achieve minimax-optimal excess-risk rates for $0<s\leq 2$ \citep{Caponnetto2006OptimalRF, caponnetto2007optimal, raskutti2014early, Yao2007OnES, Lin_Optimal_2020, zhang2023optimality}. 
For smoother signals, however, this picture changes: KRR exhibits the \emph{saturation effect} and fails to attain the optimal rate \citep{bauer2007_RegularizationAlgorithms, Gerfo_spectral_2008, dicker2017kernel, li2022saturation}, whereas KGF can still remain rate-optimal. At the opposite extreme of zero regularization, kernel interpolation can be inconsistent in fixed dimension \citep{pmlr-v99-rakhlin19a,beaglehole2022kernel,buchholz2022_KernelInterpolation,li2023kernel}.

The picture changes substantially in large dimensions. We consider the regime in which \(d\) and \(n\) grow jointly, with $n\asymp d^\gamma, \gamma>0$,
and the covariates are distributed on or near the sphere \(\mathbb S^{d-1}\).
Recent work has characterized the spectra of empirical inner-product kernels, identified polynomial-approximation barriers, and established sharp results on minimax optimality, saturation, and Pinsker-type risk bounds for large-dimensional kernel methods
\citep{Liang_Just_2019,ghorbani2021linearized,mei2022generalization,lu2023optimal,zhang2024optimal,lu2024saturation,lu2024pinsker}.
Most strikingly, unlike in the classical fixed-dimensional picture, kernel interpolation can generalize well and can even attain the minimax-optimal prediction rate for suitable relationships between \(s\) and \(\gamma\).

This behavior is closely related to a phenomenon that has motivated much of the modern study of overparameterized learning. Empirically, highly expressive neural networks can often be trained until the training error is nearly zero without a corresponding deterioration in test performance
\citep{zhang2021understanding,zhang2016understanding,belkin2018understand,belkin2019does,belkin2021fit,nakkiran2021deep}.
In particular, good prediction need not be confined to a carefully chosen early-stopping time: continued training toward interpolation can leave test performance largely unchanged. Through the connection between wide neural networks and kernel methods in the lazy-training regime \citep{Jacot_NTK_2018,Arora_on_2019,Hu_Regularization_2021}, this observation raises a natural theoretical question for spectral algorithms:
\begin{center}
\textit{
    If a regularized estimator has already reached its optimal prediction rate, what happens when regularization is weakened further and training continues toward interpolation?
    }
\end{center}

Existing analyses describe several aspects of this path: stages of KGF training, KRR bounds across regularization levels, and deterministic approximations for ridge risk. We review these results below. Here we study how prediction persistence depends jointly on source smoothness, the spectral filter, and the relation between sample size and dimension, by comparing the conditional risk orders across regularization scales.

We use \emph{learning curve} for the asymptotic order of the conditional excess risk as a function of the regularization exponent $u$, where $\lambda\asymp d^{-u}$ and $u>0$. We fix $\gamma$, the source smoothness $s$, the filter qualification $\tau$ defined in Section~\ref{sec_settings}, and the target sequence; $u=\infty$ denotes interpolation at $\lambda=0$. The conditional risk averages over the training noise for a given training design, as defined in~\eqref{eq bias var decomposition}. A conventional learning curve indexed by sample size instead varies the number of observations.

For inner-product kernels on the sphere, we derive a common expression for the conditional risk orders of the spectral algorithms considered here. Comparing these orders characterizes prediction persistence and deterioration before interpolation. We also relate the risk orders to a population sequence model and extend the KRR analysis to a structured class of kernels on general domains.

Our main contributions are as follows.
\begin{itemize}
    \item Theorem~\ref{theorem_learn_curve} gives matching upper and lower orders for the conditional excess risk of spectral algorithms on $\mathbb{S}^{d-1}$ under the source and signal nondegeneracy conditions. Its formula describes the over-regularized, under-regularized, and interpolation regimes, including the dependence on filter qualification and transitions between spectral blocks.

    \item Theorem~\ref{coroll_prediction_persistence} characterizes prediction persistence beyond optimal regularization. For $0<s\leq\Gamma(\gamma)$, where $\Gamma(\gamma)$ is defined in~\eqref{eqn_threshold_s}, the conditional risk matches the minimax risk order from an optimal scale through weaker polynomial scales to interpolation. For $s>\Gamma(\gamma)$, an interval before interpolation has a polynomially larger risk. The interpolation phase boundary was established by \citet{zhang2024phase}; the path conclusion follows here by comparing the risk orders in Theorem~\ref{theorem_learn_curve}.

    \item Proposition~\ref{prop_learning_curve_of_sequence} compares the kernel risk with the population sequence risk for the same spectral filter up to constant factors. It uses a deterministic admissible regularization scale of order $\max\{\lambda,n^{-1}\}$, which yields a tractable description of the kernel risk across the scales considered here.

    \item Theorem~\ref{theorem_learn_curve_general} extends the KRR risk orders to kernels on general domains for noninteger $\gamma$, under spectral scaling, hyper-contractivity, and tail diagonal concentration conditions, together with the source and signal conditions.
    
\end{itemize}

The remainder of the paper is organized as follows. Section~\ref{sec_settings} introduces the kernel model, interpolation spaces, and spectral algorithms used throughout the paper. Section~\ref{sec_main_results} states the learning curves and prediction persistence results for inner-product kernels on the sphere.
Section~\ref{sec_general} extends the analysis to KRR on general domains, and Section~\ref{sec_experiments} reports numerical experiments.

\subsection{ Literature review}
Recent research has renewed its focus on spectral algorithms in large-dimensional settings, where the sample size $n$ scales with the data dimension $d$ as $n \asymp d^{\gamma}$ for some $\gamma > 0$, and the input distributions are defined on (or near) the sphere  $\mathbb{S}^{d-1}$ (e.g., the standard normal distribution).
For instance, \cite{Karoui_spectrum_2010} investigated the spectral properties of empirical inner-product kernels when the inputs are drawn from normal distributions with $\gamma = 1$. 
A series of studies \citep{Liang_Just_2019, liang2020multiple, Ghorbani_When_2021, ghorbani2021linearized, mei2021learning, mei2022generalization, misiakiewicz_learning_2021} established polynomial approximation barriers and risk approximations for KRR and related kernel methods. \citet{Ghosh_three_2021} further described multiple stages of KGF training on the sphere, including stages in which training error decreases while prediction error remains approximately constant. Their additive approximations apply to targets satisfying the stated moment and concentration conditions; a remainder of order $o_{d,\mathbb P}(1)$ alone need not determine a specified vanishing polynomial risk rate.

Other works \citep{lu2023optimal, zhang2024optimal, lu2024saturation} addressed the minimax rate-optimality and saturation effect of spectral algorithms, while \cite{lu2024pinsker} established the Pinsker bound (the precise asymptotic minimax risk bound) for kernel regression problems. 
These works analyze optimal tuning and the dependence of the optimal rates on filter qualification. For interpolation, \citet{Liang_Just_2019,liang2020multiple,zhang2024phase} analyze generalization in large dimensions; \citet{zhang2024phase} give the spherical interpolation risk orders and the corresponding minimax phase boundary.

Results at positive regularization provide additional comparisons. \citet{barzilai2023generalization} give general KRR upper bounds at arbitrary regularization levels, as well as a separate spherical interpolation specialization. \citet{aerni2023strong} study bias and variance for convolutional kernels on the hypercube with specified polynomial targets, over a polynomial range of ridge parameters including interpolation. For general covariance spectra, \citet{tsigler2023benign} give nonasymptotic ridge bounds across regularization levels. Substituting spherical spectral blocks and source energies into their scalar risk expressions yields the corresponding KRR polynomial orders; the distribution, target, and cutoff conditions for their probability bounds are discussed in Appendix B.3.1.


In the fixed-dimensional regime, several studies \citep{Bordelon_Spectrum_2020, Cui2021GeneralizationER, jin2021learning, li2023asymptotic, cheng2024comprehensive} demonstrated an exact bias-variance trade-off, yielding U-shaped learning curves for the excess risk of KRR. 
Moreover, \cite{li2024generalization} emphasized the importance of studying the learning curves of KGF and related spectral algorithms, particularly because neural networks are typically trained via gradient descent. By developing an analytic functional framework, \cite{li2024generalization} derived U-shaped excess risk curves for fixed-dimensional analytic spectral algorithms.
Their operator-concentration approach requires a lower bound on $\lambda$. In the large-dimensional setting considered here, the corresponding approximations are available in the over-regularized range, so additional arguments are needed to analyze weaker regularization and interpolation.

\citet{cheng2022dimension} derive relative approximations for the conditional bias and variance of ridge regression through an effective sequence model with adjusted regularization and noise. \citet{misiakiewicz2024non} establish nonasymptotic deterministic equivalents for KRR under quantitative concentration and residual moment conditions, including a relative approximation for the prediction risk that can capture vanishing risk rates. Their analysis also includes a specialization to spherical inner-product kernels.

The assumptions of these results are distinct from the source and signal conditions considered here. In particular, the spherical specialization of \citet{misiakiewicz2024non} requires a nonvanishing fraction of the target energy above a fixed spectral degree, whereas this fraction vanishes for the nondegenerate source-condition sequences considered here when $s>0$. Their general deterministic-equivalent theorem has a broader scope, but its application to our setting requires separate verification of the residual moment, concentration, and quantitative approximation conditions. We discuss these conditions in Appendix B.3.1.


Theorem~\ref{theorem_learn_curve} gives a common expression for the conditional risk orders across the spectral filters considered here, with explicit dependence on regularization, smoothness, and qualification. Theorem~\ref{coroll_prediction_persistence} describes the persistence and deterioration obtained by comparing these orders. The earlier ridge rates, optimal tuning results, and interpolation threshold provide direct points of comparison.

A comprehensive comparison between our results and the existing literature is provided in Appendix B. The comparison of ridge bounds and effective sequence models is given in Appendix B.3.1; Remark~\ref{remark_compare_with_mei_mis} discusses general-domain kernels, and Appendix B also records the interpolation risk orders.

\subsection{Notations}\label{subsec_notations}
For any integer $q \geq 2$, we denote the norm in $L^q:=L^q(\mathcal{X}, \rho_{\mathcal{X}})$ by $\|\cdot\|_{L^q}$. 
For \( x \in [0, \infty] \), denote by \( \lfloor x \rfloor \) the greatest integer less than or equal to \( x \) if \( x < \infty \), with the convention  \( \lfloor x \rfloor = \infty \) if \( x = \infty \). Similarly, let \( \lceil x \rceil \) denote the smallest integer greater than or equal to \( x \) if \( x < \infty \), and \( \lceil x \rceil = \infty \) if \( x = \infty \).

We use asymptotic notations $O_d(\cdot),~o_d(\cdot),~\Omega_d(\cdot)$ and $\Theta_d(\cdot)$ to describe the behavior of non-negative sequences as $d \to \infty$.
For instance, we say two sequences of (deterministic) non-negative quantities $\{U(d), V(d)\}_{d=1}^{\infty}$ satisfy $U(d) =  o_{d}(V(d))$ if and only if for any $\varepsilon > 0$, there exists a constant $D_{\varepsilon}$ that depends only on $\varepsilon$ and the absolute positive constants listed in Definition A.1, such that for any $d > D_{\varepsilon}$, we have $U(d)< \varepsilon V(d)$. 
We also use the notation $V(d) \gg U(d)$ and $U(d)\ll V(d)$ if $U(d) = o_{d}(V(d))$.

For random variables, the probabilistic asymptotic notations $O_{d, \mathbb{P}}(\cdot), o_{d, \mathbb{P}}(\cdot), \Omega_{d, \mathbb{P}}(\cdot), \Theta_{d, \mathbb{P}}(\cdot)$ are also used. For instance, we say two sequences of random variables $\{X(d), Y(d)\}_{d=1}^{\infty}$ satisfy $ X(d) =  O_{d, \mathbb{P}}(Y(d)) $ if and only if for any $\delta > 0$, there exist two constants $C_{\delta} $ and $ D_{\delta}$ that depend only on $\delta$ and the absolute positive constants listed in Definition A.1,  such that for any $d > D_{\delta}$, we have $\mathbb{P}\left( |X(d)| > C_{\delta} |Y(d)| \right) < \delta$.

{For a symmetric matrix \( A\) of dimension \( k \), and any positive integer \( j \leq k \), let \( \lambda_j(A) \) denote the \( j \)-th largest eigenvalue of \( A \).} 
In particular, let \( \lambda_{\text{max}}(A) \) and \( \lambda_{\text{min}}(A) \) denote the largest and smallest eigenvalues of \( A \), respectively. 
For another symmetric matrix \( \tilde{A} \) of the same dimension, we write $\tilde{A}\geq A$, if $\tilde{A}-A$ is positive semi-definite. 
We write $\tilde{A}=\Omega_{d, \mathbb{P}}(1) A$ if and only if for any $\delta > 0$, there exist two constants $C_{\delta}>0$ and $ D_{\delta}$ that depend only on $\delta$ and the absolute positive constants listed in Definition A.1,  such that for any $d > D_{\delta}$, we have $\tilde{A} \geq C_{\delta}A$ with probability at least $1-\delta$.
For any non-symmetric matrix $U$, define $\|U\|_{2} := \sqrt{\|UU^{\top}\|_{2}}$.

{Throughout, for an event $\calE$ and $A \in \mathbb{R}\cup\{-\infty,+\infty\}$, expressions of the form $\mathbf 1\{\calE\}A$ are understood to be zero on $\calE^c$, even when $A$ is infinite.}

\section{Problem setting on the sphere}\label{sec_settings}

Suppose that we observe $n$ i.i.d. samples $(x_i, y_i)$ for $i =1,2,\ldots, n$ from the model:
\begin{equation}\label{equation:true_model}
    y=f_{\star}(x)+\epsilon,
\end{equation}
where $x_i$'s are sampled from $\rho_{\calX}$, which is the uniform distribution on $\mathcal{X} =\mathbb{S}^{d-1}\subset \bbR^{d}$,
$y \in \mathcal{Y} \subset \mathbb{R}$,
$f_{\star}$ is the regression function defined on $\mathcal{X}$, and  $\epsilon_1, \ldots, \epsilon_n$ are conditionally independent given $(x_1, \ldots, x_n)$, with conditional mean 0 and variance $\sigma^2 \in (0, \infty)$. 
Sections~\ref{sec_settings} and \ref{sec_main_results} focus on this spherical design, whereas Section~\ref{sec_general} later considers KRR on general domains. 
In high dimensions, standardized covariates are often concentrated near a thin spherical shell \citep{talagrand1996new}, so the spherical model is a natural reference setting. 
The spherical setting is also analytically convenient, because inner-product kernels admit an explicit spectral block structure that can be analyzed sharply.

We denote the response vector by $Y=\left(y_1, \ldots, y_n\right)^{\top} \in \mathbb{R}^n$ and the design matrix by $X=\left[x_1, \ldots, x_n\right]^{\top} \in \mathbb{R}^{n \times d}$.
In this work, we consider the large-dimensional setting where the dimension $d$ grows with the sample size  $n$. 
Specifically, the sample size and the dimension satisfy the following assumption:

\begin{assumption}\label{assump_asymptotic}
    There exist positive constants $c_1$, $c_2$, and $\gamma$ such that the sample size satisfies
\begin{align}\label{Asym}
    c_{1} d^{\gamma} \leq n \leq c_{2} d^{\gamma},
\end{align}
and we denote $\ell_{\gamma} := \lfloor \gamma\rfloor$.
\end{assumption}

\subsection{Inner product kernels}\label{subsec_sphere_kernel}

{In the spherical setting of Sections~\ref{sec_settings} and \ref{sec_main_results}, we consider the following inner product kernels.}
A continuous inner product kernel $\mathscr{K}$ defined on $\mathbb{S}^{d-1}$ is given by 
\begin{equation*}
    \mathscr{K}(x, x^\prime) = \Phi(\left\langle x, x^\prime \right\rangle), \forall~ x, x^\prime \in \mathbb{S}^{d-1},
\end{equation*}
where  $\Phi:[-1,1] \to \mathbb{R}$ is a continuous function satisfying the following assumption. 

\begin{assumption}\label{assu:coef_of_inner_prod_kernel}
$\Phi(t) \in \mathcal{C}^{\infty} \left([-1,1]\right)$ is a fixed function (independent of $d$) and there exists a non-negative sequence of absolute constants $\{a_i\}_{i = -1}^{\infty}$ such that
    \begin{displaymath}
        \Phi(t) = \sum\nolimits_{i=0}^\infty a_i t^i
        \quad \text{ and } \quad 
        \sum\nolimits_{i=0}^\infty a_i \leq a_{-1}, 
    \end{displaymath}
    where the constant $a_{-1}$ denotes an upper bound on the sum of coefficients. 
    Furthermore, we require  
    $a_{i} > 0$ for any $0 \leq i \leq \ell_{\gamma} + 2$, where $\ell_{\gamma}$ is defined in Assumption \ref{assump_asymptotic}.
\end{assumption}


The primary purpose of Assumption \ref{assu:coef_of_inner_prod_kernel} is to ensure the positive definiteness of the kernel and to derive the asymptotic order for the 
eigenvalues $\mu_0, \cdots, \mu_{\ell_{\gamma} + 2}$ defined later in (\ref{spherical_decomposition_of_inner_main}). 
First, Theorem 1.b in \cite{gneiting2013strictly} states that an inner product kernel on the sphere is positive definite in all dimensions if and only if all coefficients $a_{i}$ are non-negative. 
Second, the strict positivity of the coefficients $a_k$ for $0 \leq k \leq \ell_{\gamma} + 2$ implies that the eigenvalues  satisfy $\mu_k = \Theta_d( d^{-k} )$ for all $k$ in this range.
%
%

The analyticity of $\Phi$ is a standard assumption in the related literature (e.g., \cite{liang2020multiple, xiao2022precise, hu2022sharp, lu2023optimal, zhang2024optimal, lu2024saturation, zhang2024phase, lu2024pinsker}). Indeed, many commonly used inner product kernels defined on $\mathbb{S}^{d-1}$ have an analytic $\Phi$:
\begin{itemize}

    \item Gaussian kernel: When defined on the sphere, this kernel is analytic with positive coefficients.
    
    \item Random feature kernel: 
    A random-feature kernel satisfying Assumption \ref{assu:coef_of_inner_prod_kernel} is obtained by taking
$w\sim \calN(0,I_d)$ and $\sigma(t)=e^t$. For $x,x'\in \mathbb{S}^{d-1}$,
\[
\begin{aligned}
\mathscr{K}(x,x')
&=\E_w[\sigma(\langle w,x\rangle)\sigma(\langle w,x'\rangle)]\\
&=\exp\{\|x+x'\|^2/2\}
 =\exp\{1+\langle x,x'\rangle\}.
\end{aligned}
\]
Thus $\Phi(t)=e^{1+t}$, with $a_k=e/k!>0$ for all $k\ge0$ and
$\sum_{k\ge0}a_k=e^2$. 
    
\end{itemize}

Define the integral operator as
\begin{equation}\label{def_T}
    T(f)(x)=\int \mathscr{K}(x, x^{\prime}) f(x^{\prime}) ~\mathsf{d} \rho_{\calX}(x^{\prime}). 
\end{equation}
By Theorem 1.b in \cite{gneiting2013strictly}, Assumption \ref{assu:coef_of_inner_prod_kernel} implies that the kernel $\mathscr{K}$ is positive definite. Moreover, both $\mathscr{K}$ and the trace of $T$ are bounded:
\begin{equation*}
\begin{aligned}
    K_{\max}&:=\sup_{x \in \mathcal X} \mathscr{K}(x, x) 
    = \sup_{x \in \mathcal X} \Phi(\left\langle x, x \right\rangle) 
    \leq \sum\nolimits_{i=0}^{\infty} a_i \leq a_{-1}\\
    \mathrm{tr}(T)&:=\int \mathscr{K}(x, x)  ~\mathsf{d} \rho_{\calX}(x) \leq a_{-1}. 
\end{aligned}
\end{equation*}
Consequently, standard results (e.g., \cite{steinwart2012mercer}) ensure that $T$ is a positive, self-adjoint, trace-class, and compact operator.  
{Mercer's theorem then gives the decomposition}
\begin{align}\label{eqn:mercer_decomp}
    \mathscr{K}(x, x^{\prime})=\sum_{j=1}^{\infty}\lambda_{j}\phi_{j}(x)\phi_{j}( x^{\prime}),
\end{align}
where the eigenvalues satisfy $\lambda_{1} \geq \lambda_{2} \geq \ldots \geq 0$, and the corresponding eigenfunctions are $\{\phi_{j}(\cdot)\}_{j\geq 1}$. 
Furthermore, since $\mathscr{K}$ is an inner product kernel defined on the sphere, the Funk-Hecke formula provides a more explicit spectral decomposition:
\begin{equation}\label{spherical_decomposition_of_inner_main}
\begin{aligned}
\mathscr{K}(x,x^\prime) = \sum_{k=0}^{\infty} \mu_{k} \sum_{j=1}^{N(d, k)} \psi_{k, j}(x) \psi_{k, j}\left(x^\prime\right),
\end{aligned}
\end{equation}
where for each $k\geq 0$, $\mu_{k}$ is the degree-$k$ block eigenvalue with multiplicity $N(d, k)$, and 
$\{\psi_{k, j}(\cdot): j=1,\ldots, N(d, k)\}$ are spherical harmonic polynomials of degree $k$. 
The dimension of the degree-$k$ subspace is given by $N(d,0)=1$ and $N(d, k) = \frac{2k+d-2}{k} \cdot \frac{(k+d-3)!}{(d-2)!(k-1)!}$ for $k\geq 1$. 
Notably, in contrast to the ordered sequence $\{\lambda_j\}$, the sequence of $\{\mu_{k}\}_{k \geq 0}$ is not necessarily non-increasing. 
For more details on inner product kernels, see, e.g.,  \citep{gallier2009notes, ghorbani2021linearized}.

The following proposition characterizes the asymptotic behavior and ordering of the leading eigenvalues $\lambda_j$. Its proof is provided in Appendix J.5.

\begin{proposition}\label{prop:inner_edr}
Suppose Assumptions~\ref{assump_asymptotic} and
\ref{assu:coef_of_inner_prod_kernel} hold, with fixed $\gamma>0$. Therefore, $\Phi(t) \in \mathcal{C}^{\infty} \left([-1,1]\right)$ is a fixed function independent of $d$.
Then the following statements hold.
\begin{itemize}
\item For every integer $0\leq k\leq\ell_{\gamma}+2$,
\[
\mu_k=\Theta_d(d^{-k}),\qquad N(d,k)=\Theta_d(d^k).
\]
In particular, writing $N_{-1}=0$ and
$N_k=\sum_{k'=0}^k N(d,k')$, we have $N_k=\Theta_d(d^k)$
for $0\leq k\leq\ell_{\gamma}+2$.
\item There exists a constant $\mathfrak C>0$, depending only on
$\gamma,a_{-1},a_0,\ldots,a_{\ell_{\gamma}+2}$, such that for
$d\geq\mathfrak C$ and every $0\leq k\leq\ell_{\gamma}+1$,
\[
\lambda_{N_{k-1}+1}=\cdots=\lambda_{N_k}=\mu_k,
\qquad
\{\phi_{N_{k-1}+1},\ldots,\phi_{N_k}\}
=\{\psi_{k,1},\ldots,\psi_{k,N(d,k)}\},
\]
where the second equality means that the eigenfunctions may be chosen
as the indicated spherical-harmonic basis (up to an orthogonal change
of basis within each eigenspace).
\end{itemize}
\end{proposition}

Proposition \ref{prop:inner_edr} reveals the spectral block structure of large-dimensional inner product kernels. 
In comparison, most fixed-dimensional kernels do not have such properties. We list some examples:
\begin{itemize}
    \item  For Sobolev kernels defined on a bounded domain with smooth boundary with smoothness $r>d/2$, we have $C_1(d) j^{-2r/d} \leq \lambda_j \leq C_2(d) j^{-2r/d}$ \citep{edmunds1996function, zhang2023optimality}.

    \item {For the rectified linear unit (ReLU) neural tangent kernel (NTK) defined on $\mathcal{S}^{d-1}$, we have} $C_1(d) j^{-d/(d-1)} \leq \lambda_j \leq C_2(d) j^{-d/(d-1)}$ \citep{Bietti_on_2019, Bietti_deep_2021}.

    \item  For Gaussian kernels defined on $\mathcal{S}^{d-1}$, we have\\
    $C_1(d) j^{1/(d-1)}\log j\leq \log(1/\lambda_j) \leq C_2(d) j^{1/(d-1)}\log j$ \citep{Minh_Mercer_2006}.
   
\end{itemize}

In the large-dimensional setting considered throughout this paper, we assume $d \geq \mathfrak{C}$ such that the results of Proposition \ref{prop:inner_edr} hold.

\subsection{The interpolation space}\label{subsec:interpolation_space}

Enumerate the strictly positive population eigenvalues by an index set $J_d$, which may be finite. Define
\[
 [\HH]^s=\left\{f=\sum_{j\in J_d}f_j\phi_j:
                 \sum_{j\in J_d}\lambda_j^{-s}f_j^2<\infty\right\}.
\]
In particular, $[\HH]^0$ is the $L^2$ closure of the positive spectral subspace. 
Generally speaking, functions in $[\mathcal{H}]^s$ become smoother as $s$ increases (see, e.g., the example of Sobolev spaces in \cite{edmunds1996function, zhang2023optimality}).

Standard kernel regression analysis typically assumes the well-specified case where the true regression function $f_{\star}$ falls into the RKHS $\calH$ (e.g., \cite{Caponnetto2006OptimalRF, caponnetto2007optimal, Yao2007OnES, raskutti2014early, Liang_Just_2019, lu2023optimal}). 
However, subsequent research has suggested that the RKHS assumption might be too restrictive, prompting interest in the performance of kernel regression in the misspecified case with $s \in (0, 1)$ (see \citet{fischer2020_SobolevNorm}, \citet{zhang2023optimality,zhang2023optimality_2}, \citet{zhang2024optimal}, and \citet{lu2024pinsker,lu2024saturation}). 
Recently, several studies on large-dimensional kernel regression have considered the extreme case where $s=0$ (e.g., \cite{ghorbani2021linearized, mei2021learning, mei2022generalization, misiakiewicz_learning_2021}).
To provide a unified analysis of large-dimensional kernel regression, 
{we assume that $f_{\star}$ falls in a ball of radius $\sqrt{R}$ in $[\calH]^{s}$:}

\begin{assumption}\label{assump_function_calss}
There exist absolute constants $s \geq 0$ and $R>0$, such that we have
    \begin{equation}\label{eqn:function_calss}
    f_{\star} \in 
    \left\{f \in [\calH]^{s} \mid \|f\|_{[\calH]^{s}} \leq \sqrt{R}\right\}.
\end{equation}
\end{assumption}

From Assumption~\ref{assump_function_calss} and Proposition \ref{prop:inner_edr}, the regression function admits the following decomposition in \(L^2\): 
\begin{equation}\label{eqn_decompose_reg_func}
    f_{\star}(x) = \sum_{k=0}^{\ell_{\gamma}+1} \sum_{j=1}^{N(d,k)} \theta_{k,j} \psi_{k,j}(x) + \sum_{j\in J_d, j > N_{\ell_{\gamma}+1}} f_j \phi_{j}(x),
\end{equation}
where the second series converges in \(L^2\), and $\sum_{k=0}^{\ell_{\gamma}+1} \sum_{j=1}^{N(d,k)} \mu_k^{-s} \theta_{k,j}^2+ \sum\nolimits_{j\in J_d, j > N_{\ell_{\gamma}+1}} \lambda_j^{-s} f_j^2 = \|f\|_{[\calH]^{s}}^2 \leq R$.
Throughout, \(f_{\star}(X)\) denotes the evaluation of a fixed measurable representative of the \(L^2\)-equivalence class \(f_{\star}\).

\subsection{Spectral algorithms}

Below we recall the definition of analytic filter functions (see \cite{bauer2007_RegularizationAlgorithms, li2024generalization, lu2024saturation} for further details).
KRR, KGF, iterated ridge, and kernel gradient descent are all covered by this filter framework. 

\begin{definition}[Analytic filter functions]
  \label{def:filter}
  Let $\left\{ \reg : (0, K_{\max}] \to \R_{\geq 0} \mid \lambda \in\Lambda\subset(0,1] \right\}$ be a family of analytic functions indexed by the regularization parameter $\lambda$.
  We call $\left\{ \reg \mid \lambda \in\Lambda\right\}$ an analytic filter family with qualification $\tau \in [1,\infty]$, if there exist positive absolute constants $C_1$ and $C_2$, such that:
  \begin{enumerate}
      \item 
      For all $\lambda \in\Lambda$ and $z\in(0,K_{\max}]$, we have $z \reg(z) \in [0, 1]$ and the function $\rem(z):= 1- z\reg(z)$ is non-increasing in $z$.
      
      \item 
        For any $z\in(0,K_{\max}]$ and any $t \in [0,\tau]$ for finite $\tau$, or any finite $t \geq 0$ for $\tau=\infty$,
        there exists a constant $\mathfrak{C}_3$ only depending on $t$, such that 
\begin{equation}\label{eq:Filter_case1}
    \begin{aligned}
        &~ C_1 (z+\lambda)^{-1} \leq \reg(z) \leq C_2 (z+\lambda)^{-1}\\
        &~ \rem(z) \leq \mathfrak{C}_3 \lambda^{t}z^{-t}.
    \end{aligned}
\end{equation}

        \item If $\tau<\infty$, then for any constant $\kappa_1 \in (0, K_{\max}]$, there exists a positive constant $\mathfrak{C}_{4}$ only depending on $\tau$ and $\kappa_1$, such that we have
    \begin{align}
    \label{eq:Filter_Rem_finite_case2}
        \rem(\kappa_1) \geq \mathfrak{C}_{4} \lambda^{\tau}.
    \end{align}

    \item If $\tau = \infty$, then for any constant $\kappa_2>0$, there exists a positive constant $\mathfrak{C}_{5}$ only depending on $\tau$ and $\kappa_2$, such that for any $z \leq \kappa_2 \lambda$, we have
            \begin{align}
    \label{eq:Filter_Rem_infinite_case}
        \rem(z) \geq \mathfrak{C}_{5}.
    \end{align}

    \item 
    
    When $s>0$, we also assume that the analytic filter function $\reg$ satisfies the remaining conditions listed in Appendix C of \cite{lu2024saturation} (i.e., equations (30), (31), (C1), and (C2) in \cite{lu2024saturation}). For convenience, these requirements are restated in Appendix K.1.
    
        
  \end{enumerate}
  \end{definition}


\begin{remark}
    The conditions (C1)-(C2) restated in Appendix K.1 essentially require that $\reg$ admits an extension to a larger domain, while preserving all relevant properties that hold on the original domain. This allows us to apply the ``analytic functional argument'' developed by \cite{li2024generalization}.
\end{remark}

Define the matrix $K$ as $(K)_{i, j} := \mathscr{K}(x_i, x_j)$. Since \(K/n\) is real symmetric positive semi-definite, write $K/n = Q\text{diag}(\widehat{\lambda}_1, \cdots, \widehat{\lambda}_n)Q^{\top}$, where $Q$ is an orthogonal matrix. We define
\begin{equation}\label{eqn_def_filter_matrix_function}
    \reg(K/n) := Q\text{diag}(\reg(\widehat{\lambda}_1), \cdots, \reg(\widehat{\lambda}_n))Q^{\top}. 
\end{equation}
The spectral estimator for any $\lambda\in\Lambda$ is then given by 
\begin{align}
  \label{eq:SA}
  \hat{f}_{\lambda}(\cdot) = \frac{1}{n}\mathscr{K}(\cdot, X) \reg\left(\frac{1}{n}K\right) Y,
\end{align}
see, e.g., \cite{bauer2007_RegularizationAlgorithms, li2024generalization}. 

On zero empirical eigenvalues, use any fixed finite value for $\varphi_\lambda(0)$ when $\lambda>0$; the estimator is independent of this value because every Gram-null vector also annihilates $\mathscr{K}(\cdot,X)$. For the classical filters one may use the continuous extension. At $\lambda=0$, set $\varphi_0(z)=z^{-1}$ for $z>0$ and $\varphi_0(0)=0$. Then
\[
 \widehat f_0(\cdot)=\mathscr{K}(\cdot, X) K^\dagger Y.
\]

{In the spherical setting of Sections~\ref{sec_settings} and \ref{sec_main_results}, Lemma~\ref{lemma_invert_of_K} shows that $K$ is invertible almost surely for every sufficiently large $d$. On this event, $\widehat f_0(\cdot)$ fits the training responses exactly.}

\begin{lemma}\label{lemma_invert_of_K}
Under Assumptions 
\ref{assump_asymptotic} and 
\ref{assu:coef_of_inner_prod_kernel}, $K$ is invertible almost surely for every sufficiently large $d$.
\end{lemma}

Throughout, we assume that the regularization parameter $\lambda$ follows a polynomial scaling in the dimension $d$. 
\begin{assumption}\label{assume_lambda}
    There exist constants $c_3,c_4>0$ and a constant exponent $u \in (0, \infty]$, such that we have
    $$
    c_3 d^{-u} \leq \lambda \leq c_4 d^{-u},
    $$
    where the case $u = \infty$ corresponds to $\lambda=0$ (i.e., the interpolation limit). Define  $\ell_{\lambda} = \mathbf{1}\{u < \infty\}\lfloor u \rfloor + \mathbf{1}\{u = \infty\}\infty$ and $\tilde{\ell}=\mathbf{1}\{u < \infty\}\min\{\ell_{\gamma}, \ell_{\lambda}\} + \mathbf{1}\{u = \infty\}\ell_{\gamma}$.
\end{assumption}

\begin{remark}
For a discrete regularization index set $\Lambda$, statements under
Assumption \ref{assume_lambda} concern admissible sequences $\lambda = \lambda(n, d)\in\Lambda$ with the
specified polynomial scaling. Having zero as an accumulation point alone
does not assert the existence of such a sequence for every exponent $u$.
The construction in Example 4 below supplies an admissible sequence for every
fixed $u>0$ for gradient descent.
\end{remark}

\subsubsection{Examples of analytic spectral algorithms}\label{subsec_example_spe}

We now illustrate the concept of analytic filter functions using several classical spectral algorithms.


\begin{example}[Kernel ridge regression]
  \label{example:KRR}
For kernel ridge regression (KRR), the filter is  
  \begin{align*}
    \reg^{\krr}(z) = \frac{1}{z+\lambda},\quad \rem^{\krr}(z) = \frac{\lambda}{z+\lambda},\quad \tau=1.
  \end{align*}
\end{example}

\begin{example}[Kernel gradient flow]
\label{example:gradient_flow}
For kernel gradient flow with time parameter $\lambda^{-1}$, the filter is
  \begin{align*}
    \reg^{\gf}(z) = \frac{1-e^{-z / \lambda}}{z},\quad \rem^{\gf}(z) = e^{-z / \lambda}, \quad \tau=\infty.
  \end{align*}
\end{example}

\begin{example}[Iterated ridge regression]
  \label{example:IteratedRidge}
  Let $q = 1, 2, \cdots$ be a fixed integer.
  The $q$-fold iterated ridge estimator corresponds to the filter 
  \begin{align*}
    \reg^{\mathrm{IT}, q}(z) = \frac{1}{z}\left[ 1 - \frac{\lambda^q}{(z+\lambda)^q} \right],
    \quad
    \psi^{\mathrm{IT}, q}_\lambda(z) =\frac{\lambda^q}{(z+\lambda)^q},\quad \tau=q.
  \end{align*}
\end{example}

\begin{example}[Kernel gradient descent]
  \label{example:GradientDescent}
  Fix a step size $0<\eta\le (2K_{\max}+1)^{-1}$.
For kernel gradient descent, use $\Lambda_\eta=\{(\eta t)^{-1}:t\in\mathbb N,\ t\geq\lceil\eta^{-1}\rceil\}$.
For a fixed $u>0$ and $C_\lambda>0$, choose $t_d=\max\{\lceil\eta^{-1}\rceil,
                   \lfloor d^u/(\eta C_\lambda)\rfloor\}$ and set $\lambda=(\eta t_d)^{-1}$.
Then $\lambda\sim C_\lambda d^{-u}$, and the real-axis and complex-domain
bounds in Appendix K.1 apply on this grid.  The case $u=\infty$ is the
separately defined interpolation limit, not a finite iteration count.
The resulting estimator is associated with the spectral filter
  \begin{align*}
    \reg^{\mr{GD}}(z) &= \eta \sum_{k=0}^{t-1} (1-\eta z)^{k} = \frac{1-(1-\eta z)^t}{z},\quad \lambda = (\eta t)^{-1}, \\
    \rem^{\mr{GD}}(z) &= (1-\eta z)^t, \quad\tau=\infty.
  \end{align*}
\end{example}

\subsection{Nondegenerate signal}\label{subsec_non_de_signal}


To derive matching lower bounds for the excess risk, we require that the signal has nontrivial energy in the degrees around $\tilde{\ell}$ defined in Assumption~\ref{assume_lambda}.

\begin{assumption}\label{assump_source condition_lower_bound}
There exists an absolute constant $c_{0} > 0$ such that for any $ d \geq \mathfrak{C}$, we have
\begin{equation*}
\begin{aligned}
    &~ \mu_{m}^{-s} \sum\limits_{j=1}^{N(d, m)}  \theta_{m,j}^{2} \ge c_{0}, ~~m=\tilde{\ell}, \tilde{\ell}+1;\\
    &~ \sum\limits_{k=0}^{\tilde{\ell}} \sum\limits_{j =1}^{N(d, k)}  \theta_{k, j}^{2} \ge c_{0}.
\end{aligned}
   \end{equation*}
\end{assumption}

Assumption~\ref{assump_source condition_lower_bound} is a non-degeneracy condition needed for establishing a matching lower bound. Together with Assumption~\ref{assump_function_calss}, it implies
\begin{equation*}
\sup_d \|f\|_{[\calH]^{s}} \leq \sqrt{R} <\infty,
\qquad
\|f\|_{[\calH]^{t}}\to\infty
\quad\text{for every }t>s.
\end{equation*}
Specifically, recall that Proposition \ref{prop:inner_edr} implies $\mu_{\tilde{\ell}+1} = \Theta_d(d^{-(\tilde{\ell}+1)})$, hence Assumption~\ref{assump_source condition_lower_bound} implies
$$
\|f\|_{[\calH]^{t}}^2
\geq \mu_{\tilde{\ell}+1}^{-t}\sum\limits_{j=1}^{N(d, \tilde{\ell}+1)}  \theta_{\tilde{\ell}+1,j}^{2}
\geq c_0 \mu_{\tilde{\ell}+1}^{s-t} \asymp d^{(\tilde{\ell}+1)(t-s)}\to\infty .
$$

Thus, \(s\) characterizes the largest smoothness level for which the target sequence belongs to a source ball with a dimension-independent radius. 

For the regularization path, we fix one target sequence and assume the required non-degeneracy simultaneously on the finitely many spectral blocks arising over the considered range of \(u\). Our results are then applied to this same target sequence for each fixed \(u\); no uniform-in-\(u\) probability statement is intended.

Without Assumption~\ref{assump_source condition_lower_bound}, the spectral terms used in our lower-bound argument may be of smaller order, so the stated matching lower bound is no longer guaranteed. Similar non-degeneracy assumptions appear in \cite{Cui2021GeneralizationER,li2023asymptotic,zhang2024optimal,lu2024saturation,zhang2024phase}.

\begin{remark}
    Several lower bound results in the literature remain valid, in the sense that the same proofs apply, if their respective non-degeneracy conditions are replaced by Assumption \ref{assump_source condition_lower_bound}. 
    These include Lemmas~21--23 of \cite{zhang2024optimal}, Lemma~D.14 of \cite{lu2024saturation}, and Lemmas~5.4.3--5.4.4 of \cite{zhang2024phase}. 
    %
\end{remark}

\section{Main results}\label{sec_main_results}

This section is devoted to determining the convergence rate of the excess risk of $\hat{f}_{\lambda}$ under the setup in Section~\ref{sec_settings}. 
For a fixed design $X=(x_1,\ldots,x_n)$, write
\begin{align}\label{eq bias var decomposition}
\calR_{d}(u;X):=\E_{\epsilon\mid X}\|\widehat f_\lambda-f_\star\|_{L^2(\rho_X)}^2 =  \mathrm{var}(\hat{f}_{\lambda}) + \mathrm{bias}^2(\hat{f}_{\lambda}),
\qquad \lambda\asymp d^{-u}.
\end{align}
where
\begin{equation}\label{eq var term}
    \mathrm{var}(\hat{f}_{\lambda}) := \E_{x, \epsilon}\left[\left(\hat{f}_{\lambda}(x)-\E_\epsilon\left(\hat{f}_{\lambda}(x)\right)\right)^2\right],
\end{equation}
and
\begin{equation}\label{eq bias term}
    \mathrm{bias}^2(\hat{f}_{\lambda}) := \E_x\left[\left(\E_\epsilon\left(\hat{f}_{\lambda}(x)\right)-f_{\star}(x)\right)^2\right].
\end{equation}
The randomness in our probabilistic asymptotic statements is the training design. The test point is independent of the training sample. We take $0<\sigma^2<\infty$ to be fixed independently of $d$. Our statements along the regularization path concern each fixed exponent $u$ and the same target sequence, with the simultaneous block non-degeneracy specified in Section \ref{subsec_non_de_signal}.

{By determining the convergence rate of each term in \eqref{eq bias var decomposition}, we obtain the complete learning curve of large-dimensional spectral algorithms.}

\begin{theorem}\label{theorem_learn_curve}
    Suppose Assumptions \ref{assump_asymptotic}, \ref{assu:coef_of_inner_prod_kernel}, 
\ref{assump_function_calss},
\ref{assume_lambda}, 
and \ref{assump_source condition_lower_bound} hold with $\tau \leq \infty$, $s \geq 0$, and $u \in (0, \infty]$. 
Denote $\tilde{s}=\min\{s, 2\tau\}$. Then the excess risk of the spectral algorithm estimator in (\ref{eq bias var decomposition}) satisfies
\begin{equation*}
\begin{aligned}
    &~\calR_{d}(u;X)\\
= &~ \Theta_{d, \mathbb{P}} \left( 
d^{\gamma - \tilde{\ell} - 1 - 2  \max\{\gamma - u, 0\}}
+ d^{\tilde{\ell} - \gamma}
+ d^{-(\tilde{\ell} + 1)  s}
+ \mathbf{1}\{\tau<\infty, u<\infty\} d^{-2  \tau  u + (2  \tau - \tilde{s})  \tilde{\ell}}
\right)\\
&~ + \mathbf{1}\{\tau = \infty, u = \tilde{\ell}<\infty\}\Theta_{d, \mathbb{P}} \left( d^{-\tilde{\ell} s}\right).
\end{aligned}
\end{equation*}
\end{theorem}

Theorem \ref{theorem_learn_curve} fully characterizes the convergence rate of the excess risk in different asymptotic regimes $n = \Theta_d( d^{\gamma} )$, different smoothness $s \geq 0$, different filter qualification $\tau$, and different choices of the regularization parameter $\lambda=\Theta_d(d^{-u})$. Therefore, it provides the complete learning curve for large-dimensional spectral algorithms.

When $\ell_\lambda<\ell_\gamma$, we have $u<\ell_\gamma\leq\gamma$ and hence
\[
 n\lambda\asymp d^{\gamma-u}\longrightarrow\infty.
\]
Thus this case lies on the side $\lambda\gg n^{-1}$. Its proof must account for the intermediate degrees $\ell_\lambda<k\leq\ell_\gamma$ separately.
One must isolate and control this intermediate spectral block in both the bias and the variance. 
The remaining ingredients combine the large-dimensional kernel-matrix concentration results of \cite{ghorbani2021linearized} with the analytic functional framework for spectral filters. For the reader's convenience, Appendix~C provides a proof sketch for a representative case.

The full proof is deferred to the appendices. Appendix~D bounds the variance term in \eqref{eq var term}, Appendix~E bounds the bias term in \eqref{eq bias term}, and Appendix~F establishes the remaining operator approximations and combines them with these bounds to prove Theorem~\ref{theorem_learn_curve}.

Theorem~\ref{theorem_learn_curve} will be interpreted in the following subsections from the perspectives of phase transitions, robustness to under-regularization, saturation, and comparison with the sequence model.

\subsection{Three regimes of the learning curve} 
For interpreting the total-risk formula, let $a=\tilde{\ell}$ and define
\[
\begin{aligned}
 V_d(u)&=d^{\gamma-a-1-2(\gamma-u)_+}+d^{a-\gamma},\\
 B_d(u)&=d^{-(a+1)s}+Q_d(u)+J_d(u),\\
 Q_d(u)&=\begin{cases}
 d^{-2\tau u+(2\tau-\min\{s,2\tau\})a},&\tau<\infty,\ u<\infty,\\
 0,&\text{otherwise},
 \end{cases}\\
 J_d(u)&=\begin{cases}
 d^{-as},&\tau=\infty,\ u=a<\infty,\\
 0,&\text{otherwise}.
 \end{cases}
\end{aligned}
\]
Theorem~\ref{theorem_learn_curve} states that $\Risk(u;X)=\mathrm{var}(\hat{f}_{\lambda}) + \mathrm{bias}^2(\hat{f}_{\lambda})=\TP(V_d(u)+B_d(u))$, and Theorem D.1 states that $\mathrm{var}(\hat{f}_{\lambda}) =\TP(V_d(u))$. 
Consequently, if $B_d(u)=O_{d}(V_d(u))$, the variance determines the rate of the conditional excess risk. Conversely, if $B_d(u)\gg V_d(u)$, then the squared bias determines the rate of the conditional excess risk and we have $\mathrm{bias}^2(\hat{f}_{\lambda}) = \TP(B_d(u))$. We emphasize, however, that $B_d(u)$ need not always give the exact order of the squared bias when $B_d(u)=O(V_d(u))$.

This decomposition also reveals the shape of the learning curve. At the level of powers of $d$, the variance expression is non-decreasing in $u$, whereas $B_d(u)$ is non-increasing.
Let $u'$ denote a selected rate-optimal exponent; the selection need not be unique when the minimum extends over a plateau, and we provide a potential choice in Proposition F.4. The three ranges below describe the risk exponent:
\begin{itemize}
    \item {\bf Over-regularized regime: } When $u < u^{\prime}$, there is too much regularization and the risk exponent is non-increasing.
    
    \vspace{3pt}
    
    \item {\bf Under-regularized regime: } When $u \in (u^{\prime}, \gamma)$, 
    the risk exponent is non-decreasing.

    \vspace{3pt}
    
    \item {\bf Interpolation regime: } The interpolation exponent is constant for $u \geq \gamma$.
\end{itemize}

As illustrated in Figure \ref{fig_theory_benign_overfit}, these are weak monotonicity statements: plateaus may occur, so exponents preceding the selected \(u'\) can already attain the optimal rate.
At $u=\infty$ the parameter is exactly zero; finite $u\geq\gamma$ has positive regularization but the same asymptotic risk order as interpolation.

\begin{figure}[btp]
\centering
\subfigure[]{\includegraphics[width=0.48\columnwidth]{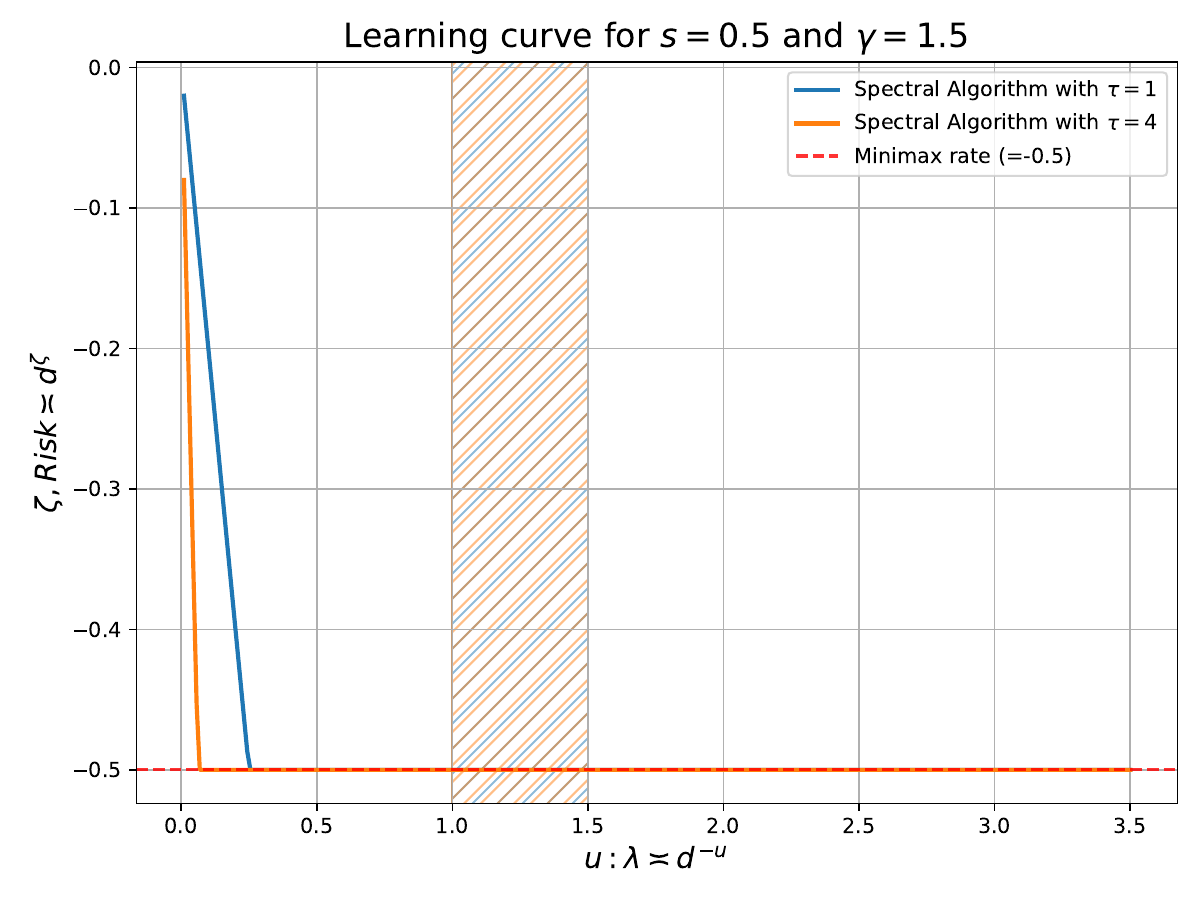}}
\subfigure[]{\includegraphics[width=0.48\columnwidth]{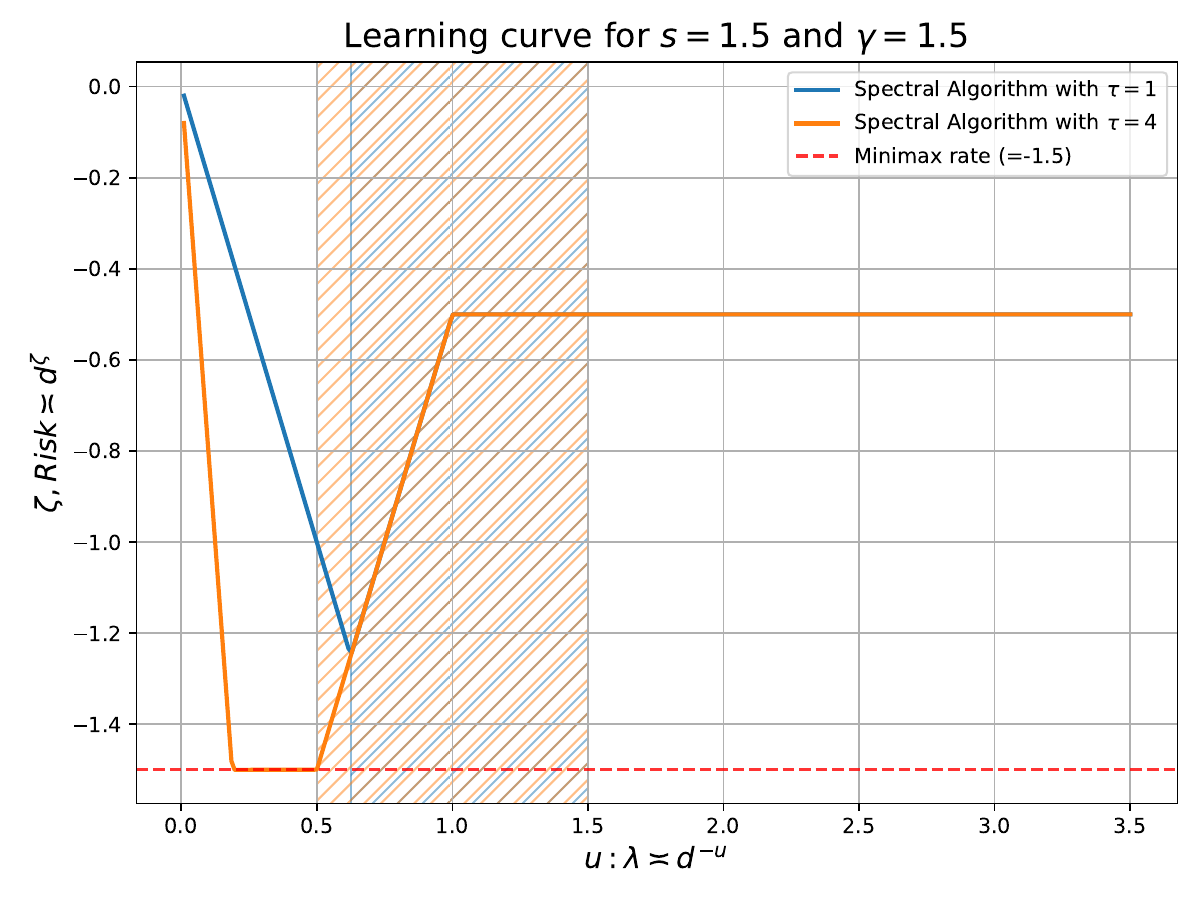}}

\vspace{-10pt}

\subfigure[]{\includegraphics[width=0.48\columnwidth]{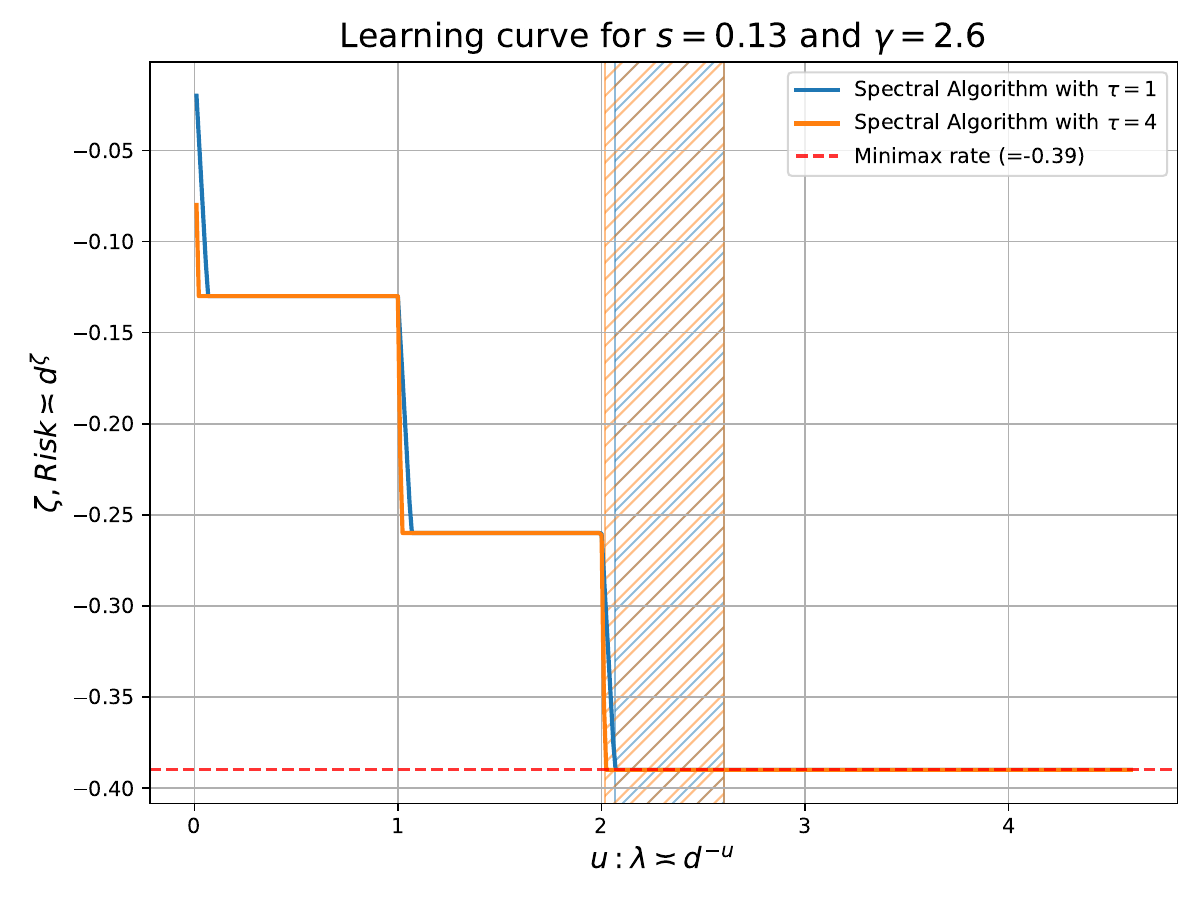}}
\subfigure[]{\includegraphics[width=0.48\columnwidth]{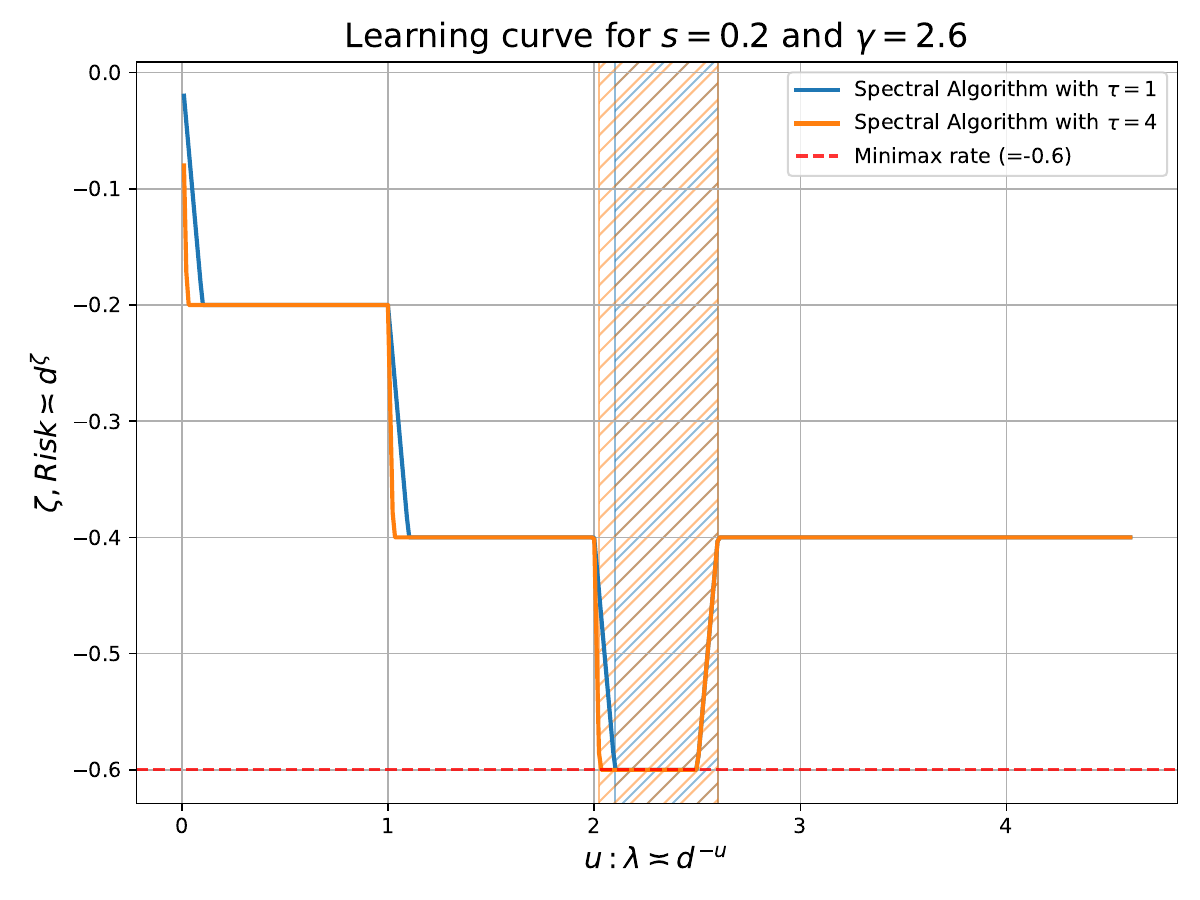}}

\vspace{-10pt}

\caption{
A graphical representation of the learning curves of large dimensional spectral algorithms with $\tau=1$ (blue) and $\tau=4$ (orange) obtained in Theorem \ref{theorem_learn_curve};
and the minimax rate (red) obtained in Proposition I.1.
Specifically, for any given $(s, \gamma, \tau, u)$, we set 
$\zeta = \max\{\gamma - \tilde{\ell} - 1 - 2  \max\{\gamma - u, 0\}, \tilde{\ell} - \gamma, -(\tilde{\ell} + 1)  s, -2  \tau  u + (2  \tau - \tilde{s})  \tilde{\ell}\}$ with $\tau \in \{1, 4\}$, and set $\zeta = \max\{p-\gamma, -(p+1)s\}$ for the minimax rate, where $p=\left\lfloor\gamma/(s+1)\right\rfloor$.
In subfigures (a) and (c), $s \leq \Gamma(\gamma)$, while in subfigures (b) and (d), $s > \Gamma(\gamma)$. 
Under-regularized regimes are marked by hatched patterns with corresponding colors.
}
\label{fig_theory_benign_overfit}
\end{figure}

A particularly interesting phenomenon emerges when the target smoothness \(s\) is small: The rate of the conditional excess risk can persist far into the under-regularized regime and all the way to interpolation. For example, when $\gamma = 0.9$, $s=0.05$, and $\tau=\infty$, the risk satisfies $\Risk(u;X)=\TP(d^{-0.05})$ for every fixed $u>0$, including $u = \infty$. Thus, although the regularization level moves far beyond an optimally tuned value, the risk rate does not deteriorate. We formalize this prediction persistence phenomenon in Section~\ref{sec: prediction persistence}.

\subsection{Prediction persistence beyond optimal regularization}
\label{sec: prediction persistence}

The full learning curve in Theorem~\ref{theorem_learn_curve} allows us to address a question beyond optimal tuning and interpolation alone: once the optimal prediction rate is reached, how far can regularization be reduced without deteriorating this rate?

Recall that, with $p:=\left\lfloor\gamma / ({s+1})\right\rfloor$,
the minimax rate for expected excess risk over a fixed-radius source ball in $[\mathcal H]^s$, with expectation over the training design and noise, is
$$
d^{-\zeta_{\star}}, \quad \zeta_{\star} = \min\{\gamma - p,  (p+1)s\};
$$
see Proposition~I.1. Here, matching this rate means that $\mathscr R_d(u;X)=\Theta_{d,\mathbb P}(d^{-\zeta_\star})$ for the specified target sequence and each fixed exponent $u$. We say that the prediction rate \emph{persists beyond optimal regularization} if this rate is maintained as \(u\) increases beyond an optimal exponent \(u'\). Persistence up to \(u=\infty\) is particularly notable: since \(\lambda=d^{-u}\) and $K$ is invertible, this endpoint corresponds to exact interpolation and zero training loss.

To characterize when such persistence occurs, define
\begin{equation}\label{eqn_threshold_s}
\Gamma(\gamma) := \begin{cases}
\infty,  & \gamma \in (0, 0.5], \\
(\gamma-\ell_{\gamma})/\ell_{\gamma}, & \gamma \in (\ell_{\gamma}, \ell_{\gamma}+0.5] ~\text{ and }~ \ell_{\gamma} \geq 1, \\
(\ell_{\gamma}+1-\gamma)/(\ell_{\gamma}+1), & \gamma \in (\ell_{\gamma}+0.5, \ell_{\gamma}+1),\\
0, & \gamma =1, 2, \cdots .
\end{cases}
\end{equation}
The following result gives a sharp dichotomy.

\begin{theorem}\label{coroll_prediction_persistence}
Suppose the assumptions of Theorem~\ref{theorem_learn_curve} hold. When $s>0$,  let $p:=\left\lfloor\gamma / ({s+1})\right\rfloor$, and let \(u'=u'(\tau)\) be an optimal regularization exponent specified in Proposition~F.4.

\begin{itemize}
\item [(i)] If $s=0$, then for any fixed $u \in (0, \infty]$, we have
$$
\calR_{d}(u;X) =
    \Theta_{d,\mathbb P}(1).
$$

\item[(ii)] If \(0 < s\leq\Gamma(\gamma)\), then for every fixed \(u\in[u',\infty]\),
$$
    \calR_{d}(u;X) =
    \Theta_{d,\mathbb P}\!\left(
    d^{-\zeta_{\star}}
    \right).
$$

\item[(iii)] If \(s>\Gamma(\gamma)\), then there exist constants
\(\varepsilon,c>0\), independent of \(d\), such that
\(u'<\gamma-\varepsilon\) and, for every fixed
\(u\in(\gamma-\varepsilon,\gamma)\),
$$
    \frac{\calR_{d}(u;X)}
{d^{-\zeta_{\star}}}
=
\Omega_{d,\mathbb P}(d^c)
\longrightarrow\infty
\qquad\text{in probability}.
$$

\end{itemize}
\end{theorem}

Thus, the smoothness threshold \(\Gamma(\gamma)\) determines whether optimal prediction persists under increasingly weak regularization. When $0<s\leq\Gamma(\gamma)$, the conditional risk retains the minimax risk order from an optimal regularization scale to exact interpolation. In particular, continued training beyond the optimal stopping point does not deteriorate the prediction rate. When \(s>\Gamma(\gamma)\), this persistence breaks down already in the under-regularized regime: the risk becomes polynomially larger than the minimax rate on an interval immediately preceding the interpolation scale.

Figure~\ref{fig_theory_benign_overfit} illustrates this dichotomy. In panels (a) and (c), where \(s\leq\Gamma(\gamma)\), the learning curve stays at the minimax rate throughout the under-regularized and interpolation regimes. In panels (b) and (d), where \(s>\Gamma(\gamma)\), the risk increases above the minimax rate as the estimator becomes under-regularized.

A heuristic intuition for the role of \(s\) can be obtained from Sobolev RKHSs. Consider \(\mathcal H=W^{m,2}(\mathcal X)\) with \(m>d/2\), for which
$
[\mathcal H]^s\cong W^{ms,2}(\mathcal X)
$.
Larger \(s\) therefore corresponds to a smoother regression function. When \(s\) is small, the target itself may contain substantial high-frequency components, and fitting increasingly fine-scale structure can remain compatible with the optimal prediction rate. For larger \(s\), by contrast, the target is smoother, while interpolation must still fit the noise in the observations; the resulting high-frequency fluctuations can then lead to a deterioration in prediction. Theorem~\ref{coroll_prediction_persistence} identifies the precise threshold \(\Gamma(\gamma)\) separating these two behaviors.

\subsection{Explaining the saturation effect}

Theorem~\ref{theorem_learn_curve}, together with the minimax benchmark in Proposition~I.1, characterizes the \emph{saturation effect} of large-dimensional spectral algorithms. In particular, finite qualification introduces the additional bias term
\[
\mathbf{1}\{\tau<\infty\}
d^{-2\tau u+(2\tau-\tilde{s})\tilde{\ell}},
\qquad
\tilde{s}=\min\{s,2\tau\},
\]
which may prevent the optimally tuned conditional risk from matching the minimax risk order.
Whether this happens, however, depends jointly on $(\gamma,s,\tau)$ and the tuning parameter $u$; in particular, $\tau<s$ alone does not imply saturation.

Theorem~\ref{theorem_learn_curve} also provides a geometric interpretation of this phenomenon. For finite $\tau$, increasing $u$ makes the qualification-dependent bias decay, while the other terms in the learning curve eventually limit the benefit of weaker regularization. Saturation occurs when this additional bias cannot be sufficiently reduced, over all choices of $u$, without another term becoming dominant above the minimax benchmark. Conversely, even when $\tau<s$, the conditional risk can still match the minimax risk order when the qualification-dependent bias becomes negligible before another term exceeds that benchmark, for example, $(\gamma, s, \tau, u)=(1.5, 1.2, 1, 0.6)$.

\begin{figure}[btp]
\centering
\subfigure[]{\includegraphics[width=0.48\columnwidth]{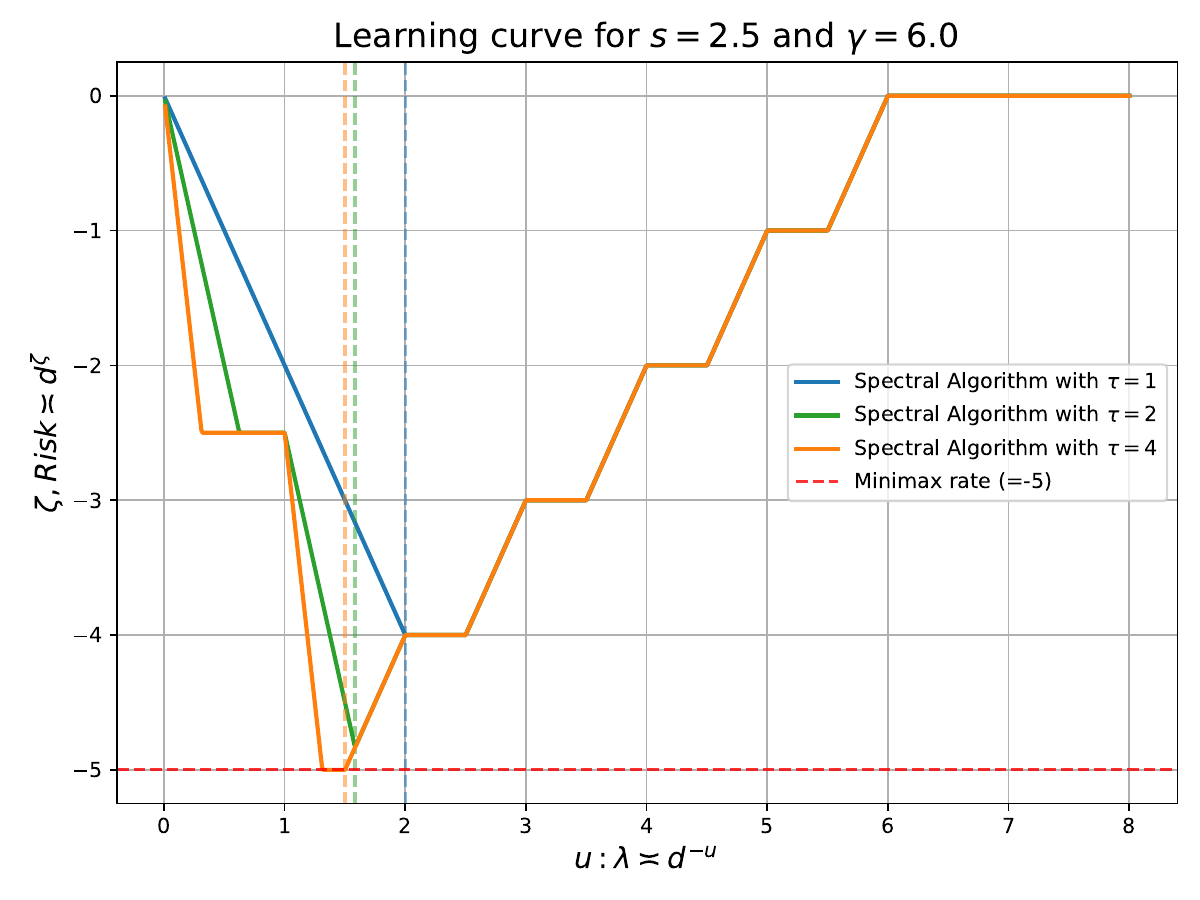}}
\subfigure[]{\includegraphics[width=0.48\columnwidth]{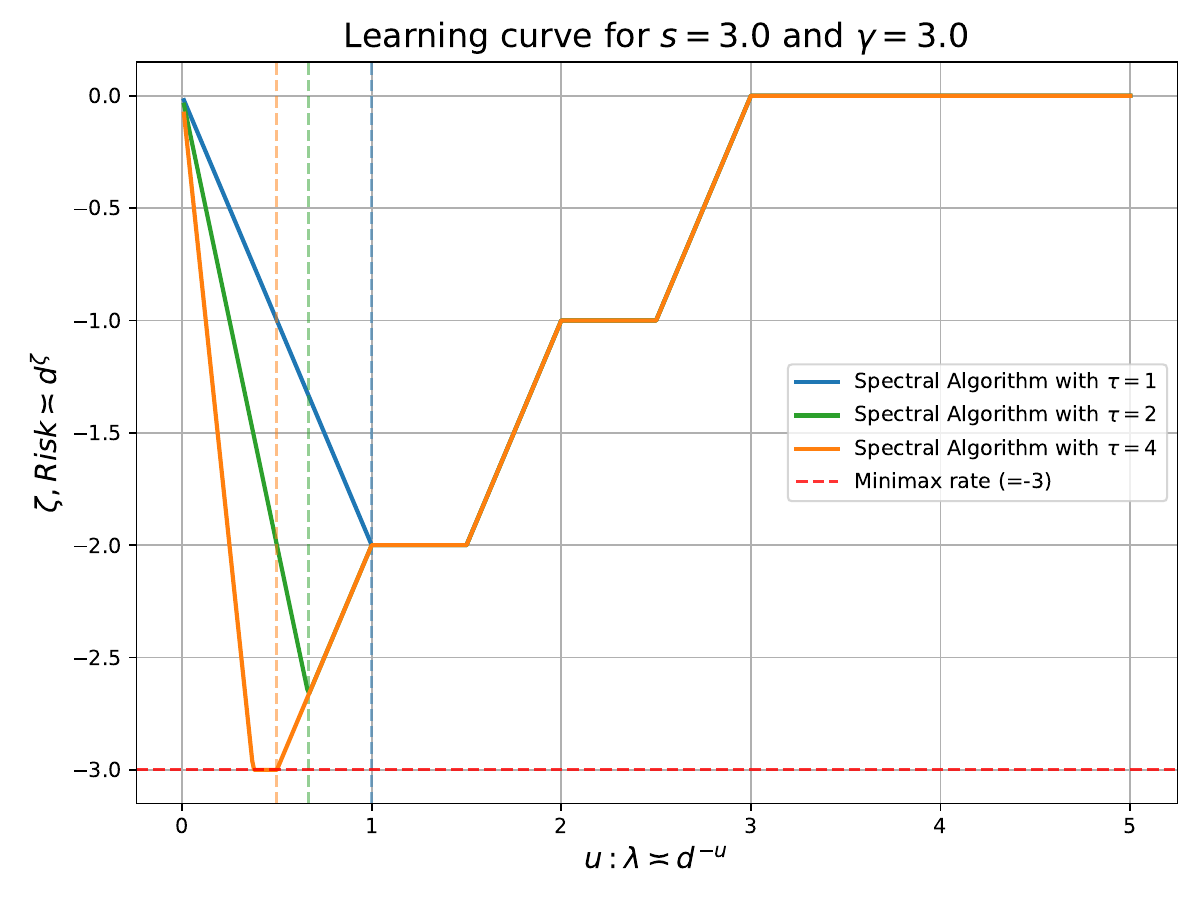}}

\caption{
Another graphical representation of the learning curves of large-dimensional spectral algorithms with $\tau=1$ (blue), $\tau=2$ (green), and $\tau=4$ (orange) obtained in Theorem~\ref{theorem_learn_curve}, together with the minimax benchmark (red) obtained in Proposition~I.1.
Dashed vertical lines indicate global minimizers of the corresponding learning curves.
}
\label{fig_theory_saturation}
\end{figure}

Figure~\ref{fig_theory_saturation} illustrates this mechanism for two particular parameter configurations. A smaller qualification makes the qualification-dependent bias decrease more slowly with $u$, which can shift the optimal tuning parameter and increase the minimum excess risk. The examples in the figure illustrate regimes where this effect leads to saturation.

\subsection{Comparison with the sequence model}

Effective sequence models for ridge regression have been studied by \citet{cheng2022dimension,misiakiewicz2024non}, using adjusted regularization and noise levels. The comparison below uses sequence noise variance $\sigma^2/n$ and the spectral filters considered here. Proposition~\ref{prop_learning_curve_of_sequence} compares the conditional kernel risk with the population sequence risk up to constant factors at a deterministic admissible regularization scale.

Consider the sequence model
\begin{equation}\label{eqn_sequence_model}
    z_j = f_j + \xi_j,\quad j=1, 2, \cdots,
\end{equation}
where $(z_j)_{j=1}^{\infty}$ is the observation sequence, $f_j$ is the coefficient of $f_{\star}$ with respect to $\phi_j$, that is, we have $f_j = \langle f_{\star}, \phi_j \rangle_{L^2}$, $j=1, 2, \cdots$. Moreover, $\xi_j$, $j=1, 2, \cdots$ are (not necessarily independent) random variables with mean zero and variance $\sigma^2 / n$.
Suppose $\lambda_j$'s are known and consider the parametrization $f_j = \lambda_j^{1/2} \beta_j$, $j=1, 2, \cdots$. 
For $\lambda>0$, define ridge estimation through the continuous, strongly convex quadratic functional on $\ell^2$,
\[
 Q_\lambda(\beta)=\frac12\sum_j\lambda_j\beta_j^2
 -\sum_j\sqrt{\lambda_j}z_j\beta_j+\frac\lambda2\sum_j\beta_j^2.
\]
The random vector $(\sqrt{\lambda_j}z_j)_j$ is in $\ell^2$ almost surely, because
\[
 \E\sum_j\lambda_j z_j^2
 =\sum_j\lambda_j f_j^2+\frac{\sigma^2}{n}\tr(T)<\infty.
\]
The unique minimizer is $\widehat\beta_j=\sqrt{\lambda_j}z_j/(\lambda_j+\lambda)$, and hence $\widehat f_j=\lambda_jz_j/(\lambda_j+\lambda) = \lambda_j \reg^{\krr}(\lambda_j) z_j$, $j=1, 2, \cdots$. Gradient flow for the same unregularized quadratic, initialized at zero and stopped at time $t=\lambda^{-1}$, gives $\widehat f_j=(1-e^{-\lambda_jt})z_j = \lambda_j \reg^{\gf}(\lambda_j) z_j$.

More generally, for any filter function $\reg$ in Definition \ref{def:filter},
we define the estimator of $(f_j)_{j=1}^{\infty}$ as:
\begin{equation}\label{eqn_est_sequence_model}
    \hat{f}_j^{\lambda} := \lambda_j \reg(\lambda_j) z_j,\quad j=1, 2, \cdots
\end{equation}
whose excess risk is
\begin{equation}\label{eqn_excess_risk_sequence}
    \begin{aligned}
        \E_{\xi}\left[\sum_{j=1}^{\infty}(\hat{f}_j^{\lambda} - f_j)^2\right]
        =
        \frac{\sigma^2}{n}\sum_{j =1}^\infty \left[ \lambda_j \reg(\lambda_j) \right]^2
        + \sum\limits_{j=1}^{\infty} \left( \rem(\lambda_j) f_{j}\right)^{2}.
    \end{aligned}
\end{equation}

\begin{remark}
    {All coordinates in the sequence model are indexed by the positive spectral
index set $J_d$ introduced in Section \ref{subsec:interpolation_space}.}
Finite spectra may be padded by zero eigenvalues and zero target coefficients, since we have set $\varphi_0(0)=0$; the padded coordinates then contribute no risk. 
{Therefore, we write $j=1, 2, \cdots$ rather than $j \in J_{d}$ for notational simplicity.}
\end{remark}

The next proposition gives the population sequence risk at the original regularization parameter for $0<u\leq\gamma$. It then compares the kernel risk, including interpolation, with the sequence risk at an admissible scale of order $\max\{\lambda,n^{-1}\}$.

\begin{proposition}\label{prop_learning_curve_of_sequence}
    Adopt all the notations and assumptions in Theorem \ref{theorem_learn_curve}. 
    Then for any $0<u\leq \gamma$, the excess risk of $(\hat{f}_j^{\lambda})_{j=1}^{\infty}$ satisfies
\begin{equation}\label{eqn_prop_learning_curve_of_sequence_0}
\begin{aligned}
     \E_{\xi}\left[\sum_{j=1}^{\infty}(\hat{f}_j^{\lambda} - f_j)^2\right]
=&~ 
\Theta_{d} \left( 
d^{2u - \gamma - \ell_{\lambda} - 1}
+ d^{\ell_{\lambda} - \gamma}
+ d^{-(\ell_{\lambda} + 1)  s}
+ \mathbf{1}\{\tau<\infty\} d^{-2  \tau  u + (2  \tau - \tilde{s})  \ell_{\lambda}}
\right)\\
&~+\mathbf{1}\{\tau = \infty, u = \ell_{\lambda}\}\Theta_{d} \left( d^{-\ell_{\lambda} s}\right).
\end{aligned}
\end{equation}
{Consequently, for $\lambda\asymp d^{-u}$, or $\lambda=0$ when $u=\infty$, let
$\bar\lambda_d\in\Lambda$ be a deterministic admissible sequence such that
$\bar\lambda_d\asymp\max\{\lambda,n^{-1}\}$.
Whenever such a sequence is available, Theorem \ref{theorem_learn_curve} implies that the sequence estimator at this
admissible regularization level satisfies}
\begin{equation}\label{eqn_prop_learning_curve_of_sequence}
    \calR_{d}(u;X) = \Theta_{d,\mathbb P}\!\left(
\E_\xi\sum_{j=1}^{\infty}
  (\widehat f_j^{\bar\lambda_d}-f_j)^2\right).
\end{equation}
\end{proposition}

For $0<u\leq\gamma$, Proposition~\ref{prop_learning_curve_of_sequence} compares the two risks at the original regularization scale. For weaker regularization and for interpolation, it uses a deterministic admissible sequence parameter of order $n^{-1}$. This comparison describes the risk orders for the spectral filters considered here; the following remark specifies admissible choices of the sequence parameter.

\begin{remark}
    For continuous filters with $\Lambda=(0,1]$, one may take
$\bar\lambda_d=\max\{\lambda,n^{-1}\}$ for all sufficiently large $d$.
{For gradient descent, let $t_{\min}=\lceil\eta^{-1}\rceil$, the minimum iteration count in Example~\ref{example:GradientDescent}, and define}
\[
\begin{aligned}
T_n=\max\{t_{\min},\lfloor n/\eta\rfloor\},\quad
\rho_n=(\eta T_n)^{-1},\quad
\bar\lambda_d=\max\{\lambda,\rho_n\}\in\Lambda_\eta.
\end{aligned}
\]
For all sufficiently large $n$,
\[
\frac1n\le\rho_n\le\frac1{n-\eta}
=\frac1n\{1+O(n^{-1})\}.
\]
Consequently the required comparison of regularization scales holds.
For finite iteration counts, this construction corresponds to stopping the
sequence-model algorithm after $\min\{t_d,T_n\}$ iterations. At the
interpolation endpoint, the sequence-model comparison uses $T_n$ iterations.
In particular, the population sequence risk describes the kernel risk up to
constant factors at an admissible regularization scale bounded below by a
constant multiple of $n^{-1}$.

\end{remark}

\section{An extension for KRR on general domains}
\label{sec_general}

Theorem~\ref{theorem_learn_curve} gives the conditional risk orders for spectral algorithms with inner-product kernels on the sphere. We now extend the KRR analysis to a structured class of kernels on general domains for noninteger $\gamma$.

The assumptions specify spectral scaling, hyper-contractivity on whole eigenspaces, and concentration of the tail diagonal, together with the source and signal conditions. Under these assumptions, Theorem~\ref{theorem_learn_curve_general} gives the same risk orders as in the spherical KRR case. Comparing these orders yields the corresponding prediction persistence criterion.

\subsection{Notation and assumptions for general domains}

We introduce the settings for this section, which parallel those in Section \ref{sec_settings}. 
The main components are as follows: 
Let $\mathcal{X} \subset \bbR^{d}$ be a measurable space, and let $\rho_{\mathcal{X}}$ be a probability measure on $\mathcal{X}$.
{Let $\widetilde{\mathscr{K}}$ be a continuous and positive-definite kernel on $\mathcal{X} \times \mathcal{X}$ with respect to $\rho_{\mathcal{X}}$. It admits a Mercer decomposition}
\begin{equation}\label{eqn_mercer_general}
    \widetilde{\mathscr{K}}(x,x^\prime) = \sum_{k=0}^{\infty} \sum_{j=1}^{\widetilde{N}(d, k)} \widetilde{\lambda}_{k, j}  \widetilde{\psi}_{k, j}(x) \widetilde{\psi}_{k, j}\left(x^\prime\right),
\end{equation}
where the eigenvalues $\{\widetilde{\lambda}_{0, 1}, \cdots, \widetilde{\lambda}_{0, \widetilde{N}(d, 0)}, \widetilde{\lambda}_{1, 1}, \cdots, \widetilde{\lambda}_{1, \widetilde{N}(d, 1)}, \cdots\}$ form a non-increasing sequence,
and the corresponding eigenfunctions of $\widetilde{\lambda}_{k, j}$ are $\widetilde{\psi}_{k, j}$ with $k=0, 1, \cdots$ and $j=1, \cdots, \widetilde{N}(d, k)$. 
{Furthermore, we let $\widetilde{\calH}$ denote the RKHS associated with the kernel $\widetilde{\mathscr{K}}$, and $[\widetilde{\calH}]^{s}$ the corresponding interpolation space. Finally, we write the regression function as}
    $$
        f_{\star}(x) = \sum_{k=0}^{\infty} \sum_{j=1}^{\widetilde{N}(d,k)} \theta_{k,j} \widetilde{\psi}_{k,j}(x).
    $$

Next, we replace certain assumptions in Section \ref{sec_settings} with the following ones.

\begin{assumption}\label{assump_asymptotic_general}
    We assume that there exist positive absolute constants $c_1$, $c_2$, and a non-integer $\gamma>0$, such that the sample size satisfies
    $c_{1} d^{\gamma} \leq n \leq c_{2} d^{\gamma}$,
and we denote $\ell_{\gamma} := \lfloor \gamma\rfloor < \gamma$.
\end{assumption}

In the spherical setting, Proposition \ref{prop:inner_edr} identifies the low-degree spectral-block scaling of inner-product kernels on $\mathbb{S}^{d-1}$.
For KRR on general domains, the next assumption abstracts the corresponding spectral scaling needed in our learning-curve analysis.

\begin{assumption}\label{assump_eigenval_general}
    We assume that the sum of the eigenvalues is bounded, that is, there exists an absolute constant $\widetilde{\mathfrak{C}}$ such that $\sum_{k=0}^{\infty} \sum_{j=1}^{\widetilde{N}(d, k)} \widetilde{\lambda}_{k, j}\leq \widetilde{\mathfrak{C}}$. Moreover, for any $k  = 0, 1, \cdots, 2\ell_{\gamma} +2$, we assume that
    $$\widetilde{N}(d, k) = \Theta_d(d^{k}) \quad \text{ and } \quad 
        \widetilde{\lambda}_{k, j} = \Theta_d(d^{-k}), \quad \text{ for } 1\leq j \leq \widetilde{N}(d, k).
    $$
\end{assumption}

Assumption \ref{assump_eigenval_general} holds for several types of common kernels in large dimensions. For example:
\begin{itemize}
    \item For the spherical example in Section~\ref{sec_main_results}, strengthen the coefficient positivity requirement to $a_k>0$ for $0\leq k\leq2\ell_{\gamma}+2$, in addition to the other parts of Assumption \ref{assu:coef_of_inner_prod_kernel}. The proof of the fixed-degree spectral estimates then applies through the degree required by Assumption \ref{assump_eigenval_general}. 

    \item Consider the fixed dot-product kernel $\mathscr{K}(x,x')=\Phi(\langle x,x'\rangle/d)$ under the uniform measure on $\{-1,1\}^d$. In addition to the smoothness and nonnegative summable coefficients of Assumption \ref{assu:coef_of_inner_prod_kernel}, require $a_k>0$ for $0\le k\le 2\ell_{\gamma}+2$.

    \item Gaussian kernels $\exp{\left(-\frac{\|x-x^{\prime}\|_{2}^{2}}{d /\beta}\right)}$ defined on $\mathbb{R}^d$ {with the standard Gaussian distribution, where $\beta>0$ is fixed} (see, e.g., \cite{rasmussen2003gaussian}).

    \item Product Laguerre kernels defined on $[0, \infty)^{d}$ with the product of independent Gamma distributions, see Appendix J.4.1 for a detailed definition.

\end{itemize}

We provide a detailed
verification in Appendix J.4.

The third assumption is known as the hyper-contractivity assumption.

\begin{assumption}\label{assump_hyper_contractivity}
    For any integer $q \geq 1$, there exists a constant $C=C(q, \ell_{\gamma})$ such that
$$
\|g\|_{L^{2q}} \leq C(q, \ell_{\gamma})\|g\|_{L^{2}}, \quad \forall g \in \text{span}\left\{\widetilde{\psi}_{k, j}, k \leq 2\ell_{\gamma}+1\right\}.
$$
\end{assumption}

{Assumption \ref{assump_hyper_contractivity} was originally adopted in \cite{mei2022generalization, misiakiewicz2024non}} to replace the addition formula property holding for spherical harmonics (see, e.g., Proposition 1.18 in \cite{gallier2009notes}). 
{Assumption \ref{assump_hyper_contractivity} holds for several types of common kernels in large dimensions, such as (i) inner-product kernels defined on $\mathbb{S}^{d-1}$ with the uniform distribution and (ii) dot-product kernels defined on the hypercube $\{-1, 1\}^{d}$ with the uniform distribution (see, e.g., Appendix E in \cite{mei2022generalization}).} 
These two examples also satisfy Assumption \ref{assump_eigenval_general}, and hence provide concrete settings covered by Theorem~\ref{theorem_learn_curve_general}.

{The fourth assumption concerns the diagonal elements of the kernel.}

\begin{assumption}
\label{assumption_tail_diagonal_general}
Define
\begin{align*}
    \kappa_1
    :=
    \sum_{k\geq \ell_\gamma+1}
    \sum_{j\leq \widetilde N(d,k)}
    \widetilde\lambda_{k,j} \quad \text{ and } \quad
    \kappa_2
    :=
    \sum_{k\geq \ell_\gamma+1}
    \sum_{j\leq \widetilde N(d,k)}
    \widetilde\lambda_{k,j}^{2}.
\end{align*}
We assume that
\begin{equation}
\label{eq_assumption_tail_diagonal_1}
\max_{i\in[n]}
\left|
\frac{
\displaystyle
\sum_{k\geq \ell_\gamma+1}
\sum_{j\leq \widetilde N(d,k)}
\widetilde\lambda_{k,j}
\widetilde\psi_{k,j}(x_i)^2
}{\kappa_1}
-1
\right|
=o_{d,\mathbb P}(1),
\end{equation}
and
\begin{equation}
\label{eq_assumption_tail_diagonal_2}
\max_{i\in[n]}
\left|
\frac{
\displaystyle
\sum_{k\geq \ell_\gamma+1}
\sum_{j\leq \widetilde N(d,k)}
\widetilde\lambda_{k,j}^{2}
\widetilde\psi_{k,j}(x_i)^2
}{\kappa_2}
-1
\right|
=o_{d,\mathbb P}(1).
\end{equation}
\end{assumption}

{One can also show that Assumption \ref{assumption_tail_diagonal_general} holds for (i) inner-product kernels defined on $\mathbb{S}^{d-1}$ with the uniform distribution, and (ii) dot-product kernels defined on the hypercube $\{-1, 1\}^{d}$ with the uniform distribution (see, e.g., (190) and Appendix E in \cite{mei2022generalization}).}


The next two assumptions are for the regression functions. 
{Under Assumptions~\ref{assume_lambda} and~\ref{assump_eigenval_general}, Assumptions~\ref{assump_function_calss_general} and~\ref{assump_source condition_lower_bound_general} imply that $\sup_{d\geq\mathfrak{C}}\|f_{\star}\|_{[\widetilde{\calH}]^{s}}\leq\sqrt{R}$ and $\|f_{\star}\|_{[\widetilde{\calH}]^{t}}\to\infty$ as $d\to\infty$ for every $t>s$. Thus, as in the spherical setting discussed in Section~\ref{subsec_non_de_signal}, $s$ is the largest smoothness level for which the target sequence belongs to a source ball with a dimension-independent radius; membership in smoother spaces at a fixed dimension is still possible.}

\begin{assumption}\label{assump_function_calss_general}
There exist absolute constants $s \geq 0$ and $R>0$, such that we have
    \begin{equation*}
    f_{\star} \in 
    \left\{f \in [\widetilde{\calH}]^{s} \mid \|f\|_{[\widetilde{\calH}]^{s}} \leq \sqrt{R}\right\}.
\end{equation*}
\end{assumption}

\begin{assumption}\label{assump_source condition_lower_bound_general}
Suppose there exists an absolute constant $c_{0} > 0$ such that for any $ d \geq \mathfrak{C}$, we have
\begin{equation*}
        \sum\limits_{j=1}^{\widetilde{N}(d, m)} \widetilde{\lambda}_{m, j}^{-s} \theta_{m,j}^{2} \ge c_{0}, ~~m=\widetilde{\ell}, \widetilde{\ell}+1;
   \end{equation*}
Moreover, if $s> 2$, we further assume that
$\sum\limits_{k=0}^{\widetilde{\ell}} \sum\limits_{j =1}^{\widetilde{N}(d, k)}  \theta_{k, j}^{2} \ge c_{0}$.
\end{assumption}

\subsection{Learning curve of large-dimensional KRR in general domains}

The following theorem determines the learning curve of large-dimensional KRR with general kernels defined on general domains.

\begin{theorem}\label{theorem_learn_curve_general}
    Suppose Assumptions \ref{assume_lambda}, 
    \ref{assump_asymptotic_general}, 
    \ref{assump_eigenval_general}, 
    \ref{assump_hyper_contractivity},
    \ref{assumption_tail_diagonal_general},
\ref{assump_function_calss_general},
and \ref{assump_source condition_lower_bound_general} hold with $s \geq 0$ and $u \in (0, \infty]$. 
Denote $\widetilde{s}=\min\{s, 2\}$. Then the excess risk of the KRR estimator in (\ref{eq bias var decomposition}) satisfies
\begin{equation*}
\begin{aligned}
    \E_{x, \epsilon} \left[ \left(\hat{f}_{\lambda}(x) - f_{\star}(x) \right)^{2} \right]
= \Theta_{d, \mathbb{P}} \left( 
d^{\gamma - \widetilde{\ell} - 1 - 2  \max\{\gamma - u, 0\}}
+ d^{\widetilde{\ell} - \gamma}
+ d^{-(\widetilde{\ell} + 1)  s}
+ d^{-2  u + (2 - \widetilde{s})  \widetilde{\ell}}
\right).
\end{aligned}
\end{equation*}
\end{theorem}



Theorem \ref{theorem_learn_curve_general} shows that the rate characterization obtained in Theorem~\ref{theorem_learn_curve} is not tied to the exact sphere kernel setting. 
The extension uses the spectral scaling in Assumption \ref{assump_eigenval_general}, the low-degree hypercontractivity in Assumption \ref{assump_hyper_contractivity},
    and the two tail-diagonal concentration conditions in Assumption \ref{assumption_tail_diagonal_general}, together with the stated sampling and target assumptions.
In this sense, the spherical analysis captures structural features that are sufficient for the same large-dimensional behavior to persist beyond the sphere.
This supports the use of spherical kernel models as a tractable basic setting for studying large-dimensional KRR.

\begin{remark}\label{remark_compare_with_mei_mis}
General kernels have also been studied under eigenspace concentration conditions; see \citet{mei2022generalization,misiakiewicz2024non}. An additive approximation with remainder $o_{d,\mathbb P}(1)$ alone need not determine a specified vanishing polynomial risk rate. Theorem 12 of \citet{misiakiewicz2024non} provide relative approximations for conditional bias and variance under concentration and residual moment conditions and quantitative error requirements. The tail energy condition discussed in Appendix B.3.1 is a sufficient condition for their spherical specialization. Theorem~\ref{theorem_learn_curve_general} gives explicit polynomial risk orders under the structural and target assumptions stated here.
\end{remark}

For a fixed kernel and design sequence, define
\[
 \mathcal B_s(R)=\{f\in[\widetilde{\mathcal H}]^s:
                  \|f\|_{[\widetilde{\mathcal H}]^s}^2\leq R\}.
\]
Let $\mathcal P_s(R,\sigma)$ consist of the laws of independent
random-design regression observations with marginal design law $\rho_X$,
regression function in $\mathcal B_s(R)$, and conditionally centered noise
with conditional variance at most $\sigma^2$.   
The following theorem establishes a minimax benchmark on an arbitrary domain. 

\begin{theorem}\label{thm:minimaxgeneral}
Fix $s>0$ and let $p=\lfloor\gamma/(s+1)\rfloor$. Suppose Assumption \ref{assump_asymptotic_general} holds, and the spectral block scaling in Assumption \ref{assump_eigenval_general} holds through $p+1$, and the span of blocks through $p$ has a dimension-independent $L^4/L^2$ norm bound. In independent random-design regression with conditional mean-zero noise of variance at most $\sigma^2$, let the minimax class include Gaussian noise with variance $\sigma^2>0$. Then
\[
 \inf_{\widehat f}\sup_{\rho_{f_\star}\in\mathcal P_s}
 \E\norm{\widehat f-f_\star}_{L^2}^2
 \asymp d^{p-\gamma}+d^{-(p+1)s}.
\]
Here $\mathcal P_s$ is the class just defined. In particular Assumptions \ref{assump_asymptotic_general}, 
    \ref{assump_eigenval_general}, 
    and \ref{assump_hyper_contractivity} suffice for this benchmark. For $s=0$, the minimax rate is constant.
\end{theorem}

The algebraic dichotomy in Theorem \ref{coroll_prediction_persistence} holds {under the present assumptions given in Theorem \ref{theorem_learn_curve_general}}, with the same target sequence and simultaneous block nondegeneracy along the path. 
In the persistence range, the conditional excess risk has, in probability, the order of the source-ball minimax benchmark in Theorem~\ref{thm:minimaxgeneral}. 
At $\lambda=0$, exact interpolation holds with probability tending to one (see, e.g., Appendix H.2).

\begin{remark}
    Theorem \ref{thm:minimaxgeneral} proves the expected minimax benchmark, not uniform integrability of the KRR risk.
\end{remark}

\section{Numerical experiments}\label{sec_experiments}
We use several experiments to illustrate Theorem \ref{theorem_learn_curve}.
{Consider the radial basis function (RBF) kernel with a fixed bandwidth:}
    $$
        \mathscr{K}^{\mathrm{rbf}}(x,x^{\prime}) = \exp{\left(-\frac{\|x-x^{\prime}\|_{2}^{2}}{2}\right)}, ~~x, x^{\prime} \in \mathbb{S}^{d-1}.
    $$
The RBF kernel satisfies Assumption \ref{assu:coef_of_inner_prod_kernel}. 

We consider the following data generation process:
$y_{i} = f_{\star}(x_{i}) + \epsilon_{i}, ~~ i = 1, \cdots, n,$ 
where each $x_{i}$ is independently sampled from the uniform distribution on $\mathbb{S}^{d-1}$, and $\epsilon_{i} \overset{\text{i.i.d.}}{\sim} \mathcal{N}(0,1)$.
For different training sample sizes $n$,  we set the corresponding dimension as $d=\lfloor n^{1/\gamma}\rfloor$. For each kernel and dimension $d$, we consider the following regression function $f_{\star}$:
\begin{equation}\label{experiment true function}
    f_{\star}(x) = 
    \left\{\begin{matrix}
\sum_{i=1}^{3} \mathscr{K}(\xi_{i},x), & \quad \text{ if } s=1; \vspace{5pt}\\
\sum_{i=1}^{3} \sum_{k=0}^{2} d^{k(1-s)/2} P_{k, d}(\langle\xi_{i}, x\rangle), & \quad \text{ if } s \neq 1;
\end{matrix}\right.
\end{equation}
where $P_{k, d}(\cdot)$ is the $k$-th normalized Gegenbauer polynomial with parameter $(d-2)/2$ (see, e.g., Appendix M.2).
{We draw a proposed triplet by sampling $\xi_{1}, \xi_{2}, \xi_{3}$ independently from the uniform distribution on $\mathbb{S}^{d-1}$. We accept it only if the constraint below holds; otherwise, we discard the entire triplet and repeat the sampling. The acceptance constraint is} 
    \begin{equation}\label{eqn_req_inner_of_xi_in_exp}
        \sum_{i \neq i^{\prime}}^{3} \left| \langle \xi_{i}, \xi_{i^{\prime}} \rangle \right| \leq 2.9,
    \end{equation}
    which excludes nearly collinear configurations. 
    Once sampled, the triplet remains fixed for each experimental run (i.e., for each combination of $n$, $d$, and $s$), but is resampled across different runs. 
One can verify that $f_{\star}$ satisfies Assumption \ref{assump_function_calss} and Assumption \ref{assump_source condition_lower_bound} (see, e.g., Appendix M.2).

\begin{remark}
The condition (\ref{eqn_req_inner_of_xi_in_exp}) is very mild. Indeed, let
\(        Q_d:=\sum_{1\leq i\neq i^{\prime}\leq 3}
        \left|\langle \xi_i,\xi_{i^{\prime}}\rangle\right|.
   \)
By rotational invariance of the uniform distribution on $\mathbb{S}^{d-1}$,
$\E[\xi_i\xi_i^\top]=\frac{1}{d}\mathrm{I}_d$, $\E\left[\langle\xi_i,\xi_{i^{\prime}}\rangle^2\right] =\frac{1}{d}$, $i\neq i^{\prime}$.
Since the sum defining $Q_d$ contains six terms, the Cauchy--Schwarz inequality gives
\(        \E[Q_d^2]
        \leq
        6\sum_{i\neq i^{\prime}}^3
        \E\left[\langle\xi_i,\xi_{i^{\prime}}\rangle^2\right]
        =36/d.
   \)
Hence $Q_d=O_{\mathbb{P}}(d^{-1/2})$. In particular, for any fixed $\varepsilon>0$,
\(        \mathbb{P}(Q_d\geq\varepsilon)
        \leq 36/(\varepsilon^2 d)
        \longrightarrow 0.
   \)
Therefore, the fixed-threshold requirement in (\ref{eqn_req_inner_of_xi_in_exp}) holds with probability tending to one as $d\to\infty$.
\end{remark}


For each $n$, we repeat the experiments 50 times and present the average generalization error $N^{-1}\sum_{i \leq N} (\hat{f}_{\lambda}(z_i) - f_{\star}(z_i))^{2}$ of (a) kernel gradient flow (KGF); (b) kernel ridge regression (KRR) on $N = 1,000$ testing samples $\{z_{i}\}_{i \leq N}$.
We fit a least squares line on the log scale to the excess risk as a function of the dimension, which yields an estimate of the convergence rate (i.e., the slope reported in the figure legends).


\begin{remark}
    We select KGF and KRR for experiments due to their dual significance. 
    On one hand, they are central and widely studied instances of spectral algorithms for kernel regression.
    On the other hand, the class of analytic spectral algorithms is characterized by a qualification parameter $\tau \in [1, \infty]$, with KRR and KGF serving as canonical representatives of the two ends, $\tau=1$ and $\tau =\infty$. 
    Hence, their agreement with the predicted slopes shows that the large-dimensional asymptotic picture in Theorem~\ref{theorem_learn_curve} is already visible at finite sample sizes for the two canonical endpoint filters, and provides empirical support for the broader theory of analytic spectral algorithms.
\end{remark}

\noindent {\bf Type 1 experiments:} We consider four parameter settings of  $(\gamma, s, u)$ as follows:
\begin{itemize}
    \item {$(\gamma, s, u) =(1.5, 1.5, 0.5)$, $n$ from $2600$ to $5000$, in increments of $200$, $d = \lfloor n^{2/3} \rfloor$.}

    \item {$(\gamma, s, u) =(0.8, 1.0, 2.0)$, $n$ from $1000$ to $3000$, in increments of $100$, $d = \lfloor n^{5/4} \rfloor$.}

    \item {$(\gamma, s, u) =(1.0, 1.0, 2.0)$, $n$ from $1000$ to $3000$, in increments of $100$, $d = n$.}

    \item {$(\gamma, s, u) =(1.2, 2.0, 1.0)$, $n$ from $1600$ to $4000$, in increments of $150$, $d = \lfloor n^{5/6} \rfloor$.}
    
\end{itemize}
We set $\lambda = C_{\lambda} d^{-u}$ with $C_{\lambda} \in \{0.001, 0.01, 0.1, 0.2, 0.4, 0.5, 0.7, 0.8, 1, 10\}$.
{We fixed the grid of regularization multipliers in advance and ran the experiments for every value of $C_\lambda$ in the grid. We display results for selected values.}
As shown in 
Figure \ref{experiment_1_settings_1}, as well as Figures 5 and 6 in Appendix M.1,
the empirical convergence rates agree with the theoretical rates predicted by Theorem \ref{theorem_learn_curve}.

\noindent {\bf Type 2 experiments:} 
In the second type of experiments, we fix $\gamma = 1.5$ and consider $s \in \{1, 2\}$. 
{We vary $n$ from $3000$ to $5000$ (in increments of $200$) and set $d = \lfloor n^{2/3} \rfloor$.} 
{For any given $(\gamma, s, u)$, as in the Type 1 experiments (see, e.g., Figure \ref{experiment_1_settings_1}), we fit a least squares line on the log scale to the excess risk as a function of the dimension, which yields an estimate of the convergence rate.}
We then report the estimated convergence rates with $u \in [0.1, 2.5]$, and compare them with the theoretical convergence rates given in Theorem \ref{theorem_learn_curve}. 
As shown in Figure \ref{experiment_2_s_1}, the estimated convergence rates align well with the theoretical convergence rates.


\begin{figure}
\centering
\subfigure{\includegraphics[width=0.45\columnwidth]{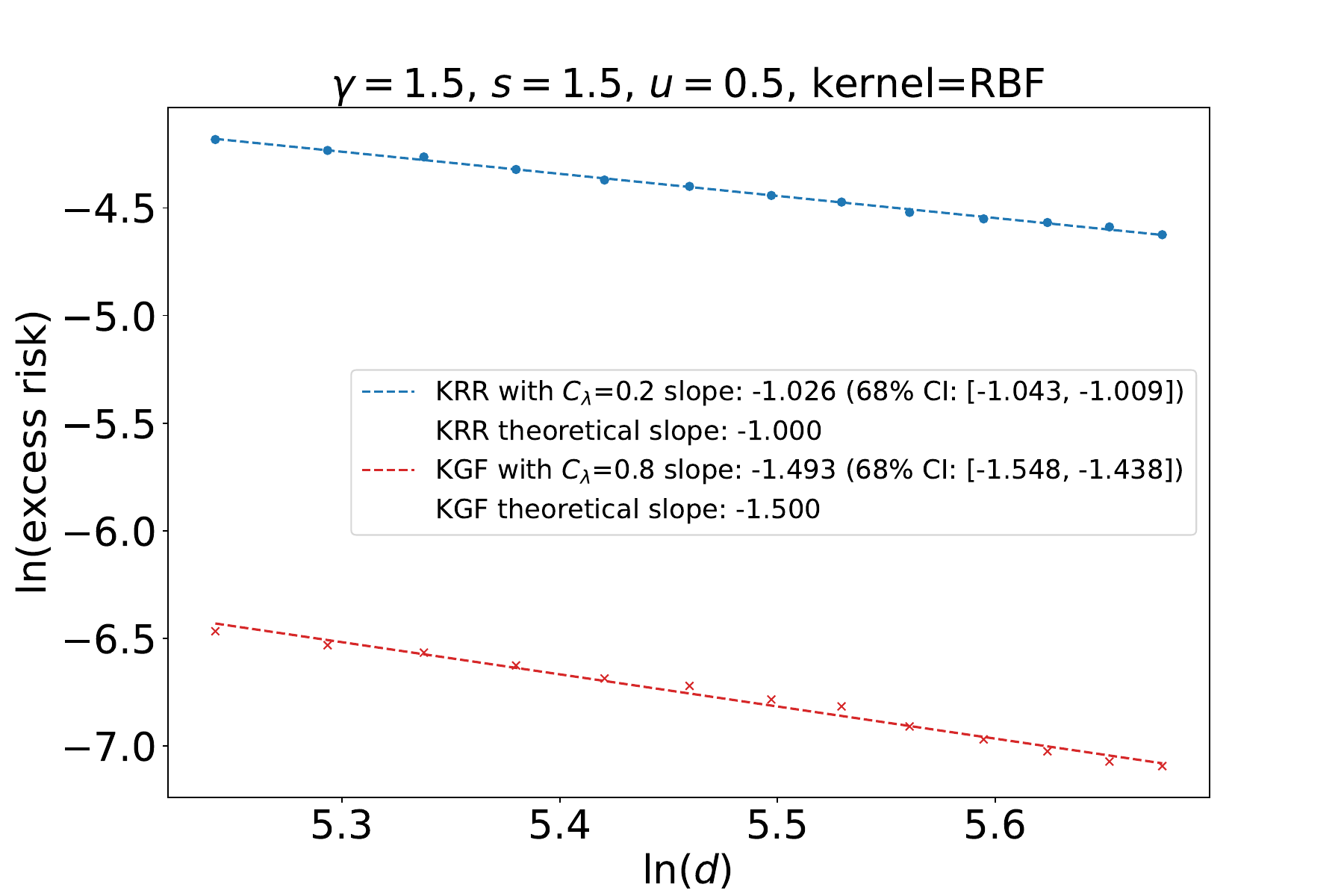}}
\subfigure{\includegraphics[width=0.45\columnwidth]{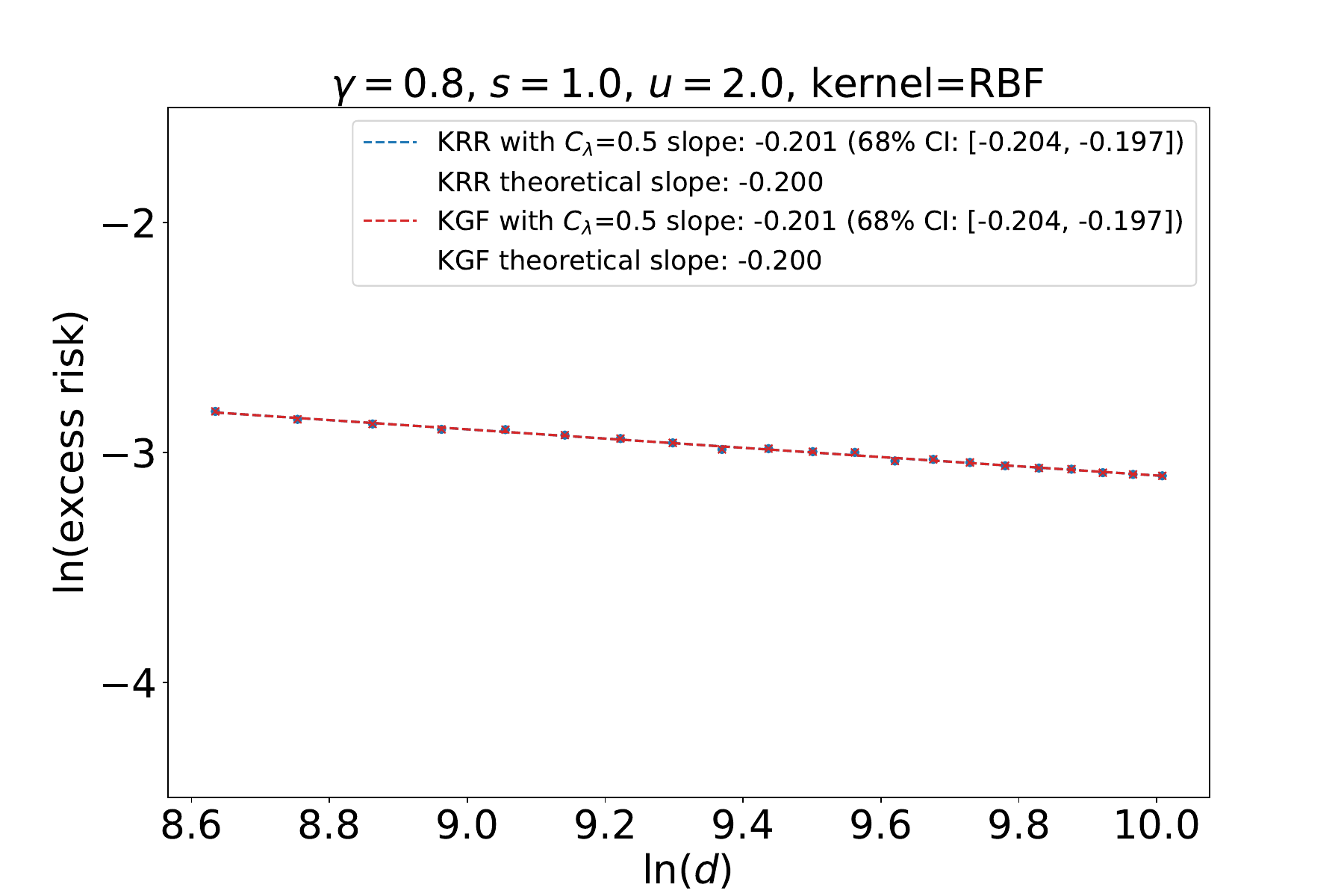}}
\caption{
Left: Type 1 experiments with parameters $(\gamma, s, u) = (1.5, 1.5, 0.5)$ and regularization $\lambda = C_{\lambda} d^{-u}$.
Right: Type 1 experiments with $(\gamma, s, u) = (0.8, 1.0, 2.0)$. 
In each panel, we plot $\ln(\text{Excess risk})$ versus $\ln(d)$ for KRR ($\tau=1$) and KGF ($\tau=\infty$), with each method using one selected $C_{\lambda}$ value.
Dashed lines show the least-squares fits. 
The legend reports the fitted slopes (with their 68\% confidence intervals), and the theoretical predictions from Theorem \ref{theorem_learn_curve}.
{Results for additional $C_{\lambda}$ values are provided in Figures 7 and 8 in Appendix M.1.}
}
\label{experiment_1_settings_1}
\end{figure}

\begin{figure}[htb]
\centering
\subfigure{\includegraphics[width=0.45\columnwidth]{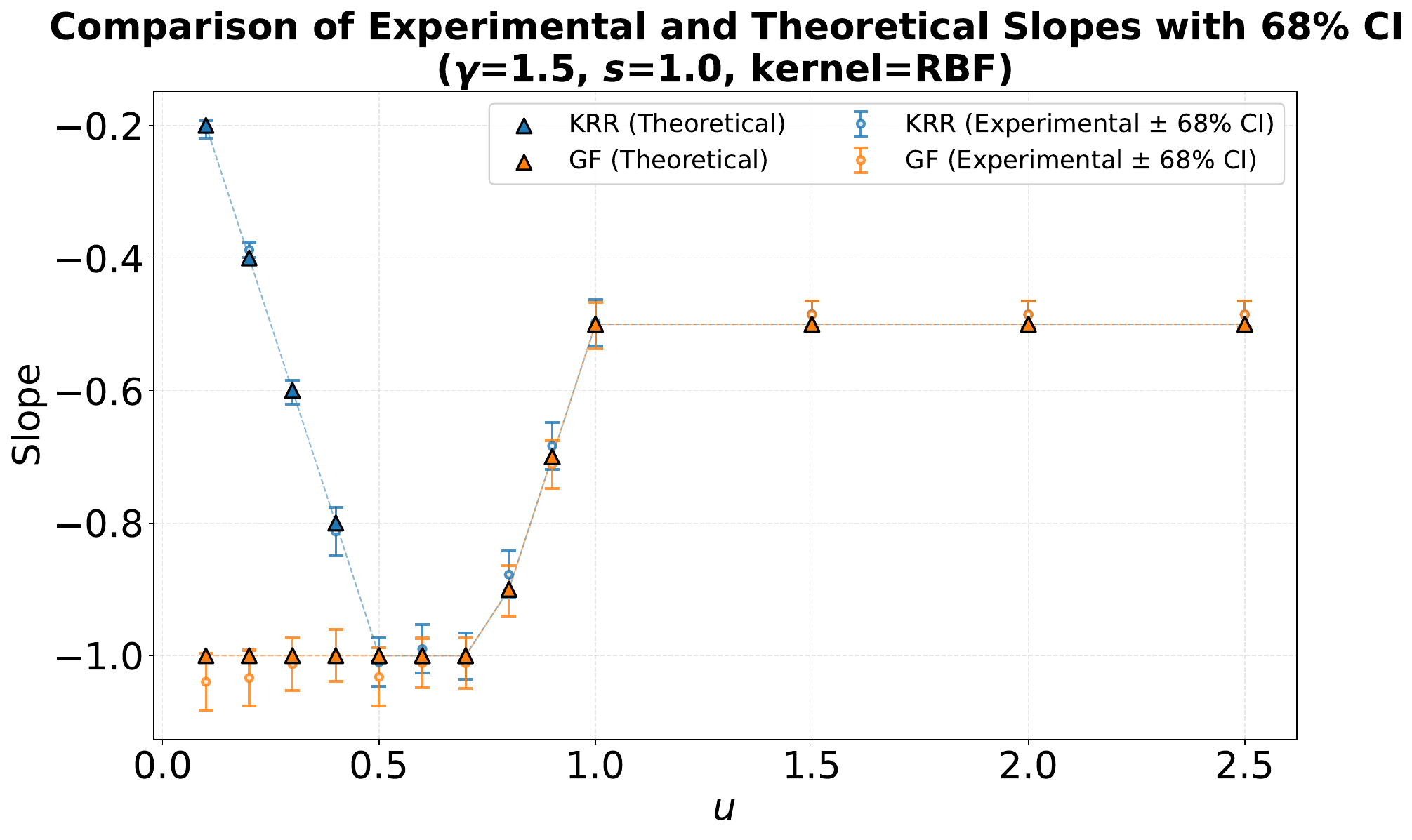}}
\subfigure{\includegraphics[width=0.45\columnwidth]{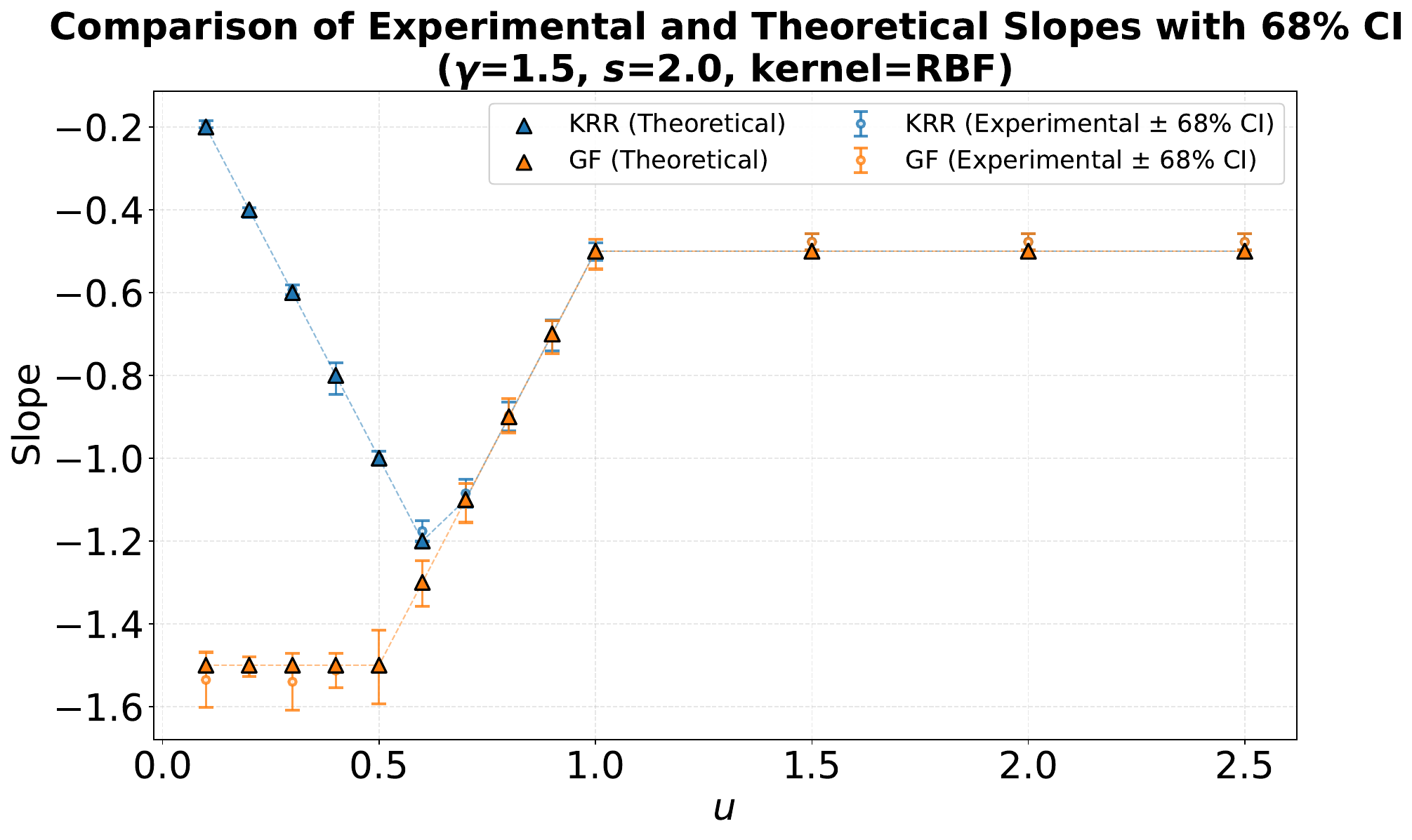}}

\caption{
Left: Comparison of the experimental and theoretical convergence rates for Type 2 experiments with $(\gamma, s) = (1.5, 1)$. 
Right: Type 2 experiments with $(\gamma, s) = (1.5, 2)$.
For each value of $u$, the experimental slope (circles with 68\% confidence intervals shown as error bars) is obtained by fitting a least squares line to $\ln(\text{Excess risk})$ against $\ln(d)$;
the theoretical slope (dashed lines with triangles) is predicted by Theorem \ref{theorem_learn_curve} (using $\tau=1$ for KRR and $\tau=\infty$ for KGF). 
}
\label{experiment_2_s_1}
\end{figure}

\section{Conclusion}\label{sec_conclusion}
We studied the learning curves of large-dimensional spectral algorithms in the large dimensional setting where
$n \asymp d^\gamma$. 
Our results for inner-product kernels on $\mathbb{S}^{d-1}$ show that under-regularization is a central part of the learning curve of large-dimensional spectral algorithms.
For the smoothness range identified by Theorem~\ref{coroll_prediction_persistence}, once the regularization reaches its optimal scale, the excess risk remains minimax rate-optimal under further training throughout the under-regularized regime and all the way to exact interpolation, where the training loss is zero. This shows that the prediction-optimal solution need not be a sharply tuned point: in part of the parameter space there is a whole path from optimal tuning to interpolation with the same prediction rate. Outside this range, we identify an under-regularized interval with a polynomial prediction-risk gap relative to the minimax benchmark. 
The general-domain KRR result further suggests that the same behavior is driven by low-degree spectral scaling and hyper-contractivity. 
Extending this general-domain analysis from KRR to broader analytic spectral algorithms is a natural next step.

%

\clearpage

\section*{Acknowledgments}
D.H. is the corresponding author. 
D.H. is supported by Singapore MOE AcRF Tier-1 Grant A-8004149-00-00.
Q.L. is supported by the National Natural Science Foundation of China (9237012, 11971257) and Tsinghua University Dushi Program (20251080059). 
Y.X. is supported by the SUSTech-NUS Joint Research Program and the National Natural Science Foundation of China (12571312). 
AI-assisted tools were used for assisting with parts of the proof development, including Theorem E.1, Lemma E.4, Theorem G.3, Lemma G.4, Proposition G.5, Lemma H.1, and Theorem 4.3.
AI-assisted tools were also used for language polishing, checking the correctness of mathematical proofs.

\bibliographystyle{chicago}
\bibliography{reference.bib} 

\clearpage

\appendix

\renewcommand{\theequation}{S\arabic{equation}}
\renewcommand{\theHequation}{supp.\arabic{equation}}
\setcounter{equation}{0}
\setcounter{figure}{4}

\begin{appendix}

The appendix is organized as follows: 
Appendix \ref{append_notation} introduces all notations used throughout this Appendix; 
Appendix \ref{sec_related_work} provides a detailed comparison between our results and the existing literature;
{Appendix \ref{append_sketch_proof_main_results} provides a sketch of the proof of Theorem 3.1;
Appendices \ref{append_variance}, \ref{append_bias}, and \ref{append_learning_curve} contain the proof of Theorem 3.1;
Appendix \ref{append_kernel_app_sphere} restates and develops several approximation results for kernels defined on the sphere;
Appendix \ref{append_kernel_app_general} generalizes results in Appendix \ref{append_kernel_app_sphere} to kernels on general domains and proves Theorem 4.1;
Appendix \ref{sec_proof_minimax_general} determines the minimax rate on a general domain, as stated in Theorem 4.3;}
Appendix \ref{appendix_proof_other_results} proves other results in the main text;
Appendix \ref{appendix_filter_property} discusses properties of the filter function defined in Definition 2.2; 
Appendix \ref{appendix_auxiliary} lists some auxiliary lemmas;
and Appendix \ref{appendix_experiments} provides further experiments and supporting proofs for Section 5.

\section{Notations}\label{append_notation}

In this section, we first specify all the absolute constants on which our asymptotic notations (e.g., $O_{d}()$ and $O_{d, \mathbb{P}}()$) depend. 
That is, any implicit constants appearing in our results depend only on the absolute constants listed below. 
Next, we introduce the notations used throughout this Appendix, most of which are adopted from \cite{ghorbani2021linearized, zhang2024phase}.


\begin{definition}\label{def:abs_constants}

    We list all the absolute constants used in this paper:
    
    \begin{itemize}

    \item $\sigma^2 \in (0, \infty)$: Fixed variance of the noise in (1).

    \vspace{3pt}
    
    \item \(c_1, c_2, \gamma > 0\): Positive constants in the asymptotic framework (2). 

    \vspace{3pt}

    \item \(a_{-1}; a_0, a_1, \ldots, a_{\lfloor \gamma \rfloor + 2} > 0\) : {An upper bound on the sum and the first} $(\lfloor \gamma \rfloor+ 3)$ coefficients of the Taylor expansion of \(\Phi(\cdot)\) as specified in Assumption 2.

    \vspace{3pt}

    \item \(s \geq 0\), $R>0$: Coefficients in the source condition  (6).

    \vspace{3pt}

    \item $c_{0} >0$: A constant given in Assumption 5.

    \vspace{3pt}
    
    \item $C_1, C_2 > 0$, $\mathfrak{C}_{i}$ $i=3, \cdots, 8$, $\tilde{E}, \tilde{F}$, $K_{\max}$: The constants related to the filter function defined in Definition 2.2 and Appendix \ref{appendix_addtion_require_of_reg}.

    \vspace{3pt}

    \item $\tau^{\prime}$: A finite constant defined in \eqref{eqn_def_of_tau_prime}. We introduce this notation to emphasize that constants used in probabilistic bounds need not be uniform over the entire class of filters or over $\tau$ unless this is separately assumed.

    \vspace{3pt}
    
    \item \(c_3, c_4 > 0\): Positive constants in Assumption 4. 

    \vspace{3pt}

    \item $\tilde{\ell}=\mathbf{1}\{u < \infty\}\min\{\lfloor \gamma \rfloor, \lfloor u \rfloor\} + \mathbf{1}\{u = \infty\}\lfloor \gamma \rfloor$ defined in Assumption 4.

    \vspace{3pt}

    \item $\min\{u - \tilde{\ell}, 1\} \in [0, 1]$: here $u$ is the decay rate of the regularization parameter $\lambda$ in Assumption 4.

    \end{itemize}
\end{definition}

For Section 4 and its appendices, implicit constants may additionally depend
on the uniform trace bound, the upper and lower block-scaling constants in
Assumption 7, and the hypercontractivity constants in Assumption 8.
Unless explicitly stated otherwise, the probabilistic asymptotic statements
are for a fixed admissible sequence of kernels and design distributions.
The starting dimension in such a statement may also depend on the
convergence moduli in Assumption 9.  No uniformity over all kernel--design
sequences satisfying a merely pointwise $o_{d,\mathbb P}(1)$ condition is
asserted.  This convention does not remove the uniformity over the target
source ball and admissible noise distributions in the minimax theorem,
where the kernel and design sequence are fixed.



\paragraph*{Notations on eigenfunctions and spherical harmonics}
Denote $\psi(x) = (\phi_j(x))^{\top}_{j \ge 1} \in \mathbb{R}^{\infty}$, where $\phi_j(\cdot)$ are the eigenfunctions given in (4). For any $k =0, 1, \cdots$ and $j \leq N(d, k)$, denote $N_{k} = \sum_{k^{\prime}=0}^{k}N(d, k^{\prime})$, $ \Psi_{k, j} = \psi_{k,j}(X) = \left(\psi_{k,j}(x_{1}), \cdots, \psi_{k,j}(x_{n}) \right)^{\top} \in \mathbb{R}^{n}$, and
\begin{displaymath}
    \Psi_{k} = \left(\Psi_{k,1}, \cdots, \Psi_{k,N(d,k)} \right) \in \mathbb{R}^{n \times N(d,k)}.
\end{displaymath}
Denote $\Psi =  \left(\psi(x_{1}), \cdots \psi(x_{n}) \right)^{\top} \in \mathbb{R}^{n \times \infty}$. For $k =0, 1, \cdots$, denote 
$$
\Psi_{\leq k} = \left(\Psi_{0}, \cdots, \Psi_{k} \right) \in \mathbb{R}^{n \times N_{k}},
$$
then Proposition 2.1 and the Funk-Hecke formula in (5) together imply that for any $k \leq \ell_{\gamma}+1$, we have $\Psi = \left( \Psi_{\leq k}, \Psi_{>k} \right)$, where $\Psi_{>k} = \left(\Psi_{j_1}, \Psi_{j_2}, \cdots \right)$ with $\{j_1, j_2, \cdots\} \subset \{k+1, k+2, \cdots\}$.\\
When $\ell_{\lambda}<\ell_{\gamma}$, we further denote $\Psi_{\ell_{\lambda}+1, \ell_{\gamma}} = \left(\Psi_{\ell_{\lambda}+1}, \cdots, \Psi_{\ell_{\gamma}} \right)$.

\paragraph*{Notations on eigenvalues}
For any $k = 0, 1, \cdots$, define the diagonal matrix
\begin{displaymath}
    \Sigma_{k} = \mu_{k} \mathrm{I}_{N(d,k)} \in \mathbb{R}^{N(d,k) \times N(d,k)}.
\end{displaymath}


Denote $\Sigma = \text{Diag}\{ \Sigma_0, \Sigma_1, \Sigma_2, \cdots\}$. For $k = 0, 1, \cdots$, denote 
$$
\Sigma_{\leq k} = \text{Diag} \left\{\Sigma_{0}, \cdots, \Sigma_{k} \right\} \in \mathbb{R}^{N_{k} \times N_{k}}.
$$
Then Proposition 2.1 and the Funk-Hecke formula in (5) imply that for any $k \leq \ell_{\gamma}+1$, we have $\Sigma = \text{Diag}  \left\{ \Sigma_{\leq k}, \Sigma_{>k} \right\}$, where 
$\Sigma_{>k} = \text{Diag} \left\{\Sigma_{j_1}, \Sigma_{j_2}, \cdots \right\}$ with $\{j_1, j_2, \cdots\} \subset \{k+1, k+2, \cdots\}$.\\
When $\ell_{\lambda}<\ell_{\gamma}$, we further denote $\Sigma_{\ell_{\lambda}+1, \ell_{\gamma}} = \text{Diag} \left\{\Sigma_{\ell_{\lambda}+1}, \cdots, \Sigma_{\ell_{\gamma}} \right\}$.

\paragraph*{Notations on the regression function}
Recall that in (7) we define
\begin{equation*}
    f_{\star}(x) =\theta^{\top}  \psi(x)=\sum_{k=0}^{\ell_{\gamma}+1} \sum_{j=1}^{N(d,k)} \theta_{k,j} \psi_{k,j}(x) + \sum_{j =N_{\ell_{\gamma}+1}+1}^{\infty}f_j \phi_{j}(x),
\end{equation*}
where $\theta = \left( \theta_{0,1}, \theta_{1,1}, \cdots, \theta_{\ell_{\gamma}+1, N(d, \ell_{\gamma}+1)}, f_{N_{\ell_{\gamma}+1}+1}, f_{N_{\ell_{\gamma}+1}+2}, \cdots\right)^{\top}$.\\
For any $k =0, 1, \cdots, \ell_{\gamma}+1$, denote 
\begin{displaymath}
    \theta_{k} = \left( \theta_{k,1}, \cdots, \theta_{k,N(d,k)} \right)^{\top} \in \mathbb{R}^{N(d,k)},
\end{displaymath}
then we can denote $ \theta = \left( \theta^{ \top}_{\leq \tilde{\ell}}, \theta^{\top}_{>\tilde{\ell}} \right)^{\top}$, where
\begin{displaymath}
    \theta_{\leq \tilde{\ell}} = \left( \theta_{0}^{ \top}, \cdots, \theta_{\tilde{\ell}}^{\top} \right)^{\top} \in \mathbb{R}^{N_{\tilde{\ell}}}.
\end{displaymath}

\paragraph*{Notations on the kernel matrix}
Denote {the kernel matrix, which is also called the Gram matrix,} by $K := \mathscr{K}(X, X) \in \mathbb{R}^{n \times n}$, where for $i,j \leq n$, its $(i, j)$-th entry is $(K)_{ij} = \mathscr{K}(x_{i},x_{j})$. \\
For $k \in \{\ell_{\gamma}, \tilde{\ell}\}$, define
$$
K_{\leq k} := \Psi_{\leq k} \Sigma_{\leq k} \Psi_{\leq k}^{\top}, 
\qquad 
K_{>k} := \Psi_{>k} \Sigma_{>k} \Psi_{>k}^{\top}. 
$$

Define $M_{>k} := K_{>k} + n\lambda \mathrm{I}_n$, and
$$
M := K + n\lambda \mathrm{I}_n.
$$
When $\ell_{\lambda}<\ell_{\gamma}$, we further denote $ K_{\ell_{\lambda}+1, \ell_{\gamma}} := \Psi_{\ell_{\lambda}+1, \ell_{\gamma}} \Sigma_{\ell_{\lambda}+1, \ell_{\gamma}} \Psi_{\ell_{\lambda}+1, \ell_{\gamma}}^{\top}$.

\paragraph*{Discussion about the notations $\mathbb{R}^{\infty}$ and $ \mathbb{R}^{\infty \times \infty} $}
We do not assume $\mathbb{R}^{\infty}$ and $ \mathbb{R}^{\infty \times \infty} $ (or $ \mathbb{R}^{n \times \infty} $) to be well-defined normed spaces with specific norms or inner products. Instead, we only use these notations to highlight the shapes of vectors and matrices, which serve as the components of another vector belonging to a well-defined space (e.g., $\ell^{2}$ space or $\mathbb{R}$). Similar notations were widely used in related work \citep{ghorbani2021linearized, mei2022generalization, zhang2024phase}.

{Specifically, recalling the notations above,} the $\mathbb{R}^{\infty}$ vectors and $ \mathbb{R}^{\infty \times \infty}$ (or $\mathbb{R}^{n \times \infty}$) matrices in our paper include
    \begin{equation*}
         \psi(x) \in \mathbb{R}^{\infty}, ~~ \theta ~(\text{ or }\theta_{>\tilde{\ell}}) \in \mathbb{R}^{\infty},~~ \Sigma ~(\text{ or }\Sigma_{>\tilde{\ell}}) \in \mathbb{R}^{\infty \times \infty}, ~~ \Psi ~(\text{ or }\Psi_{>\tilde{\ell}}) \in \mathbb{R}^{n \times \infty}.
    \end{equation*}
These vectors and matrices appear in formulas that involve the following terms:
    \begin{equation*}
        \psi(x)^{\top} \Sigma \psi(x), ~~ \psi(x)^{\top} \Sigma \Psi^{\top}, ~~ \Psi \Sigma \Psi^{\top}, ~~ \theta, ~~ \Psi \theta, ~~ \psi(x)^{\top} \theta, ~~ \| \Sigma_{>\tilde{\ell}} \|_{\mathrm{op}}.
    \end{equation*} 
Below we explain that these terms are well-defined:
    \begin{itemize}
        \item [(i)] {By the definition of the kernel in Assumption 2,} $ \sup\limits_{x \in \mathcal{X}} \psi(x)^{\top} \Sigma \psi(x) = \sup\limits_{x \in \mathcal{X}} \mathscr{K}(x,x) \le a_{-1}.$ 
        
        \item [(ii)] {The Cauchy-Schwarz inequality shows that, for any $x, x^{\prime} \in \mathcal{X}$,}
        \begin{equation*}
            \left|\psi(x)^{\top} \Sigma \psi(x^{\prime})\right| \le [\psi(x)^{\top} \Sigma \psi(x)]^{1/2} \cdot [\psi(x^{\prime})^{\top} \Sigma \psi(x^{\prime})]^{1/2} \le a_{-1}.
        \end{equation*}
        Therefore, all entries in the $n$-dimensional vector $ \psi(x)^{\top} \Sigma \Psi^{\top} $ and the $n \times n$-matrix $\Psi \Sigma \Psi^{\top}$ are well-defined and bounded by $a_{-1}$. 
        
        \item [(iii)] The $\mathbb{R}^{\infty}$ vector $\theta$ belongs to $\ell^{2}$ space, since we always have $f_{\star} \in L^{2}$. In addition, $ \Psi \theta $ belongs to $ \mathbb{R}^{n}$, and Lemma \ref{lemma ge l l2} shows that the Euclidean norm $ \|\Psi_{>\tilde{\ell}} \theta_{>\tilde{\ell}}\|_{2} $ is finite. 
        
        \item [(iv)] For $ \psi(x)^{\top} \theta$, {it only appears in the derivation of the bias term, i.e., Eq.(\ref{eq derivation of bias_spe}).} Since Eq.(\ref{eq derivation of bias_spe}) always holds, the use of notation $ \psi(x)^{\top}\theta $ is just for the convenience of Parseval's theorem.

        \item [(v)] Since $\Sigma_{>\tilde{\ell}}$ is an infinite-dimensional matrix, we consider it as a continuous linear operator from $\ell^{2}$ to itself. Therefore, $\|\Sigma_{>\tilde{\ell}} \|_{\mathrm{op}} $ below (\ref{eqn_bound_B_3_spe}) just represents the largest diagonal element, i.e., $\mu_{\tilde{\ell}+1}$.
    \end{itemize}

\paragraph*{Discussion about the filter function} For notational simplicity, we denote $\varphi_{\lambda, K}:=n^{-1}\reg(K/n) \in \mathbb{R}^{n \times n}$. 
For any orthogonal matrix $B\in \mathbb{R}^{n\times n}$ and any diagonal matrix $S\in \mathbb{R}^{n\times n}$, from (11) we have
\begin{align}
    \label{eq:Filter_Rem_finite_case3}
    \reg(BSB^{\top})=B\reg(S)B^{\top}.
    \end{align}
Therefore, $\varphi_{\lambda, K}$ is a symmetric matrix, and from Definition 2.2 we have
\begin{equation}\label{eqn:spe_properties}
    \begin{aligned}
         &C_1 M^{-1} =   \frac{C_1}{n}\left(\frac{K}{n}+\lambda \mathrm{I}_n \right)^{-1} \leq \varphi_{\lambda, K} \leq  C_2 M^{-1}.
         \\
    \end{aligned}
\end{equation}

\section{Discussion on related works}\label{sec_related_work}

This section compares the assumptions, regularization ranges, algorithms, and approximation precision of the related results

\subsection{Learning curves of fixed-dimensional spectral algorithms}\label{subsec_learning_curve_fix_d}

{The learning curves of fixed-dimensional spectral algorithms became popular and have been well-studied over the past few years.}
Several early works attempted to describe the learning curve of KRR using heuristic arguments \cite{Bordelon_Spectrum_2020, Cui2021GeneralizationER, mallinar2022benign}. However, their results relied on the unrealistic Gaussian design assumption, that is, they assumed that the eigenfunctions $\{\phi_j(x)\}_{j=0}^{\infty}$ in (4) are normal random variables.

\citet{li2023asymptotic} developed an operator approach for fixed-dimensional KRR. Writing $T_X$ for the empirical covariance operator and $T$ for its population counterpart, their argument replaces empirical resolvents by population resolvents and controls the empirical-to-population discrepancy. In particular, the resolvent identity
\[
\reg^{\krr}(T_X)-\reg^{\krr}(T)
=\reg^{\krr}(T_X)(T-T_X)\reg^{\krr}(T)
\]
provides a convenient route to variance approximations. For the bias, the empirical regularized solution is compared with the population regularized solution. Under the conditions of their results, these comparisons yield deterministic equivalents for the prediction risk.

\citet{li2024generalization} extended this methodology to analytic spectral filters by using a contour-integral representation of the difference between functions of $T_X$ and $T$. This covers algorithms beyond KRR and makes filter qualification relevant to the resulting learning curves.

The regularization range is important. In the polynomial eigendecay setting $\lambda_j\asymp j^{-\beta}$ with $\beta>1$, the relative approximation for analytic filters in Theorem 3.1 of \citet{li2024generalization} applies when $\lambda\to0$ and $\lambda=\Omega(n^{-\theta})$ for some fixed $\theta<\beta$. At $\lambda=O(n^{-\beta})$, that theorem instead supplies a logarithmic lower bound. Thus its relative approximation does not by itself describe the path to interpolation. In large dimensions, the operator concentration arguments used in \citet{zhang2024optimal,lu2024saturation} similarly require sufficiently strong regularization. Our analysis combines operator methods on their applicable range with a separate treatment of weaker regularization and the intermediate spectral blocks.

\subsection{Recent advances in large-dimensional spectral algorithms}\label{subsec_related_work_large_dimension_spe}

\paragraph*{Consistency of KRR / KGF.}

Several recent studies \citep{ghorbani2021linearized, Donhauser_how_2021, mei2022generalization, misiakiewicz_spectrum_2022, hu2022sharp} focus on the consistency of large-dimensional KRR and KGF with $s=0$. Roughly speaking, they showed the following approximations for KRR and KGF:
\begin{equation}\label{eqn_restate_thm_4_in_linear}
    \left|\E_{x, \epsilon} \left[ \left(\hat{f}_{\lambda}(x) - f_{\star}(x) \right)^{2} \right]-
    \text{E}_{\text{approx}}
    \right| =o_{d, \mathbb{P}}\left(1\right),
\end{equation}
where $\text{E}_{\text{approx}}$ is a certain deterministic quantity (see, e.g., Theorem 4 in \cite{ghorbani2021linearized}).

An additive $o_{d,\mathbb P}(1)$ approximation alone does not identify the
polynomial decay rate in regimes where the excess risk tends to zero.
For example, this distinction matters for fixed $s>0$ and fixed $u>0$
in the noninteger spherical regimes studied here.

\paragraph*{Optimal tuning and minimax benchmarks.}

\citet{lu2023optimal} analyzed optimal prediction for KGF in the RKHS case. Subsequent work \citep{zhang2024optimal,lu2024saturation} studied optimally tuned spectral algorithms and the role of source smoothness and filter qualification. These results provide important benchmarks for the regularized part of the learning curve, but their stated parameter ranges do not encompass all the small-$s$ and small-$\gamma$ cases. The learning curves for small $s$ and
small $\gamma$ are of particular interest:

\begin{itemize}
    \item As shown in Theorem 3.2 and its subsequent discussion, the strong generalization observed for highly trained or interpolating predictors suggests that the learning curves for small $s$ merit further study. 

    \item The above studies \cite{ghorbani2021linearized, Donhauser_how_2021, mei2022generalization,  misiakiewicz_spectrum_2022, hu2022sharp} that considered the extreme case where $s=0$ motivate us to investigate the learning curves for all $s$, to provide a unified explanation of these results.

    \item Many large-dimensional datasets in machine learning, especially image datasets, often have dimensions that far exceed the sample size, further encouraging our focus on small $\gamma$. 
\end{itemize}

The present Theorem 3.1 gives a common expression for conditional risk orders across the admissible polynomial regularization scales in these cases as well, including regimes not covered by the earlier operator approximations. The complete adaptation used in our proof is given in Appendix \ref{append_proof_of_previous_prop_B_1}. Comparing these orders yields the persistence and deterioration conclusions in Theorem 3.2.

Finally, \citet{lu2024pinsker} obtained a sharp asymptotic minimax risk, including its leading Pinsker constant, for large-dimensional kernel regression under stronger sampling and kernel coefficient assumptions. In particular, the associated minimax rate is described by the exponent $\min\{\gamma-p,(p+1)s\}$ with $p=\lfloor\gamma/(s+1)\rfloor$.

\paragraph*{Generalization of KRR in the interpolation regime}

A substantial body of recent work has analyzed large-dimensional KRR in the interpolation regime (i.e., when $\lambda=O_d(n^{-1})$), with particular focus on the case $\lambda=0$ (i.e., kernel interpolation). 
For example, \cite{Liang_Just_2019} showed consistency of kernel interpolation with $s=\gamma=1$. 
\cite{liang2020multiple} and the spherical specialization in Theorem 5.1 of \cite{barzilai2023generalization} provide interpolation upper bounds. The general Theorem 4.1 of \cite{barzilai2023generalization} also treats arbitrary regularization, with a coefficient depending on the empirical tail kernel matrix.

\cite{aerni2023strong} provide matching bias and variance orders over a polynomial range of positive ridge parameters and at
interpolation for convolutional kernels on the discrete hypercube. Their target is $f_{\star}(x) = x_{1}x_{2}\cdots x_{L^{*}}$, where $L^{*}$ is the fixed degree in their Theorem 1. Their ridge parameter equals $n\lambda$ in our normalization.

\citet{zhang2024phase} characterize the excess risk of spherical kernel interpolation and identify the corresponding minimax phase boundary. For noninteger $\gamma$, their interpolation risk consists of three contributions, with polynomial orders $d^{\ell_\gamma-\gamma}$, $d^{\gamma-\ell_\gamma-1}$, and $d^{-(\ell_\gamma+1)s}$, corresponding to low-degree variance, high-degree variance, and residual signal energy. These interpolation risk orders are also recovered by Theorem 3.1 when $u=\infty$. 

At integer $\gamma$, the sample size is comparable to the dimension of a full harmonic block. Weak convergence of the empirical spectral distribution does not alone control the extreme eigenvalues needed in the proof. Theorem \ref{thm:critical} and Lemma \ref{lemma psi top psi_int} supply the additional matrix estimates used here.

The distinction is also technical. Interpolation analysis separates the degrees at $\ell_\gamma$, uses concentration of the high-degree kernel tail, and exploits the Sherman--Morrison--Woodbury identity for the KRR inverse. At positive regularization, the effective cutoff can be $\ell_\lambda<\ell_\gamma$; the intermediate degrees $\ell_\lambda<k\leq\ell_\gamma$ then contribute to both bias and variance and cannot be discarded. Furthermore, inverse identities specific to KRR do not directly provide corresponding results for analytic filters of different qualifications. Our argument treats these intermediate blocks and filter-dependent contributions to obtain matching conditional risk orders along the regularization path.

\subsection[Comparison of the learning curves of kernel regression and sequence models]{Comparison of the learning curves of kernel regression and sequence models}\label{sec_sequence_model}

Sequence models have been widely considered as simplifications for the analysis of kernel regression models. For example, the celebrated work by \cite{pinsker1980optimal} and \cite{nussbaum1985spline} and subsequent ones show that {the exact asymptotics of the minimax risks}, i.e., the Pinsker bounds, for Gaussian sequence models and fixed-dimensional kernel regression models over \( [\mathcal{H}]^{s} \) for \( s > 1 \) are identical; \cite{brown1996asymptotic, carter2006continous, reiss2008asymptotic} further showed {the classical Le Cam equivalence between the above kernel regression models and the Gaussian sequence models;}
\cite{li2024improving, li2025diagonal} focus on the sequence model and kernel regression model separately, showing that the order of eigen-functions impacts regression outcomes, and that over-parameterization enhances the generalization capability.
For ridge regression, \cite{cheng2022dimension, misiakiewicz2024non} identify
deterministic risk expressions with effective sequence interpretations and quantify their approximation
to the risk under random design.

\subsubsection{Comparison with existing ridge and sequence results}\label{subsec_related_work_sequence_model}

\paragraph{Ridge risk bounds.}
The regularization parameter in \citet{tsigler2023benign} equals $n\lambda$ in our normalization. Let $N_a=\sum_{k=0}^{a}N(d,k)$ be a spherical harmonic block endpoint. For fixed $a\leq\ell_\gamma$, Proposition~2.1 and the bounded trace imply
\[
\sum_{j>N_a}\lambda_j\asymp 1,
\qquad \sum_{j>N_a}\lambda_j^2\asymp d^{-(a+1)}.
\]
Thus the effective scale in their equations (6)--(7), $\lambda+n^{-1}\sum_{j>N_a}\lambda_j$, has order $\max\{\lambda,n^{-1}\}$.

Let $v=\min\{u,\gamma\}$ for finite $u$, and $v=\gamma$ at interpolation. Write $a=\lfloor v\rfloor$ and take $t_d\asymp d^{-v}$. Under our source and corresponding signal nondegeneracy conditions, the population ridge variance and squared bias have orders
\[
\frac{\sigma^2}{n}\sum_j\left(\frac{\lambda_j}{\lambda_j+t_d}\right)^2
\asymp d^{a-\gamma}+d^{2v-\gamma-a-1},
\]
\[
\sum_j\left(\frac{t_d}{\lambda_j+t_d}\right)^2f_j^2
\asymp d^{-(a+1)s}+d^{-2v+(2-\min\{s,2\})a}.
\]
For $s\geq2$, the second bias term uses the lower bound on the entire head energy in Assumption~5. At and beyond the interpolation scale, this term is bounded by $d^{a-\gamma}$. These scalar expressions therefore give the same KRR total-risk orders as Theorem~3.1.

The statistical bounds of \citet{tsigler2023benign} have additional hypotheses. Their upper bounds allow dependent whitened coordinates under sub-Gaussianity and control of the regularized tail Gram matrix. The variance lower bound in their Lemma~7 uses independent coordinates; the bias lower bound in their Lemma~8 averages over independent random signs of the coefficients and requires a lower-eigenvalue condition. Their Theorem~10 relates the block and population expressions under effective-rank or cutoff conditions. In particular, a block cutoff of order $n$ at integer $\gamma$ requires a separate check of the restriction $k<n/c$. These conditions distinguish the transfer of probability bounds from the agreement of the scalar risk orders.

\paragraph{Effective sequence risk.}
\citet{cheng2022dimension} derive relative approximations for the conditional bias and variance of ridge regression and an effective sequence representation with adjusted regularization and noise. \citet{misiakiewicz2024non} develop deterministic equivalents for KRR, including separate relative approximations for conditional bias and variance under quantitative concentration and residual moment conditions. The ridge parameter in \citet{cheng2022dimension} uses our normalization, whereas the parameter in \citet{misiakiewicz2024non} equals $n\lambda$.

With population eigenvalues $\lambda_j$, the effective regularization $\widetilde\lambda$ is determined by
\begin{equation}\label{eqn_lambda_star}
 n\left(1-\frac{\lambda}{\widetilde\lambda}\right)
 =\sum_j\frac{\lambda_j}{\lambda_j+\widetilde\lambda}.
\end{equation}
The positive spectral rank in our setting exceeds $n$ for sufficiently large $d$, so this equation has a unique positive solution also at $\lambda=0$. Define
\[
q=\frac1n\sum_j\left(\frac{\lambda_j}{\lambda_j+\widetilde\lambda}\right)^2<1,
\qquad b=\sum_j\left(\frac{\widetilde\lambda}{\lambda_j+\widetilde\lambda}\right)^2f_j^2,
\qquad R_{\mathrm{seq}}=\frac{b+\sigma^2q}{1-q}.
\]
For independent Gaussian noise $\widetilde\xi_j$ of variance $(\sigma^2+R_{\mathrm{seq}})/n$, let $\widetilde z_j=f_j+\widetilde\xi_j$ and $\widetilde f_j=\lambda_j\widetilde z_j/(\lambda_j+\widetilde\lambda)$. The identity
\[
\mathbb E_{\widetilde\xi}\sum_j(\widetilde f_j-f_j)^2
=b+(\sigma^2+R_{\mathrm{seq}})q=R_{\mathrm{seq}}
\]
gives the sequence interpretation of this deterministic risk expression. Unlike the population sequence model in Section~3.4, this effective model adjusts both the regularization and the noise variance.

To relate this expression to conditional kernel risk, let $m$ be the spectral truncation index of \citet{misiakiewicz2024non}. Suppose their Assumption~1 and quantitative conditions (151) hold with fixed concentration constants and fixed $D,K>0$ as in their Theorem~12. Suppose also that their failure probability $p_{2,n}(m)$ and relative error $\mathcal E_{R,n}(m)$ in equation (152) tend to zero. The separate bias and variance approximations then give
\begin{equation}\label{eqn_prop_restate_misia}
\mathbb E_{\epsilon\mid X}\|\widehat f_\lambda-f_\star\|_{L^2}^2
=R_{\mathrm{seq}}\bigl(1+o_{d,\mathbb P}(1)\bigr).
\end{equation}
For the conditional risk considered here, the noise requirement can be checked using independent Gaussian noise with covariance $\sigma^2\mathrm I_n$: conditional bias and variance depend only on the conditional noise mean and covariance. Using the separate approximations in Theorem~12 avoids subtracting the test-noise variance from a relative approximation to total test error.

Proposition~3.4 instead uses fixed sequence noise variance $\sigma^2/n$ and a deterministic admissible regularization scale of order $\max\{\lambda,n^{-1}\}$. Its constant-factor comparison describes the risk orders for the spectral filters considered here, as obtained from Theorem~3.1, including the interpolation endpoint.

\paragraph{Target conditions.}
In our Mercer coordinates, the target condition of \citet{cheng2022dimension} requires $\sum_j f_j^2/\lambda_j^2<\infty$. A source ball with $s\geq2$ bounds this quantity. Applying their relative bias approximation also requires their quantitative target condition, mean-zero features, and concentration of the entire whitened feature vector; the latter allows independent sub-Gaussian coordinates or the stated convex concentration property. Thus the relative approximation and our source-condition risk orders address different sets of assumptions, even in their overlapping parameter ranges.

The general Assumption~1(c) of \citet{misiakiewicz2024non} controls the target residual after spectral truncation relative to its own $L^2$ norm. Their spherical Assumption~6(c) supplies sufficient conditions in terms of full target moments and a tail-energy fraction,
\[
\|\mathrm P_{>L}f_\star\|_{L^2}\geq\|f_\star\|_{L^2}/C_L,
\]
where $L$ is a fixed positive integer, $C_L$ is independent of $d$, and $\mathrm P_{>L}$ projects onto spherical harmonics of degrees greater than $L$.

For $L=\ell_\gamma+1$ and $s>0$, our source bound $\|f_\star\|_{[\mathcal H]^s}^2\leq R$ and bounded trace imply
\[
\|\mathrm P_{>L}f_\star\|_{L^2}^2
\leq R\left(\sup_{k>L}\mu_k\right)^s
\leq R\left(\frac{\operatorname{tr}(T)}{N(d,L+1)}\right)^s
=O_d\bigl(d^{-(L+1)s}\bigr).
\]
The second inequality uses that harmonic multiplicities increase with degree for $d\geq3$. Assumption~5 also gives $\|f_\star\|_{L^2}^2\geq c_0>0$. Hence the tail-energy fraction tends to zero, so this spherical sufficient condition does not hold for these target sequences. The general theorem has its own residual-moment and quantitative approximation conditions, which must be checked separately for the target sequence under consideration.

Our results address a complementary question by deriving explicit polynomial risk orders under the source and signal conditions considered here. Theorem 3.1 treats analytic spectral algorithms with different qualifications, including KRR and kernel gradient flow, while Theorem 3.2 uses these orders to characterize prediction persistence and deterioration along the regularization path. In the KRR case, the deterministic-equivalent framework of \citet{misiakiewicz2024non} provides a related route to risk approximation whenever its concentration, residual moment, and quantitative error conditions can be verified.

\section{Proof sketch for Theorem 3.1 in a representative regime}\label{append_sketch_proof_main_results}

To highlight the main ideas without carrying the full case analysis, we focus on the case where
\begin{equation}\label{eqn_sketch_0}
    \begin{aligned}
    &~ u \in [u^{\prime}(1), \gamma) \text{ (that is, the under-regularized regime), }\\
    &~ \ell_{\lambda}+1 < \ell_{\gamma} < \gamma, \quad s>0, \quad \text{ and } \tau < \infty
\end{aligned}
\end{equation}

This case already contains the main additional difficulty in proving the full learning curve.
Indeed, once $\ell_{\lambda}<\ell_{\gamma}$, the proof must be organized at the cutoff $\ell_{\lambda}$ rather than only at $\ell_{\gamma}$.  
This creates the intermediate block between degrees $\ell_{\lambda}+1$ and $\ell_{\gamma}$, whose contribution must be tracked explicitly, especially in the bias analysis. 
This is the key new obstacle that does not appear in the same form in interpolation-level arguments or in proofs restricted to the optimally tuned regime.

The remaining cases are handled by the same general strategy, together with appropriate case-dependent modifications of several known results from the literature.
For example, for the case where $u \in [u^{\prime}(\tau), u^{\prime}(1))$, the bias bound requires a variant of the argument in Appendix D.4.2 of \cite{lu2024saturation}; see Proposition \ref{prop_verify_bias_special}. 
The complete proof treats all cases systematically.



The proof has three ingredients. 
First, we combine kernel approximation results with a decomposition of $K$ at the cutoff index  $\ell_\lambda$. 
Second, we use this decomposition to bound the variance and the bias separately. 
Third, we compare the resulting exponents to identify the dominant term in each range of $u$.

\paragraph*{Key kernel approximation results}
We first collect some kernel approximation results, and readers can refer to Appendix \ref{append_kernel_app_sphere} for details.

From \cite{ghorbani2021linearized}, we have
\begin{equation}\label{eqn_sketch_1}
    \left\|\Psi_{>\ell_{\gamma}} \Sigma_{>\ell_{\gamma}}^j \Psi_{>\ell_{\gamma}}^{\top} - \kappa_j\mathrm{I}_n \right\|_{\mathrm{op}} = o_{d, \mathbb{P}}(\kappa_j), \quad \kappa_j = \sum_{k=\ell_{\gamma}+1}^\infty \mu_k^j N(d, k),\quad j=1,2;
\end{equation}
and that
\begin{equation}\label{eqn_sketch_2}
\left\| n^{-1}{\Psi_{\leq \ell_{\gamma}}^{\top} \Psi_{\leq \ell_{\gamma}} } - \mathrm{I}_{N_{\ell_{\gamma}}} \right\|_{\mathrm{op}} = o_{d, \mathbb{P}}(1);
\end{equation}
see Lemma \ref{lemma psi psi top} and Lemma \ref{lemma psi top psi}. 
Recall that we have the decomposition 
$$
K = \underbrace{\Psi_{\leq \ell_{\lambda}} \Sigma_{\leq \ell_{\lambda}} \Psi_{\leq \ell_{\lambda}}^{\top}}_{K_{\leq \ell_{\lambda}}} + 
\underbrace{\Psi_{\ell_{\lambda}+ 1, \ell_{\gamma}} \Sigma_{\ell_{\lambda}+ 1, \ell_{\gamma}} \Psi_{\ell_{\lambda}+ 1, \ell_{\gamma}}^{\top}}_{K_{\ell_{\lambda}+ 1, \ell_{\gamma}}}
+ \underbrace{\Psi_{>\ell_{\gamma}} \Sigma_{>\ell_{\gamma}} \Psi_{>\ell_{\gamma}}^{\top}}_{K_{>\ell_{\gamma}}}.
$$
Combining the block decomposition of $K$ with Weyl's inequality, we obtain several key properties of the kernel matrices $M := K + n\lambda \mathrm{I}_n$, $K_{>\ell_{\lambda}}$, and $M_{>\ell_{\lambda}}:=K_{>\ell_{\lambda}} + n\lambda \mathrm{I}_n$, which will be used throughout the variance and bias analysis.
Specifically, we have 
\begin{equation}\label{eqn_sketch_3}
\left\{
\begin{aligned}
    &~ \lambda_{\text{min}}\left( M\right) = \Omega_{d, \mathbb{P}}(n\lambda), \quad \mathrm{tr}(M / n)=O_{d, \mathbb{P}}(n\lambda);\\
    &~ \text{For any } i= N_{\ell_{\lambda}}+1, \cdots, N_{\ell_{\lambda}+1}, \text{ we have } \lambda_i(M) = O_{d, \mathbb{P}}(n\lambda);\\
    &~ \lambda_{\text{max}}(M_{>\ell_{\lambda}}) \asymp \lambda_{\text{min}}(M_{>\ell_{\lambda}}) = \Theta_{d, \mathbb{P}}(n\lambda), \quad \lambda_{\text{max}}(K_{>\ell_{\lambda}}) =
 O_{d, \mathbb{P}}\left({d^{\gamma-\ell_{\lambda}-1}}\right);\\
 &~ \text{Denote } A = n^{-1}\left(\Sigma_{\leq \ell_{\lambda}}^{-1} + \Psi_{\leq \ell_{\lambda}}^{\top} M_{>\ell_{\lambda}}^{-1} \Psi_{\leq \ell_{\lambda}}\right), \text{ then } \lambda_{\text{max}} \left(  A^{-1} \right) \asymp \lambda_{\text{min}} \left(  A^{-1} \right) = \Theta_{d, \mathbb{P}}(n\lambda);
\end{aligned}
\right.
\end{equation}
see Lemma \ref{lemma lambda max min}. 
For the bias, we also need the following results
\begin{equation}\label{eqn_sketch_4}
   \left\{ 
   \begin{aligned}
       & \left\| \Psi_{>\ell_{\lambda}} \theta_{>\ell_{\lambda}} \right\|_{2}^{2} = ~ O_{d, \mathbb{P}}\left( n  \| \theta_{>\ell_{\lambda}} \|_{2}^{2}\right); \\
       & \| \theta_{>\ell_{\lambda}} \|_{2}^{2}  = ~ \Theta_d \left(  d^{-(\ell_{\lambda}+1)s} \right); \\
      &  \left\| \Sigma_{\leq \ell_{\lambda}}^{-\tau}  \theta_{\leq \ell_{\lambda}} \right\|_{2}^{2}  = ~ \Theta_{d} \left(  d^{(2\tau-\min\{s, 2\tau\})\ell_{\lambda} } \right);
    \end{aligned}
    \right.
\end{equation}
see Lemma \ref{lemma ge l l2}. Finally, based on the Sherman-Morrison-Woodbury formula, for any finite cutoff $k$ such that $\Sigma_{\leq k}$ and $M_{>k}$ are invertible,
\begin{equation}\label{eqn_sketch_5}
\begin{aligned}
    M^{-1} \Psi_{\leq k} \Sigma_{\leq k}^{1/2} 
        =  
        M_{>k}^{-1} \Psi_{\leq k} \Sigma_{\leq k}^{1/2} \left( \mathrm{I}_{N_{k}} +  \Sigma_{\leq k}^{1/2}  \Psi_{\leq k}^{\top} M_{>k}^{-1} \Psi_{\leq k} \Sigma_{\leq k}^{1/2}\right)^{-1},\\
        \Sigma_{\leq k} - \Sigma_{\leq k} \Psi^{\top}_{\leq k} M^{-1} \Psi_{\leq k}\Sigma_{\leq k}
        = 
        \left( \Sigma_{\leq k}^{-1} + \Psi_{\leq k}^{\top} M_{>k}^{-1} \Psi_{\leq k}  \right)^{-1};
\end{aligned}
\end{equation}
see Lemma \ref{lemma trans in V2}.

\paragraph*{Bounding the variance}

Since $\mathscr{K}(x,X) = \psi(x)^{\top} \Sigma \Psi^{\top}$, $\varphi_{\lambda, K}:=n^{-1}\reg(K/n)$, and $\epsilon_1, \cdots, \epsilon_n$ are independent with variance $\sigma^2$, 
the variance term (14) can be decomposed according to the three spectral blocks
\[
>\ell_\gamma,\qquad \ell_\lambda+1,\ldots,\ell_\gamma,\qquad \le \ell_\lambda.
\]
More precisely,
\begin{equation*}
    \begin{aligned}
        \mathrm{var}\left( \hat{f}_{\lambda} \right) &= \E_{x, \epsilon}\left[\left(\mathscr{K}(x,X) \varphi_{\lambda, K} \epsilon \right)^2\right] = \E_{x, \epsilon} \left[\left( \psi(x)^{\top} \Sigma^{1/2} \Sigma^{1/2} \Psi^{\top} \varphi_{\lambda, K} \epsilon \right)^2\right]  \\
    =&~  \sigma^{2} \cdot \mathrm{tr} \left( \varphi_{\lambda, K} \Psi \Sigma^{2} \Psi^{\top} \varphi_{\lambda, K} \right)\\
    =&~
    \sigma^{2} \cdot \mathrm{tr} \left(\underbrace{\kappa_{2}  \varphi_{\lambda, K} \left(\mathrm{I}_{n} + \Delta_{2}\right) \varphi_{\lambda, K}}_{ V_{1}}\right)  ~+~ \sigma^{2} \cdot \mathrm{tr} \left(\underbrace{ \varphi_{\lambda, K} \Psi_{\leq \tilde{\ell}} \Sigma_{\leq \tilde{\ell}}^{2} \Psi_{\leq \tilde{\ell}}^{\top} \varphi_{\lambda, K}}_{ V_{2}}\right)\\
    &~ +~ \sigma^{2} \cdot \mathrm{tr} \left(\underbrace{ \varphi_{\lambda, K} \Psi_{\ell_{\lambda}+1, \ell_{\gamma}} \Sigma_{\ell_{\lambda}+1, \ell_{\gamma}}^{2} \Psi_{\ell_{\lambda}+1, \ell_{\gamma}}^{\top} \varphi_{\lambda, K}}_{ V_{3}}\right).
    \end{aligned}
\end{equation*}

We can then bound these three trace terms using (\ref{eqn_sketch_1}), (\ref{eqn_sketch_2}), (\ref{eqn_sketch_3}), and (\ref{eqn_sketch_5}); see Appendix \ref{append_variance}. 

To illustrate the new step, we take $\mathrm{tr}(V_{3})$ as an example. 
This term results from the intermediate block between $\ell_\lambda$ and $\ell_\gamma$, which appears only when $\ell_\lambda<\ell_\gamma$. Controlling this block is what allows the variance analysis to extend beyond the interpolation-level decomposition.

\color{black}

$\mathrm{tr}(V_{3})$ can be bounded from above as
\begin{equation*}
    \begin{aligned}
        \mathrm{tr}(V_{3})
        \overset{(\ref{eqn:spe_properties})}{\le}
        C_2^2 n \lambda_{\max}(M^{-2})
        \cdot
        \mathrm{tr}\!\left(
        \frac{\Psi_{\ell_\lambda+1,\ell_\gamma}\Sigma_{\ell_\lambda+1,\ell_\gamma}^2
        \Psi_{\ell_\lambda+1,\ell_\gamma}^\top}{n}
        \right)
        \overset{(*)}{=}
        O_{d,\mathbb P}\!\left(\frac{1}{d^{\ell_\lambda+1}n\lambda^2}\right), 
    \end{aligned}
\end{equation*}
where the $(*)$ step uses the estimates in (\ref{eqn_sketch_1})-(\ref{eqn_sketch_3}), and (\ref{eqn_sketch_5}). 
These same estimates are used repeatedly in the proof of the lower bound as well. 
Combining the upper and lower bounds, we conclude that 
$
\mathrm{tr}(V_{3})
=
\Theta_{d,\mathbb P}\!\left(\frac{1}{d^{\ell_\lambda+1}n\lambda^2}\right)$.

\paragraph*{Bounding the bias}

To keep the sketch focused, we do not reproduce the full decomposition of $\mathrm{bias}^2(\hat f_\lambda)$ into the five terms $B_1,\ldots,B_5$; see Appendix \ref{append_bias}. 
Most terms can be controlled using the same approximation and estimates in (\ref{eqn_sketch_1})--(\ref{eqn_sketch_5}); see Appendix \ref{append_bias}. 
{A genuinely new argument is needed for the mixed leakage term $B_4$; one must control the intermediate spectral block without assuming that a general filter obeys a Woodbury identity.} Put $a=\tilde{\ell}=\ell_{\lambda}$, $A_{\rm head}=K_{\le a}/n$, and $A_{\rm tail}=K_{>a}/n$. The identity
$$
 A_{\rm head}\varphi_\lambda(K/n)
 =(K/n)\varphi_\lambda(K/n)-A_{\rm tail}\varphi_\lambda(K/n)
$$
and the tail and resolvent bounds give
$\|A_{\rm head}\varphi_\lambda(K/n)\|_{\rm op}=O_{d,\mathbb P}(1)$.
{Since $a<\ell_{\gamma}$, the squared-kernel tail matrix}
$R_{>a}=n^{-1}\Psi_{>a}\Sigma_{>a}^2\Psi_{>a}^{\top}$ obeys
\[
 \|R_{>a}\|_{\rm op}
 \le\|\Sigma_{>a}\|_{\rm op}\|A_{\rm tail}\|_{\rm op}
 =O_{d,\mathbb P}(d^{-2a-2}).
\]
Consequently $\|B_4\|^2=O_{d,\mathbb P}(d^{(2-\min(s,2))a-2a-2})$.

\paragraph*{Final comparison of the bias and variance rates}

For the regime in \eqref{eqn_sketch_0}, the previous steps show that the variance contributes two candidate orders, while the bias contributes three candidate orders. 
For notational convenience, we denote the corresponding exponents by $v_1(u),v_2(u)$ and $b_1(u),b_2(u, 1),b_3(u)$, respectively; see \eqref{eqn_order_in_bias_and_variance}. 
Thus
\begin{equation*}
    \begin{aligned}
        \mathrm{var}(\hat{f}_{\lambda})
        &= \Theta_{d,\mathbb P}\!\left(d^{v_1(u)}+d^{v_2(u)}\right),\\
        \mathrm{bias}^{2}(\hat{f}_{\lambda})
        &\le O_{d,\mathbb P}\!\left(d^{b_1(u)}+d^{b_2(u, 1)}+d^{b_3(u)}\right).
    \end{aligned}
\end{equation*}
Moreover, the lower bound proved in Theorem~\ref{thm_bias_spe} shows that
\begin{equation*}
\begin{aligned}
    \mathrm{bias}^{2}(\hat{f}_{\lambda})
    \geq &~\|B_3-B_4-B_5\|_{2}^{2}
    \geq \frac{1}{2}\|B_3-B_5\|_{2}^{2} - \|B_4\|_{2}^{2},\\
    \|B_3-B_5\|_{2}^{2} = &~ \Theta_{d,\mathbb P}\!\left(d^{b_1(u)}\right),\\
    \|B_4\|^2=&~ o_{d,\mathbb P}(d^{b_2(u,1)}+d^{b_3}).
\end{aligned}
\end{equation*}

Therefore, the only remaining step is to compare these exponents. 
Proposition \ref{prop_e_2} and Proposition \ref{prop_e_3} show that, in the present regime,
\begin{equation*}
    \max\{v_1(u),v_2(u),b_1(u)\}\ge \max\{b_2(u, 1),b_3(u)\}.
\end{equation*}
In other words, the secondary bias exponents $b_2(u, 1)$ and $b_3(u)$ never dominate the leading order. 
Hence the excess risk is determined by the two variance exponents together with the principal bias exponent $b_1(u)$, and we obtain
\begin{equation*}
    \calR_{d}(u;X)
    = \Theta_{d,\mathbb P}\!\left(d^{v_1(u)}+d^{v_2(u)}+d^{b_1(u)}\right).
\end{equation*}
This is the final synthesis step that turns the separate bias and variance bounds into the learning-curve formula in the regime \eqref{eqn_sketch_0}; see Appendix \ref{append_learning_curve} for the complete case-by-case comparison.

\section{Variance}\label{append_variance}

In this section, our goal is to prove the following theorem, which bounds the variance term in (13).

\begin{theorem}\label{thm_variance_spe}
    Suppose Assumptions 1, 2, and 4 hold.
    Then the variance term (14) of the spectral algorithm estimator $\hat{f}_{\lambda}$ satisfies
    \begin{displaymath}
        \mathrm{var}(\hat{f}_{\lambda}) = \Theta_{d, \mathbb{P}} \left( 
        d^{\gamma - \tilde{\ell} - 1 - 2  \max\{\gamma - u, 0\}}
+ d^{\tilde{\ell} - \gamma}
        \right).
    \end{displaymath} 
\end{theorem}
\begin{proof}
Because $\mathscr{K}(x,X) = \psi(x)^{\top} \Sigma \Psi^{\top}$ and $\epsilon_1, \cdots, \epsilon_n$ are independent with variance $\sigma^2$, the variance term (14) can be expressed as
\begin{equation*}
    \begin{aligned}
        \mathrm{var}\left( \hat{f}_{\lambda} \right) =&~ \E_{x, \epsilon}\left[\left(\mathscr{K}(x,X) \varphi_{\lambda, K} \epsilon \right)^2\right] = \E_{x, \epsilon} \left[\left( \psi(x)^{\top} \Sigma \Psi^{\top} \varphi_{\lambda, K} \epsilon \right)^2\right]  \\
    =&~  \sigma^{2} \cdot \mathrm{tr} \left( \varphi_{\lambda, K} \Psi \Sigma^{2} \Psi^{\top} \varphi_{\lambda, K} \right)\\
    (\because \text{Lemma~\ref{lemma psi psi top}})  =&~
    \sigma^{2} \cdot \mathrm{tr} \left(\underbrace{\kappa_{2}  \varphi_{\lambda, K} \left(\mathrm{I}_{n} + \Delta_{2}\right) \varphi_{\lambda, K}}_{ V_{1}}\right)  ~+~ \sigma^{2} \cdot \mathrm{tr} \left(\underbrace{ \varphi_{\lambda, K} \Psi_{\leq \tilde{\ell}} \Sigma_{\leq \tilde{\ell}}^{2} \Psi_{\leq \tilde{\ell}}^{\top} \varphi_{\lambda, K}}_{ V_{2}}\right)\\
    &~ +~ \mathbf{1}\{\ell_{\lambda}<\ell_{\gamma}\}\sigma^{2} \cdot \mathrm{tr} \left(\underbrace{ \varphi_{\lambda, K} \Psi_{\ell_{\lambda}+1, \ell_{\gamma}} \Sigma_{\ell_{\lambda}+1, \ell_{\gamma}}^{2} \Psi_{\ell_{\lambda}+1, \ell_{\gamma}}^{\top} \varphi_{\lambda, K}}_{ V_{3}}\right),
    \end{aligned}
\end{equation*}
where $ \kappa_{2} = \Theta_{d}\left( d^{-(\ell_{\gamma}+1)} \right)$ and $\|\Delta_{2}\|_{\mathrm{op}} = o_{d, \mathbb{P}}(1)$. We then bound $\mathrm{tr}(V_1)$ and $\mathrm{tr}(V_2)$ separately.

\noindent {\bf Bounding $\mathrm{tr}(V_1)$. }
On the one hand, we have
\begin{equation*}
    \begin{aligned}
        \mathrm{tr}(V_1) \leq &~ n\lambda_{\text{max}}\left( V_{1} \right) = n\kappa_{2} \cdot \lambda_{\text{max}}\left( \varphi_{\lambda, K} \left( \mathrm{I}_{n} + \Delta_{2} \right)  \varphi_{\lambda, K}\right)  \\
        (\because \text{Prop.~\ref{prop_weyl_ine}})\quad \leq  &~n\kappa_{2} \cdot \left[ \lambda_{\text{max}}\left( \varphi_{\lambda, K}^2\right) + \lambda_{\text{max}}\left( \varphi_{\lambda, K} \Delta_{2} \varphi_{\lambda, K}\right)  \right]  \\
        \leq  &~ n\kappa_{2} \cdot \lambda_{\text{max}}\left( \varphi_{\lambda, K}^2\right) \left( 1 + o_{d, \mathbb{P}}(1) \right)\\
        \overset{(\ref{eqn:spe_properties})}{\leq}  &~ C_2^2 n\kappa_{2} \cdot \lambda_{\text{min}}\left( M\right)^{-2} \left( 1 + o_{d, \mathbb{P}}(1) \right)\\
        (\because \text{Lemma~\ref{lemma lambda max min}}) \quad =&~
        O_{d, \mathbb{P}}\left(\frac{n}{d^{\ell_{\gamma}+1}(n\lambda+1)^2}\right);
    \end{aligned}
\end{equation*}
On the other hand, we have
\begin{equation*}
    \begin{aligned}
        \mathrm{tr}(V_1) \geq &~  \kappa_{2} \cdot \mathrm{tr}\left( \varphi_{\lambda, K}^2\right) \left( 1 - o_{d, \mathbb{P}}(1) \right)\\
        \overset{(\ref{eqn:spe_properties})}{\geq}&~
        C_1^2\kappa_{2} \cdot \mathrm{tr}\left( M^{-2}\right) \left( 1 - o_{d, \mathbb{P}}(1) \right)\\
       \overset{(\ast)}{ \geq} &~  C_1^2 \kappa_{2} \cdot n \left[ \mathrm{tr}\left( \frac{M}{n} \right) \right]^{-2} \left( 1 -o_{d, \mathbb{P}}(1) \right)\\
                (\because \text{Lemma~\ref{lemma lambda max min}(i)})\quad =&~
        \Omega_{d, \mathbb{P}}\left(\frac{n}{d^{\ell_{\gamma}+1}(n\lambda+1)^2}\right),
    \end{aligned}
\end{equation*}
where the third inequality $(\ast)$ can be shown using Jensen's inequality as in Remark~\ref{rem:variance-Jensen-inequ}. 
Therefore, $\mathrm{tr}(V_1) = \Theta_{d, \mathbb{P}}\left(\frac{n}{d^{\ell_{\gamma}+1}(n\lambda+1)^2}\right)$.

\begin{remark}\label{rem:variance-Jensen-inequ}
    Suppose $A = \text{diag}(a_1, \cdots, a_n) \in \mathbb{R}^{n \times n}$ is a positive diagonal matrix. 
    Then, Jensen's inequality $f\left( \frac{a_1 + \cdots + a_n}{n} \right) \leq \frac{f(a_1) + \cdots + f(a_n)}{n}$ for the convex function \( f(x) = x^{-2}, x>0 \) implies
    $$
    \left( \mathrm{tr}\left( \frac{A}{n} \right) \right)^{-2} = \left( \frac{a_1 + \cdots + a_n}{n} \right)^{-2} \leq \frac{a_1^{-2} + \cdots + a_n^{-2}}{n} = \frac{\mathrm{tr}(A^{-2})}{n}.
    $$
Since $M$ is positive definite, we can show $(\ast)$ with eigendecomposition.
\end{remark}

\noindent {\bf Bounding $\mathrm{tr}(V_2)$. }
We have
\begin{equation}\label{eqn_v_2_large_regular_para}
    \begin{aligned}
        \mathrm{tr} \left( V_{2} \right)
        =&~
        \mathrm{tr} \left( \Sigma_{\leq \tilde{\ell}} \Psi_{\leq \tilde{\ell}}^{\top}  \varphi_{\lambda, K}^2 \Psi_{\leq \tilde{\ell}} \Sigma_{\leq \tilde{\ell}} \right)
        \overset{(\ref{eqn:spe_properties})}{=} \Theta_d\left(1\right) \cdot
        \mathrm{tr} \left( \Sigma_{\leq \tilde{\ell}} \Psi_{\leq \tilde{\ell}}^{\top}  M^{-2} \Psi_{\leq \tilde{\ell}} \Sigma_{\leq \tilde{\ell}} \right)\\
        \overset{\text{Lemma } \ref{lemma trans in V2}}{=}&~
        \Theta_d\left(\frac{1}{n}\right) \cdot \mathrm{tr} \left[ \left( \frac{\Sigma_{\leq \tilde{\ell}}^{-1}}{n} + \frac{\Psi_{\leq \tilde{\ell}}^{\top} M_{>\tilde{\ell}}^{-1} \Psi_{\leq \tilde{\ell}}}{n}   \right)^{-1} \frac{\Psi_{\leq \tilde{\ell}}^{\top} M_{>\tilde{\ell}}^{-2} \Psi_{\leq \tilde{\ell}}}{n} \left( \frac{\Sigma_{\leq \tilde{\ell}}^{-1}}{n} + \frac{\Psi_{\leq \tilde{\ell}}^{\top} M_{>\tilde{\ell}}^{-1} \Psi_{\leq \tilde{\ell}}}{n}   \right)^{-1}  \right]\\
        \overset{\text{Lemma } \ref{lemma lambda max min}}{=}&~
        \Theta_{d, \mathbb{P}}\left(\frac{(n\lambda+1)^2}{n}\right) \cdot \mathrm{tr} \left[  \frac{\Psi_{\leq \tilde{\ell}}^{\top} M_{>\tilde{\ell}}^{-2} \Psi_{\leq \tilde{\ell}}}{n}   \right]\\
        \overset{\text{Lemma } \ref{lemma lambda max min}}{=}&~
        \Theta_{d, \mathbb{P}}\left(\frac{1}{n}\right) \cdot \mathrm{tr} \left[  \frac{\Psi_{\leq \tilde{\ell}}^{\top} \Psi_{\leq \tilde{\ell}}}{n}   \right]
        =
        \Theta_{d, \mathbb{P}}\left(\frac{d^{\tilde{\ell}}}{n}\right),
    \end{aligned}
\end{equation}
where the last equality follows from Lemmas \ref{lemma psi top psi} and \ref{lemma psi top psi_int}, together with the fact that $N_{\tilde{\ell}}:= \sum_{k=0}^{\tilde{\ell}} N(d, k) = \Theta(d^{\tilde{\ell}})$ (see Proposition 2.1).

\noindent {\bf Bounding $\mathrm{tr}(V_3)$, when $\ell_{\lambda}<\ell_{\gamma}$. }

When $\ell_{\lambda}<\ell_{\gamma}$, let $t=n\lambda+1$, $a = \ell_{\lambda}$, $m = \ell_{\gamma}$, and
\[
 Q_{a+1:m}=\Psi_{a+1:m}\Sigma_{a+1:m}^2\Psi_{a+1:m}^\top.
\]
The upper bound follows from 
$$
    \norm{\varphi_{\lambda, K}}_{\op} \overset{(\ref{eqn:spe_properties})}{\leq}   C_2 \lambda_{\text{min}}\left( M\right)^{-1} \overset{\text{Lemma~\ref{lemma lambda max min}}}{=} O_{d,\Pp}(t^{-1}),
$$
and the exact trace:
\begin{equation}\label{eq:Vmid-upper}
 \tr(\varphi_{\lambda, K}^2Q_{a+1:m})
 \le\norm{\varphi_{\lambda, K}}_{\op}^2\,n\sum_{k=a+1}^m\mu_k^2N(d, k)
 =O_{d,\Pp}\left(\frac{n}{d^{a+1}t^2}\right).
\end{equation}
For the lower bound, retain block $k=a+1\leq m$ and let $s_1\geq\cdots\geq s_n\geq0$
be the eigenvalues of $\Psi_k\Psi_k^\top/n$.  Choose
\[
 \mathcal{I}_d=\begin{cases}
        N(d,k),& k<\gamma,\\
        \lfloor\eta\min\{n,N(d,m)\}\rfloor,& k=\gamma=m,
       \end{cases}
\]
where $\eta$ is the constant in Lemma \ref{lemma psi top psi_int} in the second case only.
If $k<\gamma$, Lemma \ref{lemma psi top psi}, or the below-critical part of Lemma \ref{lemma psi top psi_int} when
$\gamma=m$, gives $s_{\mathcal{I}_d}\geq1/2$ with probability tending to one.
If $k=\gamma=m$, use \eqref{eq:bulk}.  This separates the genuinely critical block
from the top retained block $k=m<\gamma$ in a noninteger regime.
In either case $\mathcal{I}_d\asymp d^{a+1}$ and $N_a=o(\mathcal{I}_d)$.

By the trace rearrangement inequality and \eqref{eq:rank}, on an event whose
probability tends to one,
\begin{align*}
 \mathrm{tr}(\varphi_{\lambda,K}^2 Q_{a+1:m})
 &\geq C_1^2n\mu_k^2\sum_{i=1}^n\frac{s_i}{\lambda_i(M)^2}\\
 &\geq c n\mu_k^2\sum_{i=N_a+1}^{\mathcal{I}_d}\frac1{\lambda_i(M)^2}
 \geq c\frac{n\mu_k^2(\mathcal{I}_d-N_a)}{(1+n\lambda)^2}
 \asymp \frac{n}{d^{a+1}(1+n\lambda)^2}.
\end{align*}
Together with \eqref{eq:Vmid-upper}, this proves the claimed order of the intermediate
variance.

Combining all the above, we have
\begin{equation*}
    \begin{aligned}
      \mathrm{var}\left( \hat{f}_{\lambda} \right) =&~ \sigma^2 \left(\mathrm{tr}(V_1) + \mathrm{tr} \left( V_{2} \right) \right) + \mathbf{1}\{\ell_{\lambda}<\ell_{\gamma}\}\sigma^{2}\mathrm{tr} \left( V_{3} \right)\\
        = &~
        \Theta_{d, \mathbb{P}}\left(\frac{n}{d^{\ell_{\gamma}+1}(n\lambda+1)^2} + \frac{d^{\tilde{\ell}}}{n}\right) + \mathbf{1}\{\ell_{\lambda}<\ell_{\gamma}\}\Theta_{d, \mathbb{P}}\left(\frac{n}{d^{\ell_{\lambda}+1}(n\lambda+1)^2}\right)\\
        = &~
        \Theta_{d, \mathbb{P}}\left(\frac{n}{d^{\tilde{\ell}+1}(n\lambda+1)^2} + \frac{d^{\tilde{\ell}}}{n}\right)\\
        = &~
        \Theta_{d, \mathbb{P}}\left(
        d^{\gamma - \tilde{\ell} - 1 - 2  \max\{\gamma - u, 0\}}
+ d^{\tilde{\ell} - \gamma}
\right),
    \end{aligned}
\end{equation*}
and we get the desired results.
\end{proof}

\section{Bias}\label{append_bias}

We first define the following notation:
\begin{equation}\label{eqn_def_of_tau_prime}
    \begin{aligned}
  \tau^{\prime} :=
  \begin{cases}
\tau, &  \tau < \infty;\\
s, &  \tau = \infty \text{ and } s \geq 1;\\
1, &  \tau = \infty \text{ and } s \in \{0\} \cup [\frac{1}{2}, 1);\\
(2s)^{-1}, &  \tau = \infty \text{ and } 0 < s < \frac{1}{2}.
  \end{cases}
  \end{aligned}
\end{equation}

In this section, our goal is to prove the following theorem, which bounds the bias term in (13).

\begin{theorem}\label{thm_bias_spe}
Suppose Assumptions 1--5 hold.
On one hand, when $\gamma > \ell_{\gamma}$, or when $\ell_{\lambda} < \gamma = \ell_{\gamma}$, the bias term (15) of the spectral algorithm estimator $\hat{f}_{\lambda}$ satisfies:
    \begin{equation*}
\begin{aligned}
    \mathrm{bias}^{2}(\hat{f}_{\lambda})  =  &~
    \|B_1+B_2\|_{2}^{2} + \|B_3-B_4-B_5\|_{2}^{2},
\end{aligned} 
\end{equation*}
where
\begin{itemize}

\item[(i)] We have
    \begin{equation*}
    \begin{aligned}
                \|B_1\|_{2}^{2}  =&~ 
                O_{d, \mathbb{P}}\left(d^{-2u + (2-\min\{s, 2\})\tilde{\ell} } + d^{(2-\min\{s, 2\})\tilde{\ell}-2\gamma }\right);
    \end{aligned}
\end{equation*}

    \item[(ii)]
    If $0 < u < \gamma <1$, then for any $\tau \leq \infty$, we have
        \begin{equation*}
        \begin{aligned}
                \|B_1\|_{2}^{2}  \leq &~ 
                O_{d, \mathbb{P}}\left(  d^{-2\tau^{\prime} u} \right) + \Delta_1,\\
                \|B_1\|_{2}^{2}  \geq &~ 
                \Omega_{d, \mathbb{P}}\left(  d^{-2\tau u} \right)- \Delta_1, \quad \text{ if } \tau<\infty,
        \end{aligned}
        \end{equation*}
        where $0 \leq \Delta_1 = O_{d, \mathbb{P}}\left( d^{-2\gamma }\right)$;

\item[(iii)] We have
\begin{equation*}
    \begin{aligned}
                \|B_2\|_{2}^{2} =&~ 
            O_{d, \mathbb{P}}\left(  d^{-(\tilde{\ell}+1)s}\right);
    \end{aligned}
\end{equation*}

\item[(iv)] We have
\begin{equation*}
    \begin{aligned}
                \|B_3-B_5\|_{2}^{2} 
        =
        \Theta_{d, \mathbb{P}} \left( d^{-(\tilde{\ell}+1)s}\right);
    \end{aligned}
\end{equation*}

\item[(v)] We have
\begin{equation*}
    \begin{aligned}
                &~\|B_4\|_{2}^{2} = 
                O_{d, \mathbb{P}} \left(  d^{(2-\min\{s, 2\})\tilde{\ell} - (\gamma  + \ell_{\gamma} + 1)\mathbf{1}\{\ell_{\gamma} \leq \ell_{\lambda}\} - 2(\ell_{\lambda}+1)\mathbf{1}\{\ell_{\gamma} > \ell_{\lambda}\} } \right).
    \end{aligned}
\end{equation*}
\end{itemize}
Otherwise, when $\ell_{\lambda} \geq \gamma = \ell_{\gamma}$, we have
\begin{equation}\label{eq:bias}
 \operatorname{bias}^2(\widehat f_\lambda)
 =O_{d,\Pp}\left(d^{-\ell_{\gamma}s}+d^{-2-\min(s,2)(\ell_{\gamma}-1)}\right)
 =O_{d,\Pp}(1).
\end{equation}
\end{theorem}

    \begin{remark}
        We remark that the upper bound in Theorem~\ref{thm_bias_spe} is not tight in certain cases, for which we will provide a tight bound in Proposition \ref{prop_verify_bias_special}.
    \end{remark}
    
\begin{proof}
We consider the following two cases: (I) we have $\gamma > \ell_{\gamma}$, or we have $\ell_{\lambda} < \gamma = \ell_{\gamma}$;
(II) we have $\ell_{\lambda} \geq \gamma = \ell_{\gamma}$.

\paragraph{Case (I). } Suppose that we have $\gamma > \ell_{\gamma}$, or we have $\ell_{\lambda} < \gamma = \ell_{\gamma}$. 
We first prove that \begin{equation}\label{eqn_bound_eigen_psi_top_psi}
    \lambda_{\text{max}}(\Psi_{\leq \tilde{\ell}}^{\top}\Psi_{\leq \tilde{\ell}}) = \Theta_{d, \mathbb{P}}(n) \quad \text{ and } \quad \lambda_{\text{min}}(\Psi_{\leq \tilde{\ell}}^{\top}\Psi_{\leq \tilde{\ell}}) = \Theta_{d, \mathbb{P}}(n).
\end{equation}
Eq.~\eqref{eqn_bound_eigen_psi_top_psi} can be seen by considering the following: 
\begin{itemize}
    \item Suppose $\ell_{\lambda} \geq \ell_{\gamma}$. Since $\tilde{\ell}=\ell_{\gamma}$, from Lemma \ref{lemma psi top psi} and Lemma \ref{lemma psi top psi_int}, {we have that Eq. \eqref{eqn_bound_eigen_psi_top_psi} holds.} 

    \item Suppose  $\ell_{\lambda}<\ell_{\gamma}<\gamma$. Recall that $\Psi_{\leq \ell_{\gamma}} = \left[\Psi_{\leq \ell_{\lambda}}, \Psi_{\ell_{\lambda}+1, \ell_{\gamma}}\right]$. 
    Since $\tilde{\ell}=\ell_{\lambda}$, from Lemma \ref{lemma psi top psi}, we have
\begin{align*}
&~ \left\| \frac{1}{n}
\Psi_{\leq \ell_\lambda}^{\top}\,\Psi_{\leq \ell_\lambda} - 
\mathrm{I}_{N_{\ell_\lambda}} \right\|_{\mathrm{op}}\\
    \leq &~ \left\| \frac{1}{n}
\begin{pmatrix}
  \Psi_{\leq \ell_\lambda}^{\top}\,\Psi_{\leq \ell_\lambda}
  & 
  \Psi_{\leq \ell_\lambda}^{\top}\,\Psi_{\ell_\lambda+1,\ell_{\gamma}}
  \\[1em]
  \Psi_{\ell_\lambda+1,\ell_{\gamma}}^{\top}\,\Psi_{\leq \ell_\lambda}
  & 
  \Psi_{\ell_\lambda+1,\ell_{\gamma}}^{\top}\,\Psi_{\ell_\lambda+1,\ell_{\gamma}}
\end{pmatrix} - \begin{pmatrix}
\mathrm{I}_{N_{\ell_\lambda}} & 0 \\[0.5em]
0 & \mathrm{I}_{\,N_{\ell_{\gamma}}-N_{\ell_\lambda}}
\end{pmatrix} \right\|_{\mathrm{op}} = o_{d, \mathbb{P}}(1),
\end{align*}
    which implies Eq. (\ref{eqn_bound_eigen_psi_top_psi}).

    \item Suppose $\ell_{\lambda}<\ell_{\gamma} = \gamma$. As in the second case, we can use Lemma \ref{lemma psi top psi_int} to show Eq. (\ref{eqn_bound_eigen_psi_top_psi}).
\end{itemize}





The bias term Eq.~(15) can be decomposed as
\begin{equation}\label{eq derivation of bias_spe}
    \begin{aligned}
        \mathrm{bias}^2(\hat{f}_{\lambda}) &= \E_{x} \left[\left( \mathscr{K}(x,X) \varphi_{\lambda, K} f_{\star}(X) -f_{\star}(x)\right)^2\right]\\
    =&~ \E_{x} \left[\left( \psi(x)^{\top} \Sigma \Psi^{\top} \varphi_{\lambda, K} \Psi \theta - \psi(x)^{\top} \theta\right)^2\right]\\
    =&~ \| \Sigma \Psi^{\top} \varphi_{\lambda, K} \Psi \theta - \theta  \|_{2}^{2}\\
    =&~
    \| \Sigma_{\leq \tilde{\ell}} \Psi^{\top}_{\leq \tilde{\ell}} \varphi_{\lambda, K} \Psi \theta - \theta_{\leq \tilde{\ell}}  \|_{2}^{2} + \| \Sigma_{>\tilde{\ell}} \Psi^{\top}_{>\tilde{\ell}} \varphi_{\lambda, K} \Psi \theta - \theta_{>\tilde{\ell}}  \|_{2}^{2}\\
    =&~
    \| \underbrace{\Sigma_{\leq \tilde{\ell}} \Psi^{\top}_{\leq \tilde{\ell}} \varphi_{\lambda, K} \Psi_{\leq \tilde{\ell}} \theta_{\leq \tilde{\ell}} - \theta_{\leq \tilde{\ell}}  }_{B_{1}}+  \underbrace{\Sigma_{\leq \tilde{\ell}} \Psi^{\top}_{\leq \tilde{\ell}} \varphi_{\lambda, K} \Psi_{>\tilde{\ell}} \theta_{>\tilde{\ell}}}_{B_{2}} \|_{2}^{2}\\
    &~+
     \|\underbrace{\theta_{>\tilde{\ell}}}_{B_{3}} - \underbrace{ \Sigma_{>\tilde{\ell}} \Psi^{\top}_{>\tilde{\ell}} \varphi_{\lambda, K} \Psi_{\leq \tilde{\ell}} \theta_{\leq \tilde{\ell}} }_{B_{4}} - \underbrace{  \Sigma_{>\tilde{\ell}} \Psi^{\top}_{>\tilde{\ell}} \varphi_{\lambda, K} \Psi_{>\tilde{\ell}} \theta_{>\tilde{\ell}} }_{ B_{5}}\|_{2}^{2}.
    \end{aligned}
\end{equation}



In the following calculations, we denote $Z = n^{-1/2}\Psi_{\leq \tilde{\ell}}$. 

\noindent {\bf Bounding $\|B_1\|_{2}^{2}$. }
Further adopt the following notations:
\begin{equation*}
    \begin{aligned}
        &~ A_{\rm head} = Z\Sigma_{\leq \tilde{\ell}} Z^{\top} = \frac{K_{\leq \tilde{\ell}}}{n}, \quad A_{\rm tail} = \frac{K_{> \tilde{\ell}}}{n}, \quad A = \frac{K}{n}, \quad \xi = Z (Z^{\top} Z)^{-1} \Sigma_{\leq \tilde{\ell}}^{-1} \theta_{\leq \tilde{\ell}}.
    \end{aligned}
\end{equation*}
Then,
\begin{itemize}
\item Recall that $\varphi_{\lambda, K}:=n^{-1}\reg(A)$. Hence we have
    \begin{equation}\label{eqn_B_1_collect_1}
    \begin{aligned}
                &~\|B_1\|_{2}^{2} = 
                \left\| \left( \mathrm{I}_{N_{\tilde{\ell}}} - \Sigma_{\leq \tilde{\ell}} \Psi^{\top}_{\leq \tilde{\ell}} \varphi_{\lambda, K} \Psi_{\leq \tilde{\ell}}\right)    \theta_{\leq \tilde{\ell}} \right\|_{2}^{2}\\
                =&~ 
                \Theta_{d, \mathbb{P}}\left(\frac{1}{n}\right)
                \left\| \Psi_{\leq \tilde{\ell}}\left( \mathrm{I}_{N_{\tilde{\ell}}} - \Sigma_{\leq \tilde{\ell}} \Psi^{\top}_{\leq \tilde{\ell}} \varphi_{\lambda, K} \Psi_{\leq \tilde{\ell}}\right)    \theta_{\leq \tilde{\ell}} \right\|_{2}^{2}\\
                =&~ 
                \Theta_{d, \mathbb{P}}\left(\frac{1}{n}\right)
                \left\| \left( \mathrm{I}_{n} - \Psi_{\leq \tilde{\ell}}\Sigma_{\leq \tilde{\ell}} \Psi^{\top}_{\leq \tilde{\ell}} \varphi_{\lambda, K} \right) 
                \Psi_{\leq \tilde{\ell}}\theta_{\leq \tilde{\ell}} \right\|_{2}^{2}\\
                {=}&~
        \Theta_{d, \mathbb{P}}\left(1\right)
        \left\| \left( \mathrm{I}_n - Z\Sigma_{\leq \tilde{\ell}} Z^{\top} \reg(A) \right) 
                Z \left[\Sigma_{\leq \tilde{\ell}} Z^{\top} Z (Z^{\top} Z)^{-1} \Sigma_{\leq \tilde{\ell}}^{-1} \right]\theta_{\leq \tilde{\ell}} \right\|_{2}^{2}\\
                {=}&~
        \Theta_{d, \mathbb{P}}\left(1\right)
        \left\| \left( \mathrm{I}_n - A_{\rm head} \reg(A) \right) 
                A_{\rm head}  \xi\right\|_{2}^{2}.
    \end{aligned}
\end{equation}

    \item From (11), we have
    \begin{equation}\label{eqn_B_1_collect_2}
    \begin{aligned}
        A\reg(A)= \reg(A) A.
    \end{aligned}
    \end{equation}

    \item We have
    \begin{equation}\label{eqn_B_1_collect_3}
    \begin{aligned}
        &~\left[\mathrm{I}_{n} - A_{\rm head} \reg(A) \right] A_{\rm head}\\
        = &~ A_{\rm head} + A_{\rm head} \reg(A) A_{\rm tail} - A_{\rm head} \reg(A) A\\
        (\because \eqref{eqn_B_1_collect_2}) \quad=&~ A_{\rm head} (\mathrm{I}_{n} - A\reg(A)) + A_{\rm head} \reg(A) A_{\rm tail}\\
        =&~ A_{\rm head} \rem(A) + A_{\rm head} \reg(A) A_{\rm tail}.
    \end{aligned}
    \end{equation}

\item From Lemmas \ref{lemma psi psi top} and \ref{lemma lambda max min} (iii), we have
\begin{equation}\label{eqn_B_1_collect_4}
    \begin{aligned}
        \|A_{\rm tail}\|_{2} = O_{d, \mathbb{P}}(n^{-1})\mathbf{1}\{\ell_{\gamma} \leq \ell_{\lambda}\} + O_{d, \mathbb{P}}\left({d^{-\ell_{\lambda}-1}}\right)\mathbf{1}\{\ell_{\gamma} > \ell_{\lambda}\}.
    \end{aligned}
    \end{equation}

\item We have
    \begin{equation}\label{eqn_B_1_collect_5}
    \begin{aligned}
        &~ \|A_{\rm head} \rem(A)\|_{2}^{2}\\
        \leq &~ 2 \|A_{\rm tail} \rem(A)\|_{2}^{2} + 2 \|A \rem(A)\|_{2}^{2}\\
        (\because \eqref{eqn_B_1_collect_4} \text{ and } (8)) \quad = &~  O_{d, \mathbb{P}}(n^{-2})\mathbf{1}\{\ell_{\gamma} \leq \ell_{\lambda}\} + O_{d, \mathbb{P}}\left({d^{-2\ell_{\lambda}-2}}\right)\mathbf{1}\{\ell_{\gamma} > \ell_{\lambda}\} + O_{d, \mathbb{P}}(\lambda^2)\\
        = &~  O_{d, \mathbb{P}}(n^{-2} +\lambda^2).
    \end{aligned}
    \end{equation}

\item We have
    \begin{equation}\label{eqn_B_1_collect_6}
    \begin{aligned}
        &~ \|A_{\rm head} \reg(A) A_{\rm tail}\|_{2}^{2}\\
        \leq &~ 2 \|A_{\rm tail} \reg(A) A_{\rm tail}\|_{2}^{2} + 2 \|A \reg(A) A_{\rm tail}\|_{2}^{2}\\
        (\because \eqref{eqn_B_1_collect_4} \text{ and } (8)) \quad = &~  O_{d, \mathbb{P}}(\|A_{\rm tail}\|_{2}^{4}[\lambda_{\min}(A + \lambda \mathrm{I}_n)]^{-2} + \|A_{\rm tail}\|_{2}^2)\\
        (\because \text{ Lemma \ref{lemma lambda max min}}) \quad = &~  O_{d, \mathbb{P}}(n^{-2})\mathbf{1}\{\ell_{\gamma} \leq \ell_{\lambda}\} + O_{d, \mathbb{P}}\left({d^{-2\ell_{\lambda}-2}}\right)\mathbf{1}\{\ell_{\gamma} > \ell_{\lambda}\}.
    \end{aligned}
    \end{equation}

\item Finally, we have
\begin{equation}\label{eqn_B_1_collect_7}
    \begin{aligned}
   \|\xi\|_{2}^{2} = &~    
        \left\| Z (Z^{\top} 
                Z)^{-1} \Sigma_{\leq \tilde{\ell}}^{-1} \theta_{\leq \tilde{\ell}}  \right\|_{2}^{2}\\
\overset{(\ref{eqn_bound_eigen_psi_top_psi})}{=}&~
        \Theta_{d, \mathbb{P}} \left( 1 \right)   
        \left\| \Sigma_{\leq \tilde{\ell}}^{-1} \theta_{\leq \tilde{\ell}}  \right\|_{2}^{2}\\
        \overset{\text{Lemma } \ref{lemma ge l l2}}{=}&~ 
         \Theta_{d, \mathbb{P}} \left(  d^{(2-\min\{s, 2\})\tilde{\ell} } \right).
    \end{aligned}
\end{equation}
    
\end{itemize}

Now let's begin to bound $\|B_1\|_{2}^{2}$. From \eqref{eqn_B_1_collect_1}, we have
\begin{equation}\label{eqn_B_1_part_1}
\begin{aligned}
    \|B_1\|_{2}^{2} = &~ \Theta_{d, \mathbb{P}}\left(1\right)
        \left\| \left( \mathrm{I}_n - A_{\rm head} \reg(A) \right) 
                A_{\rm head}  \xi\right\|_{2}^{2}\\
                \leq &~ O_{d, \mathbb{P}}\left(1\right)
        \left\| \left( \mathrm{I}_n - A_{\rm head} \reg(A) \right) 
                A_{\rm head}\right\|_{2}^{2}  \left\|\xi\right\|_{2}^{2}\\
(\because \eqref{eqn_B_1_collect_3}, \eqref{eqn_B_1_collect_5}, \text{ and } \eqref{eqn_B_1_collect_6}) \quad \leq &~ 
O_{d, \mathbb{P}}(n^{-2} + \lambda^2)\left\|\xi\right\|_{2}^{2}\\
(\because \eqref{eqn_B_1_collect_7}) \quad = &~
O_{d, \mathbb{P}}\left(d^{-2u + (2-\min\{s, 2\})\tilde{\ell} } + d^{(2-\min\{s, 2\})\tilde{\ell}-2\gamma }\right).
\end{aligned}
\end{equation}

\noindent {\it Proof of the case $0 < u<\gamma<1$.}

In this regime, we have $\tilde{\ell}=0$. Hence the low-frequency
subspace is one-dimensional. Let
$$
q=n^{-1/2}\mathbf 1_n.
$$
Then Proposition 2.1 implies
$$
\Sigma_{\leq \tilde{\ell}}
=
\Sigma_{\leq 0}
=
\mu_0
=
\Theta_d(1),
\qquad
\Psi_{\leq \tilde{\ell}}
=
\Psi_{\leq 0}
=
n^{1/2}q,
\qquad
\theta_{\leq \tilde{\ell}}
=
\theta_0\in\mathbb R.
$$
Moreover,
$$
A_{\rm head}
=
\mu_0 qq^\top,
\qquad
A=A_{\rm head}+A_{\rm tail}.
$$
Let
$$
b=\|A_{\rm tail}\|_{\rm op}.
$$
Since $n \asymp d^\gamma$ and $\lambda \asymp d^{-u}$, the condition
$0\le u<\gamma$ and \eqref{eqn_B_1_collect_4} imply
$$
b=O_{d,\mathbb P}(d^{-\gamma})
=o_{d,\mathbb P}(\lambda).
$$

\medskip

\noindent{\bf Step 1: Reduction to a scalar directional quantity.}

By the definition of $B_1$, we have
$$
\begin{aligned}
\|B_1\|_2
&=
|\theta_0|
\left|
1-\mu_0 q^\top\reg(A)q
\right| \\
&=
|\theta_0|
\left|
q^\top
\left(\mathrm{I}_n-A_{\rm head}\reg(A)\right)
q
\right| \\
&=
|\theta_0|
\left|
q^\top
\left(\mathrm{I}_n-(A-A_{\rm tail})\reg(A)\right)
q
\right| \\
&=
|\theta_0|
\left|
q^\top\rem(A)q
+
q^\top A_{\rm tail}\reg(A)q
\right|.
\end{aligned}
$$
Since Assumptions 3 and 5 imply $|\theta_0| = \Theta_{d}(1)$, it remains to control
$
q^\top\rem(A)q
$
and
$
q^\top A_{\rm tail}\reg(A)q
$ separately.

\medskip

\noindent{\bf Step 2: Directional residual bound for $q^\top\rem(A)q$.}

Let $a_1$ be the largest eigenvalue of $A$, and let $p_1$ be a
corresponding unit eigenvector. By Weyl's inequality,
$$
\mu_0\le a_1\le \mu_0+\|A_{\rm tail}\|_{2} \overset{\eqref{eqn_B_1_collect_4}}{\leq} \mu_0 + o_{d, \mathbb{P}}(1).
$$
Since Proposition 2.1 implies $\mu_0 = \Theta_{d}(1)$, we further have $a_1  = \Theta_{d, \mathbb{P}}(1)$.

Projecting the eigenequation
$$
Ap_1=a_1p_1
$$
onto $q^\perp$, and using $A=\mu_0qq^\top+A_{\rm tail}$, we obtain
$$
a_1P_{q^\perp}p_1
=
P_{q^\perp}A_{\rm tail}p_1.
$$
Hence
$$
\|P_{q^\perp}p_1\|_2
\le
\frac{\|A_{\rm tail}\|_{2}}{a_1}
=o_{d, \mathbb{P}}(1).
$$
Consequently,
\begin{equation}\label{eqn_b1_special_small_gamma_case_1}
    |q^\top p_1|^2
=
1-O_{d,\mathbb P}(b^2).
\end{equation}

Let $\{p_j\}_{j=1}^n$ be an orthonormal eigenbasis of $A$, with
corresponding eigenvalues $\{a_j\}_{j=1}^n$. Since $\rem(A)$ is a
matrix function of $A$,
$$
q^\top\rem(A)q
=
\sum_{j=1}^n
\rem(a_j)|q^\top p_j|^2.
$$
Recall that Definition 2.2 implies $0\le \rem(z)\le 1$ on the spectrum of
$A$, and that
$$
\rem(a_1)\le \mathfrak{C}_3 a_1^{-t} \lambda^{t}, \quad 1\leq t\leq \tau.
$$
Therefore,
$$
\begin{aligned}
q^\top\rem(A)q
&=
\rem(a_1)|q^\top p_1|^2
+
\sum_{j\ge 2}\rem(a_j)|q^\top p_j|^2 \\
&\le
O_{d, \mathbb{P}}(\lambda^{\tau^{\prime}})
+
\sum_{j\ge 2}|q^\top p_j|^2 \\
&=
O_{d, \mathbb{P}}(\lambda^{\tau^{\prime}})
+
\left(1-|q^\top p_1|^2\right) \\
&=
O_{d, \mathbb{P}}(\lambda^{\tau^{\prime}} + n^{-2}).
\end{aligned}
$$

If $\tau<\infty$, choose $\kappa_1=K_{\max}$. Since $A=K/n$ is positive semidefinite and $\operatorname{tr}(A)=K_{\max}$, its largest eigenvalue satisfies $a_1\leq\kappa_1$, and hence
$$
\begin{aligned}
q^\top\rem(A)q
&\ge
\rem(a_1)|q^\top p_1|^2 \\
(\because \rem(z) \text{ is non-increasing in } z) &\ge
\rem(\kappa_1)|q^\top p_1|^2 \\
(\because (9) \text{ and } \eqref{eqn_b1_special_small_gamma_case_1}) \quad &\ge
\Omega_{d, \mathbb{P}}(\lambda^{\tau}).
\end{aligned}
$$

\medskip

\noindent{\bf Step 3: Control of $\left|q^\top A_{\rm tail}\reg(A)q \right|$.}

Since
$$
A_{\rm head}=\mu_0qq^\top,
$$
we have
$$
q=\mu_0^{-1}A_{\rm head}q.
$$
Thus
$$
\reg(A)q
=
\mu_0^{-1}\reg(A)A_{\rm head}q.
$$
Using $A_{\rm head}=A-A_{\rm tail}$, we get
$$
\begin{aligned}
\reg(A)A_{\rm head}
&=
\reg(A)(A-A_{\rm tail}) \\
&=
\reg(A)A-\reg(A)A_{\rm tail}\\
&= \mathrm{I}_n-\rem(A)-\reg(A)A_{\rm tail}
\end{aligned}
$$
Taking operator norms and using the filter bounds in Definition 2.2
$$
\|\rem(A)\|_{\rm op}\le 1,
\qquad
\|\reg(A)\|_{\rm op}\le C_2\lambda^{-1},
$$
we obtain
$$
\begin{aligned}
\|\reg(A)A_{\rm head}\|_{\rm op}
&\le
1+\|\rem(A)\|_{\rm op}
+
\|\reg(A)\|_{\rm op}\|A_{\rm tail}\|_{\rm op} \\
&\le
2+C_2\frac{b}{\lambda}.
\end{aligned}
$$
Since $b=o_{d,\mathbb P}(\lambda)$, this implies
$$
\|\reg(A)A_{\rm head}\|_{\rm op}
=
O_{d,\mathbb P}(1).
$$
Consequently,
$$
\|\reg(A)q\|_2
\le
\mu_0^{-1}
\|\reg(A)A_{\rm head}\|_{\rm op}
=
O_{d,\mathbb P}(1).
$$
Hence
\begin{equation}
    \begin{aligned}
        &~\left|q^\top A_{\rm tail}\reg(A)q \right|\\
        (\text{Cauchy-Schwarz inequality}) \quad \leq &~ \|q\|_2 \|A_{\rm tail}\reg(A)q\|_2\\
        \leq &~ \|A_{\rm tail}\|_2 \|\reg(A)q\|_2\\
        =&~ O_{d,\mathbb P}(n^{-1}).
    \end{aligned}
\end{equation}

\medskip

\noindent{\bf Step 4: Bounds for $\|B_1\|_2^2$.}

For any $x, y \in \mathbb{R}$, we have $|x|^2/2 - |y|^2 \leq |x+y|^2 \leq 2|x|^2 + 2|y|^2$. Therefore, combining Steps 1--3, we have
\begin{equation*}
    \begin{aligned}
                \|B_1\|_{2}^{2}  \leq &~ 
                O_{d, \mathbb{P}}\left(  d^{-2\tau^{\prime} u} \right) + \Delta_1;\\
                \|B_1\|_{2}^{2}  \geq &~ 
                \Omega_{d, \mathbb{P}}\left(  d^{-2\tau u} \right)- \Delta_1, \quad \text{ if } \tau<\infty,
    \end{aligned}
\end{equation*}
where $0 \leq \Delta_1 = O_{d, \mathbb{P}}\left( d^{-2\gamma }\right)$.

In summary, we have completed the proof for the following statements about $B_{1}$:
\begin{itemize}
    \item We have
    \begin{equation}\label{eqn_bound_B_1_spe}
    \begin{aligned}
                \|B_1\|_{2}^{2}  =&~ 
                O_{d, \mathbb{P}}\left(d^{-2u + (2-\min\{s, 2\})\tilde{\ell} } + d^{(2-\min\{s, 2\})\tilde{\ell}-2\gamma }\right);
    \end{aligned}
\end{equation}

    \item If $0 < u < \gamma <1$, then for any $\tau \leq \infty$, we have
    \begin{equation*}
    \begin{aligned}
                \|B_1\|_{2}^{2}  \leq &~ 
                O_{d, \mathbb{P}}\left(  d^{-2\tau^{\prime} u} \right) + \Delta_1;\\
                \|B_1\|_{2}^{2}  \geq &~ 
                \Omega_{d, \mathbb{P}}\left(  d^{-2\tau u} \right)- \Delta_1, \quad \text{ if } \tau<\infty,
    \end{aligned}
\end{equation*}
where $0 \leq \Delta_1 = O_{d, \mathbb{P}}\left( d^{-2\gamma }\right)$.

\end{itemize}


\noindent {\bf Bounding $\|B_2\|_2^2$.}
Let
$$
g_{>\tilde\ell}
=
n^{-1/2}\Psi_{>\tilde\ell}\theta_{>\tilde\ell}.
$$
Since $\varphi_{\lambda,K}=n^{-1}\reg(A)$ and
$\Psi_{\le\tilde\ell}=n^{1/2}Z$, we have
$$
B_2
= \Sigma_{\leq \tilde{\ell}} \Psi^{\top}_{\leq \tilde{\ell}} \varphi_{\lambda, K} \Psi_{>\tilde{\ell}} \theta_{>\tilde{\ell}} =
\Sigma_{\le\tilde\ell}Z^\top\reg(A)g_{>\tilde\ell}.
$$
By the concentration of $Z^\top Z$ in \eqref{eqn_bound_eigen_psi_top_psi},
$$
\|B_2\|_2^2
=
\Theta_{d,\mathbb P}(1)\|ZB_2\|_2^2.
$$
Moreover,
$$
ZB_2
=
Z\Sigma_{\le\tilde\ell}Z^\top\reg(A)g_{>\tilde\ell}
=
A_{\rm head}\reg(A)g_{>\tilde\ell}.
$$
Therefore, by Lemma~\ref{lemma_head_filter_product},
$$
\begin{aligned}
\|B_2\|_2^2
&=
\Theta_{d,\mathbb P}(1)
\|A_{\rm head}\reg(A)g_{>\tilde\ell}\|_2^2 \\
&\le
O_{d,\mathbb P}(1)
\|g_{>\tilde\ell}\|_2^2.
\end{aligned}
$$
Since
$$
\|g_{>\tilde\ell}\|_2^2
=
\frac{1}{n}
\|\Psi_{>\tilde\ell}\theta_{>\tilde\ell}\|_2^2,
$$
Lemma~\ref{lemma ge l l2} gives
\begin{equation}\label{eqn_bound_B_2_spe}
    \begin{aligned}
                \|B_2\|_{2}^{2} 
=&~ 
            O_{d, \mathbb{P}}\left(  d^{-(\tilde{\ell}+1)s}\right).
    \end{aligned}
\end{equation}

\noindent {\bf Bounding $\|B_3\|_{2}^{2}$. }
We have
\begin{equation}\label{eqn_bound_B_3_spe}
    \begin{aligned}
                \|B_3\|_{2}^{2} 
        \overset{\text{Lemma } \ref{lemma ge l l2}}{=}&~
        \Theta_{d, \mathbb{P}} \left( d^{-(\tilde{\ell}+1)s}\right).
    \end{aligned}
\end{equation}

\noindent {\bf Bounding $\|B_4\|_2^2$.}
Define
$$
R_{>\tilde\ell}
=
\frac{1}{n}
\Psi_{>\tilde\ell}
\Sigma_{>\tilde\ell}^2
\Psi_{>\tilde\ell}^\top.
$$
Recall the notations $\varphi_{\lambda,K}=n^{-1}\reg(A)$,
$\Psi_{\le\tilde\ell}=n^{1/2}Z$, and $\xi
=
Z(Z^\top Z)^{-1}
\Sigma_{\le\tilde\ell}^{-1}
\theta_{\le\tilde\ell}$. Since $A_{\rm head}\xi
=
Z\theta_{\le\tilde\ell}$, we can rewrite $B_4$ as
\begin{equation*}
    \begin{aligned}
        \|B_4\|_2^2
= &~
\theta_{\le\tilde\ell}^\top
Z^\top
\reg(A)
R_{>\tilde\ell}
\reg(A)
Z
\theta_{\le\tilde\ell}\\
= &~
\xi^\top
A_{\rm head}
\reg(A)
R_{>\tilde\ell}
\reg(A)
A_{\rm head}
\xi.
    \end{aligned}
\end{equation*}

Hence, by Lemma~\ref{lemma_head_filter_product},
$$
\begin{aligned}
\|B_4\|_2^2
&\le
\|A_{\rm head}\reg(A)\|_{\rm op}
\|R_{>\tilde\ell}\|_{\rm op}
\|\reg(A)A_{\rm head}\|_{\rm op}
\|\xi\|_2^2 \\
&=
O_{d,\mathbb P}(1)
\|R_{>\tilde\ell}\|_{\rm op}
\|\xi\|_2^2\\
(\eqref{eqn_B_1_collect_7}) \quad &= O_{d,\mathbb P}(d^{(2-\min\{s, 2\})\tilde{\ell} })
\|R_{>\tilde\ell}\|_{\rm op}
\end{aligned}
$$

It remains to control $\|R_{>\tilde\ell}\|_{\rm op}$. Since
$$
\Sigma_{>\tilde\ell}^2
\preceq
\|\Sigma_{>\tilde\ell}\|_{\rm op}
\Sigma_{>\tilde\ell},
$$
we have
$$
R_{>\tilde\ell}
\preceq
\|\Sigma_{>\tilde\ell}\|_{\rm op}
\frac{1}{n}
\Psi_{>\tilde\ell}
\Sigma_{>\tilde\ell}
\Psi_{>\tilde\ell}^\top
=
\|\Sigma_{>\tilde\ell}\|_{\rm op}
A_{\rm tail}.
$$
Thus
$$
\|R_{>\tilde\ell}\|_{\rm op}
\le
\|\Sigma_{>\tilde\ell}\|_{\rm op}
\|A_{\rm tail}\|_{\rm op}.
$$
By the spectral decay of $\Sigma$,
$$
\|\Sigma_{>\tilde\ell}\|_{\rm op}
=
O_d(d^{-(\tilde\ell+1)}).
$$
By \eqref{eqn_B_1_collect_4},
$$
\|A_{\rm tail}\|_{2} = O_{d, \mathbb{P}}(n^{-1})\mathbf{1}\{\ell_{\gamma} \leq \ell_{\lambda}\} + O_{d, \mathbb{P}}\left({d^{-\ell_{\lambda}-1}}\right)\mathbf{1}\{\ell_{\gamma} > \ell_{\lambda}\}.
$$
Therefore,
$$
\|R_{>\tilde\ell}\|_{\rm op}
=
O_{d,\mathbb P}
\left(
d^{-(\tilde\ell+1)-\gamma}
\right)\mathbf{1}\{\ell_{\gamma} \leq \ell_{\lambda}\} + O_{d, \mathbb{P}}\left({d^{-2\ell_{\lambda}-2}}\right)\mathbf{1}\{\ell_{\gamma} > \ell_{\lambda}\}.
$$
Therefore, we have
\begin{equation}\label{eqn_bound_B_4_spe}
\begin{aligned}
    &~\|B_4\|_{2}^{2} = O_{d, \mathbb{P}} \left(  d^{(2-\min\{s, 2\})\tilde{\ell} - (\gamma  + \ell_{\gamma} + 1)\mathbf{1}\{\ell_{\gamma} \leq \ell_{\lambda}\} - 2(\ell_{\lambda}+1)\mathbf{1}\{\ell_{\gamma} > \ell_{\lambda}\} } \right).
\end{aligned}
\end{equation}

\noindent {\bf Bounding $\|B_5\|_{2}^{2}$. }
We have
\begin{equation}\label{eqn_bound_B_5_spe}
    \begin{aligned}
\|B_5\|_{2}^{2} 
        =&~
        \left( \Psi_{>\tilde{\ell}} \theta_{>\tilde{\ell}}  \right)^{\top} 
        \left( \varphi_{\lambda, K}   \Psi_{>\tilde{\ell}} \Sigma_{>\tilde{\ell}}^2 \Psi_{>\tilde{\ell}}^{\top} \varphi_{\lambda, K} \right) 
        \left( \Psi_{>\tilde{\ell}} \theta_{>\tilde{\ell}}  \right)\\
        \leq &~
        \|\Sigma_{> \tilde{\ell}}\|_{\mathrm{op}} \cdot \|\varphi_{\lambda, K} K_{> \tilde{\ell}}  \varphi_{\lambda, K}\|_{2} \cdot \|\Psi_{>\tilde{\ell}} \theta_{>\tilde{\ell}}\|_{2}^{2}\\
        \overset{\text{Lemma } \ref{lemma ge l l2}}{\leq}&~
         O_{d, \mathbb{P}} \left(  d^{-(\tilde{\ell}+1)(s+1) + \gamma} \right) \cdot \|\varphi_{\lambda, K} K  \varphi_{\lambda, K}\|_{2}\\
         =&~
         O_{d, \mathbb{P}} \left(  d^{-(\tilde{\ell}+1)(s+1) + \gamma} \right) \cdot \|\varphi_{\lambda, K}^{1/2} \left(\varphi_{\lambda, K}^{1/2} K \varphi_{\lambda, K}^{1/2}\right)  \varphi_{\lambda, K}^{1/2}\|_{2}\\
         \leq &~
         O_{d, \mathbb{P}} \left(  d^{-(\tilde{\ell}+1)(s+1) + \gamma} \right) \cdot \|\varphi_{\lambda, K}\|_{2}\\
         \overset{(\ref{eqn:spe_properties})}{=} &~
         O_{d, \mathbb{P}} \left(  d^{-(\tilde{\ell}+1)(s+1) + \gamma} \right) \cdot \|M^{-1}\|_{2}\\
         \overset{\text{Lemma } \ref{lemma lambda max min}}{=}&~
         O_{d, \mathbb{P}}\left( \frac{n }{d^{(\tilde{\ell}+1)(s+1)} (n\lambda+1)}\right) = O_{d, \mathbb{P}}\left( \frac{1}{d^{\tilde{\ell}+1}(\lambda +n^{-1})}\cdot d^{-(\tilde{\ell}+1)s} \right)\\
         {=}&~
         o_{d, \mathbb{P}}\left(   d^{-(\tilde{\ell}+1)s} \right).
    \end{aligned}
\end{equation}

\paragraph{\bf Case (II). } Suppose there exists a positive integer $m$ such that $\ell_{\lambda} \geq \gamma = \ell_{\gamma} = m$. 

Put $b=m-1$ and partition the coefficients and features at degree $b$. Write $H=\{0,\ldots,b\}$ and $T=\{m,m+1,\ldots\}$. Set
\[
 Z=\Psi_H/\sqrt n,\quad \mathsf A=K/n,\quad
 \mathsf A_H=Z\Sigma_HZ^\top,\quad \mathsf A_T=K_T/n,
 \quad \varphi=\varphi_\lambda(\mathsf A),\quad F=\varphi/n.
\]
Here $\Psi_T$ includes the critical block; in particular, $K_T=K_m+K_{>m}$.

\paragraph{Preliminary bounds.}
Lemma \ref{lemma psi top psi_int} and \eqref{eq:G1} give
\begin{equation}\label{eq:bias-prelim}
 \begin{aligned}
 Z^\top Z&\asympp \mathrm{I},& \norm{\mathsf A_T}_{\op}&=O_{d,\Pp}(n^{-1}),\\
 \lambda_{\min}(\mathsf A)&\ge c/n,& \norm{\varphi}_{\op}&=O_{d,\Pp}(n),
 &\norm{F}_{\op}&=O_{d,\Pp}(1).
 \end{aligned}
\end{equation}
The source condition and the population eigenvalue bounds imply
\begin{equation}\label{eq:sourcebounds}
 \norm{\theta_T}_2^2\le Cd^{-ms},\qquad
 \norm{\Sigma_H^{-1}\theta_H}_2^2\le Cd^{(2-\min(s,2))b},
 \qquad \norm{\Sigma_T}_{\op}\le Cd^{-m}\asymp n^{-1}.
\end{equation}
Since $f_\star$ is fixed independently of the design,
\begin{equation}\label{eq:responsebound}
 \norm{\Psi_T\theta_T}_2^2=O_{d,\Pp}(n\norm{\theta_T}_2^2)
\end{equation}
by its exact expectation and Markov's inequality. These comparisons are also valid if $\theta_T=0$.

Using $\mathsf A_H=\mathsf A-\mathsf A_T$ and $\norm{\mathsf A\varphi}_{\op}\le1$,
\begin{equation}\label{eq:headfilter-strong}
 \norm{\mathsf A_H\varphi}_{\op}
 +\norm{\varphi\mathsf A_H}_{\op}=O_{d,\Pp}(1).
\end{equation}
Let
\[
 \xi=Z(Z^\top Z)^{-1}\Sigma_H^{-1}\theta_H.
\]
Then $\mathsf A_H\xi=Z\theta_H$ and
\begin{equation}\label{eq:xibound}
 \norm{\xi}_2^2=O_{d,\Pp}(d^{(2-\min(s,2))b}).
\end{equation}

\paragraph{The exact decomposition.}
Define
\begin{equation}\label{eq:Bdef}
 \begin{aligned}
 B_1&=\Sigma_H\Psi_H^\top F\Psi_H\theta_H-\theta_H,&
 B_2&=\Sigma_H\Psi_H^\top F\Psi_T\theta_T,\\
 B_3&=\theta_T,& B_4&=\Sigma_T\Psi_T^\top F\Psi_H\theta_H,&
 B_5&=\Sigma_T\Psi_T^\top F\Psi_T\theta_T.
 \end{aligned}
\end{equation}
Parseval's identity gives precisely
\begin{equation}\label{eq:exactbias}
 \operatorname{bias}^2(\widehat f_\lambda)
 =\norm{B_1+B_2}_2^2+\norm{B_3-B_4-B_5}_2^2.
\end{equation}
Infinite coefficient vectors are used only through their $L^2$ norms and the convergent kernel quadratic forms, as in the manuscript.

\paragraph{The first term.}
Since $Z^\top Z\asympp \mathrm{I}$, multiplication by $Z$ preserves coefficient norms up to constants. Thus
\[
 \norm{B_1}_2^2\asympp
 \norm{(\mathrm{I}_n-\mathsf A_H\varphi)\mathsf A_H\xi}_2^2.
\]
{Write $\psi_\lambda(\mathsf A)=\mathrm{I}_n-\mathsf A\varphi$. Commutation of $\mathsf A$ and $\varphi$ gives}
\[
 (\mathrm{I}_n-\mathsf A_H\varphi)\mathsf A_H
 =\mathsf A_H\psi_\lambda(\mathsf A)+\mathsf A_H\varphi\mathsf A_T.
\]
Qualification at least one gives $\norm{\mathsf A\psi_\lambda(\mathsf A)}_{\op}\le C\lambda$. Since $\lambda=O(n^{-1})$ and $\norm{\psi_\lambda(\mathsf A)}_{\op}\le1$, the first term on the right is $O_{d,\Pp}(n^{-1})$. By \eqref{eq:headfilter-strong}, the second has the same bound. Therefore
\begin{equation}\label{eq:B1bound}
 \norm{B_1}_2^2=O_{d,\Pp}(n^{-2}\norm{\xi}_2^2).
\end{equation}
At $\lambda=0$, the residual is zero on the positive spectrum, so the argument remains valid.

\paragraph{The second and third terms.}
For $g=\Psi_T\theta_T/\sqrt n$, $ZB_2=\mathsf A_H\varphi g$. Equations \eqref{eq:responsebound} and \eqref{eq:headfilter-strong} imply
\begin{equation}\label{eq:B23bound}
 \norm{B_2}_2^2=O_{d,\Pp}(\norm{\theta_T}_2^2),
 \qquad \norm{B_3}_2^2=\norm{\theta_T}_2^2.
\end{equation}

\paragraph{The fourth term.}
Put $R_T=n^{-1}\Psi_T\Sigma_T^2\Psi_T^\top$. Since
$\Sigma_T^2\preceq\norm{\Sigma_T}_{\op}\Sigma_T$,
\[
 \norm{R_T}_{\op}\le\norm{\Sigma_T}_{\op}\norm{\mathsf A_T}_{\op}
 =O_{d,\Pp}(n^{-2}).
\]
Using $Z\theta_H=\mathsf A_H\xi$ yields
\begin{equation}\label{eq:B4bound}
 \norm{B_4}_2^2
 =\xi^\top\mathsf A_H\varphi R_T\varphi\mathsf A_H\xi
 =O_{d,\Pp}(n^{-2}\norm{\xi}_2^2).
\end{equation}

\paragraph{The fifth term.}
The real-axis filter bound $0\le z\varphi_\lambda(z)\le1$ implies $FKF\preceq F$, by simultaneous diagonalization. Since $K_T\preceq K$,
\begin{equation}\label{eq:B5bound}
 \begin{aligned}
 \norm{B_5}_2^2
 &\le\norm{\Sigma_T}_{\op}\norm{FK_TF}_{\op}\norm{\Psi_T\theta_T}_2^2\\
 &\le\norm{\Sigma_T}_{\op}\norm{F}_{\op}\norm{\Psi_T\theta_T}_2^2
 =O_{d,\Pp}(\norm{\theta_T}_2^2).
 \end{aligned}
\end{equation}

Finally, apply
\[
 \norm{B_1+B_2}^2\le2\norm{B_1}^2+2\norm{B_2}^2,
 \quad
 \norm{B_3-B_4-B_5}^2\le3\norm{B_3}^2+3\norm{B_4}^2+3\norm{B_5}^2.
\]
Equations \eqref{eq:sourcebounds}--\eqref{eq:B5bound} give
\[
 \operatorname{bias}^2(\widehat f_\lambda)
 =O_{d,\Pp}\big(d^{-ms}+n^{-2}d^{(2-\min(s,2))(m-1)}\big),
\]
which is \eqref{eq:bias} because $n\asymp d^m$.

\end{proof}

\subsection[Bounds for the bias term in the under-regularized regimes]{Bounds for the bias term in the under-regularized regimes}

The following proposition adapts the bias estimates in Appendix D.4.2 of \cite{lu2024saturation} to $u^\prime(\tau)\leq u<u^\prime(1)$ and $\gamma\geq1$. Here $u^\prime(t)$ is the exponent specified in the proof of Proposition \ref{prop_e_1}.

\begin{proposition}\label{prop_verify_bias_special}
    Suppose Assumptions 1--5 hold.
Suppose that $\tau \leq \infty$, $s>0$,
$u^{\prime}(\tau) \leq u < u^{\prime}(1)$, and $\gamma \geq 1$. 
Then, the bias term (15) of the spectral algorithm estimator $\hat{f}_{\lambda}$ satisfies:

    \begin{equation*}
\begin{aligned}
    \mathrm{bias}^{2}(\hat{f}_{\lambda})  \leq  &~
    \Theta_d\left( 
    d^{-(\tilde{\ell} + 1)  s}
+ \mathbf{1}\{\tau<\infty\} d^{-2  \tau  u + (2  \tau - \tilde{s})  \tilde{\ell}} + \mathbf{1}\{\tau = \infty, u = \tilde{\ell}\} d^{-\tilde{\ell} s}
\right) 
    + 
    \Delta_2,\\
    \mathrm{bias}^{2}(\hat{f}_{\lambda})  \geq  &~
    \Theta_d\left( 
    d^{-(\tilde{\ell} + 1)  s}
+ \mathbf{1}\{\tau<\infty\} d^{-2  \tau  u + (2  \tau - \tilde{s})  \tilde{\ell}} + \mathbf{1}\{\tau = \infty, u = \tilde{\ell}\} d^{-\tilde{\ell} s}
\right) 
    -
    \Delta_2,
\end{aligned} 
\end{equation*}
where $\Delta_2 \geq 0$ and
\begin{equation*}
    \begin{aligned}
        \Delta_2 =&~ o_{d, \mathbb{P}} \left(
d^{\gamma - \tilde{\ell} - 1 - 2  \max\{\gamma - u, 0\}}
+ d^{\tilde{\ell} - \gamma}
+ d^{-(\tilde{\ell} + 1)  s}
+ \mathbf{1}\{\tau<\infty\} d^{-2  \tau  u + (2  \tau - \tilde{s})  \tilde{\ell}}\right)\\
&~+o_{d, \mathbb{P}} \left( \mathbf{1}\{\tau = \infty, u = \tilde{\ell}\} d^{-\tilde{\ell} s}
\right).
    \end{aligned}
\end{equation*}
\end{proposition}
\begin{proof} 
We use the notation in \cite{lu2024saturation}, in particular:
   \begin{align*}
       &~ \mathcal{N}_{1, \varphi}(\lambda), \mathcal{N}_{2, \varphi}(\lambda), \mathcal{M}_{1, \varphi}(\lambda), 
       \mathcal{M}_{2, \varphi}(\lambda) \text{ in (42) of \cite{lu2024saturation}}\\
       &~ \mathcal{N}_{1}(\lambda) := \mathcal{N}_{1, \varphi^{\krr}}(\lambda) \text{ below (42) of \cite{lu2024saturation}}\\
       &~ \tilde{f}_{\lambda}, f_{\lambda} \text{ in (36) and (38) of \cite{lu2024saturation}}.
   \end{align*}

The operator argument in Appendix \ref{append_proof_of_previous_prop_B_1} yields ~\eqref{eq:p1_bias_approximation} for every fixed $0<u<\gamma$ satisfying ~\eqref{eq:p1_concentration} and the stated bias conditions. Here $u<u^\prime(1)<\gamma$, and $N_1(\lambda)\asymp d^u$ gives ~\eqref{eq:p1_concentration}. We verify the remaining conditions in \eqref{eqn:check_condition_bias_1} and \eqref{eqn:check_condition_bias_3} below.

We introduce the following lemma about $\mathcal{M}_{j, \varphi}(\lambda)$, $j=1,2$, which fixes a mistake in Lemma D.14 of \cite{lu2024saturation}.

\begin{lemma}\label{lemma_bounds_on_quantities}
    Suppose Assumptions 1--5 hold.
Suppose that $\tau = \infty$, $s>0$, and $0<u \leq \gamma$. Define
$$
    \mathcal M_{1,\varphi}(\lambda)=\norm{\psi_\lambda(T)f_\star}_{L^\infty},\quad
 \mathcal M_{2,\varphi}(\lambda)=\norm{\psi_\lambda(T)f_\star}_{L^2}^2.
$$
Then we have
    \begin{equation}\label{eq:M2}
    \begin{aligned}
        \mathcal{M}_{2, \varphi}(\lambda) 
        = &~ \Theta_{d}\left(d^{-(\ell_{\lambda}+1)s}
    + \mathbf{1}\{u = \ell_{\lambda}\} d^{-\ell_{\lambda} s}
         \right)\\
        \mathcal{M}_{1, \varphi}^2(\lambda) 
        = &~  O_{d}\left( d^{- (\ell_{\lambda}+1)(s-1)}+ \mathbf{1}\{u = \ell_{\lambda}\} d^{-\ell_{\lambda} (s-1)}
         \right), \quad s \geq 1.
    \end{aligned}
    \end{equation}
\end{lemma}
\begin{proof}
In this proof, we use $C$ and $C^{\prime}$ to represent absolute constants.

Let $a = \ell_{\lambda}$ and $e_k=\mu_k^{-s}\sum_j\theta_{k,j}^2$, so that $\sum_ke_k\le R$. On the tail $k>a$, spectral ordering gives $\mu_k\le C d^{-(a+1)}$. Since $0\le\psi_\lambda\le1$,
\[
 \sum_{k>a}\psi_\lambda(\mu_k)^2\mu_k^s e_k
 \le CRd^{-(a+1)s}.
\]
Furthermore, $\mu_{a+1}/\lambda\to0$, so the upper filter bound implies $\psi_\lambda(\mu_{a+1})\ge \mathfrak{C}_{5} > 0$ for large $d$. Nondegeneracy in block $a+1$, i.e., Assumption 5, yields the corresponding lower bound.

For $k\le a$, choose any fixed qualification exponent $t$ large enough that $2t\ge s$. The head contribution is at most
\[
 \mathfrak{C}_3^2\lambda^{2t}\max_{k\le a}\mu_k^{s-2t}\sum_{k\le a}e_k
 \le \mathfrak{C}_3^2 d^{-2t(u-a)-as}.
\]
If $u>a$, choose $2t(u-a)\ge s$. If $u=a$, this is $O(d^{-as})$, and Definition 2.2(4), applied with a fixed upper bound on $\mu_a/\lambda$, together with block-$a$ nondegeneracy, gives the lower bound $\Omega(d^{-as})$. This proves
$$
    \mathcal{M}_{2, \varphi}(\lambda) 
        =  \Theta_{d}\left(d^{-(a+1)s}
    + \mathbf{1}\{u = a\} d^{-a s}
         \right).
$$

For the supremum norm, the addition formula and Cauchy--Schwarz give, for every $x$,
\[
 \left|\sum_{k}\sum_j\psi_\lambda(\mu_k)\theta_{k,j}\psi_{k,j}(x)\right|^2
 \le R\sum_{k}\psi_\lambda(\mu_k)^2\mu_k^sN(d,k).
\]
If $s\ge1$, the tail is bounded by
\[
 R\sum_{k>a}\mu_k^sN(d,k)
 \le R\left(\sup_{k>a}\mu_k\right)^{s-1}\sum_{k>a}\mu_kN(d,k)
 \le C d^{-(a+1)(s-1)}.
\]
The head, using qualification $t$ with $2t+1-s\ge0$, is at most
\[
 CR\mathfrak{C}_3^2 \lambda^{2t}\sum_{k=0}^a d^{k(2t+1-s)}
 \le C^{\prime} d^{-2t(u-a)-a(s-1)}.
\]
At $u=a$ this has the asserted order. At $u>a$, choose $t$ so that $2t(u-a)\ge s-1$. Finally, using $|h+t|^2\le2|h|^2+2|t|^2$, we obtain
$$
    \mathcal{M}_{1, \varphi}^2(\lambda) 
        =  O_{d}\left( d^{- (a+1)(s-1)}+ \mathbf{1}\{u = a\} d^{-a (s-1)}
         \right), \quad s \geq 1.
$$
\end{proof}

Recall that $\lambda \asymp d^{-u}$.
Set $a=\ell_{\lambda} =\lfloor u\rfloor$, $0<u<\gamma$, and
\[
 \mathcal{P}_{d}=\mathcal M_{2,\varphi}(\lambda)+\frac{\sigma^2}{n}\mathcal N_{2,\varphi}(\lambda),
 \qquad \mathcal{W}_{d}=n^{-1}d^a.
\]
In particular, Lemmas \ref{lem:popall} and \ref{lemma_bounds_on_quantities} imply $P_d\gtrsim W_d+d^{-(a+1)s}$. We also need the following two lemmas.

\begin{lemma}\label{lemma_bias_special_condition_2}
    Define
    \[
 \mathcal{S}_\lambda=\sum_j\frac{\lambda^2\lambda_j\varphi_\lambda(\lambda_j)^2}{\lambda+\lambda_j}f_j^2.
\]
Then for every fixed $s>0$, we have
$$
    \frac{\mathcal{N}_{1}(\lambda)}{n} \ln(n) (\ln\lambda^{-1})^2 \cdot \mathcal{S}_\lambda  
    = 
    o_{d}\left( \mathcal{P}_{d}  \right).
$$
\end{lemma}
\begin{proof}
    Splitting at $a$ and using $\varphi_\lambda(z)\lesssim(z+\lambda)^{-1}$ gives
\[
 \mathcal{S}_\lambda\lesssim \lambda^2d^{a\max(2-s,0)}+d^{-(a+1)s}.
\]
Lemma \ref{lem:popall} implies $\mathcal{N}_{1}(\lambda) \asymp \lambda^{-1}$. Hence, with $L_d=(\log d)^3$,
\[
 \frac{\mathcal N_1}{n}L_d \mathcal{S}_\lambda
 \lesssim L_d\left[n^{-1}\lambda d^{a\max(2-s,0)}
       +n^{-1}\lambda^{-1}d^{-(a+1)s}\right].
\]
For the first term, division by $\mathcal{W}_d$ gives
$L_dd^{-u+a(\max(2-s,0)-1)}=o(1)$: for $0<s<1$ its power is $-(u-a)-as<0$, and for $s\ge1$ it is at most $-u<0$. The second term divided by $d^{-(a+1)s}$ gives $L_d d^{u-\gamma}=o(1)$. Therefore, we have
$$
    \frac{\mathcal{N}_{1}(\lambda)}{n} \ln(n) (\ln\lambda^{-1})^2 \cdot \mathcal{S}_\lambda  
    = o_{d}\left( \mathcal{W}_{d}+d^{-(a+1)s} \right)
    =o_{d}\left( \mathcal{P}_{d}  \right).
$$
\end{proof}

\begin{lemma}\label{lem:flambda}
For $0<s<1$, the population regularized target $f_\lambda=T\varphi_\lambda(T)f_\star$ satisfies
\begin{align}
 \norm{f_\lambda}_{\infty}
 &\lesssim d^{(1-s)a/2}+\lambda^{-1}d^{-(1+s)(a+1)/2},\label{eq:flambda}\\
 \norm{f_\lambda}_{\infty}^2
 &\lesssim d^{(1-s)a}+\lambda^{-2}d^{-(1+s)(a+1)}.\label{eq:flambdasq}
\end{align}
\end{lemma}
\begin{proof}
On the head, the transfer $\mu_k\varphi_\lambda(\mu_k)$ is at most one. Source Cauchy--Schwarz and the addition formula bound the squared supremum by
$R\sum_{k\le a}\mu_k^sN(d,k)\lesssim d^{(1-s)a}$.
On the tail, the squared supremum is at most
\[
 C R\lambda^{-2}\sum_{k>a}\mu_k^{s+2}N(d,k)
 \le C R\lambda^{-2}\left(\sup_{k>a}\mu_k\right)^{s+1}\tr(T).
\]
Take square roots and add the two parts to get \eqref{eq:flambda}; square with a factor two to obtain \eqref{eq:flambdasq}.
\end{proof}

Now we return to the proof of Proposition \ref{prop_verify_bias_special}.
We also note that the following notations are different in our paper and in \cite{lu2024saturation, zhang2024optimal}:
\begin{itemize}
    \item In \cite{lu2024saturation}, $\ell:=\lfloor u \rfloor$; while in our paper, we denote $\ell_{\lambda}:=\lfloor u \rfloor$.

    \item 
    Consider the case where $0<s<1$. 
    In this case,  $\tilde{s}:=\min(s, 2\tau)=s$, so the notation $\tilde{p} = \lfloor \gamma / (\tilde{s}+1) \rfloor$ used in \cite{zhang2024optimal} and in Proposition \ref{prop_e_1} is the same as the notation $p = \lfloor \gamma / (s+1) \rfloor$ in \cite{lu2024saturation}.
    Notice that one can show that $p= \tilde{p} < \gamma$ (see Proposition \ref{prop_e_1}).


\end{itemize}

We divide the proof into the following two cases: (i) $s \geq 1$, (ii) $0<s<1$.

\noindent {\bf Case (i): $s \geq 1$. }  
    Notice that:
    \begin{itemize}
        \item From (36) and (41) of \cite{lu2024saturation}, we have that $\mathrm{bias}(\hat{f}_{\lambda})$ equals the bias term defined in (41) of \cite{lu2024saturation}.

        \item From (60) of \cite{lu2024saturation}, we have that $\mathrm{bias}(\hat{f}_{\lambda}) \leq \left\| f_{\lambda} - f_{\star}\right\|_{L^{2}} 
    +
    \left\| \tilde{f}_{\lambda} - f_{\lambda}\right\|_{L^{2}}$ and $\mathrm{bias}(\hat{f}_{\lambda}) \geq \left\| f_{\lambda} - f_{\star}\right\|_{L^{2}} 
    -
    \left\| \tilde{f}_{\lambda} - f_{\lambda}\right\|_{L^{2}}$. Moreover, the latter two terms can be upper bounded separately:
    \begin{itemize}
        \item By the definition of $\mathcal{M}_{2,\varphi}(\lambda)$, we have
        \begin{equation*}
            \begin{aligned}
                \left\| f_{\lambda} - f_{\star}\right\|_{L^{2}}^2 =&~ \mathcal{M}_{2, \varphi}(\lambda).
            \end{aligned}
        \end{equation*}

        \item Suppose the following inequalities hold (which we will prove later)
        \begin{equation}\label{eqn:check_condition_bias_1}
      \begin{aligned}
          \frac{ \mathcal{N}_{1}(\lambda) \mathcal{M}_{1, \varphi}^2(\lambda)}{n^2}
    &= 
    o_{d}\left( \mathcal{P}_{d}  \right),\\
    \frac{\mathcal{N}_{1}(\lambda)}{n} \ln(n) (\ln\lambda^{-1})^2 \cdot \sum_{i=1}^\infty \frac{\lambda^2 \lambda_i \reg^2(\lambda_i)}{\lambda + \lambda_i} f_i^2  
    &= 
    o_{d}\left( \mathcal{P}_{d}  \right).
      \end{aligned}
  \end{equation}
  Then, on one hand, since $u < u^{\prime}(1) < \gamma$, from Lemma 23 in \cite{zhang2024optimal}, there exists $\epsilon>0$, such that
        $$
        n^{\epsilon - 1}\mathcal{N}_{1}(\lambda) = \Theta_d(n^{\epsilon - 1} \lambda^{-1})= o_d(1);
        $$
        Together with \eqref{eqn:check_condition_bias_1}, this verifies \eqref{eq:p1_concentration}, \eqref{eq:p1_common_bias_condition}, and \eqref{eq:p1_bounded_bias_condition}. The bias approximation \eqref{eq:p1_bias_approximation} therefore gives
 \begin{equation*}
        \left\| \tilde{f}_{\lambda} - f_{\lambda}\right\|_{L^{2}}^2 = o_{d, \bbP}\left( \mathcal{P}_{d}  \right).
    \end{equation*}
  
    \end{itemize}

        \item Proposition \ref{prop_e_1} implies that $u < u^{\prime}(1) < \gamma$, hence $\tilde{\ell}=a$.
    \end{itemize}

Therefore, from the above arguments, we only need to verify that (\ref{eqn:check_condition_bias_1}) holds.

\begin{remark}
    We remark that the above argument is the same as that in \cite{lu2024saturation}, where they verify an analogous version of (\ref{eqn:check_condition_bias_1}) 
    via their Lemmas D.21-D.23.
\end{remark}

For infinite qualification and noninteger $u$, Lemma \ref{lemma_bounds_on_quantities} gives
\[
 \frac{\mathcal N_1\mathcal M_1^2/n^2}{\mathcal{W}_d}
 \lesssim d^{u-\gamma-a-(a+1)(s-1)}=o(1).
\]
At integer $u=a$, the additional term has ratio
\[
 d^{u-\gamma-as}=o(1).
\]
Thus we have
$$
    \frac{ \mathcal{N}_{1}(\lambda) \mathcal{M}_{1, \varphi}^2(\lambda)}{n^2}
    = 
    o_{d}\left( \mathcal{P}_{d}  \right)
$$

For finite qualification, the addition formula implies (see also Lemma~\ref{lem:popall})
\[
 \mathcal M_1^2\lesssim
 \lambda^{2\tau}d^{a\max(2\tau+1-s,0)}+d^{-(a+1)(s-1)}.
\]
If $s\le2\tau$, divide the head contribution to $\mathcal N_1\mathcal M_1^2/n^2$ by the head lower bound for $\mathcal M_2$ in Lemma~\ref{lem:popall}. The ratio is at most $C d^{u+a-2\gamma}=o(1)$. If $s>2\tau$, the head lower bound is $c\lambda^{2\tau}$, and $\max(2\tau+1-s,0)\le1$, giving the same upper bound on the ratio. The tail contribution is $o(W_d)$ as above. This verifies that
$$
    \frac{ \mathcal{N}_{1}(\lambda) \mathcal{M}_{1, \varphi}^2(\lambda)}{n^2}
    = 
    o_{d}\left( \mathcal{P}_{d}  \right)
$$

Finally, Lemma \ref{lemma_bias_special_condition_2} implies
$$
    \frac{\mathcal{N}_{1}(\lambda)}{n} \ln(n) (\ln\lambda^{-1})^2 \cdot \sum_{i=1}^\infty \frac{\lambda^2 \lambda_i \reg^2(\lambda_i)}{\lambda + \lambda_i} f_i^2  
    = 
    o_{d}\left( \mathcal{P}_{d}  \right),
$$
hence (\ref{eqn:check_condition_bias_1}) holds for any fixed $s \geq 1$.

\noindent {\bf Case (ii): $0< s <1$. } 
The proof is based on the following  argument, which is similar to the one in the case where $s \geq 1$:  
    \begin{itemize}
        \item []
        
        Suppose there exists a constant $\epsilon>0$ only depending on $u$, $s$, and $\gamma$, such that the following inequalities hold (which we will prove later)
\begin{equation}\label{eqn:check_condition_bias_3}
      \begin{aligned}
   \frac{\mathcal{N}_{1}(\lambda)}{n} \ln(n) (\ln\lambda^{-1})^2 \cdot \sum_{i=1}^\infty \frac{\lambda^2 \lambda_i \reg^2(\lambda_i)}{\lambda + \lambda_i} f_i^2  
    &= o_{d} 
    \left( \mathcal{P}_{d}  \right);\\
    n^{-2} \mathcal{N}_{1}(\lambda) 
    \left(
    \left\| f_{\lambda} \right\|_{L^{\infty}}^2  
    +
    n^{1-s+\epsilon}
    \right)
    &= o_{d}\left( \mathcal{P}_{d}  \right).
      \end{aligned}
  \end{equation}
  For the second line of \eqref{eqn:check_condition_bias_3}, choose $0<\epsilon_t<\epsilon/2$ and $t_n=n^{(1-s)/2+\epsilon_t}$. The integrability estimate \eqref{eq:p1_embedding} and the truncation argument in Appendix \ref{append_proof_of_previous_prop_B_1} then yield \eqref{eq:p1_bias_approximation}, which gives the required bounds for the squared bias.

    \end{itemize}

Therefore, we only need to establish (\ref{eqn:check_condition_bias_3}).

Lemma \ref{lemma_bias_special_condition_2} implies
$$
    \frac{\mathcal{N}_{1}(\lambda)}{n} \ln(n) (\ln\lambda^{-1})^2 \cdot \sum_{i=1}^\infty \frac{\lambda^2 \lambda_i \reg^2(\lambda_i)}{\lambda + \lambda_i} f_i^2  
    = 
    o_{d}\left( \mathcal{P}_{d}  \right).
$$

Lemma~\ref{lem:popall} and Lemma \ref{lem:flambda} imply
    \begin{align*}
    &~n^{-2} \mathcal{N}_{1}(\lambda) 
    \left(
    \left\| f_{\lambda} \right\|_{L^{\infty}}^2  
    +
    n^{1-s+\epsilon}
    \right)
        =
        O_{d}\left(
   d^{u-2\gamma}  \left[ d^{(1-s)a} + d^{2u-(1+s)(a+1)} \right] +  d^{u-\gamma (1+s-\epsilon)} \right)\\
   &~\mathcal{P}_{d} \geq \frac{\sigma^2}{n}\mathcal N_{2,\varphi}(\lambda) + \Omega_{d}(d^{-(a+1)s})= \Omega_{d}(d^{2u-a-1-\gamma} + d^{a-\gamma} + d^{-(a+1)s}).
    \end{align*}

Denote 
$\mathbf{I}= 
         u-2\gamma + (1-s)a$,
$\mathbf{II}= 
         3u-2\gamma -(1+s)(a+1)$,
$\mathbf{III}=u-\gamma (1+s-\epsilon)$.     

Denote $A_1 = 2u-a-1-\gamma$, $A_2 = a-\gamma$, $A_3 = -(a+1)s$.

\begin{itemize}
    \item We have
    $$
        \mathbf{I} - A_2 = u - \gamma -sa \overset{u<\gamma}{<} 0.
    $$

    \item We have
    $$
        \mathbf{II} - A_1 = u - \gamma - s(a+1) < 0.
    $$

    \item Write $\delta = u-a$. \\
    If $\delta < \gamma s$, take $\epsilon=(\gamma s - \delta) / (2\gamma)$. Then we have
    $$
        \mathbf{III} - A_2 = \delta -\gamma s + \epsilon \gamma <0.
    $$
    Otherwise, if $\delta \geq \gamma s$, we can show $a<\ell_{\gamma}$ by contradiction: suppose instead $a \geq \ell_{\gamma}$. Then $u < \gamma$ implies $a=\ell_{\gamma}$. Since $u<u'(1)$, the finite-qualification selector in Case III of Proposition~\ref{prop_e_1} implies $a\le p=\lfloor\gamma/(1+s)\rfloor\le \ell_{\gamma}$, so $p = a=\ell_{\gamma}$. In this case
\[
 \delta<u'(1)-p\le s/2<\gamma s\qquad(\gamma\ge1),
\]
contradicting $\delta\ge\gamma s$. Therefore $a+ 1 \leq  \ell_{\gamma} \leq \gamma$. Taking $\epsilon = (\gamma - u)/(2\gamma)$, we have
$$
\mathbf{III} - A_3 \leq  u-\gamma (1+s-\epsilon) + \gamma s <0.
$$
\end{itemize}

Combining all the above steps, we obtained (\ref{eqn:check_condition_bias_3}).    
\end{proof}

\section{Learning curve}\label{append_learning_curve}

In this section, we use Theorems \ref{thm_variance_spe} and \ref{thm_bias_spe}, Proposition \ref{prop_verify_bias_special}, and the operator approximation proved in Appendix \ref{append_proof_of_previous_prop_B_1} to determine the convergence rate of the excess risk in (13).

\subsection{Notations and restatement of results}

    Notice that, when $\ell_{\lambda} \geq \gamma = \ell_{\gamma}$, we have $d^{\tilde{\ell}-\gamma}=\Theta_{d}(1)$.
    Since
    $$
        \gamma - \tilde{\ell} - 1 - 2  \max\{\gamma - u, 0\}, \quad \tilde{\ell} - \gamma, \quad -2  \tau  u + (2  \tau - \tilde{s})  \tilde{\ell}, \quad -\tilde{\ell} s,
    $$
    are non-positive,  from (13), Theorem \ref{thm_variance_spe}, and Theorem \ref{thm_bias_spe}, we have
    \begin{equation*}
        \begin{aligned}
            &~\calR_{d}(u;X)\\
            = &~ \Theta_{d, \mathbb{P}} \left(1\right)\\
            = &~ \Theta_{d, \mathbb{P}} \left( 
d^{\gamma - \tilde{\ell} - 1 - 2  \max\{\gamma - u, 0\}}
+ d^{\tilde{\ell} - \gamma}
+ d^{-(\tilde{\ell} + 1)  s}
+ \mathbf{1}\{\tau<\infty\} d^{-2  \tau  u + (2  \tau - \tilde{s})  \tilde{\ell}}
\right)\\
&~ + \mathbf{1}\{\tau = \infty, u = \tilde{\ell}\}\Theta_{d, \mathbb{P}} \left( d^{-\tilde{\ell} s}\right).
        \end{aligned}
    \end{equation*}

Therefore, we only need to consider the case where we have $\gamma > \ell_{\gamma}$, or we have $\ell_{\lambda} < \gamma = \ell_{\gamma}$.

We first introduce some notations.

Recall that $u$ is an absolute constant given in Assumption 4, and we have $\lambda = \Theta_d\left(d^{-u}\right)$. Denote
\begin{equation}\label{eqn_order_in_bias_and_variance}
    \begin{aligned}
    &~v_1 = v_1(u) = \gamma - \tilde{\ell} - 1 - 2  \max\{\gamma - u, 0\},\\
    &~v_2 = v_2(u) =  \tilde{\ell} - \gamma,\\
    &~b_1 = b_1(u) = -(\tilde{\ell} + 1)  s,\\
    &~b_2 = b_2(u, t) = -2  t  u + (2  t - \min\{s, 2t\})  \tilde{\ell}, \quad 1 \leq t<\infty,\\
    &~v_1^{\prime} = v_1^{\prime}(u) = \gamma - \ell_{\lambda} - 1 - 2(\gamma - u),\\
    &~v_2^{\prime} = v_2^{\prime}(u) =  \ell_{\lambda} - \gamma,\\
    &~b_1^{\prime} = b_1^{\prime}(u) = -(\ell_{\lambda} + 1)  s,\\
    &~b_2^{\prime} = b_2^{\prime}(u, t) = -2  t  u + (2  t - \min\{s, 2t\})  \ell_{\lambda}, \quad 1 \leq t<\infty,\\
    &~b_3 = b_3(u) =  (2 - \min\{s, 2\})  \tilde{\ell} -2  \gamma,\\
    &~b_4 = b_4(u) = (2 - \min\{s, 2\})  \tilde{\ell} -\gamma - \tilde{\ell} - 1,\\
    &~b_5 = b_5(u) = -\tilde{\ell} s,\\
    &~b_5^{\prime} = b_5^{\prime}(u) = -\ell_{\lambda} s,\\
    &~ b_\tau=b_\tau(u)=\begin{cases}b_2(u,\tau),&\tau<\infty,\\-\infty,&\tau=\infty.\end{cases}.
    \end{aligned}
\end{equation}
Furthermore, when $s>0$,
let $u^{\prime} = u^{\prime}(\tau)$ be the regularization exponent specified in the proof of Proposition \ref{prop_e_1}, with qualification $\tau \leq \infty$ (ignoring logarithmic factors).

\begin{itemize}
    \item Let's first restate results in Theorem \ref{thm_variance_spe} and Theorem \ref{thm_bias_spe} with notations defined in (\ref{eqn_order_in_bias_and_variance}):
\begin{equation}\label{eqn_learning_curve_pre_results_1}
    \begin{aligned}
        \mathrm{var}(\hat{f}_{\lambda}) = &~ \Theta_{d, \mathbb{P}}\left( d^{v_1} + d^{v_2} \right)\\
        \mathrm{bias}^{2}(\hat{f}_{\lambda})  =  &~
    \|B_1+B_2\|_{2}^{2} + \|B_3-B_4-B_5\|_{2}^{2}\\
    \|B_1\|_{2}^{2}  =&~ 
                O_{d, \mathbb{P}}\left(d^{b_2(u, 1)}+d^{b_3}\right)\\
                \|B_2\|_{2}^{2} =&~ 
            O_{d, \mathbb{P}}\left(  d^{b_1}\right)\\
            \|B_3-B_5\|_{2}^{2} 
        =&~
        \Theta_{d, \mathbb{P}} \left( d^{b_1}\right)\\
        (\because \ell_{\gamma} + 1 > \gamma \text{ and } \ell_{\lambda}+1 > u) \quad \|B_4\|_{2}^{2} =&~ o_{d, \mathbb{P}} \left(d^{b_2(u, 1)}+d^{b_3}\right);
    \end{aligned}
\end{equation}
Moreover, if $0 < u < \gamma <1$, we have $\tilde{\ell}:=\min\{\ell_{\gamma}, \ell_{\lambda}\}=0$. 
Therefore, for any $\tau \leq \infty$, we have

\begin{equation}\label{eqn_learning_curve_pre_results_2}
    \begin{aligned}
    \|B_1\|_{2}^{2}  \leq&~ 
                O_{d, \mathbb{P}}\left( d^{b_2(u, \tau^{\prime})}  \right) + \Delta_1\\
                \|B_1\|_{2}^{2}  \geq &~ 
                \Omega_{d, \mathbb{P}}\left( d^{b_2(u, \tau)} \right)- \Delta_1, \quad \text{ if } \tau<\infty\\
    \|B_4\|_{2}^{2} =&~ O_{d, \mathbb{P}} \left(d^{-\gamma-1}\right) = O_{d, \mathbb{P}} \left(d^{b_4}\right),           
    \end{aligned}
\end{equation}
where $\tau^{\prime}$ is defined as in (\ref{eqn_def_of_tau_prime}), and $0 \leq \Delta_1 = O_{d, \mathbb{P}}\left( d^{b_3}\right)$.

\item 
Suppose Assumptions 1--5 hold with $\tau \leq \infty$, $s>0$, and $0<u<u^{\prime}(\tau)$. Suppose also that at least one of the following holds: (i) $\tau=\infty$, (ii) $s>1/(2\tau)$, or (iii) $\gamma>(2\tau+1)s/[2\tau(1+s)]$. Then the operator argument in Appendix \ref{append_proof_of_previous_prop_B_1} gives
\begin{equation}\label{eqn_learning_curve_pre_results_3}
    \begin{aligned}
&~\calR_{d}(u;X)\\
=&~
\Theta_{d, \mathbb{P}} \left( d^{v_1^{\prime}} + d^{v_2^{\prime}} + d^{b_1^{\prime}} + \mathbf{1}\{\tau<\infty\}d^{b_2^{\prime}(u, \tau)}+ \mathbf{1}\{\tau = \infty, u = \tilde{\ell}\}d^{b_{5}^{\prime}}\right)\\
=&~
\Theta_{d, \mathbb{P}} \left( d^{v_1} + d^{v_2} + d^{b_1} + \mathbf{1}\{\tau<\infty\}d^{b_2(u, \tau)}+ \mathbf{1}\{\tau = \infty, u = \tilde{\ell}\}d^{b_{5}}\right),
    \end{aligned}
\end{equation}
where the last line follows from Proposition \ref{prop_e_1};

\item Finally, let's restate results in Proposition \ref{prop_verify_bias_special} with notations defined in (\ref{eqn_order_in_bias_and_variance}). Suppose all assumptions and conditions in Proposition \ref{prop_verify_bias_special} hold; this implies that $\tau \leq \infty$, $s>0$, $u^{\prime}(\tau) \leq u < u^{\prime}(1)$, and $\gamma \geq 1$.
Then for any $u^{\prime}(\tau) \leq u < u^{\prime}(1)$, we have

\begin{equation}\label{eqn_learning_curve_pre_results_4}
    \begin{aligned}
        \mathrm{bias}^{2}(\hat{f}_{\lambda}) \leq &~ \Theta_{d} \left( d^{b_1} + \mathbf{1}\{\tau<\infty\}d^{b_2(u, \tau)} + \mathbf{1}\{\tau = \infty, u = \tilde{\ell}\}d^{b_{5}}\right) + \Delta_2\\
        \mathrm{bias}^{2}(\hat{f}_{\lambda}) \geq &~ \Theta_{d} \left( d^{b_1} + \mathbf{1}\{\tau<\infty\}d^{b_2(u, \tau)} + \mathbf{1}\{\tau = \infty, u = \tilde{\ell}\}d^{b_{5}}\right) - \Delta_2;
    \end{aligned}
\end{equation}
where $0 \leq \Delta_2 = o_{d, \mathbb{P}} \left( d^{v_1} + d^{v_2} + d^{b_1} + \mathbf{1}\{\tau<\infty\}d^{b_2(u, \tau)} + \mathbf{1}\{\tau = \infty, u = \tilde{\ell}\}d^{b_{5}}\right)$.

\end{itemize}

In the next two subsections, we will consider the case with $s=0$ in Appendix \ref{appendix_f2} and with $s>0$ in Appendix \ref{appendix_f3}.

\subsection{$s=0$}\label{appendix_f2}

When $s=0$, Theorem \ref{thm_variance_spe} and Theorem \ref{thm_bias_spe} together imply the following results:

\begin{corollary}\label{coroll_learn_curve_case_1}
    Suppose Assumptions 1,2,3,4, and 5 hold with $\tau \leq \infty$ and $s=0$. Then, the excess risk of the spectral algorithm estimator in (13) satisfies
\begin{equation*}
    \calR_{d}(u;X) = \Theta_{d, \mathbb{P}} \left( 1 \right).
\end{equation*}
\end{corollary}
\begin{proof}
    Notice that $b_1=0$ and these four quantities are non-positive:  $v_1, v_2,  b_2(u, 1), b_3$. Hence (\ref{eqn_learning_curve_pre_results_1}) implies that 
    $$
    \calR_{d}(u;X) = O_{d, \mathbb{P}} \left( d^{v_1} + d^{v_2} \right) + O_{d, \mathbb{P}} \left( d^{b_1} + d^{b_2(u, 1)} + d^{b_3} \right) = O_{d, \mathbb{P}} \left( 1 \right),
    $$
    and that
    $$
    \calR_{d}(u;X) \geq 
    \|B_3-B_4-B_5\|_{2}^{2} \geq \frac{1}{2}\|B_3-B_5\|_{2}^{2} - \|B_4\|_{2}^{2}
    = \Omega_{d, \mathbb{P}} \left( 1 \right),
    $$
    and we obtain the desired results.
\end{proof}

\subsection{$s>0$}\label{appendix_f3}

We will prove the following results.

\begin{corollary}\label{coroll_learn_curve_case_2}
Suppose Assumptions 1,2,3,4, and 5 hold with $s>0$. 
Denote $\tilde{s}=\min\{s, 2\tau\}$.
Assume in addition that either $\tau=\infty$, or $s>1/(2\tau)$, or $\gamma > ((2\tau+1)s) / (2\tau(1+s))$.
Then the excess risk of the spectral algorithm estimator in (13) satisfies
\begin{equation*}
\begin{aligned}
    &~\calR_{d}(u;X)\\
    =&~ \Theta_{d, \mathbb{P}} \left( \frac{n}{d^{\tilde{\ell}+1}(n\lambda+1)^{2}} + \frac{d^{\tilde{\ell}}}{n} + 
    d^{-(\tilde{\ell}+1)s}
    +\mathbf{1}\{\tau<\infty\}\lambda^{2\tau}d^{(2\tau-\tilde{s})\tilde{\ell} }\right)\\
    &~ + \mathbf{1}\{\tau = \infty, u = \tilde{\ell}<\infty\}\Theta_{d, \mathbb{P}} \left( d^{-\tilde{\ell} s}\right).
\end{aligned}
\end{equation*}
\end{corollary}

\begin{corollary}\label{coroll_learn_curve_case_2_complex}
    Suppose Assumptions 1,2,3,4, and 5 hold with $s>0$. 
Denote $\tilde{s}=\min\{s, 2\tau\}$.
Assume in addition that $\tau<\infty$, $s \leq 1/(2\tau)$, and $\gamma \leq ((2\tau+1)s) / (2\tau(1+s))$.
Then, the excess risk of the spectral algorithm estimator in (13) satisfies
\begin{equation*}
\begin{aligned}
    &~\calR_{d}(u;X)\\
=&~ \Theta_{d, \mathbb{P}} \left( \frac{n}{d^{\tilde{\ell}+1}(n\lambda+1)^{2}} + \frac{d^{\tilde{\ell}}}{n} + 
    d^{-(\tilde{\ell}+1)s}
    +\lambda^{2\tau}d^{(2\tau-\tilde{s})\tilde{\ell} }\right).
\end{aligned}
\end{equation*}
\end{corollary}

\begin{proof}[Proof of Corollary \ref{coroll_learn_curve_case_2}]
Let's separate the proof into four parts:
\begin{itemize}
    \item [(i)] $u < u^{\prime}(\tau)$; 
    \item [(ii)] $u^{\prime}(\tau) \geq u^{\prime}(1)$ and $u \geq u^{\prime}(\tau)$;
    \item [(iii)] $u^{\prime}(\tau) < u^{\prime}(1)$ and $u \geq u^{\prime}(1)$; 
    \item [(iv)] $u^{\prime}(\tau) < u^{\prime}(1)$ and $u \in [u^{\prime}(\tau), u^{\prime}(1))$.
\end{itemize}

\noindent {\bf Case (i). } Suppose $u < u^{\prime}(\tau)$. (\ref{eqn_learning_curve_pre_results_3}) implies
\begin{align*}
&~\calR_{d}(u;X)\\
=&~
\Theta_{d, \mathbb{P}} \left( d^{v_1} + d^{v_2} + d^{b_1} + \mathbf{1}\{\tau<\infty\}d^{b_2(u, \tau)}+ \mathbf{1}\{\tau = \infty, u = \tilde{\ell}< u^{\prime}(\tau)\}d^{b_{5}}\right).
\end{align*}

\noindent {\bf Case (ii). } Suppose $u^{\prime}(\tau) \geq u^{\prime}(1)$. For any $u \geq u^{\prime}(\tau) \geq u^{\prime}(1)$, from Proposition \ref{prop_e_2} and Proposition \ref{prop_e_3}, we have
\begin{equation*}
    \max\{v_1(u), v_2(u), b_1(u)\} \geq \max\{b_\tau(u), b_2(u, 1), b_3(u)\},
\end{equation*}
Therefore from (\ref{eqn_learning_curve_pre_results_1}) we have
\begin{itemize}
    \item Upper bound:
    \begin{align*}
&~\calR_{d}(u;X)\\
=&~
O_{d, \mathbb{P}} \left( d^{v_1} + d^{v_2} + d^{b_1}\right)\\
=&~
O_{d, \mathbb{P}} \left( d^{v_1} + d^{v_2} + d^{b_1} + \mathbf{1}\{\tau<\infty\}d^{b_2(u, \tau)}\right);
\end{align*}

\item Lower bound:
\begin{align*}
&~ \calR_{d}(u;X)
\geq  \Omega_{d, \mathbb{P}} \left( d^{v_1} + d^{v_2} \right) + \|B_3-B_4-B_5\|_{2}^{2}\\
\geq &~ \Omega_{d, \mathbb{P}} \left( d^{v_1} + d^{v_2} \right) + \frac{1}{2}\|B_3-B_5\|_{2}^{2} - \|B_4\|_{2}^{2}\\
= &~
\Omega_{d, \mathbb{P}} \left( d^{v_1} + d^{v_2} + d^{b_1}\right)
=
\Omega_{d, \mathbb{P}} \left( d^{v_1} + d^{v_2} + d^{b_1} + \mathbf{1}\{\tau<\infty\}d^{b_2(u, \tau)}\right).
\end{align*}

\end{itemize}

\noindent {\bf Case (iii). } Suppose $u^{\prime}(\tau) < u^{\prime}(1)$. For any $u \geq u^{\prime}(1) \geq u^{\prime}(\tau)$, from Proposition \ref{prop_e_2} and Proposition \ref{prop_e_3}, we have
\begin{equation*}
    \max\{v_1(u), v_2(u), b_1(u)\} \geq \max\{b_\tau(u), b_2(u, 1), b_3(u)\},
\end{equation*}
Therefore, as in {\bf Case (ii)}, we can show that
\begin{align*}
\calR_{d}(u;X)
=
\Theta_{d, \mathbb{P}} \left( d^{v_1} + d^{v_2} + d^{b_1}\right)
=
\Theta_{d, \mathbb{P}} \left( d^{v_1} + d^{v_2} + d^{b_1} + \mathbf{1}\{\tau<\infty\}d^{b_2(u, \tau)}\right).
\end{align*}

\noindent {\bf Case (iv). } Suppose $u^{\prime}(\tau) < u^{\prime}(1)$ and $u \in [u^{\prime}(\tau), u^{\prime}(1))$. From Proposition \ref{prop_e_2} and Proposition \ref{prop_e_3}, we have
\begin{equation}
\label{eqn_tau_finite_b234_bound}
    \max\{v_1(u), v_2(u), b_1(u)\} \geq \max\{b_\tau(u), b_3(u), b_4(u)\}.
\end{equation}
Therefore, 
\begin{itemize}
    \item[(a)] When $\gamma \geq 1$, from the variance bound in (\ref{eqn_learning_curve_pre_results_1}) and the bias bound in (\ref{eqn_learning_curve_pre_results_4}), we have
    \begin{align*}
&~\calR_{d}(u;X)\\
= &~
\Theta_{d, \mathbb{P}} \left( d^{v_1} + d^{v_2} + d^{b_1} + \mathbf{1}\{\tau<\infty\}d^{b_2(u, \tau)} + \mathbf{1}\{\tau = \infty, u = \tilde{\ell}<\infty\}d^{b_{5}}\right)\\
= &~
\Theta_{d, \mathbb{P}} \left( d^{v_1} + d^{v_2} + d^{b_1} + \mathbf{1}\{\tau<\infty\}d^{b_2(u, \tau)} + \mathbf{1}\{\tau = \infty, u^{\prime}(\tau) \leq u = \tilde{\ell}<u^{\prime}(1), \gamma \geq 1\}d^{b_{5}}\right)\\
= &~
\Theta_{d, \mathbb{P}} \left( d^{v_1} + d^{v_2} + d^{b_1} + \mathbf{1}\{\tau<\infty\}d^{b_2(u, \tau)} + \mathbf{1}\{\tau = \infty, u^{\prime}(\tau) \leq u = \tilde{\ell}<u^{\prime}(1)\}d^{b_{5}}\right),
\end{align*}
where the last equality holds since we have $\tilde{\ell}=0 \neq u$ for any $\gamma < 1$; 

    \item[(b)] When $\gamma<1$ and $\tau<\infty$, Proposition \ref{prop_e_1} with $t=1$ further implies $0<u<u^{\prime}(1) <\gamma<1$. Hence from (\ref{eqn_learning_curve_pre_results_1}) and (\ref{eqn_learning_curve_pre_results_2}) we have
    \begin{itemize}
    \item Upper bound:
    \begin{align*}
\calR_{d}(u;X)
\leq &~ O_{d, \mathbb{P}} \left( d^{v_1} + d^{v_2}\right) + 2\|B_1\|_{2}^{2} + 2\|B_2\|_{2}^{2} + 2\|B_3 - B_5\|_{2}^{2} + 2\|B_4\|_{2}^{2}\\
= &~
O_{d, \mathbb{P}} \left( d^{v_1} + d^{v_2}\right) + O_{d, \mathbb{P}}\left( d^{b_2(u, \tau)}  + d^{b_3} + d^{b_1} + d^{b_4} \right)\\
(\because \text{ (\ref{eqn_tau_finite_b234_bound})}) \quad = &~ O_{d, \mathbb{P}} \left( d^{v_1} + d^{v_2} + d^{b_1} + d^{b_2(u, \tau)}\right);
\end{align*}

\item Lower bound:
    \begin{itemize}
        \item If $b_1>b_4$, then
\begin{align*}
&~\calR_{d}(u;X)
\geq  \Omega_{d, \mathbb{P}} \left( d^{v_1} + d^{v_2} \right) + \|B_3-B_4-B_5\|_{2}^{2}\\
=&~
\Omega_{d, \mathbb{P}} \left( d^{v_1} + d^{v_2} + d^{b_1}\right)=\Omega_{d, \mathbb{P}} \left( d^{v_1} + d^{v_2} + d^{b_1} + d^{b_2(u, \tau)}\right);
\end{align*}

\item If $b_1 \leq b_4$, since Proposition \ref{prop_e_2} implies that $b_4<v_2$, (\ref{eqn_tau_finite_b234_bound}) implies
\begin{equation*}
    \max\{v_1(u), v_2(u)\} \geq \max\{b_1(u), b_2(u, \tau)\}.
\end{equation*}
Therefore, 
\begin{align*}
\calR_{d}(u;X)
\geq
\Omega_{d, \mathbb{P}} \left( d^{v_1} + d^{v_2}\right)
=
\Omega_{d, \mathbb{P}} \left( d^{v_1} + d^{v_2} + d^{b_1} + d^{b_2(u, \tau)}\right),
\end{align*}

    \end{itemize}

\end{itemize}

    \item[(c)] Finally, suppose $\gamma<1$ and $\tau = \infty$. Recall that from (\ref{eqn_def_of_tau_prime}), we have
    $$
        \tau^{\prime}:= 
        \left\{\begin{matrix}
s, &  s \geq 1;\\
1, &  s \in [\frac{1}{2}, 1);\\
(2s)^{-1}, &  0 < s < \frac{1}{2}.
\end{matrix}\right.
        $$

 Therefore, from Proposition \ref{prop_e_1} we have
    \begin{align*}
    u^{\prime}(\infty) =
    \left\{\begin{matrix}
    u^{\prime}(s), & \quad s \geq 1;\\
    \vspace{5pt}
\gamma/2 = u^{\prime}(1) & \quad s \in [\frac{1}{2}, 1) \text{ and } \gamma < s;\\
\vspace{5pt}
\min\{\gamma (1+s)-s, \gamma / 2 \} \geq s/2 = u^{\prime}(1) & \quad s \in [\frac{1}{2}, 1) \text{ and } s \leq \gamma < 1;\\
\vspace{5pt}
\gamma s = \frac{\gamma}{2(2s)^{-1}} = u^{\prime}((2s)^{-1}) & \quad 0 < s < \frac{1}{2} \text{ and } \gamma < s;\\
\vspace{5pt}
\min\{\gamma (1+s)-s, \gamma / 2 \} \geq s^2 = \frac{s}{2(2s)^{-1}} = u^{\prime}((2s)^{-1}) & \quad 0 < s < \frac{1}{2} \text{ and } s \leq \gamma < 1;
\end{matrix}\right.
\end{align*}
implying that $u^{\prime}(\infty) \geq u^{\prime}(\tau^{\prime})$. 
Furthermore, the assumption of {\bf Case (iv)} implies $u \geq u^{\prime}(\infty)\geq u^{\prime}(\tau^{\prime})$. 
Therefore, from Proposition \ref{prop_e_2} and Proposition \ref{prop_e_3}, we have
\begin{equation*}
    \max\{v_1(u), v_2(u), b_1(u)\} \geq \max\{b_2(u, \tau^{\prime}), b_3(u), b_4(u)\}.
\end{equation*}
From the assumption of {\bf Case (iv)}, we have
$u<u^{\prime}(1)$. From Proposition \ref{prop_e_1} we have $0<u<u^{\prime}(1) <\gamma<1$.

    Hence from (\ref{eqn_learning_curve_pre_results_1}) and (\ref{eqn_learning_curve_pre_results_2}) we have
    \begin{itemize}
    \item Upper bound:
    \begin{align*}
\calR_{d}(u;X)
\leq &~ O_{d, \mathbb{P}} \left( d^{v_1} + d^{v_2}\right) + 2\|B_1\|_{2}^{2} + 2\|B_2\|_{2}^{2} + 2\|B_3 - B_5\|_{2}^{2} + 2\|B_4\|_{2}^{2}\\
= &~
O_{d, \mathbb{P}} \left( d^{v_1} + d^{v_2}\right) + O_{d, \mathbb{P}}\left( d^{b_2(u, \tau^{\prime})}  + d^{b_3} + d^{b_1} + d^{b_4} \right)\\
= &~ O_{d, \mathbb{P}} \left( d^{v_1} + d^{v_2} + d^{b_1}\right);
\end{align*}

\item Lower bound:
    \begin{itemize}
        \item If $b_1>b_4$, then
\begin{align*}
&~\calR_{d}(u;X)
\geq  \Omega_{d, \mathbb{P}} \left( d^{v_1} + d^{v_2} \right) + \|B_3-B_4-B_5\|_{2}^{2}\\
=&~
\Omega_{d, \mathbb{P}} \left( d^{v_1} + d^{v_2} + d^{b_1}\right);
\end{align*}

\item If $b_1 \leq b_4$, since Proposition \ref{prop_e_2} implies that $b_1 \leq b_4<v_2$, we have
\begin{align*}
\calR_{d}(u;X)
\geq
\Omega_{d, \mathbb{P}} \left( d^{v_1} + d^{v_2}\right)
=
\Omega_{d, \mathbb{P}} \left( d^{v_1} + d^{v_2} + d^{b_1}\right).
\end{align*}

    \end{itemize}

\end{itemize}

\end{itemize}

Combining all the above cases, we obtain
\begin{equation*}
\begin{aligned}
    &~\calR_{d}(u;X)\\
=&~ \Theta_{d, \mathbb{P}} \left( \frac{n}{d^{\tilde{\ell}+1}(n\lambda+1)^{2}} + \frac{d^{\tilde{\ell}}}{n} + 
    d^{-(\tilde{\ell}+1)s}
    +\mathbf{1}\{\tau<\infty\}\lambda^{2\tau}d^{(2\tau-\tilde{s})\tilde{\ell} }\right)\\
    &~ + \mathbf{1}\{\tau = \infty\}\left(\mathbf{1}\{u = \tilde{\ell}< u^{\prime}(\tau)\} + \mathbf{1}\{u^{\prime}(\tau) \leq u = \tilde{\ell}<u^{\prime}(1)\}\right) \Theta_{d, \mathbb{P}} \left( d^{-\tilde{\ell} s}\right)\\
    =&~ \Theta_{d, \mathbb{P}} \left( \frac{n}{d^{\tilde{\ell}+1}(n\lambda+1)^{2}} + \frac{d^{\tilde{\ell}}}{n} + 
    d^{-(\tilde{\ell}+1)s}
    +\mathbf{1}\{\tau<\infty\}\lambda^{2\tau}d^{(2\tau-\tilde{s})\tilde{\ell} }\right)\\
    &~ + \mathbf{1}\{\tau = \infty, u = \tilde{\ell}< \max\{u^{\prime}(1), u^{\prime}(\tau)\}\}\Theta_{d, \mathbb{P}} \left( d^{-\tilde{\ell} s}\right).
\end{aligned}
\end{equation*}
Finally, notice that when $\tau = \infty, u = \tilde{\ell} \geq \max\{u^{\prime}(1), u^{\prime}(\tau)\}$, we have
$$
    -\tilde{\ell} s \leq -\min\{s, 2\}\tilde{\ell} = b_2(u, 1) \overset{\text{Proposition \ref{prop_e_3}}}{\leq} \max\{v_1(u), v_2(u), b_1(u)\}.
$$
Therefore, 
\begin{equation*}
\begin{aligned}
    &~\calR_{d}(u;X)\\
    =&~ \Theta_{d, \mathbb{P}} \left( \frac{n}{d^{\tilde{\ell}+1}(n\lambda+1)^{2}} + \frac{d^{\tilde{\ell}}}{n} + 
    d^{-(\tilde{\ell}+1)s}
    +\mathbf{1}\{\tau<\infty\}\lambda^{2\tau}d^{(2\tau-\tilde{s})\tilde{\ell} }\right)\\
    &~ + \mathbf{1}\{\tau = \infty, u = \tilde{\ell}< \max\{u^{\prime}(1), u^{\prime}(\tau)\}\}\Theta_{d, \mathbb{P}} \left( d^{-\tilde{\ell} s}\right)\\
    &~ + \mathbf{1}\{\tau = \infty, u = \tilde{\ell}\geq \max\{u^{\prime}(1), u^{\prime}(\tau)\}\}\Theta_{d, \mathbb{P}} \left( d^{-\tilde{\ell} s}\right)\\
    =&~ \Theta_{d, \mathbb{P}} \left( \frac{n}{d^{\tilde{\ell}+1}(n\lambda+1)^{2}} + \frac{d^{\tilde{\ell}}}{n} + 
    d^{-(\tilde{\ell}+1)s}
    +\mathbf{1}\{\tau<\infty\}\lambda^{2\tau}d^{(2\tau-\tilde{s})\tilde{\ell} }\right)\\
    &~ + \mathbf{1}\{\tau = \infty, u = \tilde{\ell}<\infty\}\Theta_{d, \mathbb{P}} \left( d^{-\tilde{\ell} s}\right).
\end{aligned}
\end{equation*}
\end{proof}


\begin{proof}[Proof of Corollary \ref{coroll_learn_curve_case_2_complex}]
Let's separate the proof into two parts:
\begin{itemize}
    \item [(i)] $u < \gamma$; 
    \item [(ii)] $u \geq \gamma$.
\end{itemize}

\noindent {\bf Case (i). } Suppose $u < \gamma$. Since $s \leq 1/(2\tau)$ and
$$
\gamma \leq \frac{2\tau + 1}{2\tau} \frac{s}{s+1} \leq \frac{2\tau + 1}{2\tau} \frac{\frac{1}{2\tau}}{\frac{1}{2\tau}+1} < 1,
$$
we have $0 < u < \gamma <1$. Therefore, from (\ref{eqn_learning_curve_pre_results_1}) and (\ref{eqn_learning_curve_pre_results_2})
we have 
\begin{itemize}
    \item Upper bound:
    \begin{align*}
\calR_{d}(u;X)
\leq &~ O_{d, \mathbb{P}} \left( d^{v_1} + d^{v_2}\right) + 2\|B_1\|_{2}^{2} + 2\|B_2\|_{2}^{2} + 2\|B_3 - B_5\|_{2}^{2} + 2\|B_4\|_{2}^{2}\\
= &~
O_{d, \mathbb{P}} \left( d^{v_1} + d^{v_2}\right) + O_{d, \mathbb{P}}\left( d^{b_2(u, \tau)}  + d^{b_3} + d^{b_1} + d^{b_4} \right)\\
(\because \text{ Proposition \ref{prop_e_2}}) \quad = &~ O_{d, \mathbb{P}} \left( d^{v_1} + d^{v_2} + d^{b_1} + d^{b_2(u, \tau)}\right);
\end{align*}

\item Lower bound:
    \begin{itemize}
        \item If $b_1>b_4$ and $b_1 \geq b_2(u, \tau)$, then
\begin{align*}
&~ \calR_{d}(u;X)
\geq  \Omega_{d, \mathbb{P}} \left( d^{v_1} + d^{v_2} \right) + \|B_3-B_4-B_5\|_{2}^{2}\\
\overset{b_1 > b_4}{=} &~ \Omega_{d, \mathbb{P}} \left( d^{v_1} + d^{v_2} + d^{b_1}\right)
=
\Omega_{d, \mathbb{P}} \left( d^{v_1} + d^{v_2} + d^{b_1} + d^{b_2(u, \tau)}\right);
\end{align*}

\item If $b_1 \leq b_4$ and $b_1 \geq b_2(u, \tau)$, since Proposition \ref{prop_e_2} implies that $b_4<v_2$, we have\\
$\max\{b_1, b_2(u, \tau)\}<v_2$, hence
\begin{align*}
\calR_{d}(u;X)
\geq
\Omega_{d, \mathbb{P}} \left( d^{v_1} + d^{v_2}\right)
=
\Omega_{d, \mathbb{P}} \left( d^{v_1} + d^{v_2} + d^{b_1} + d^{b_2(u, \tau)}\right),
\end{align*}

\item If $b_2(u, \tau)>b_3$ and $b_1 < b_2(u, \tau)$, then
\begin{align*}
&~ \calR_{d}(u;X)
\geq   \Omega_{d, \mathbb{P}} \left( d^{v_1} + d^{v_2} \right) + \|B_1+B_2\|_{2}^{2}\\
\overset{b_2(u, \tau)>\max\{b_1, b_3\}}{=} &~
\Omega_{d, \mathbb{P}} \left( d^{v_1} + d^{v_2} + d^{b_2}(u, \tau)\right)
=
\Omega_{d, \mathbb{P}} \left( d^{v_1} + d^{v_2} + d^{b_1} + d^{b_2(u, \tau)}\right),
\end{align*}

\item If $b_2(u, \tau) \leq b_3$ and $b_1 < b_2(u, \tau)$, since Proposition \ref{prop_e_2} implies that $b_3 \leq v_2$, we have $\max\{b_1, b_2(u, \tau)\} \leq v_2$, hence
\begin{align*}
\calR_{d}(u;X)
\geq  \Omega_{d, \mathbb{P}} \left( d^{v_1} + d^{v_2} \right) 
=
\Omega_{d, \mathbb{P}} \left( d^{v_1} + d^{v_2} + d^{b_1} + d^{b_2(u, \tau)}\right).
\end{align*}
    \end{itemize}

\end{itemize}

\noindent {\bf Case (ii). } Suppose $u \geq \gamma$. From Proposition \ref{prop_e_1}  with $t=1$ and $t=\tau$, we have $\gamma > u^{\prime}(1)$ and $\gamma > u^{\prime}(\tau)$, hence from Proposition \ref{prop_e_2} and Proposition \ref{prop_e_3}, we have
\begin{equation*}
    \max\{v_1(u), v_2(u), b_1(u)\} \geq \max\{b_2(u, 1), b_2(u, \tau), b_3(u)\},
\end{equation*}
Therefore, from (\ref{eqn_learning_curve_pre_results_1}) we have
\begin{itemize}
    \item Upper bound:
    \begin{align*}
\calR_{d}(u;X)
=
O_{d, \mathbb{P}} \left( d^{v_1} + d^{v_2} + d^{b_1}\right)
=
O_{d, \mathbb{P}} \left( d^{v_1} + d^{v_2} + d^{b_1} + d^{b_2(u, \tau)}\right);
\end{align*}

\item Lower bound:
\begin{align*}
&~ \calR_{d}(u;X)
\geq  \Omega_{d, \mathbb{P}} \left( d^{v_1} + d^{v_2} \right) + \|B_3-B_4-B_5\|_{2}^{2}\\
=&~
\Omega_{d, \mathbb{P}} \left( d^{v_1} + d^{v_2} + d^{b_1}\right)
=
\Omega_{d, \mathbb{P}} \left( d^{v_1} + d^{v_2} + d^{b_1} + d^{b_2(u, \tau)}\right).
\end{align*}

\end{itemize}
Putting the two bounds together yields the desired result. 
\end{proof}

\subsection[Calculations of quantities for the learning curve results]{Calculations of quantities for the learning curve results}


\begin{proposition}\label{prop_e_1}
    Let $u^{\prime}=u^{\prime}(t)$ be the exponent specified by the four cases in the proof, for qualification $t\in[1,\infty]$ and source parameter $s>0$. These formulas agree with \cite{lu2024saturation} on its stated range; for $s>2t$, Case II uses $\tilde{s}=2t$, consistent with the population squared-bias exponent. Then, we have $u^{\prime} < \gamma$ and $\tilde{p} \leq u^{\prime} < \tilde{p}+1$, where $\tilde{s} := \min\{s, 2t\}$ and $\tilde{p}:=\lfloor \gamma/(\tilde{s}+1)\rfloor$.\\
    As a direct corollary, for any $u \leq u^{\prime}$, we have $\ell_{\lambda} = \lfloor u \rfloor \leq \ell_{\gamma}$, hence
$$
v_1^{\prime} = v_1, ~ v_2^{\prime} = v_2, ~
b_1^{\prime} = b_1, ~ b_2^{\prime}(u, t) = b_2(u, t), ~
b_5^{\prime} = b_5.
$$
\end{proposition}
\begin{proof}
According to Corollaries D.15--17 in \cite{lu2024saturation}, we divide the proof into four cases.

\noindent {\bf Case I: $1 \leq s \leq t$}\\
\noindent Denote $p:=\lfloor \gamma/(s+1)\rfloor$.\\
   (1.1) Suppose $p \geq 1$ and $p(s+1) \leq \gamma < ps+p+s$. We have
   $$
   u^{\prime} = p+\frac{1}{2} < p+ps \leq \gamma \quad \text{ and } \quad p \leq u^{\prime} < p+1.
   $$
   (1.2) Suppose $p \geq 1$ and $ps+p+s \leq \gamma < ps+p+s+1$. We have
   $$
   u^{\prime} = \frac{\gamma - (p+1)(s-1)}{2} < p+1 < \gamma \quad \text{ and } \quad p \leq u^{\prime} < p+1.
   $$
   (1.3) Suppose $\gamma < s$. We have
   $$
   u^{\prime} = \min\{\gamma, 1\}/2 < \gamma \quad \text{ and } \quad 0 \leq u^{\prime} < 1.
   $$
   (1.4) Suppose $s \leq \gamma < s+1$. We have
   $$
   u^{\prime} = \frac{\gamma - (s-1)}{2} < \gamma \quad \text{ and } \quad 0 \leq u^{\prime} < 1.
   $$

\noindent {\bf Case II: $\tilde{s} > t$}\\
\noindent Denote $\tilde{p}:=\lfloor \gamma/(\tilde{s}+1)\rfloor$.\\
   (2.1) Suppose $\gamma \geq 1$ and $0 \leq \gamma - \tilde{p}(\tilde{s} +1) < t$. We have
   $$
   u^{\prime} = \tilde{p} + \frac{\gamma - \tilde{p}(\tilde{s} +1)}{2t} < \tilde{p}+\frac{1}{2} < \frac{\gamma + 1}{2} \leq \gamma \quad \text{ and } \quad \tilde{p} \leq u^{\prime} < \tilde{p}+1.
   $$  
   (2.2) Suppose $\gamma \geq 1$ and $t \leq \gamma - \tilde{p}(\tilde{s} +1) < \tilde{s} + \tilde{s} /t - 1$. We have
   $$
   u^{\prime} = \tilde{p} + \frac{\gamma - \tilde{p}(\tilde{s} +1)+1}{2(t+1)} \leq 
   \left\{\begin{matrix}
\tilde{p}+\frac{\tilde{s} }{2t} < t +\tilde{p} + \tilde{p}\tilde{s}  \leq \gamma, & \tilde{p} \geq 1  \\
\frac{\gamma}{t + 1} < \gamma, & \tilde{p}=0,
\end{matrix}\right.
   $$  
   and $\tilde{p} \leq u^{\prime} < \tilde{p}+1$.\\
   (2.3) Suppose $\gamma \geq 1$ and $\tilde{s} + \tilde{s} /t - 1 \leq \gamma - \tilde{p}(\tilde{s} +1) < \tilde{s} +1$. We have
   $$
   u^{\prime} = \tilde{p} + \frac{\gamma - \tilde{p}(\tilde{s} +1)+1-\tilde{s} }{2} < \tilde{p}+1<\gamma \quad \text{ and } \quad \tilde{p} \leq u^{\prime} < \tilde{p}+1.
   $$ 
   (2.4) Suppose $\gamma < 1$. We have
   $$
   u^{\prime} = \frac{\gamma}{2} < \gamma \quad \text{ and } \quad 0 \leq u^{\prime} < 1.
   $$

\noindent {\bf Case III: $t < \infty$ and $s<1$}\\
\noindent Denote $p:=\lfloor \gamma/(s+1)\rfloor$.\\
(3.1) Suppose $p(s+1) \leq \gamma < ps+p+s$. We have
   $$
   u^{\prime} = \frac{\gamma +2t p - sp-p}{2t} 
   \leq 
   \left\{\begin{matrix}
p+\frac{s}{2t} < p(s+1) \leq \gamma, & p \geq 1  \\
\frac{\gamma}{2t} < \gamma, & p=0,
\end{matrix}\right.
   $$ 
   and $p \leq u^{\prime} < p+1$.\\
   (3.2) Suppose $ps+p+s \leq \gamma < ps+p+s+1$. We have
   $$
   u^{\prime} = p + \frac{s}{2t} < ps+p+s \leq \gamma \quad \text{ and } \quad p \leq u^{\prime} < p+1.
   $$

\noindent {\bf Case IV: $t = \infty$ and $s<1$}\\
\noindent Denote $p:=\lfloor \gamma/(s+1)\rfloor$.\\
   (4.1) Suppose $p \geq 1$ and $p(s+1) \leq \gamma < ps+p+s$. We have
   $$
   u^{\prime} = p + \frac{s}{2} < p(s+1) \leq \gamma \quad \text{ and } \quad p \leq u^{\prime} < p+1.
   $$  
   (4.2) Suppose $p \geq 1$ and $ps+p+s \leq \gamma < ps+p+s+1$. We have
   $$
   u^{\prime} = \frac{\gamma + p(1-s)}{2} < \gamma \quad \text{ and } \quad p \leq u^{\prime} < p+1.
   $$  
   (4.3) Suppose $\gamma < s$. We have
   $$
   u^{\prime} = \gamma \cdot \min\left\{s, \frac{1}{2}\right\} < \gamma \quad \text{ and } \quad 0 \leq u^{\prime} < 1.
   $$  
   (4.4) Suppose $s \leq \gamma < s+1$. We have
   $$
   u^{\prime} = \min\{\gamma (1+s)-s, \gamma / 2 \} < \gamma \quad \text{ and } \quad 0 \leq u^{\prime} < 1.
   $$ 
\end{proof}

\begin{proposition}\label{prop_e_2}
For any $u \in (0, \infty]$, we have $b_3(u) \leq v_2(u)$ and $b_4(u) < v_2(u)$. 
\end{proposition}
\begin{proof}
    We have
    \begin{align*}
        b_3 \leq &~ 2(\tilde{\ell} - \gamma) \overset{\tilde{\ell} \leq \ell_{\gamma} \leq \gamma}{\leq} \tilde{\ell} - \gamma = v_2,\\
        b_4 \leq &~ \tilde{\ell} - \gamma - 1 < v_2.
    \end{align*}
\end{proof}

\begin{proposition}\label{prop_e_3}
    Let $u^{\prime} = u^{\prime}(t)$ be the optimal rate of the regularization parameter $\lambda$ with qualification $t < \infty$ and coefficient of the source condition $s>0$ defined in \cite{lu2024saturation} (ignoring the logarithm term). Then, for any $u \geq u^{\prime}$, we have
    $$
    \max\{v_1(u), v_2(u), b_1(u)\} \geq b_2(u, t).
    $$
\end{proposition}
\begin{proof}
For notational simplicity, we denote  $\tilde{s}:=\min\{s, 2t\}$  and $b_2(u)=b_2(u, t)$ for any $u \in [0, \infty]$.
 We divide the proof into three steps:
 \begin{itemize}
     \item[(i)] We first show that for any $u^{\prime} \leq u < \lceil u^{\prime} \rceil$, we have
     $$
    \max\{v_1(u), v_2(u), b_1(u)\} \geq b_2(u);
    $$
     
     \item[(ii)] We then show that we have
        $$
        \max\{v_1(\lceil u^{\prime} \rceil), v_2(\lceil u^{\prime} \rceil), b_1(\lceil u^{\prime} \rceil)\} \geq b_2(\lceil u^{\prime} \rceil);
        $$
    \item[(iii)] Finally, we show that, for any integer $q \geq 0$ and any $u\geq q$, if $\max\{v_1(q), v_2(q), b_1(q)\} \geq b_2(q)$, then we have
    $$
        \max\{v_1(u), v_2(u), b_1(u)\} \geq b_2(u);
    $$
 \end{itemize}
Combining (i), (ii), and (iii) with $q=\lceil u^{\prime} \rceil$, we obtain the desired result.

\noindent {\bf Proof of (i). } 
For any integer $q \geq 0$ and any $q \leq q_1 \leq q_2 < q+1$, we have
\begin{equation}\label{eqn_prop_e3_1}
    \begin{aligned}
        v_1(q_2) &\geq v_1(q_1)\\
        v_1(q+1) &= v_1(q) + \mathbf{1}\{q < \ell_{\gamma}\} + 2(\gamma -\ell_{\gamma})\mathbf{1}\{q = \ell_{\gamma}\}\\
        v_2(q_2) &= v_2(q_1)\\
        v_2(q+1) &= v_2(q) + \mathbf{1}\{q < \ell_{\gamma}\}\\
        b_1(q_2) &= b_1(q_1)\\
        b_1(q+1) &= b_1(q) - s\mathbf{1}\{q < \ell_{\gamma}\}\\
        b_2(q_2) &\leq b_2(q_1)\\
        b_2(q+1) &= b_2(q) - \tilde{s}\mathbf{1}\{q < \ell_{\gamma}\} - 2t \mathbf{1}\{q \geq \ell_{\gamma}\}.
    \end{aligned}
\end{equation}
Therefore, for any $u^{\prime} \leq u < \lceil u^{\prime} \rceil$, we have
     \begin{align*}
         \max\{v_1(u), v_2(u), b_1(u)\} 
         \geq&~
         \max\{v_1(u^{\prime}), v_2(u^{\prime}), b_1(u^{\prime})\}\\
         \overset{\text{Corollaries D.15-17 in \cite{lu2024saturation}} }{\geq}&~
         b_2(u^{\prime}) \geq b_2(u).
     \end{align*}

\noindent {\bf Proof of (ii). } 
If $u^{\prime}$ is an integer, then $\lceil u^{\prime} \rceil = u^{\prime}$. Hence Corollaries D.15-17 in \cite{lu2024saturation} directly imply $\max\{v_1(\lceil u^{\prime} \rceil), v_2(\lceil u^{\prime} \rceil), b_1(\lceil u^{\prime} \rceil)\} \geq b_2(\lceil u^{\prime} \rceil)$. 

Suppose $u^{\prime}$ is not an integer. 
From Proposition \ref{prop_e_1}, we have $\tilde{p} \leq u^{\prime} < \tilde{p}+1$, where $\tilde{p}:=\lfloor \gamma/(\tilde{s}+1)\rfloor$.
Denote $\tilde{\ell}_{u^{\prime}} := \min\{\ell_{\gamma}, \tilde{p}+1\}$. Then we have $\lceil u^{\prime} \rceil = \tilde{p}+1$ and
\begin{equation*}
    \begin{aligned}
    v_2(\lceil u^{\prime} \rceil) =&~  \tilde{\ell}_{u^{\prime}} - \gamma\\
    b_2(\lceil u^{\prime} \rceil) =&~ -2  t  (\tilde{p}+1) + (2  t - \tilde{s})  \tilde{\ell}_{u^{\prime}}.
    \end{aligned}
\end{equation*}

\begin{itemize}
    \item If $\gamma < \tilde{p}+1$, then by the definition of $\tilde{p} \leq \ell_{\gamma} = \lfloor \gamma \rfloor$, we have $\tilde{p}=\ell_{\gamma}$ and $\tilde{\ell}_{u^{\prime}} = \ell_{\gamma}$. Hence
$$
b_2(\lceil u^{\prime} \rceil) = -2  t - \tilde{s}\ell_{\gamma} \overset{t \geq 1}{<} -1 < \ell_{\gamma} - \gamma = v_2(\lceil u^{\prime} \rceil);
$$

    \item If $\gamma \geq \tilde{p}+1$, then $\tilde{\ell}_{u^{\prime}} = \tilde{p}+1$.
    Since $\gamma<(\tilde{p}+1)(\tilde{s}+1)$, we have
$$
b_2(\lceil u^{\prime} \rceil) = -\tilde{s}\tilde{p}-\tilde{s} < \tilde{p}+1-\gamma = v_2(\lceil u^{\prime} \rceil).
$$
\end{itemize}
Combining the above two cases, we have
$$
        \max\{v_1(\lceil u^{\prime} \rceil), v_2(\lceil u^{\prime} \rceil), b_1(\lceil u^{\prime} \rceil)\} \geq b_2(\lceil u^{\prime} \rceil). 
        $$

\noindent {\bf Proof of (iii). } 
Suppose that for some integer $q \geq 0$, we have $\max\{v_1(q), v_2(q), b_1(q)\} \geq b_2(q)$. 
By (\ref{eqn_prop_e3_1}), for any $q^{\prime \prime}\in[q, q+1)$ we have
\begin{align}\label{eqn_step_1}
  \max\{v_1(q^{\prime \prime}), v_2(q^{\prime \prime}), b_1(q^{\prime \prime})\} &\geq  b_2(q^{\prime \prime}).
\end{align}
Thus, the inequality holds throughout the interval $[q, q+1)$.

Next, consider the integer point $q+1$. We analyze three possible cases:
\begin{itemize}
    \item If $s \leq 2 t$, then $\tilde{s}=s$ and (\ref{eqn_prop_e3_1}) implies 
    \begin{align*}
        &~ \max\{v_1(q+1), v_2(q+1), b_1(q+1)\} \geq \max\{v_1(q), v_2(q), b_1(q)\} - s\mathbf{1}\{q < \ell_{\gamma}\}\\
        \geq &~ b_2(q) - s\mathbf{1}\{q < \ell_{\gamma}\} \geq b_2(q+1).
    \end{align*}
    
    \item If $s>2t$ and $q < \ell_{\gamma}$, then from (\ref{eqn_order_in_bias_and_variance}) we have $b_1(q)=-s(q+1) < -2t q = b_2(q)$, so the inequality $\max\{v_1(q), v_2(q), b_1(q)\} \geq b_2(q)$ must come from the first two components:
    $$
        \max\{v_1(q), v_2(q)\} \geq b_2(q).
    $$
    Therefore, (\ref{eqn_prop_e3_1}) implies
    \begin{align*}
        &~ \max\{v_1(q+1), v_2(q+1), b_1(q+1)\} \geq \max\{v_1(q+1), v_2(q+1)\}\\
        \geq &~ \max\{v_1(q), v_2(q)\} \geq  b_2(q) \geq b_2(q+1).
    \end{align*}

    \item If $s>2t$ and $q \geq \ell_{\gamma}$, then (\ref{eqn_prop_e3_1}) implies 
    \begin{align*}
        &~ \max\{v_1(q+1), v_2(q+1), b_1(q+1)\} \geq \max\{v_1(q), v_2(q), b_1(q)\} \\
        \geq &~ b_2(q) \geq b_2(q+1).
    \end{align*}
\end{itemize}
Combining (\ref{eqn_step_1}) and all three cases above, we conclude that if 
$$
\max\{v_1(u), v_2(u), b_1(u)\} \geq b_2(u),
$$
holds at integer $u=q$, it also holds for all $u \in [q, q+1]$.

Finally, by induction over integers $q+1, q+2, \cdots$, we have 
$$
\max\{v_1(u), v_2(u), b_1(u)\} \geq b_2(u),\quad u \geq q,
$$
which completes the proof of (iii).
\end{proof}

\subsection{Verification of the operator approximations: Proof of \eqref{eqn_learning_curve_pre_results_3}}\label{append_proof_of_previous_prop_B_1}
    
We adapt the proofs of Theorem D.5, Lemmas D.7--D.8, and Theorems
D.9--D.10 in the NeurIPS 2024 proceedings version of
\cite{lu2024saturation}. We first verify their approximation conditions
under Assumptions 1--5, and then justify the operator estimates, including
the case $0<s<1$.

\paragraph{Population quantities.}
Fix $0<u<u'(\tau)$. Proposition~\ref{prop_e_1} gives $u'<\gamma$,
so $a:=\ell_\lambda=\lfloor u\rfloor=\tilde\ell$.
All spectral sums in this proof use the strictly positive eigenvalues
$\{\lambda_j:j\in J_d\}$ and their orthonormal eigenfunctions
$\{\phi_j:j\in J_d\}$; $J_d$ may be finite. Write
$f_\star=\sum_{j\in J_d}f_j\phi_j$.
Use $\mathcal N_1$, $\mathcal N_{2,\varphi}$ and
$\mathcal M_{2,\varphi}$ as defined before Lemma~\ref{lem:popall}.
For the integral operator $T$ on $L^2(\rho_{\calX})$, define
$f_\lambda=T\varphi_\lambda(T)f_\star$ and
$\mathcal M_{1,\varphi}(\lambda)=\|f_\star-f_\lambda\|_{L^\infty}$.

Since $a+1\le\ell_\gamma+1$, Proposition~\ref{prop:inner_edr}
gives $\mu_k\asymp d^{-k}$ and $N(d,k)\asymp d^k$ for
$0\le k\le a+1$, and $\sup_{k>a}\mu_k\lesssim d^{-(a+1)}$.
Degree $a+1$ and the trace bound therefore give
$\sum_{k>a}\mu_kN(d,k)\asymp1$ and
$\sum_{k>a}\mu_k^2N(d,k)\asymp d^{-(a+1)}$.
Thus Lemmas~\ref{lem:popall} and \ref{lemma_bounds_on_quantities} imply
\begin{equation}\label{eq:p1_population}
\begin{aligned}
\mathcal N_1(\lambda)&\asymp d^u,\qquad
\mathcal N_{2,\varphi}(\lambda)\asymp d^a+d^{2u-a-1},\\
\mathcal M_{2,\varphi}(\lambda)&\asymp
\begin{cases}
d^{-(a+1)s}+d^{-2\tau u+(2\tau-\tilde s)a},&\tau<\infty,\\
d^{-(a+1)s}+\mathbf1\{u=a\}d^{-as},&\tau=\infty.
\end{cases}
\end{aligned}
\end{equation}
These lemmas have direct population proofs. Their lower bounds use only
the two normalized block energies and the unweighted head energy in
Assumption 5. In particular, for $s>2\tau$, residual monotonicity and
$\psi_\lambda(K_{\max})\gtrsim\lambda^\tau$ give the head term
$\lambda^{2\tau}$. For $\tau=\infty$ and $u=a$,
Definition~\ref{def:filter}(4) gives the additional term $d^{-as}$.
No positivity beyond the degrees in Assumption 2 is required.

Define the comparison quantities
\begin{equation}\label{eq:p1_comparison_quantities}
\begin{aligned}
\mathcal P_d&=\mathcal M_{2,\varphi}(\lambda)
+\frac{\sigma^2}{n}\mathcal N_{2,\varphi}(\lambda),\qquad
\mathcal W_d=n^{-1}d^a,\\
\mathcal S_\lambda&=\sum_{j\in J_d}
\frac{\lambda^2\lambda_j\varphi_\lambda(\lambda_j)^2}
{\lambda+\lambda_j}f_j^2.
\end{aligned}
\end{equation}
Equation~\eqref{eq:p1_population} yields
$\mathcal P_d\gtrsim\mathcal W_d+d^{2u-a-1-\gamma}+d^{-(a+1)s}$.
For sufficiently small constants $\epsilon,\eta>0$, the required
approximation conditions are
\begin{align}
n^{\epsilon-1}\mathcal N_1&\longrightarrow0,
\label{eq:p1_concentration}\\
\frac{\mathcal N_1^2}{n\mathcal N_{2,\varphi}}
\log n(\log\lambda^{-1})^2&\longrightarrow0,
\label{eq:p1_variance_condition}\\
\frac{\mathcal N_1}{n}\log n(\log\lambda^{-1})^2
\mathcal S_\lambda&=o_d(\mathcal P_d),
\label{eq:p1_common_bias_condition}\\
\frac{\mathcal N_1\mathcal M_{1,\varphi}^2}{n^2}
&=o_d(\mathcal P_d)\quad(s\ge1),
\label{eq:p1_bounded_bias_condition}\\
\frac{\mathcal N_1}{n^2}
\left(\|f_\lambda\|_\infty^2+n^{1-s+\eta}\right)
&=o_d(\mathcal P_d)\quad(0<s<1).
\label{eq:p1_truncated_bias_condition}
\end{align}
The quantities in these conditions are evaluated at $\lambda$.

\paragraph{Concentration and variance.}
Equation~\eqref{eq:p1_concentration} follows from $u<\gamma$ and
\eqref{eq:p1_population}, with $0<\epsilon<(\gamma-u)/\gamma$.
Let $p=\lfloor\gamma/(1+\tilde s)\rfloor$. The formulas in the proof of
Proposition~\ref{prop_e_1} give $p\le u'<p+1$ and $2u'\le\gamma+p$.
To verify the second inequality, first consider $1\le s\le\tau$.
For $u'=p+1/2$ one has $\gamma-p\ge ps\ge1$;
the other choices satisfy $u'\le\gamma/2$.
For $\tilde s>\tau$, let $\Delta=\gamma-p(1+\tilde s)$.
When $\gamma\ge1$, the three respective values of $2u'-\gamma-p$ are
$-p\tilde s+(\tau^{-1}-1)\Delta$,
$-p\tilde s+(1-\tau\Delta)/(\tau+1)$, and
$-p\tilde s+1-\tilde s$.
All are nonpositive on their defining intervals; the middle interval
has $\Delta\ge\tau\ge1$. For $\gamma<1$, $u'=\gamma/2$.
For $s<1$ and finite $\tau$, the two values are
$-ps+(\tau^{-1}-1)\Delta$ and $-ps+s/\tau-\Delta$,
where $\Delta\ge s$ in the second case.
For $s<1$, $\tau=\infty$ and $p\ge1$, they are
$p+s-\gamma\le s(1-p)$ and $-ps$; for $p=0$,
both selectors are at most $\gamma/2$.

By \eqref{eq:p1_population},
$\mathcal N_1^2/(n\mathcal N_{2,\varphi})
\asymp\min\{d^{2u-\gamma-a},d^{a+1-\gamma}\}$.
If $a+1<\gamma$, the second exponent is negative.
Otherwise $a\le p<\gamma\le a+1$ implies $a=p$, and
$u<u'$ gives $2u-\gamma-a<0$.
Each fixed negative power absorbs $\log n(\log\lambda^{-1})^2$.
This proves \eqref{eq:p1_variance_condition}, also when
$\gamma=a+1$ is an integer.

\paragraph{Bias conditions.}
Lemma~\ref{lemma_bias_special_condition_2} proves
\eqref{eq:p1_common_bias_condition} for every $0<u<\gamma$.
Its proof uses only the source bound, the spectral estimates and
Lemma~\ref{lem:popall}.
For $s\ge1$, the addition formula and source Cauchy--Schwarz give
\begin{equation}\label{eq:p1_M1_bound}
\mathcal M_{1,\varphi}^2\lesssim
\begin{cases}
\lambda^{2\tau}d^{a\max\{2\tau+1-s,0\}}
+d^{-(a+1)(s-1)},&\tau<\infty,\\
d^{-(a+1)(s-1)}+\mathbf1\{u=a\}d^{-a(s-1)},&\tau=\infty.
\end{cases}
\end{equation}
For finite $\tau$, split the sum at degree $a$ and bound the tail by
$(\sup_{k>a}\mu_k)^{s-1}\operatorname{tr}(T)$.
The second line is Lemma~\ref{lemma_bounds_on_quantities}.
The tail contribution to $n^{-2}\mathcal N_1\mathcal M_{1,\varphi}^2$,
divided by $\mathcal W_d$, is
$O(d^{u-\gamma-a-(a+1)(s-1)})=o(1)$.
For finite $\tau$, divide the head contribution by the head term of
$\mathcal M_{2,\varphi}$ in \eqref{eq:p1_population}; the ratio is at most
$C d^{u+a-2\gamma}=o(1)$, including $s>2\tau$.
For infinite $\tau$ at $u=a$, the extra term in
\eqref{eq:p1_M1_bound}, divided by $\mathcal W_d$, has ratio
$O(d^{u-\gamma-as})=o(1)$.
This proves \eqref{eq:p1_bounded_bias_condition}, including $s=1$.

For $0<s<1$, Lemma~\ref{lem:flambda} gives
$\|f_\lambda\|_\infty^2\lesssim
d^{(1-s)a}+d^{2u-(1+s)(a+1)}$.
After multiplication by $n^{-2}\mathcal N_1$, divide these terms by
$\mathcal W_d$ and $d^{2u-a-1-\gamma}$, respectively.
The resulting exponents are $u-\gamma-sa<0$ and
$u-\gamma-s(a+1)<0$.
It remains to show that $d^{u-\gamma(1+s)+\gamma\eta}=o_d(\mathcal P_d)$
for some $\eta>0$. Now $p=\lfloor\gamma/(1+s)\rfloor$ and $a\le p$.
We verify a strict exponent margin at $\eta=0$.

If $a<p$, comparison with $d^{-(a+1)s}$ gives exponent
$(u-\gamma)-s(\gamma-a-1)<0$, since $a+1\le p<\gamma$.
If $a=p\ge1$ and $\tau<\infty$, then
$u-p<u'-p\le s/(2\tau)<\gamma s$; comparison with $\mathcal W_d$ applies.
For $\tau=\infty$, the selector $u'=p+s/2$ gives the same inequality.
For its other selector, $u'=p+\Delta/2$ with
$\Delta=\gamma-p(1+s)\ge s$, comparison with $d^{-(p+1)s}$ gives
$$
u-\gamma(1+s)+(p+1)s
<s-ps(1+s)-\Delta(s+1/2)
\le\frac{s}{2}\{1-2s-2p(1+s)\}<0.
$$
If $a=p=0$ and $\tau=\infty$, then $\gamma<s$ gives
$u<u'\le\gamma s$, while $\gamma\ge s$ gives
$u<u'\le\gamma(1+s)-s$.
Comparison with $\mathcal W_d$ and $d^{-s}$, respectively, applies.
If $a=p=0$, $\tau<\infty$, and $\gamma<s$, the parameter conditions
preceding \eqref{eqn_learning_curve_pre_results_3} imply $s>1/(2\tau)$: otherwise
$(2\tau+1)s/[2\tau(1+s)]\ge s>\gamma$.
Hence $u<u'=\gamma/(2\tau)<\gamma s$.
If instead $\gamma\ge s$, then $u<u'=s/(2\tau)$ and
$u-\gamma(1+s)+s<(2\tau+1)s/(2\tau)-\gamma(1+s)<0$.
The last inequality follows from $\gamma\ge s$ when $s>1/(2\tau)$,
and from condition (iii) otherwise.
Choosing $\eta>0$ smaller than these fixed margins divided by $\gamma$
proves \eqref{eq:p1_truncated_bias_condition}.

\paragraph{Operator estimates under the current assumptions.}
We justify the approximation argument for any fixed $0<u<\gamma$
satisfying \eqref{eq:p1_concentration} and
\eqref{eq:p1_common_bias_condition}, together with
\eqref{eq:p1_bounded_bias_condition} or
\eqref{eq:p1_truncated_bias_condition} as appropriate.
The variance estimate additionally uses \eqref{eq:p1_variance_condition}.
For every nonnegative spectral multiplier $h$ considered here, the
spherical addition formula gives
\begin{equation}\label{eq:p1_addition}
\sum_{k:\mu_k>0}h(\mu_k)\sum_{j=1}^{N(d,k)}\psi_{k,j}(x)^2
=\sum_{k:\mu_k>0}h(\mu_k)N(d,k).
\end{equation}
In particular, $h(z)=(z\varphi_\lambda(z))^2$,
$h(z)=z\varphi_\lambda(z)$ and $h(z)=z/(z+\lambda)$
verify equations (43)--(45) of \cite{lu2024saturation} with equality.
The trace bound makes each sum finite for fixed $\lambda>0$.

Let $S:\mathcal H\to L^2(\rho_{\calX})$ be the inclusion,
$T_H=S^*S$, and $U=T_H+\lambda I$.
The empirical covariance $T_X$ was defined in
Appendix~\ref{subsec_learning_curve_fix_d}.
Both $T_H$ and $T_X$ have norm at most $K_{\max}$.
By \eqref{eq:p1_addition},
$\|U^{-1/2}\mathscr K(x,\cdot)\|_{\mathcal H}^2=\mathcal N_1$.
Operator Bernstein and \eqref{eq:p1_concentration} give
\begin{equation}\label{eq:p1_operator_concentration}
\delta_d:=\|U^{-1/2}(T_X-T_H)U^{-1/2}\|_{\mathrm{op}}
=O_{d,\mathbb P}\bigl((\mathcal N_1\log n/n)^{1/2}\bigr).
\end{equation}
The term $\mathcal N_1\log n/n$ is absorbed by its square root.
On the event $\delta_d<1/2$, $T_X+\lambda I$ and $U$ are comparable.
On the contour $\Gamma_\lambda$ in equation (112) of the cited paper,
$\sup_{0\le t\le K_{\max}}(t+\lambda)/|t-z|\le C$.
Consequently the two resolvents weighted on both sides by $U^{1/2}$
have bounded operator norm.
Conditions (C1)--(C2) in Appendix~\ref{appendix_addtion_require_of_reg}
give
$\int_{\Gamma_\lambda}|\varphi_\lambda(z)|\,|dz|=O(\log\lambda^{-1})$
and $\int_{\Gamma_\lambda}|\psi_\lambda(z)|\,|dz|
=O(\lambda\log\lambda^{-1})$.
Use the analytic extension at zero for the operator calculus;
the estimator is independent of this choice on the empirical null space.

For $0<s<1$, the truncation argument also needs a dimension independent
integrability bound. For every $f\in[\mathcal H]^s$,
\begin{equation}\label{eq:p1_embedding}
\|f\|_{L^{2/(1-s)}}\le K_{\max}^{s/2}\|f\|_{[\mathcal H]^s}.
\end{equation}
Indeed, on a finite positive spectral truncation the analytic family
$c\mapsto\sum_j\lambda_j^{z/2}c_j\phi_j$ has norm at most one
from $\ell^2$ to $L^2$ on $\Re z=0$, and norm at most
$K_{\max}^{1/2}$ from $\ell^2$ to $L^\infty$ on $\Re z=1$.
Interpolation and passage to the spectral limit prove
\eqref{eq:p1_embedding}.
Choose $0<\epsilon_t<\eta/2$, let $t_n=n^{(1-s)/2+\epsilon_t}$,
and choose $2<q<2/(1-s)$ close enough to $2/(1-s)$ that
$t_n^q/n\to\infty$.
The source condition and \eqref{eq:p1_embedding} imply
$\mathbb P(\max_{i\le n}|f_\star(x_i)|>t_n)
\le n\|f_\star\|_q^q/t_n^q\to0$.

Let $g=S^*f_\star$, $h_\lambda=\varphi_\lambda(T_H)g$, and
$g_X=n^{-1}\sum_i\mathscr K(x_i,\cdot)f_\star(x_i)$.
Then $Sh_\lambda=f_\lambda$ and
$\tilde f_\lambda:=\mathbb E_\epsilon(\hat f_\lambda\mid X)
=S\varphi_\lambda(T_X)g_X$.
For $r_X=g_X-T_Xh_\lambda-g+T_Hh_\lambda$, an identity in $\mathcal H$ is
\begin{equation}\label{eq:p1_operator_identity}
\begin{aligned}
\varphi_\lambda(T_X)g_X-h_\lambda
={}&\varphi_\lambda(T_X)r_X\\
&+(\varphi_\lambda(T_X)-\varphi_\lambda(T_H))\psi_\lambda(T_H)g\\
&-(\psi_\lambda(T_X)-\psi_\lambda(T_H))\varphi_\lambda(T_H)g.
\end{aligned}
\end{equation}
After applying $S$, the last two terms in
\eqref{eq:p1_operator_identity} have norm at most
$C\delta_d\log\lambda^{-1}
(\mathcal M_{2,\varphi}^{1/2}+\mathcal S_\lambda^{1/2})$
by the resolvent identity and the two contour bounds.
For the first term, the real filter bound and covariance comparison
give $\|S\varphi_\lambda(T_X)U^{1/2}\|_{\mathrm{op}}\le C$.
The centered vector $U^{-1/2}r_X$ has second moment at most
$\mathcal N_1\mathcal M_{2,\varphi}/n$.
For $s\ge1$, its summands have envelope
$C\mathcal N_1^{1/2}\mathcal M_{1,\varphi}$.
For $0<s<1$, restrict to $|f_\star|\le t_n$; the envelope is
$C\mathcal N_1^{1/2}(\|f_\lambda\|_\infty+t_n)$.
The discarded expectation is bounded by
$C\mathcal N_1^{1/2}\mathcal M_{2,\varphi}^{1/2}t_n^{-q/2}
=o(\mathcal P_d^{1/2})$ by \eqref{eq:p1_concentration} and
$t_n^q/n\to\infty$.
Hilbert space Bernstein, \eqref{eq:p1_operator_concentration}, and
the bias conditions therefore imply
\begin{equation}\label{eq:p1_bias_approximation}
\|\tilde f_\lambda-f_\lambda\|_{L^2}
=o_{d,\mathbb P}(\mathcal P_d^{1/2}),\qquad
\mathrm{bias}^2(\hat f_\lambda)
=\mathcal M_{2,\varphi}+o_{d,\mathbb P}(\mathcal P_d).
\end{equation}
The second assertion follows from the triangle inequality and
$\|f_\star-f_\lambda\|_{L^2}^2=\mathcal M_{2,\varphi}\le\mathcal P_d$.
This argument uses elements of $\mathcal H$ even when $f_\star\notin\mathcal H$.

Finally, conditional on $X$, the noise covariance is $\sigma^2I_n$.
Equation~\eqref{eq:p1_addition} gives, for every $x$,
$$
\|S\varphi_\lambda(T_H)\mathscr K(x,\cdot)\|_{L^2}^2
=\mathcal N_{2,\varphi}.
$$
The resolvent comparison gives
\begin{equation}\label{eq:p1_variance_error}
\begin{aligned}
&\left|\mathrm{var}(\hat f_\lambda)^{1/2}
-\sigma n^{-1/2}\mathcal N_{2,\varphi}^{1/2}\right|\\
&\qquad=O_{d,\mathbb P}\left(\sigma n^{-1/2}\delta_d
\mathcal N_1^{1/2}\log\lambda^{-1}\right).
\end{aligned}
\end{equation}
Equations~\eqref{eq:p1_operator_concentration} and
\eqref{eq:p1_variance_condition} make the error in
\eqref{eq:p1_variance_error} relatively negligible, so
$\mathrm{var}(\hat f_\lambda)
=(\sigma^2/n)\mathcal N_{2,\varphi}(1+o_{d,\mathbb P}(1))$.
Combining this with \eqref{eq:p1_bias_approximation} and
\eqref{eq:p1_population} proves \eqref{eqn_learning_curve_pre_results_3}.
All probability statements concern each fixed $u<u'(\tau)$.

\ifcsname old@section\endcsname
  \expandafter\let\expandafter\lcSphereSection\csname old@section\endcsname
\else
  \let\lcSphereSection\section
\fi
\lcSphereSection[Kernel approximation results for kernels on the sphere]{Kernel approximation results for kernels on the sphere}\label{append_kernel_app_sphere}

In this section, we will restate and develop several approximation results for kernels defined on the sphere (see, e.g., Section 2.1).

Throughout this section, we assume that all conditions and assumptions in Theorem 3.1 hold.

The following lemma is borrowed from Proposition 3 in \cite{ghorbani2021linearized} as well as Lemma S5 in \cite{zhang2024phase}.

\begin{lemma}\label{lemma psi psi top}
   We have
    \begin{equation}\label{eq k ge l in lemma}
    \begin{aligned}
        \Psi_{>\ell_{\gamma}} \Sigma_{>\ell_{\gamma}} \Psi_{>\ell_{\gamma}}^{\top} &= \kappa_1 \left(\mathrm{I}_{n} + \Delta_{1}\right), \quad \|\Delta_{1}\|_{\mathrm{op}} = o_{d, \mathbb{P}}(1)\\
        \Psi_{>\ell_{\gamma}} \Sigma_{>\ell_{\gamma}}^{2} \Psi_{>\ell_{\gamma}}^{\top} &= \kappa_2 \left(\mathrm{I}_{n} + \Delta_{2}\right), \quad \|\Delta_{2}\|_{\mathrm{op}} = o_{d, \mathbb{P}}(1);
    \end{aligned}
    \end{equation}
    with
    \begin{align*}
    \kappa_{1} = \Theta_{d}(1), \quad \kappa_{2} = \Theta_{d}\left( d^{-(\ell_{\gamma}+1)} \right).
\end{align*}
\end{lemma}

The following lemma is borrowed from Lemma 11 in \cite{ghorbani2021linearized}; it also appears as Lemma S6 in \cite{zhang2024phase}.

\begin{lemma}[Lemma 11 in \cite{ghorbani2021linearized}; Lemma S6 of \cite{zhang2024phase}]
\label{lemma psi top psi}
When $\gamma > \ell_{\gamma}$, we have
    \begin{displaymath}
         \left\| \frac{\Psi_{\leq \ell_{\gamma}}^{\top} \Psi_{\leq \ell_{\gamma}} }{n} - \mathrm{I}_{N_{\ell_{\gamma}}} \right\|_{\mathrm{op}} = o_{d, \mathbb{P}}(1).
    \end{displaymath}
\end{lemma}

Suppose $\gamma = \ell_{\gamma}$ and denote $S_m=\Psi_m^\top\Psi_m/n$, $m=1,2, \cdots$. 
We note that Theorem 1 of \cite{xiao2022precise} only provides weak convergence results on the eigenvalue distribution of $S_{\ell_{\gamma}}$, which is insufficient to bound the maximum / minimum eigenvalues with high probability. Therefore, their proof (e.g., Appendix C.6) and the proof in \cite{zhang2024phase} (e.g., Appendix 2.3) contain gaps. The following two results fix this gap.

\begin{theorem}[Fixed-degree harmonic sample covariance]\label{thm:critical}
Suppose $\gamma = \ell_{\gamma}$. Fix an integer $m\ge1$. There is a constant $K_m<\infty$ such that, for every $n\ge1$ and $D=N(d, m)$,
\begin{equation}\label{eq:all-n}
 \Pp\left\{\norm{S_m}_{\op}\le K_m\left(1+\frac Dn\right)\right\}
 \ge1-D^{-1}.
\end{equation}
In particular, by letting $m= {\ell_{\gamma}}$, we have
\begin{equation}\label{eq:Cm-proved}
 \lambda_{\max}(S_{\ell_{\gamma}})=\Theta_{d,\Pp}(1),
 \qquad \mathrm{tr}(S_{\ell_{\gamma}})=D\asymp n.
\end{equation}
\end{theorem}
\begin{proof}
The vector $Y=\psi_m(x)$ satisfies all conditions of Lemma~\ref{lem:tikh} with $B=5^{3m}$ by \eqref{eq:moment}.

If $n\ge2D$, apply that lemma with $N=n$. Since $\norm{Y_i}_2^2=D$ exactly, with probability at least $1-D^{-1}$,
\begin{equation}\label{eq:large-n}
 \norm{S_m}_{\op}\le1+C_6\left(\frac Dn+5^m\sqrt{\frac Dn}\right)
 \le K_m\left(1+\frac Dn\right).
\end{equation}
For the last inequality, use $\sqrt r\le(1+r)/2$ for $r\ge0$.

If $n<2D$, enlarge the probability space by adjoining $2D-n$ independent copies. Define
\[
 S^{\mathrm{aug}}=\frac1{2D}\sum_{i=1}^{2D}Y_iY_i^\top.
\]
On every outcome of this enlarged probability space,
\[
 \frac1n\sum_{i=1}^nY_iY_i^\top\preceq\frac{2D}{n}S^{\mathrm{aug}}.
\]
Lemma~\ref{lem:tikh}, applied with $N=2D$, therefore gives
\begin{equation}\label{eq:small-n}
 \norm{S_m}_{\op}
 \le\frac{2D}{n}\left[1+C_6\left(\frac12+\frac{5^m}{\sqrt2}\right)\right]
 \le K_m\left(1+\frac Dn\right)
\end{equation}
with probability at least $1-D^{-1}$. The failure probability for the original matrix is bounded by the failure probability of the augmented event, because the former event depends only on the first $n$ observations. Marginalizing the auxiliary observations proves \eqref{eq:all-n}. 

Finally, \eqref{eq:add} gives $\mathrm{tr} S_{\ell_{\gamma}}=N(d, {\ell_{\gamma}})$ exactly, so $\lambda_{\max}(S_{\ell_{\gamma}})\ge1$. Under $n\asymp d^{{\ell_{\gamma}}}\asymp N(d, {\ell_{\gamma}})$, the upper bound in \eqref{eq:all-n} is a constant independent of $d$, and $D^{-1}\to0$. This proves \eqref{eq:Cm-proved}.
\end{proof}

\begin{lemma}\label{lemma psi top psi_int}
Suppose $\gamma= \ell_{\gamma}=m$ and $n\asymp d^m$. Then
\begin{equation}\label{eqn_lemma psi top psi_int_1}
 \norm{S_{\le m-1}-\mathrm{I}_{N_{m-1}}}_{\op}=o_{d,\Pp}(1),
 \qquad \norm{\Psi_{\le m}/\sqrt n}_{\op}=O_{d,\Pp}(1).
\end{equation}
Moreover, there is a fixed $\eta>0$ such that, with
\[
 \mathcal{I}_d=\left\lfloor\eta\min(n,N(d, m))\right\rfloor,
\]
one has $\mathcal{I}_d\asymp d^m$ and
\begin{equation}\label{eq:bulk}
 \Pp\{\lambda_{\mathcal{I}_d}(S_m)\ge1/2\}\longrightarrow1.
\end{equation}
The same lower bound holds for the $\mathcal{I}_d$-th largest eigenvalue of $\Psi_m\Psi_m^\top/n$.
\end{lemma}
\begin{proof}
For the first assertion, put $q=N_{m-1}$ and $v(x)=\psi_{\le m-1}(x)$. Orthonormality and \eqref{eq:add} give
\[
 \E vv^\top=\mathrm{I}_q,\quad \norm{v}_2^2=q,\quad
 \E(vv^\top-\mathrm{I}_q)^2=(q-1)\mathrm{I}_q.
\]
The matrix Bernstein inequality applied to $v(x_i)v(x_i)^\top-\mathrm{I}_q$ yields
\[
\norm{S_{\le m-1}-\mathrm{I}_{q}}_{\op} =O_{d,\Pp}\left(\sqrt{\frac{q\log(2q)}n}+\frac{q\log(2q)}n\right)=o_{d,\Pp}(1).
\]
Here $q\asymp d^{m-1}$ and $n\asymp d^m$; for $m=1$, $q=1$ and the Gram matrix is exactly $\mathrm{I}_1$.

By Theorem~\ref{thm:critical}, choose a deterministic $B\ge1$ such that $\Pp\{\norm{S_m}_{\op}\le B\}\to1$. The concatenation satisfies
\[
 \norm{\Psi_{\le m}/\sqrt n}_{\op}^2
 \le\norm{\Psi_{\le m-1}/\sqrt n}_{\op}^2+
      \norm{\Psi_m/\sqrt n}_{\op}^2=O_{d,\Pp}(1).
\]
On the event $\norm{S_m}_{\op}\le B$, let $r$ be the number of eigenvalues of $S_m$ that are at least $1/2$. Its exact trace gives
\[
 N(d, m)=\mathrm{tr} S_m\le rB+(N(d, m)-r)/2\le rB+N(d, m)/2,
 \qquad r\ge\frac{N(d, m)}{2B}.
\]
Choose $\eta=(4B)^{-1}$. Then $\mathcal{I}_d\le N(d, m)/(4B)\le r$ for all sufficiently large $d$, while 
$$
\mathcal{I}_d\asymp\min(n,N(d, m))\asymp d^m.
$$
The nonzero eigenvalues of the row and column Gram matrices agree, so \eqref{eq:bulk} holds for both.
\end{proof}

\begin{proposition}\label{prop_lemma_3_in_xiao}
Suppose $\gamma= \ell_{\gamma}=m$, $n\asymp d^m$, and $0\le n\lambda\le C_\lambda$. Put $t=1+n\lambda$ and
\begin{equation}\label{eq:Dreg}
 \mathcal D=\frac{t}{n}\Sigma_{\le m}^{-1}+S_{\le m}.
\end{equation}
There are deterministic constants $0<c<C<\infty$ such that
\[
 \Pp\{c\mathrm{I}_{N_{m}}\preceq\mathcal D\preceq C\mathrm{I}_{N_{m}}\}\longrightarrow1.
\]
In particular, both extreme eigenvalues of $\mathcal D$ have order $\Theta_{d,\Pp}(1)$, including when $N_{m}>n$.
\end{proposition}
\begin{proof}
Write $Z_0=\Psi_{\le m-1}/\sqrt n$ and $Z_1=\Psi_m/\sqrt n$. On an event whose probability tends to one, Lemma \ref{lemma psi top psi_int} gives
\[
 Z_0^\top Z_0\succeq c_0\mathrm{I},\quad \norm{Z_0}_{\op}^2\le C_0,
 \quad \norm{Z_1}_{\op}^2\le B.
\]
Since $t\asymp1$ and $n\mu_m\asymp1$, the degree-$m$ part of $t\Sigma_{\le m}^{-1}/n$ is at least $r_0\mathrm{I}$ for a fixed $r_0>0$. All eigenvalues of this diagonal matrix are bounded above by a fixed constant. The upper bound for $\mathcal D$ follows immediately.

For any vector $(v,w)\in\R^{N_{m-1}}\times\R^{N(d, m)}$, set
\[
 Q=\norm{Z_0v+Z_1w}_2^2+r_0\norm{w}_2^2.
\]
The quadratic form of $\mathcal D$ is at least $Q$. Furthermore,
\begin{equation}\label{eq:coercivity}
 \begin{aligned}
 c_0\norm{v}_2^2
 &\le\norm{Z_0v}_2^2
 \le2\norm{Z_0v+Z_1w}_2^2+2\norm{Z_1w}_2^2
 \le(2+2B/r_0)Q,\\
 \norm{w}_2^2&\le Q/r_0.
 \end{aligned}
\end{equation}
Thus
\[
 \norm{v}_2^2+\norm{w}_2^2
 \le\left[\frac{2+2B/r_0}{c_0}+\frac1{r_0}\right]Q.
\]
This proves the asserted lower bound.
\end{proof}

Denote $m = \ell_{\gamma}$.
The proof of the next lemma uses the above-critical tail estimates of the manuscript's Lemma \ref{lemma psi psi top}:
\begin{equation}\label{eq:G1}
 \begin{aligned}
 K_{>m}&=\kappa_1(\mathrm{I}_n+\Delta_1),
 &\kappa_1&\asymp1,&\norm{\Delta_1}_{\op}&=o_{d,\Pp}(1),\\
 \Psi_{>m}\Sigma_{>m}^2\Psi_{>m}^\top
 &=\kappa_2(\mathrm{I}_n+\Delta_2),
 &\kappa_2&\asymp d^{-(m+1)},&\norm{\Delta_2}_{\op}&=o_{d,\Pp}(1).
 \end{aligned}
\end{equation}

\begin{lemma}\label{lemma lambda max min}
Take a fixed $u\in(0,\infty]$, with $\lambda\asymp d^{-u}$ for finite $u$ and $\lambda=0$ for $u=\infty$. Put $m=\lfloor\gamma\rfloor$, $a=\min(m,\lfloor u\rfloor)$ for finite $u$, and $a=m$ for $u=\infty$. Set $t=1+n\lambda$, $M=K+n\lambda \mathrm{I}_n$, and $M_{>a}=K_{>a}+n\lambda \mathrm{I}_n$.
We have
\begin{itemize}
\item[(i)] $\lambda_{\min}(M)\ge ct$ with probability tending to one, and
\begin{equation}\label{eq:traceM}
 \mathrm{tr}(M)/n=\Phi(1)+n\lambda
\end{equation}
exactly.
\item[(ii)] $ct\mathrm{I}_n\preceq M_{>m}\preceq Ct\mathrm{I}_n$ with probability tending to one.
\item[(iii)] If $a<m$, then
\begin{equation}\label{eq:tailbound}
 \norm{K_{>a}}_{\op}=O_{d,\Pp}(nd^{-(a+1)}+1)=o_{d,\Pp}(n\lambda),
 \qquad \norm{M_{>a}/(n\lambda)-\mathrm{I}_n}_{\op}=o_{d,\Pp}(1).
\end{equation}
For every valid index $N_a<i\le n$, simultaneously,
\begin{equation}\label{eq:rank}
 \lambda_i(M)\le\norm{K_{>a}}_{\op}+n\lambda=O_{d,\Pp}(t).
\end{equation}
\item[(iv)] For
\begin{equation}\label{eq:Aa}
 A_a=\frac1n\Sigma_{\le a}^{-1}
       +\frac1n\Psi_{\le a}^\top M_{>a}^{-1}\Psi_{\le a},
\end{equation}
there are $0<c<C$ such that
\begin{equation}\label{eq:Acompare}
 ct\mathrm{I}_{N_a}\preceq A_a^{-1}\preceq Ct\mathrm{I}_{N_a}
\end{equation}
with probability tending to one.
\end{itemize}
\end{lemma}
\begin{proof}
The above-critical tail estimate gives, with probability tending to one,
\[
 c\mathrm{I}_n\preceq K_{>m}\preceq C\mathrm{I}_n.
\]
Consequently $ct\mathrm{I}_n\preceq M_{>m}\preceq Ct\mathrm{I}_n$ and $\lambda_{\min}(M)\ge ct$. On the sphere,
\[
 \frac{\mathrm{tr} M}{n}=\Phi(1)+n\lambda
\]
holds exactly. This proves (i) and (ii).

Suppose $a<m$. The bound $\|\Psi_{a+1:m}\|_{\op}^2=O_{\PP}(n)$ follows from Lemma \ref{lemma psi top psi} when $\gamma>m$, and from Lemma \ref{lemma psi top psi_int} when $\gamma=m$. Therefore
\[
 \|K_{>a}\|_{\op}
 \le\|\Psi_{a+1:m}\|_{\op}^2\|\Sigma_{a+1:m}\|_{\op}
       +\|K_{>m}\|_{\op}
 =O_{\PP}(nd^{-a-1}+1).
\]
Here $u<a+1\le m\le\gamma$, so $n\lambda\to\infty$ and
\[
 \frac{\|K_{>a}\|_{\op}}{n\lambda}
 =O_{\PP}(d^{u-a-1}+d^{u-\gamma})=o_{\PP}(1).
\]
Thus $\|M_{>a}/(n\lambda)-\mathrm{I}_n\|_{\op}=o_{\PP}(1)$. Since $\rank(K_{\le a})\le N_a$, the min--max principle gives, simultaneously for $N_a<i\le n$,
\[
 \lambda_i(M)\le\|K_{>a}\|_{\op}+n\lambda=O_{\PP}(t).
\]
This proves (iii).

For (iv), write
\[
 A_a=\frac1n\Sigma_{\le a}^{-1}
       +\frac1n\Psi_{\le a}^{\top}M_{>a}^{-1}\Psi_{\le a}.
\]
First suppose $\gamma>m$. Lemma \ref{lemma psi top psi} implies $S_{\le a}:=\Psi_{\le a}^{\top}\Psi_{\le a}/n=\mathrm{I}+o_{\PP}(1)$. If $a<m$, part (iii) gives $M_{>a}\asymp_{\PP}t\mathrm{I}_n$; if $a=m$, part (ii) gives the same conclusion. Hence the second term of $A_a$ is comparable to $t^{-1}\mathrm{I}_{N_a}$. The first term is positive semidefinite and satisfies
\[
 t\left\|\frac{\Sigma_{\le a}^{-1}}n\right\|_{\op}
 \le C\left(d^{a-\gamma}+\lambda d^a\right)=O(1),
\]
because $a\le\gamma$ and $a\le u$; the term $\lambda d^a$ is zero at interpolation. Thus $A_a\asymp_{\PP}t^{-1}\mathrm{I}_{N_a}$.

Next suppose $\gamma=m$ and $a<m$. The same proof applies using the below-critical part of Lemma \ref{lemma psi top psi_int} instead of Lemma \ref{lemma psi top psi}.

Finally suppose $\gamma=m=a$. Only in this case do $u\ge m$ and $t\asymp1$ hold together. Part (ii) implies
\[
 tA_m\asymp_{\PP}\frac{t}{n}\Sigma_{\le m}^{-1}+S_{\le m}.
\]
Proposition \ref{prop_lemma_3_in_xiao} shows that the matrix on the right is bounded above and below by positive constant multiples of $\mathrm{I}_{N_m}$, with probability tending to one. This conclusion also permits $N_m>n$. Inverting the comparison proves (iv) in every case.
\end{proof}

\begin{lemma}\label{lemma_head_filter_product}
Let $a=\widetilde{\ell}$ be the cutoff in Assumption~4, and let $\mathsf A=K/n$, $\mathsf A_H=K_{\leq a}/n$, and $t=1+n\lambda$. Suppose the real-axis filter conditions in Definition~2.2 hold for $\lambda\in\Lambda$. At $\lambda=0$, set $\varphi_0(\mathsf A)=\mathsf A^{-1}$, which is well-defined almost surely for all sufficiently large $d$ by Lemma~2.4. Then, for $\lambda\in\Lambda\cup{0}$,
\[
 \norm{\mathsf A_H\varphi_\lambda(\mathsf A)}_{\op}
 +\norm{\varphi_\lambda(\mathsf A)\mathsf A_H}_{\op}=O_{d,\Pp}(1).
\]
\end{lemma}
\begin{proof}
Put $\mathsf A_T=K_{>a}/n$. Lemma~\ref{lemma lambda max min} gives
\[
 \norm{\varphi_\lambda(\mathsf A)}_{\op}=O_{d,\Pp}(n/t),
 \qquad \norm{\mathsf A_T}_{\op}\,n/t=O_{d,\Pp}(1).
\]
Use $\mathsf A_H\varphi_\lambda(\mathsf A)=\mathsf A\varphi_\lambda(\mathsf A)-\mathsf A_T\varphi_\lambda(\mathsf A)$ and $0\le z\varphi_\lambda(z)\le1$. The second asserted bound is the transpose of the first.
\end{proof}

With some modifications, the following lemma restates Lemmas 8, 9, and 10 in \cite{zhang2024phase}.

\begin{lemma}\label{lemma ge l l2}
We have
\begin{itemize}
    \item[(i)] $\left\| \Psi_{>\tilde{\ell}} \theta_{>\tilde{\ell}} \right\|_{2}^{2} = O_{d, \mathbb{P}}\left( n  \| \theta_{>\tilde{\ell}} \|_{2}^{2}\right)$.

    \item[(ii)] $\| \theta_{>\tilde{\ell}} \|_{2}^{2}  = \Theta_d \left(  d^{-(\tilde{\ell}+1)s} \right)$.

    \item[(iii)] $\left\| \Sigma_{\leq \tilde{\ell}}^{-t}  \theta_{\leq \tilde{\ell}} \right\|_{2}^{2}  = \Theta_{d} \left(  d^{(2t-\min\{s, 2t\})\tilde{\ell} } \right)$, $t \geq 1$.
\end{itemize}
\end{lemma}
\begin{proof}
(i) and (ii) are direct results of Lemmas S8 and S10 in \cite{zhang2024phase}. Now let's prove (iii). We have
\begin{equation}\label{eq proof le l-1}
        \left\| \Sigma_{\leq \tilde{\ell}}^{-t}  \theta_{\leq \tilde{\ell}} \right\|_{2}^{2} = 
        \sum\limits_{m = 0}^{\tilde{\ell}} \mu_{m}^{-2t} \sum\limits_{j = 1 }^{N(d, m)} (\theta_{m,j})^{2} 
        = \sum\limits_{m = 0}^{\tilde{\ell}} \mu_{m}^{s-2t} \mu_{m}^{-s}\sum\limits_{j = 1 }^{N(d, m)} (\theta_{m,j})^{2}.
    \end{equation}
On the one hand, when $0 \leq s \leq 2t$, (\ref{eq proof le l-1}) implies
\begin{equation*}
\begin{aligned}
    \left\| \Sigma_{\leq \tilde{\ell}}^{-t}  \theta_{\leq \tilde{\ell}} \right\|_{2}^{2} & \overset{\text{Assumption } 3}{=} 
    O_{d} \left(  \mu_{\tilde{\ell}}^{s-2t} \right) =
        O_{d} \left(  d^{(2t-s)\tilde{\ell} } \right)\\
        \left\| \Sigma_{\leq \tilde{\ell}}^{-t}  \theta_{\leq \tilde{\ell}} \right\|_{2}^{2} &\geq \mu_{\tilde{\ell}}^{s-2t} \mu_{\tilde{\ell}}^{-s}\sum\limits_{j = 1 }^{N(d, \tilde{\ell})} (\theta_{\tilde{\ell},j})^{2} \overset{\text{Assumption } 5}{=}
        \Omega_{d} \left( d^{(2t-s)\tilde{\ell} } \right).
\end{aligned}
    \end{equation*}
On the other hand, when $s > 2t$,
since $\mu_m = O_d(1)$, $m=0, 1, \cdots, \tilde{\ell}$, (\ref{eq proof le l-1}) implies
\begin{equation*}
\begin{aligned}
    \left\| \Sigma_{\leq \tilde{\ell}}^{-t}  \theta_{\leq \tilde{\ell}} \right\|_{2}^{2} &\overset{\text{Assumption } 3}{=} 
    O_{d} \left( 1 \right)\\
        \left\| \Sigma_{\leq \tilde{\ell}}^{-t}  \theta_{\leq \tilde{\ell}} \right\|_{2}^{2} &= 
        \Omega_d(1) \cdot \sum\limits_{m = 0}^{\tilde{\ell}}  \sum\limits_{j = 1 }^{N(d, m)} (\theta_{m,j})^{2}
        \overset{\text{Assumption } 5}{=}
        \Omega_d(1),
\end{aligned}
    \end{equation*}
and we obtain the desired results.
\end{proof}

The following two identities are based on the Sherman-Morrison-Woodbury formula (see, e.g., Proposition \ref{prop_smw_formula}).


\begin{lemma}\label{lemma trans in V2}
For any finite cutoff $k$ such that $\Sigma_{\leq k}$ and $M_{>k}$ are invertible, 
\begin{equation*}
    \begin{aligned}
        M^{-1} \Psi_{\leq k} \Sigma_{\leq k}^{1/2} 
        = &~ 
        M_{>k}^{-1} \Psi_{\leq k} \Sigma_{\leq k}^{1/2} \left( \mathrm{I}_{N_{k}} +  \Sigma_{\leq k}^{1/2}  \Psi_{\leq k}^{\top} M_{>k}^{-1} \Psi_{\leq k} \Sigma_{\leq k}^{1/2}\right)^{-1}, 
    \end{aligned}
\end{equation*}
and 
\begin{equation*}
    \begin{aligned}
    \Sigma_{\leq k} - \Sigma_{\leq k} \Psi^{\top}_{\leq k} M^{-1} \Psi_{\leq k}\Sigma_{\leq k}
        = &~
        \left( \Sigma_{\leq k}^{-1} + \Psi_{\leq k}^{\top} M_{>k}^{-1} \Psi_{\leq k}  \right)^{-1}.
    \end{aligned}
\end{equation*}
        
\end{lemma}
\begin{proof}
Recall that we have $M = \Psi_{\leq k} \Sigma_{\leq k} \Psi_{\leq k}^{\top} + M_{>k}$. 
\begin{itemize}
    \item Applying the Sherman-Morrison-Woodbury formula with $Z= \Psi_{\leq k} \Sigma_{\leq k}^{1/2}$ and $A=M_{>k}$, we have
    \begin{equation*}
    \begin{aligned}
        M^{-1} \Psi_{\leq k} \Sigma_{\leq k}^{1/2} 
        = &~  (ZZ^{\top} + A)^{-1} Z
        = A^{-1}Z - A^{-1} Z(\mathrm{I}_{N_{k}} + Z^{\top} A^{-1} Z)^{-1} Z^{\top} A^{-1}Z\\
        =&~
        A^{-1}Z \left(\mathrm{I}_{N_{k}} - (\mathrm{I}_{N_{k}} + Z^{\top} A^{-1} Z)^{-1} \left(\mathrm{I}_{N_{k}} + Z^{\top} A^{-1} Z - \mathrm{I}_{N_{k}}\right)\right)
        \\
        =&~
        A^{-1} Z \left( \mathrm{I}_{N_{k}} +  Z^{\top} A^{-1} Z\right)^{-1}
        \\
        =&~
        M_{>k}^{-1} \Psi_{\leq k} \Sigma_{\leq k}^{1/2} \left( \mathrm{I}_{N_{k}} +  \Sigma_{\leq k}^{1/2}  \Psi_{\leq k}^{\top} M_{>k}^{-1} \Psi_{\leq k} \Sigma_{\leq k}^{1/2}\right)^{-1}.
    \end{aligned}
\end{equation*}

\item Applying the Sherman-Morrison-Woodbury formula with $Z= \Psi_{\leq k}^{\top} M_{> k}^{-1/2}$ and $A=\Sigma_{\leq k}^{-1}$, we have
\begin{equation*}
    \begin{aligned}
        \left( \Sigma_{\leq k}^{-1} + \Psi_{\leq k}^{\top} M_{>k}^{-1} \Psi_{\leq k}  \right)^{-1}
        = &~
        (ZZ^{\top} + A)^{-1}\\
        = &~
        A^{-1} - A^{-1} Z(\mathrm{I}_{n} + Z^{\top} A^{-1} Z)^{-1} Z^{\top} A^{-1}\\
        = &~
        \Sigma_{\leq k} - \Sigma_{\leq k} \Psi^{\top}_{\leq k} M^{-1} \Psi_{\leq k}\Sigma_{\leq k}.
    \end{aligned}
\end{equation*}
\end{itemize}
\end{proof}

\section{Proof of Theorem 4.1
}\label{append_kernel_app_general}

In this section, we will generalize the results for spherical inner-product kernels defined on the sphere $\mathbb{S}^{d-1}$ to kernels on general domains and prove Theorem 4.1.

Throughout this section, we assume that all conditions and assumptions in Theorem 4.1 hold.

We use the same notations as in Appendix \ref{append_notation}, except that each symbol is augmented with a tilde, e.g., $\Psi \to \widetilde{\Psi}$ and $\Sigma \to \widetilde{\Sigma}$.

\subsection{Kernel approximation results for general kernels }

The following lemma is a corollary of Theorem 6 in \cite{mei2022generalization}, which generalizes the results in Lemma \ref{lemma psi psi top} and Lemma \ref{lemma psi top psi}.

\begin{lemma}\label{lemma psi psi top_general}
   We have
    \begin{equation}\label{eq k ge l in lemma_general_1}
    \begin{aligned}
        \widetilde{\Psi}_{>\ell_{\gamma}} \widetilde{\Sigma}_{>\ell_{\gamma}} \widetilde{\Psi}_{>\ell_{\gamma}}^{\top} &= \kappa_1 \left(\mathrm{I}_{n} + \Delta_{1}\right), \quad \|\Delta_{1}\|_{\mathrm{op}} = o_{d, \mathbb{P}}(1),
    \end{aligned}
    \end{equation}
    and
    \begin{equation}\label{eq k ge l in lemma_general_2}
    \begin{aligned}
        \widetilde{\Psi}_{>\ell_{\gamma}} \widetilde{\Sigma}_{>\ell_{\gamma}}^{2} \widetilde{\Psi}_{>\ell_{\gamma}}^{\top} &= \kappa_2 \left(\mathrm{I}_{n} + \Delta_{2}\right), \quad \|\Delta_{2}\|_{\mathrm{op}} = o_{d, \mathbb{P}}(1);
    \end{aligned}
    \end{equation}
    with
    \begin{align*}
    \kappa_{1} = \Theta_{d}(1), \quad \kappa_{2} = \Theta_{d}\left( d^{-(\ell_{\gamma}+1)} \right);
\end{align*}
Moreover, denote $\widetilde{N}_{\ell_{\gamma}} = \sum_{k=0}^{\ell_{\gamma}}\widetilde{N}(d, k)$. Then we have
    \begin{equation}\label{eq k ge l in lemma_general_3}
         \left\| \frac{\widetilde{\Psi}_{\leq \ell_{\gamma}}^{\top} \widetilde{\Psi}_{\leq \ell_{\gamma}} }{n} - \mathrm{I}_{\widetilde{N}_{\ell_{\gamma}}} \right\|_{\mathrm{op}} = o_{d, \mathbb{P}}(1).
    \end{equation}
\end{lemma}
\begin{proof}
While Theorem 6 in \cite{mei2022generalization} is stated under the setting of random feature kernels, its proofs (mainly relying on Propositions 3 and 4 in \cite{mei2022generalization}) do not rely on any property specific to this class. Consequently, the results remain valid for general kernels, and we adopt them in this broader setting.

Let $m= \ell_{\gamma}$ and $L=2m+2$ and, for $h\in\{m+1,L\}$, define
\[
 T_q(h)=\sum_{k\ge h,j}\widetilde\lambda_{k,j}^{\,q}, \quad q=1,2,4.
\]
We first state the following conditions (i.e., Assumption 6 (a), Assumption 6 (b), and Assumption 7 (a) in \cite{mei2022generalization} with the following substitution of notations: $u(d) \to \sum_{k=0}^{2\ell_{\gamma} +1}\widetilde{N}(d, k)$, $N(d) \to n$, $M(d) \to \sum_{k=0}^{\ell_{\gamma}}\widetilde{N}(d, k)$, and $\{\lambda_{d, j}^2\} \to \{\widetilde{\lambda}_{k, j}\}$ or $\{\lambda_{d, j}^2\} \to \{\widetilde{\lambda}_{k, j}^2\}$):
\begin{itemize}
    \item For any integer $q \geq 1$, there exists a constant $C=C(q, \ell_{\gamma})$ such that
\begin{equation}\label{eqn_condi_general_1}
\|g\|_{L^{2q}} \leq C(q, \ell_{\gamma})\|g\|_{L^{2}}, \quad g \in \text{span}\left\{\widetilde{\psi}_{k, j}, k \leq 2\ell_{\gamma}+1\right\}.
\end{equation}

    \item There exists $\delta_1>0$, such that
 \begin{equation}\label{eqn_condi_general_2}
    n^{2+\delta_1} \leq \min\left\{\frac{T_{1}^2(L) }{T_2(L)}, \frac{T_{2}^2(L) }{T_4(L)}\right\}.
    \end{equation}

    \item There exists $\delta_2>0$, such that
\begin{equation}\label{eqn_condi_general_3}
    n^{1+\delta_2} \leq \min\left\{\frac{T_{1}(L/2)}{\widetilde{\lambda}_{L/2, 1}}, \frac{T_{2}(L/2)}{\widetilde{\lambda}_{L/2, 1}^2}\right\}..
    \end{equation}

\end{itemize}

Now let's verify these conditions. Assumption 8 ensures (\ref{eqn_condi_general_1}). 

For $q\in\{1,2,4\}$ and $h \in \{L, L/2\}$, spectral scaling, ordering, and the trace bound imply
\[
 T_q(h)\asymp d^{-h(q-1)}.
\]
Indeed, block $h$ proves the lower bound, and
$T_q(h)\le\widetilde\lambda_{h,1}^{q-1}T_1(h)$ proves the upper bound. For each $q=1,2$,
\[
 \frac{T_q(L)^2}{T_{2q}(L)}\asymp d^L,
 \qquad
 \frac{T_q(L/2)}{\widetilde\lambda_{L/2,1}^{\,q}}\asymp d^{L/2}.
\]
Since $m<\gamma<m+1$, choose positive constants
\[
 \delta_1<L/\gamma-2,\qquad
 \delta_2<(m+1)/\gamma-1,\qquad
 \delta_0<1-m/\gamma.
\]
Then \eqref{eqn_condi_general_2} and \eqref{eqn_condi_general_3} hold with powers $n^{2+\delta_1}$ and $n^{1+\delta_2}$. 
Also $M(d)\le n^{1-\delta_0}$ for large $d$.
Therefore, the conditions required to apply Theorem 6 in
\cite{mei2022generalization} to $\widetilde{\mathscr K}$ and 
\[
\widetilde{\mathscr K}^{(2)}(x,x')
:=
\sum_{k=0}^{\infty}
\sum_{j=1}^{\widetilde N(d,k)}
\widetilde\lambda_{k,j}^{2}
\widetilde\psi_{k,j}(x)\widetilde\psi_{k,j}(x'),
\]
are satisfied.

We first apply Theorem 6(a) in \cite{mei2022generalization} to the
kernel $\widetilde{\mathscr K}$, with
\[
M(d)=\sum_{k=0}^{\ell_\gamma}\widetilde N(d,k).
\]
It follows that
\begin{equation}\label{eq_general_theorem6_intermediate_1}
\widetilde{\Psi}_{>\ell_\gamma}
\widetilde{\Sigma}_{>\ell_\gamma}
\widetilde{\Psi}_{>\ell_\gamma}^{\top}
=
\kappa_1(\Lambda_1+\bar\Delta_1),
\end{equation}
where
\[
\Lambda_1
=
\operatorname{diag}\left\{
\frac{
\sum_{k\geq\ell_\gamma+1}
\sum_{j\leq\widetilde N(d,k)}
\widetilde\lambda_{k,j}\widetilde\psi_{k,j}(x_i)^2
}{\kappa_1}
\right\}_{i=1}^n,
\]
and, for some constant $c>0$,
\[
\E\|\bar\Delta_1\|_{\mathrm{op}}
=O_d(n^{-c}).
\]
In particular, by Markov's inequality,
\[
\|\bar\Delta_1\|_{\mathrm{op}}
=o_{d,\mathbb P}(1).
\]

On the other hand, Assumption 9 implies
\[
\|\Lambda_1-\mathrm{I}_n\|_{\mathrm{op}}
=
\max_{i\in[n]}
\left|
\frac{
\sum_{k\geq\ell_\gamma+1}
\sum_{j\leq\widetilde N(d,k)}
\widetilde\lambda_{k,j}\widetilde\psi_{k,j}(x_i)^2
}{\kappa_1}
-1
\right|
=o_{d,\mathbb P}(1).
\]
Hence, defining
\[
\Delta_1:=(\Lambda_1-\mathrm{I}_n)+\bar\Delta_1,
\]
we obtain
\[
\widetilde{\Psi}_{>\ell_\gamma}
\widetilde{\Sigma}_{>\ell_\gamma}
\widetilde{\Psi}_{>\ell_\gamma}^{\top}
=
\kappa_1(\mathrm{I}_n+\Delta_1),
\qquad
\|\Delta_1\|_{\mathrm{op}}
=o_{d,\mathbb P}(1),
\]
which proves \eqref{eq k ge l in lemma_general_1}.

Furthermore, Theorem 6(b) in \cite{mei2022generalization}, together
with \eqref{eqn_condi_general_1}, yields
\[
\left\|
\frac{\widetilde{\Psi}_{\leq\ell_\gamma}^{\top}
\widetilde{\Psi}_{\leq\ell_\gamma}}{n}
-\mathrm{I}_{\widetilde N_{\ell_\gamma}}
\right\|_{\mathrm{op}}
=o_{d,\mathbb P}(1),
\]
which proves \eqref{eq k ge l in lemma_general_3}.

We now apply Theorem 6(a) in \cite{mei2022generalization} to $\widetilde{\mathscr K}^{(2)}$.
The preceding estimates verify the corresponding eigenvalue
conditions for this kernel. Therefore,
\begin{equation}\label{eq_general_theorem6_intermediate_2}
\widetilde{\Psi}_{>\ell_\gamma}
\widetilde{\Sigma}_{>\ell_\gamma}^{2}
\widetilde{\Psi}_{>\ell_\gamma}^{\top}
=
\kappa_2(\Lambda_2+\bar\Delta_2),
\end{equation}
where
\[
\Lambda_2
=
\operatorname{diag}\left\{
\frac{
\sum_{k\geq\ell_\gamma+1}
\sum_{j\leq\widetilde N(d,k)}
\widetilde\lambda_{k,j}^{2}\widetilde\psi_{k,j}(x_i)^2
}{\kappa_2}
\right\}_{i=1}^n,
\]
and
\[
\|\bar\Delta_2\|_{\mathrm{op}}
=o_{d,\mathbb P}(1).
\]
Here the latter conclusion follows from the expectation bound in
Theorem 6(a) and Markov's inequality.

By (24),
\[
\|\Lambda_2-\mathrm{I}_n\|_{\mathrm{op}}
=o_{d,\mathbb P}(1).
\]
Defining
\[
\Delta_2:=(\Lambda_2-\mathrm{I}_n)+\bar\Delta_2,
\]
we conclude that
\[
\widetilde{\Psi}_{>\ell_\gamma}
\widetilde{\Sigma}_{>\ell_\gamma}^{2}
\widetilde{\Psi}_{>\ell_\gamma}^{\top}
=
\kappa_2(\mathrm{I}_n+\Delta_2),
\qquad
\|\Delta_2\|_{\mathrm{op}}
=o_{d,\mathbb P}(1).
\]
This proves \eqref{eq k ge l in lemma_general_2}.
\end{proof}

The next lemma generalizes the results in Lemma \ref{lemma ge l l2}. We can obtain its proof by modifying the corresponding proofs of Lemmas S8 and S10 in \cite{zhang2024phase} and Lemma \ref{lemma ge l l2}, and hence we omit the proof.

\begin{lemma}\label{lemma ge l l2_general}
We have
\begin{itemize}
    \item[(i)] $\left\| \widetilde{\Psi}_{>\widetilde{\ell}} \theta_{>\widetilde{\ell}} \right\|_{2}^{2} = O_{d, \mathbb{P}}\left( n  \| \theta_{>\widetilde{\ell}} \|_{2}^{2}\right)$.

    \item[(ii)] $\| \theta_{>\widetilde{\ell}} \|_{2}^{2}  = \Theta_d \left(  d^{-(\widetilde{\ell}+1)s} \right)$.

    \item[(iii)] $\left\| \widetilde{\Sigma}_{\leq \widetilde{\ell}}^{-1}  \theta_{\leq \widetilde{\ell}} \right\|_{2}^{2}  = \Theta_{d} \left(  d^{(2-\min\{s, 2\})\widetilde{\ell} } \right)$.
\end{itemize}
\end{lemma}

\subsection{Proof of Theorem 4.1}
\begin{proof}
Lemma \ref{lemma psi psi top_general} gives, on an event whose probability tends to one,
\[
 \widetilde K\succeq\widetilde K_{>\ell_{\gamma}}
 =\kappa_1(\mathrm{I}_n+\Delta_1)\succeq(\kappa_1/2)\mathrm{I}_n.
\]
Thus the endpoint is invertible and interpolates on this event. The proof in this section applies there at $\lambda=0$, with $\lambda+n^{-1}=n^{-1}$. On the complement retain the pseudoinverse definition. A statement in probability is unaffected by that complement.

    Adopt the notations in (\ref{eqn_order_in_bias_and_variance}). 
 Also, let $m= \ell_{\gamma}<\gamma$, $a=\min(m,\lfloor u\rfloor)$ for finite $u$, and $a=m$ for $u=\infty$. Set $t=1+n\lambda$, $\widetilde{M}=\widetilde{K}+n\lambda \mathrm{I}_n$, $\widetilde{D}_k = \widetilde{N}(d, k)$, and $\widetilde{N}_{a} = \sum_{k=0}^{a}\widetilde{D}_k$.
    We divide the proof into four steps.

{\bf Step (I). } Lemma \ref{lemma psi psi top_general} gives, with probability tending to one,
\begin{align}
\widetilde K_{>m}
&=\kappa_1(\mathrm{I}_n+\Delta_1),
& \kappa_1&\asymp 1,
& \|\Delta_1\|_{\op}&=o_{\Pp}(1),
\label{eq:r07-tail1}\\
\widetilde Q_{>m}
:=\widetilde\Psi_{>m}\widetilde\Sigma_{>m}^2\widetilde\Psi_{>m}^{\top}
&=\kappa_2(\mathrm{I}_n+\Delta_2),
& \kappa_2&\asymp d^{-(m+1)},
& \|\Delta_2\|_{\op}&=o_{\Pp}(1),
\label{eq:r07-tail2}\\
\widetilde S_{\le m}
:=n^{-1}\widetilde\Psi_{\le m}^{\top}\widetilde\Psi_{\le m}
&=\mathrm{I}+E_m,
&&& \|E_m\|_{\op}&=o_{\Pp}(1).
\label{eq:r07-lowgram}
\end{align}
In particular, every principal block $\widetilde S_B$ of $\widetilde S_{\le m}$ satisfies
\[
  \|\widetilde S_B-\mathrm{I}\|_{\op}=o_{\Pp}(1).
\]

We first record the analogue of the resolvent-scale estimates used in Appendix \ref{append_variance}.  From \eqref{eq:r07-tail1},
\[
  \widetilde M\succeq \widetilde K_{>m}+n\lambda \mathrm{I}_n\succeq c t \mathrm{I}_n
\]
with probability tending to one.  Also,
\[
  c t \mathrm{I}_n\preceq \widetilde M_{>m}\preceq C t \mathrm{I}_n.
\]
If $a<m$, then $a=\lfloor u\rfloor$ and $u<a+1\le m<\gamma$.  By \eqref{eq:r07-lowgram},
\[
  \|\widetilde\Psi_{a+1:m}\|_{\op}^2=O_{\Pp}(n),
\]
and Assumption 7 yields
\begin{align*}
  \|\widetilde K_{>a}\|_{\op}
  &\le
  \|\widetilde\Psi_{a+1:m}\|_{\op}^2
  \|\widetilde\Sigma_{a+1:m}\|_{\op}
  +\|\widetilde K_{>m}\|_{\op}\\
  &=O_{\Pp}(n d^{-(a+1)}+1).
\end{align*}
Since $n\lambda\asymp d^{\gamma-u}\to\infty$,
\[
  \frac{\|\widetilde K_{>a}\|_{\op}}{n\lambda}
  =O_{\Pp}\!\left(d^{u-a-1}+d^{u-\gamma}\right)
  =o_{\Pp}(1),
\]
so that
\[
  \widetilde M_{>a}\asymp_{\Pp} t\mathrm{I}_n.
\]
When $a=m$, the same conclusion follows directly from \eqref{eq:r07-tail1}; importantly, no boundedness of $t$ is needed.

Define
\begin{equation}
  \widetilde A_a
  :=\frac1n\widetilde\Sigma_{\le a}^{-1}
   +\frac1n\widetilde\Psi_{\le a}^{\top}
      \widetilde M_{>a}^{-1}\widetilde\Psi_{\le a}.
  \label{eq:r07-Aa}
\end{equation}
Since $\widetilde M_{>a}\asymp_{\Pp}t\mathrm{I}_n$ and
$n^{-1}\widetilde\Psi_{\le a}^{\top}\widetilde\Psi_{\le a}=\mathrm{I}+o_{\Pp}(1)$,
\[
 \frac1n\widetilde\Psi_{\le a}^{\top}
      \widetilde M_{>a}^{-1}\widetilde\Psi_{\le a}
 \asymp_{\Pp} t^{-1}\mathrm{I}.
\]
The first term in \eqref{eq:r07-Aa} is positive semidefinite and, by Assumption 7,
\begin{align*}
 t\left\|\frac1n\widetilde\Sigma_{\le a}^{-1}\right\|_{\op}
 &\lesssim \frac{t d^a}{n}
 =d^{a-\gamma}+\lambda d^a
 =O(1),
\end{align*}
where at interpolation the second term is interpreted as zero.  Therefore
\begin{equation}
  \widetilde A_a\asymp_{\Pp}t^{-1}\mathrm{I},
  \qquad
  \widetilde A_a^{-1}\asymp_{\Pp}t\mathrm{I}.
  \label{eq:r07-Aainv}
\end{equation}
Equivalently, with
\[
  \widetilde C_a
  :=\left(
      \widetilde\Sigma_{\le a}^{-1}
      +\widetilde\Psi_{\le a}^{\top}
       \widetilde M_{>a}^{-1}\widetilde\Psi_{\le a}
      \right)^{-1},
\]
we have
\begin{equation}
  \widetilde C_a=\frac1n\widetilde A_a^{-1}
  \asymp_{\Pp}\frac{t}{n}\mathrm{I}
  \asymp_{\Pp}(\lambda+n^{-1})\mathrm{I}.
  \label{eq:r07-Ca}
\end{equation}

We next replace the exact spherical trace identities.  Let $B$ be any collection of blocks contained in $\{0,\ldots,m\}$ and write
\[
  \widetilde S_B:=\frac1n\widetilde\Psi_B^{\top}\widetilde\Psi_B.
\]
For every positive semidefinite matrix $W$ of the corresponding dimension,
\begin{equation}
  \left|\tr\!\left(W\widetilde S_B\right)-\tr(W)\right|
  \le \|\widetilde S_B-\mathrm{I}\|_{\op}\,\tr(W)
  =o_{\Pp}(\tr W).
  \label{eq:r07-weightedtrace}
\end{equation}
Indeed,
$-\delta \mathrm{I}\preceq\widetilde S_B-\mathrm{I}\preceq\delta \mathrm{I}$
with $\delta=\|\widetilde S_B-\mathrm{I}\|_{\op}$; pre- and post-multiplication by $W^{1/2}$ and taking traces gives \eqref{eq:r07-weightedtrace}.
Consequently, for every fixed low block $k\le m$,
\begin{equation}
  \frac1n\tr\!\left(
  \widetilde\Psi_k\widetilde\Sigma_k^2\widetilde\Psi_k^{\top}
  \right)
  =(1+o_{\Pp}(1))\tr(\widetilde\Sigma_k^2)
  \asymp_{\Pp}d^{-k}.
  \label{eq:r07-blocktrace}
\end{equation}
More generally, if $a<m$,
\begin{equation}
  \frac1n\tr\!\left(
  \widetilde\Psi_{a+1:m}
  \widetilde\Sigma_{a+1:m}^2
  \widetilde\Psi_{a+1:m}^{\top}
  \right)
  \asymp_{\Pp}d^{-(a+1)}.
  \label{eq:r07-midtrace}
\end{equation}
Here Assumption 7 gives
$\tr(\widetilde\Sigma_k^2)\asymp d^k d^{-2k}=d^{-k}$,
and the first block $k=a+1$ determines the order of the finite sum.

The total trace is also controlled without using the spherical identity
$K(x_i,x_i)=\Phi(1)$.  By \eqref{eq:r07-lowgram}, \eqref{eq:r07-tail1}, and Assumption 7,
\begin{align}
 \frac1n\tr(\widetilde M)
 &=\tr\!\left(\widetilde\Sigma_{\le m}\widetilde S_{\le m}\right)
   +\frac1n\tr(\widetilde K_{>m})+n\lambda\notag\\
 &=\tr(\widetilde\Sigma_{\le m})+\kappa_1+n\lambda+o_{\Pp}(1)\notag\\
 &=\tr(\widetilde T)+n\lambda+o_{\Pp}(1)
 =O_{\Pp}(t).
 \label{eq:r07-totaltrace}
\end{align}
Only the final $O_{\Pp}(t)$ bound is needed for the variance proof.  In fact, this coarser conclusion also follows from Markov's inequality and the uniform trace bound in Assumption 7.

{\bf Step (II). } We next show that the results for the variance term in Appendix \ref{append_variance} can also be extended to the case where kernels are defined on general domains.

For KRR,
\[
  \widetilde\varphi_{\lambda,K}
  =(\widetilde K+n\lambda \mathrm{I}_n)^{-1}=\widetilde M^{-1}.
\]
Decompose the scalar variance contribution, before multiplication by the fixed factor $\sigma^2$, as
\[
 \mathcal V_1+\mathcal V_2+\mathbf 1\{a<m\}\mathcal V_3,
\]
where
\begin{align*}
 \mathcal V_1
 &:=\tr\!\left(
 \widetilde M^{-1}\widetilde Q_{>m}\widetilde M^{-1}
 \right),\\
 \mathcal V_2
 &:=\tr\!\left(
 \widetilde\Sigma_{\le a}
 \widetilde\Psi_{\le a}^{\top}\widetilde M^{-2}
 \widetilde\Psi_{\le a}\widetilde\Sigma_{\le a}
 \right),\\
 \mathcal V_3
 &:=\tr\!\left(
 \widetilde M^{-1}
 \widetilde\Psi_{a+1:m}\widetilde\Sigma_{a+1:m}^2
 \widetilde\Psi_{a+1:m}^{\top}
 \widetilde M^{-1}
 \right).
\end{align*}

\paragraph{Tail variance.}
From \eqref{eq:r07-tail2},
\[
  \mathcal V_1
  =\kappa_2\tr(\widetilde M^{-2})(1+o_{\Pp}(1)).
\]
The lower spectral bound $\widetilde M\succeq ct\mathrm{I}_n$ gives
\[
  \tr(\widetilde M^{-2})\le Cn/t^2.
\]
Conversely, Jensen's inequality and \eqref{eq:r07-totaltrace} give
\[
  \tr(\widetilde M^{-2})
  \ge n\left\{\frac{\tr(\widetilde M)}{n}\right\}^{-2}
  =\Omega_{\Pp}(n/t^2).
\]
Therefore
\begin{equation}
  \mathcal V_1
  \asymp_{\Pp}\frac{n}{d^{m+1}t^2}.
  \label{eq:r07-V1}
\end{equation}

\paragraph{Low-frequency variance.}
The rectangular Woodbury identity gives
\begin{equation}
  \widetilde M^{-1}\widetilde\Psi_{\le a}\widetilde\Sigma_{\le a}
  =\widetilde M_{>a}^{-1}\widetilde\Psi_{\le a}\widetilde C_a.
  \label{eq:r07-woodbury}
\end{equation}
Since $\widetilde M_{>a}\asymp_{\Pp}t\mathrm{I}_n$ and
$n^{-1}\widetilde\Psi_{\le a}^{\top}\widetilde\Psi_{\le a}=\mathrm{I}+o_{\Pp}(1)$,
\begin{equation}
  \widetilde H_a
  :=\widetilde\Psi_{\le a}^{\top}
    \widetilde M_{>a}^{-2}\widetilde\Psi_{\le a}
  \asymp_{\Pp}\frac{n}{t^2}\mathrm{I}.
  \label{eq:r07-Ha}
\end{equation}
Combining \eqref{eq:r07-Ca}--\eqref{eq:r07-Ha},
\begin{align}
 \mathcal V_2
 &=\tr\!\left(\widetilde C_a\widetilde H_a\widetilde C_a\right)
 \asymp_{\Pp}
 \left(\frac{t}{n}\right)^2\frac{n}{t^2}\widetilde N_a
 \asymp\frac{d^a}{n}.
 \label{eq:r07-V2}
\end{align}

\paragraph{Intermediate variance.}
Assume $a<m$.  The upper bound follows from
$\|\widetilde M^{-1}\|_{\op}=O_{\Pp}(t^{-1})$
and \eqref{eq:r07-midtrace}:
\begin{equation}
  \mathcal V_3
  =O_{\Pp}\!\left(\frac{n}{d^{a+1}t^2}\right).
  \label{eq:r07-V3upper}
\end{equation}
For the lower bound retain only block $k=a+1$.  Put
\[
  \widetilde D_k:=\widetilde N(d,k),
  \qquad
  s_1\ge\cdots\ge s_n\ge0
\]
for the eigenvalues of
$n^{-1}\widetilde\Psi_k\widetilde\Psi_k^{\top}$.
Because $k\le m<\gamma$, \eqref{eq:r07-lowgram} implies
$s_{\widetilde D_k}\ge 1/2$
with probability tending to one.  Moreover,
$\operatorname{rank}(\widetilde K_{\le a})\le \widetilde N_a$
and
$\widetilde M_{>a}\preceq Ct\mathrm{I}_n$
with probability tending to one, so the min--max principle yields
\[
  \lambda_i(\widetilde M)\le Ct,
  \qquad i>\widetilde N_a.
\]
Since
\[
  \widetilde\Psi_k\widetilde\Sigma_k^2\widetilde\Psi_k^{\top}
  \succeq
  \left(\min_j\widetilde\lambda_{k,j}\right)^2
  \widetilde\Psi_k\widetilde\Psi_k^{\top},
\]
the trace rearrangement inequality gives
\begin{align*}
  \mathcal V_3
  &\ge
  c n d^{-2k}
  \sum_{i=\widetilde N_a+1}^{\widetilde D_k}
  \frac{s_i}{\lambda_i(\widetilde M)^2}\\
  &\ge
  c\,\frac{n d^{-2k}}{t^2}
  (\widetilde D_k-\widetilde N_a)
  \asymp
  \frac{n}{d^{a+1}t^2},
\end{align*}
where $\widetilde D_k\asymp d^{a+1}$ and
$\widetilde N_a\asymp d^a$.  Hence
\begin{equation}
  \mathcal V_3
  \asymp_{\Pp}\frac{n}{d^{a+1}t^2}.
  \label{eq:r07-V3}
\end{equation}

Combining \eqref{eq:r07-V1}, \eqref{eq:r07-V2}, and \eqref{eq:r07-V3}, and recalling that $a\le m$,
\begin{align*}
 \operatorname{var}(\widehat f_\lambda)
 &=\sigma^2\left(\mathcal V_1+\mathcal V_2+
    \mathbf 1\{a<m\}\mathcal V_3\right)\\
 &\asymp_{\Pp}
 \frac{n}{d^{a+1}(1+n\lambda)^2}
 +\frac{d^a}{n}\\
 &\asymp_{\Pp}
 d^{\gamma-a-1-2\max\{\gamma-u,0\}}
 +d^{a-\gamma}.
\end{align*}
Thus the variance rate in Theorem 4.1 is unchanged.  The role of
\eqref{eq:r07-weightedtrace}--\eqref{eq:r07-totaltrace} is precisely to replace the spherical addition-formula trace identities that are no longer exact on a general domain.

\

{\bf Step (III). } We then bound the bias term by modifying the proof in Appendix \ref{append_bias}.

From (\ref{eq derivation of bias_spe}) we have
\begin{equation*}
\begin{aligned}
    \mathrm{bias}^{2}(\hat{f}_{\lambda})  =  &~
    \|B_1+B_2\|_{2}^{2} + \|B_3-B_4-B_5\|_{2}^{2}.
\end{aligned} 
\end{equation*}

\begin{itemize}
    \item In the proof of Theorem \ref{thm_bias_spe}, since results for $B_2, B_3, B_4$ and $B_5$ depend on the results established in Appendix \ref{append_kernel_app_sphere}, Step (I) directly implies that (see also (\ref{eqn_learning_curve_pre_results_1})):
    \begin{equation*}
    \begin{aligned}
                \|B_2\|_{2}^{2} =&~ 
            O_{d, \mathbb{P}}\left(  d^{b_1}\right)\\
            \|B_3-B_5\|_{2}^{2} 
        =&~
        \Theta_{d, \mathbb{P}} \left( d^{b_1}\right)\\
        \|B_4\|_{2}^{2} =&~ o_{d, \mathbb{P}} \left(d^{b_2(u, 1)}+d^{b_3}\right).
    \end{aligned}
\end{equation*}

    \item We then bound the term $\|B_1\|_2^2$. Notice that for kernel ridge regression, we have 
    \begin{equation}\label{eqn_inter_general_1}
        \varphi_{\lambda, \widetilde{K}}:=n^{-1}\reg^{\krr}(\widetilde{K}/n) = (\widetilde{K} + n\lambda \mathrm{I}_n)^{-1} = \widetilde{M}^{-1}.
    \end{equation}
    Moreover, from \eqref{eq:r07-Ca} and the generalization of Lemma \ref{lemma trans in V2}, we have
    \begin{align}
            \widetilde{\Sigma}_{\leq \widetilde{\ell}} - \widetilde{\Sigma}_{\leq \widetilde{\ell}} \widetilde{\Psi}^{\top}_{\leq \widetilde{\ell}} \widetilde{M}^{-1} \widetilde{\Psi}_{\leq \widetilde{\ell}} \widetilde{\Sigma}_{\leq \widetilde{\ell}}
        =
            \left( \widetilde{\Sigma}_{\leq \widetilde{\ell}}^{-1} + \widetilde{\Psi}_{\leq \widetilde{\ell}}^{\top} \widetilde{M}_{>\widetilde{\ell}}^{-1} \widetilde{\Psi}_{\leq \widetilde{\ell}}  \right)^{-1}, \label{eqn_inter_general_2} \\[6pt]
         \lambda_{\text{max}} \left(  \left( \widetilde{\Sigma}_{\leq \widetilde{\ell}}^{-1} + \widetilde{\Psi}_{\leq \widetilde{\ell}}^{\top} \widetilde{M}_{>\widetilde{\ell}}^{-1} \widetilde{\Psi}_{\leq \widetilde{\ell}}  \right)^{-1} \right)
        = \Theta_{d, \mathbb{P}}\left(\lambda + \frac{1}{n}\right),
        \label{eqn_inter_general_3} \\[6pt]
        \lambda_{\text{min}} \left(  \left( \widetilde{\Sigma}_{\leq \widetilde{\ell}}^{-1} + \widetilde{\Psi}_{\leq \widetilde{\ell}}^{\top} \widetilde{M}_{>\widetilde{\ell}}^{-1} \widetilde{\Psi}_{\leq \widetilde{\ell}}  \right)^{-1} \right) = \Theta_{d, \mathbb{P}}\left(\lambda + \frac{1}{n}\right). \label{eqn_inter_general_4} 
    \end{align}
    Therefore, for $B_1$, we have
    \begin{equation*}
        \begin{aligned}
            \|B_1\|_{2}^{2} =&~ \left\| \left( \mathrm{I}_{\widetilde{N}_{\widetilde{\ell}}} - \widetilde{\Sigma}_{\leq \widetilde{\ell}} \widetilde{\Psi}^{\top}_{\leq \widetilde{\ell}} \varphi_{\lambda, \widetilde{K}} \widetilde{\Psi}_{\leq \widetilde{\ell}}\right)    \theta_{\leq \widetilde{\ell}} \right\|_{2}^{2}\\
           (\because (\ref{eqn_inter_general_1})) \quad {=}&~ \left\| \left( \widetilde{\Sigma}_{\leq \widetilde{\ell}} - \widetilde{\Sigma}_{\leq \widetilde{\ell}} \widetilde{\Psi}^{\top}_{\leq \widetilde{\ell}} \widetilde{M}^{-1} \widetilde{\Psi}_{\leq \widetilde{\ell}} \widetilde{\Sigma}_{\leq \widetilde{\ell}} \right) \widetilde{\Sigma}_{\leq \widetilde{\ell}}^{-1}   \theta_{\leq \widetilde{\ell}} \right\|_{2}^{2}\\
           (\because (\ref{eqn_inter_general_2})) \quad {=}&~ \left\| \left( \widetilde{\Sigma}_{\leq \widetilde{\ell}}^{-1} + \widetilde{\Psi}_{\leq \widetilde{\ell}}^{\top} \widetilde{M}_{>\widetilde{\ell}}^{-1} \widetilde{\Psi}_{\leq \widetilde{\ell}}  \right)^{-1}
           \widetilde{\Sigma}_{\leq \widetilde{\ell}}^{-1}   \theta_{\leq \widetilde{\ell}} \right\|_{2}^{2}\\
           (\because (\ref{eqn_inter_general_3}) \text{ and } (\ref{eqn_inter_general_4})) \quad {=}&~
           \Theta_{d, \mathbb{P}}(d^{-2u} + d^{-2\gamma}) \left\| 
           \widetilde{\Sigma}_{\leq \widetilde{\ell}}^{-1}   \theta_{\leq \widetilde{\ell}} \right\|_{2}^{2}\\
           (\because \text{ Lemma \ref{lemma ge l l2_general}}) \quad {=}&~ \Theta_{d, \mathbb{P}}\left((d^{-2u} + d^{-2\gamma}) d^{(2-\min\{s, 2\})\widetilde{\ell} }\right)\\
           {=}&~
           \Theta_{d, \mathbb{P}}\left(d^{b_2(u, 1)}+d^{b_3}\right).
        \end{aligned}
    \end{equation*}
\end{itemize}

{\bf Step (IV). } Finally we combine the estimates from Step (II) (variance) and Step (III) (bias decomposition) to obtain the learning curve.

From Proposition \ref{prop_e_2} we have $b_3 \leq v_2$. Therefore: 
\begin{itemize}
    \item Upper bound:
            \begin{align*}
&~ \calR_{d}(u;X)
\leq  \mathrm{var}\left( \hat{f}_{\lambda} \right) + 2\|B_1\|_{2}^{2} + 2\|B_2\|_{2}^{2} + 2\|B_3 - B_5\|_{2}^{2} + 2\|B_4\|_{2}^{2}\\
= &~
O_{d, \mathbb{P}} \left( d^{v_1} + d^{v_2}\right) + O_{d, \mathbb{P}}\left( d^{b_1} + d^{b_2(u, 1)}  + d^{b_3} \right) =  O_{d, \mathbb{P}} \left( d^{v_1} + d^{v_2} + d^{b_1} + d^{b_2(u, 1)}\right).
\end{align*}

    \item For the lower bound, we analyze three possible cases:
    \begin{itemize}
    
    \item Suppose $b_3 < b_2(u, 1)$ and $b_1 < b_2(u, 1)$. Then
\begin{align*}
  &~  \calR_{d}(u;X)
\geq 
\Omega_{d, \mathbb{P}} \left( d^{v_1} + d^{v_2}\right) + \frac{1}{2}\|B_1\|_{2}^{2} - \|B_2\|_{2}^{2} \\
=&~  \Omega_{d, \mathbb{P}}\left( d^{v_1} + d^{v_2} + d^{b_2(u, 1)}\right) = \Omega_{d, \mathbb{P}}\left( d^{v_1} + d^{v_2} + d^{b_1} + d^{b_2(u, 1)}\right).
\end{align*}

    \item Suppose $b_3 < b_2(u, 1)$ and $b_1 \geq b_2(u, 1)$. Then $\|B_4\|_{2}^{2} = o_{d, \Pp}(d^{b_1})$. Hence
    \begin{align*}
  &~  \calR_{d}(u;X)
\geq 
\Omega_{d, \mathbb{P}} \left( d^{v_1} + d^{v_2}\right) + \frac{1}{2}\|B_3 - B_5\|_{2}^{2} - \|B_4\|_{2}^{2} \\
=&~  \Omega_{d, \mathbb{P}}\left( d^{v_1} + d^{v_2} + d^{b_1}\right) = \Omega_{d, \mathbb{P}}\left( d^{v_1} + d^{v_2} + d^{b_1} + d^{b_2(u, 1)}\right).
\end{align*}

    \item Suppose $b_3 \geq b_2(u, 1)$. Then we further have $\max\{b_2(u, 1), b_3\} \leq v_2$, and hence
    \begin{align*}
  &~  \calR_{d}(u;X)
\geq 
\Omega_{d, \mathbb{P}} \left( d^{v_1} + d^{v_2}\right) + \frac{1}{2}\|B_3 - B_5\|_{2}^{2} - \|B_4\|_{2}^{2} \\
=&~  \Omega_{d, \mathbb{P}}\left( d^{v_1} + d^{v_2} + d^{b_1}\right) = \Omega_{d, \mathbb{P}}\left( d^{v_1} + d^{v_2} + d^{b_1} + d^{b_2(u, 1)}\right).
\end{align*}
\end{itemize}
\end{itemize}
In sum, we have
\begin{align*}
   \calR_{d}(u;X)
= 
\Theta_{d, \mathbb{P}}\left( d^{v_1} + d^{v_2} + d^{b_1} + d^{b_2(u, 1)}\right),
\end{align*}
and we obtain the desired results.
\end{proof}

\section{Proof of Theorem 4.3}\label{sec_proof_minimax_general}
\begin{proof}
\emph{Upper bound.} Let $v(x)\in\R^D$ list the orthonormal eigenfunctions through block $p$, where $D\asymp d^p$, and write
\[
f_\star(x)=v(x)^\top\theta+r(x).
\]
The spectral source condition gives
\[
 \norm{r}_2^2\le C d^{-(p+1)s},\qquad \norm{f_\star}_2^2\le C R.
\]
Also $p<\gamma$, so $n/D$ grows polynomially.

Let
\[
w(x)=\mathbf{1}_{\{\norm{v(x)}^2\le BD\}},
\]
where $B$ is a sufficiently large fixed constant, and set $G=\E[wvv^\top]$. For each unit vector $b$,
\begin{align*}
 b^\top(\mathrm{I}-G)b
 &=\E[(b^\top v)^2\mathbf{1}_{\{\norm{v}^2>BD\}}]\\
 &\le\bigl(\E(b^\top v)^4\bigr)^{1/2}
       \Pp(\norm{v}^2>BD)^{1/2}
 \le C B^{-1/2}.
\end{align*}
Here we used $\E\norm{v}^2=D$ and the assumed $L^4/L^2$ bound. Choosing $B$ sufficiently large gives
\[
G\succeq \frac12 \mathrm{I}.
\]

Put
\[
\widehat G=\frac1n\sum_{i=1}^n w_i v_i v_i^\top.
\]
Each summand has operator norm at most $BD$, and its matrix variance is bounded by $BD\mathrm{I}$. Matrix Bernstein therefore gives
\[
 \Pp \left(\norm{\widehat G-G}_\op>1/4\right)
 \le 2D\exp(-cn/D).
\]
Define the good event
\[
\mathcal E:=\left\{\norm{\widehat G-G}_\op\le\frac14\right\}.
\]
Since $G\succeq \mathrm{I}/2$, on $\mathcal E$ we have
\[
\lambda_{\min}(\widehat G)\ge\frac14,
\qquad
\norm{\widehat G^{-1}}_\op\le4.
\]
Define the estimator to be zero on $\mathcal E^c$, and on $\mathcal E$ define
\[
 \widehat\theta=\widehat G^{-1}\frac1n\sum_{i=1}^n w_i v_i y_i,
 \qquad
 \widehat f(x)=v(x)^\top\widehat\theta.
\]
Since $y_i=v_i^\top\theta+r_i+\epsilon_i$, on $\mathcal E$,
\[
 \widehat\theta-\theta
 =\widehat G^{-1}\left(
       \frac1n\sum_{i=1}^n w_i v_i r_i
       +\frac1n\sum_{i=1}^n w_i v_i\epsilon_i
   \right).
\]
Let
\[
 b_r:=\frac1n\sum_{i=1}^n w_i v_i r_i,
 \qquad
 b_\epsilon:=\frac1n\sum_{i=1}^n w_i v_i\epsilon_i.
\]
Independence and Cauchy--Schwarz yield
\[
 \E\norm{b_r}^2
 \le\norm{\E[wvr]}^2+\frac1n\E[w\norm{v}^2 r^2]
 \le\left(1+\frac{BD}{n}\right)\norm{r}_2^2.
\]
Indeed, for every unit vector $b$,
\[
 \left|b^\top\E[wvr]\right|
 \le \bigl(\E(b^\top v)^2\bigr)^{1/2}
      \bigl(\E r^2\bigr)^{1/2}
 =\norm{r}_2,
\]
where we used $0\le w\le1$ and the orthonormality of the coordinates of $v$. Hence
$\norm{\E[wvr]}\le\norm{r}_2$.

Conditional centering and the variance bound on the noise give
\[
 \E\norm{b_\epsilon}^2
 =\frac1{n^2}\sum_{i=1}^n
   \E\!\left[w_i^2\norm{v_i}^2
   \E(\epsilon_i^2\mid x_i)\right]
 \le\frac{\sigma^2D}{n}.
\]
On $\mathcal E$, orthonormality of $v$ gives
\[
 \norm{\widehat f-f_\star}_2^2
 =\norm{v^\top(\widehat\theta-\theta)-r}_2^2
 \le 2\norm{\widehat\theta-\theta}^2+2\norm{r}_2^2.
\]
Therefore,
\begin{align*}
 \E\!\left[
   \norm{\widehat f-f_\star}_2^2\mathbf{1}_{\mathcal E}
 \right]
 &\le
 2\E\!\left[
   \norm{\widehat\theta-\theta}^2\mathbf{1}_{\mathcal E}
 \right]
 +2\norm{r}_2^2\\
 &\le
 64\E\norm{b_r}^2
 +64\E\norm{b_\epsilon}^2
 +2\norm{r}_2^2\\
 &\le
 C\left(\frac{D}{n}+\norm{r}_2^2\right),
\end{align*}
where we also used $D/n=o(1)$.

On $\mathcal E^c$, the estimator is zero, so
\[
 \E\!\left[
   \norm{\widehat f-f_\star}_2^2\mathbf{1}_{\mathcal E^c}
 \right]
 \le CR\,\Pp(\mathcal E^c)
 \le 2CRD\exp(-cn/D).
\]
Since $n/D$ grows polynomially in $d$, this last term is negligible compared with any fixed negative power of $d$. Consequently,
\[
 \sup_{\rho_{f_\star}\in\mathcal P_s}
 \E\norm{\widehat f-f_\star}_2^2
 \le
 C\left(
   \frac{D}{n}+d^{-(p+1)s}
 \right)
 \asymp
 d^{p-\gamma}+d^{-(p+1)s}.
\]
This proves the upper bound uniformly over the source ball and all admissible noise distributions. 

\emph{Lower bound.} Work in the admissible Gaussian-noise submodel. On any fixed block $k$ of dimension $N(d,k)$, consider the sign family
\[
 f_\omega
 =a_k\sum_{j=1}^{N(d,k)}\omega_j\widetilde\psi_{k,j},
 \qquad
 a_k^2
 =c\min\left\{
   \frac{R\min_j\widetilde\lambda_{k,j}^{s}}{N(d,k)},
   \frac{\sigma^2}{n}
 \right\},
 \qquad
 \omega_j\in\{-1,1\}.
\]
For sufficiently small fixed $c$, all functions belong to the source ball. Moreover, the Kullback--Leibler divergence between two adjacent sign models is
\[
 \frac{n}{2\sigma^2}
 \norm{f_\omega-f_{\omega'}}_2^2
 =\frac{2na_k^2}{\sigma^2}
 \le 2c.
\]
Hence, testing each sign coordinate, or equivalently applying the elementary Assouad argument, yields a squared-risk lower bound of order
\[
 N(d,k)a_k^2
 \gtrsim
 \min\{d^{k-\gamma},d^{-ks}\}.
\]
For $k=p$, since $p(s+1)\le\gamma$,
\[
 \min\{d^{p-\gamma},d^{-ps}\}=d^{p-\gamma}.
\]
For $k=p+1$, since $\gamma<(p+1)(s+1)$,
\[
 \min\{d^{p+1-\gamma},d^{-(p+1)s}\}
 =d^{-(p+1)s}.
\]
Therefore,
\[
 \inf_{\widehat f}\sup_{\rho_{f_\star}\in\mathcal P_s}
 \E\norm{\widehat f-f_\star}_2^2
 \gtrsim
 \max\left\{
   d^{p-\gamma},d^{-(p+1)s}
 \right\}
 \gtrsim
 d^{p-\gamma}+d^{-(p+1)s}.
\]
Combining this with the upper bound proves the claim.

When $s=0$, the zero estimator gives a constant upper bound, while the same sign-family argument on block $\ell_\gamma+1$, whose dimension is much larger than $n$, gives a constant lower bound.
\end{proof}

\subsection{Application: minimax rate for spherical inner-product kernels}

The proof of Theorem 4.3 can be used to obtain the minimax rate of kernel regression for large-dimensional spherical inner-product kernels. For completeness, we provide a proof.

\begin{proposition}[Source-ball minimax rate on the sphere]
\label{prop:audit-spherical-minimax}
Suppose Assumptions 1 and 2 hold, let $s\geq0$ and $R>0$, and put
\[
 \mathcal F_s(R)=\{f\in[\mathcal H]^s:\|f\|_{[\mathcal H]^s}^2\leq R\}.
\]
The design is uniform on $\mathbb{S}^{d-1}$.  Let $\mathcal P_s(R,\sigma)$ be the
independent random-design regression models with target in $\mathcal F_s(R)$,
conditional mean-zero noise and conditional variance at most the fixed
$\sigma^2>0$.  This class includes independent $\calN(0,\sigma^2)$ noise.
For $s>0$, set $p=\lfloor\gamma/(s+1)\rfloor$.  Then
\[
 \inf_{\widehat f}\sup_{P\in\mathcal P_s(R,\sigma)}
 \E_P\|\widehat f-f\|_{L^2}^2
 \asymp d^{p-\gamma}+d^{-(p+1)s}.
\]
For $s=0$ the minimax risk is of constant order.  The same rates hold when
the noise class is restricted to conditional variance exactly $\sigma^2$.
No signal nondegeneracy assumption is imposed on the minimax class.
\end{proposition}
\begin{proof}
First let $s>0$.  Let $v(x)$ list the orthonormal spherical harmonics of
degrees at most $p$, and write $D=N_p\asymp d^p$ and
$f(x)=v(x)^\top\theta+r(x)$, where $r$ is orthogonal to these harmonics.
The addition formula gives $\|v(x)\|_2^2=D$ for every $x$.  By the source
condition and the spectral bounds,
\[
 \|f\|_{L^2}^2\leq C R,\qquad
 \|r\|_{L^2}^2\leq R\sup_{j>N_p}\lambda_j^s
 \leq C R d^{-(p+1)s}.
\]
Use the orthogonal-series estimator
\[
 \widehat\theta=\frac1n\sum_{i=1}^n y_i v(x_i),\qquad
 \widehat f(x)=v(x)^\top\widehat\theta.
\]
It is unbiased for $\theta$.  Independence, conditional centering and the
addition formula yield
\begin{align*}
 \E\|\widehat\theta-\theta\|_2^2
 &\leq\frac1n\E\{y^2\|v(x)\|_2^2\}\\
 &\leq \frac{D}{n}(\|f\|_{L^2}^2+\sigma^2)
 \leq C\frac{D}{n}.
\end{align*}
Orthogonality gives the required upper bound uniformly over the source ball
and the admissible noise distributions.

For the lower bound, restrict to independent Gaussian noise.  On a fixed
block $k$, of dimension $D_k=N(d,k)$ and eigenvalue $\mu_k$, consider
\[
 f_\omega=a_k\sum_{j=1}^{D_k}\omega_j\psi_{k,j},\qquad
 a_k^2=c\min\left\{\frac{R\mu_k^s}{D_k},\frac{\sigma^2}{n}\right\},
 \qquad\omega\in\{-1,1\}^{D_k},
\]
where $0<c\leq1/16$.  Each target is in $\mathcal F_s(R)$.  For neighboring
sign vectors the joint data distributions have Kullback--Leibler divergence
\[
 \frac{n}{2\sigma^2}\|f_\omega-f_{\omega'}\|_{L^2}^2
 =\frac{2n a_k^2}{\sigma^2}\leq2c.
\]
Their total variation distance is at most $\sqrt c\leq1/4$.  Given an
arbitrary $L^2$-valued estimator, estimate each sign by the sign of its
coefficient in the corresponding harmonic direction.  An incorrectly
estimated sign incurs squared error at least $a_k^2$.  Pairing neighboring
sign vectors shows that, under the uniform distribution on the sign family,
each coordinate has error probability at least
$(1-\sqrt c)/2\geq3/8$.  Consequently
\[
 \inf_{\widehat f}\sup_\omega
 \E_\omega\|\widehat f-f_\omega\|_{L^2}^2
 \geq \frac38D_k a_k^2
 \gtrsim \min\{d^{k-\gamma},d^{-ks}\}.
\]
An estimator not taking values in $L^2$ on a positive-probability event has
infinite risk there and cannot improve this bound.
Taking $k=p$ and $k=p+1$, respectively, gives $d^{p-\gamma}$ and
$d^{-(p+1)s}$, because
$p(s+1)\leq\gamma<(p+1)(s+1)$.  Their maximum is at least half their sum.
Both blocks are covered by Assumption 2 since $p+1\leq\lfloor\gamma\rfloor+1$.
When $s=0$, the zero estimator has risk at most $R$, whereas the same sign
family on block $k=\lfloor\gamma\rfloor+1$, whose dimension is much larger
than $n$, gives a positive constant lower bound.  The Gaussian lower bound
and the upper bound for all admissible noise also prove the assertion for
the fixed-variance subclass.
\end{proof}

\section{Proof of Other Results in the Main Text}\label{appendix_proof_other_results}

\subsection{Proof of Lemma 2.4}
\begin{proof}
Take $r=\ell_{\gamma}+2$. Then $\mu_r>0$ and $N(d,r)\asymp d^{\ell_{\gamma}+2}\gg n$. Choose $n$ linearly independent degree-$r$ spherical harmonics. Their evaluation determinant at $n$ points is not identically zero: otherwise their evaluation vectors would span a proper subspace, contradicting linear independence of the functions. This determinant is real analytic on the connected product of spheres and hence its zero set has surface measure zero. Consequently the full $n$-by-$N(d,r)$ harmonic evaluation matrix has row rank $n$ almost surely. Its positive contribution $\mu_r\Psi_r\Psi_r^\top$ makes the full Gram matrix positive definite.
\end{proof}

\subsection{Proof of Theorem 3.2}

\begin{proof}
{\bf Proof of (i). } It is a direct corollary of Theorem 3.1.

{\bf Proof of (ii). } Under $0<s\le\Gamma(\gamma)$, one has $p=\ell_{\gamma}$ and $\lfloor u\rfloor\ge \ell_{\gamma}$ for every $u\ge u'(\tau)$. For finite qualification, Proposition \ref{prop_e_3} removes the qualification term. For infinite qualification, the displayed selectors satisfy $u'(\infty)>\ell_{\gamma}$, so the possible integer-point term at $u=\ell_{\gamma}$ is excluded. The first variance exponent is no larger than its interpolation value, which the preceding inequalities show is dominated by the desired minimax rate. The second variance term and the tail bias provide that rate. This proves persistence for each fixed $u\ge u'$.

{\bf Proof of (iii). } For noninteger $\gamma=\ell_{\gamma}+r$ with $0<r<1$, the interpolation decay exponent is
\[
 \zeta_I=\min\{r,1-r,(\ell_{\gamma}+1)s\}.
\]
If $\ell_{\gamma}=0$ and $r\le1/2$, then $p=0$ and $\zeta_I=\min(r,s)=\zeta_\star$ for every $s\ge0$.

If $\ell_{\gamma}\ge1$ and $r\le1/2$, then $s\le r/\ell_{\gamma}$ implies $p=\ell_{\gamma}$, and both decay exponents equal $\min\{r,(\ell_{\gamma}+1)s\}$. Conversely, if $s>r/\ell_{\gamma}$, then $p\le \ell_{\gamma}-1$ and $\zeta_I=r$. One has $\gamma-p>r$. If $p=\ell_{\gamma}-1$, then $(p+1)s=\ell_{\gamma}s>r$; if $p\le \ell_{\gamma}-2$, then $\gamma<(p+1)(s+1)$ gives $(p+1)s>\gamma-(p+1)\ge r+1$. Thus $\zeta_\star>\zeta_I$.

If $r>1/2$, then $s\le(1-r)/(\ell_{\gamma}+1)$ implies $p=\ell_{\gamma}$ and both decay exponents equal $(\ell_{\gamma}+1)s$. If $s>(1-r)/(\ell_{\gamma}+1)$, then $\zeta_I=1-r$. For $p=\ell_{\gamma}$, both $r$ and $(\ell_{\gamma}+1)s$ exceed $1-r$. For $p<\ell_{\gamma}$, $\gamma-p\ge1+r$ and $(p+1)s>\gamma-(p+1)\ge r>1-r$. Again $\zeta_\star>\zeta_I$.

At an integer $\gamma=\ell_{\gamma}$, the interpolation variance is of constant order. For $s>0$, $p<\ell_{\gamma}$ and $\zeta_\star>0$; for $s=0$ both rates are constant. These cases give exactly the manuscript's piecewise $\Gamma(\gamma)$.

For breakdown, let $\Delta=\zeta_\star-\zeta_I>0$. For noninteger $\gamma$, choose
\[
 0<\varepsilon<\min\{r,\gamma-u',\Delta/4\}.
\]
The first variance exponent can decrease by at most $2\varepsilon$ on $(\gamma-\varepsilon,\gamma)$; the other relevant exponents stay fixed. Thus the risk divided by the minimax rate is $\Omega_{d,\Pp}(d^{\Delta/2})$ there. At integer $\gamma=m$, choose $\varepsilon<\min\{1,m-u',\zeta_\star/4\}$. The variance term $d^{-2(m-u)}$ gives the same conclusion with exponent at least $\zeta_\star/2$.
\end{proof}

\subsection{Proof of Proposition 3.4}
\begin{proof}
Using the strictly positive eigenvalues indexed by $J_d$, define
\[
\mathcal{N}_{1}(\lambda) = \sum_{j\in J_d} \frac{\lambda_j}{\lambda_j+\lambda},\qquad
\mathcal{N}_{2,\varphi}(\lambda) = \sum_{j\in J_d}[\lambda_j\reg(\lambda_j)]^2,\qquad
\mathcal{M}_{2,\varphi}(\lambda) = \sum_{j\in J_d}[\rem(\lambda_j)f_j]^2.
\]

We have the following results.

\begin{lemma}\label{lem:popall}
Under the spherical assumptions, for $0<u\le\gamma$ and $a=\lfloor u\rfloor$,
\[
 \mathcal N_1(\lambda)\asymp\lambda^{-1},\qquad
 \mathcal N_{2,\varphi}(\lambda)\asymp d^a+\lambda^{-2}d^{-(a+1)}.
\]
For finite qualification, with the additional finite-qualification residual condition in Appendix \ref{appendix_addtion_require_of_reg},
\[
 \mathcal M_{2,\varphi}(\lambda)
 \asymp d^{-(a+1)s}+\lambda^{2\tau}d^{(2\tau-\min(s,2\tau))a}.
\]
For infinite qualification use \eqref{eq:M2}; for $s=0$ the population squared bias is $\Theta(1)$.
\end{lemma}
\begin{proof}
The head eigenvalues are at least a constant times $\lambda$, so $\mu_k\varphi_\lambda(\mu_k)$ is bounded above and below by positive constants there. The tail eigenvalues are $o(\lambda)$. Moreover,
\[
 \sum_{k>a}\mu_kN(d,k)\asymp1,\qquad
 \sum_{k>a}\mu_k^2N(d,k)\asymp d^{-(a+1)}.
\]
Both lower bounds follow from block $a+1$; the upper bounds follow from the trace bound and the largest tail eigenvalue. These observations prove the effective-dimension assertions.

For finite qualification, the upper residual bound and the source condition give the displayed upper bound. The tail lower bound follows as in Lemma~\ref{lemma_bounds_on_quantities}. When $s\le2\tau$, the lower residual condition at scales $z>\lambda$, its continuity at $z=\lambda$, and residual monotonicity for $z<\lambda$ imply
$\psi_\lambda(\mu_a)\gtrsim(\lambda/\mu_a)^\tau$; the ratio is bounded above on this head block. Block-$a$ nondegeneracy then supplies the head lower bound. When $s>2\tau$, use $\psi_\lambda(K_{\max})\gtrsim\lambda^\tau$ and the unweighted head energy lower bound. The case $s=0$ follows directly from the nonvanishing tail energy and the bounded source norm.
\end{proof}
Equation (20) follows from Lemma \ref{lem:popall} and the exact risk identity (19), for
admissible sequences in $\Lambda$. To prove (21), put $v=\min\{u,\gamma\}$,
with $v=\gamma$ when $u=\infty$. Then
$\bar\lambda_d\asymp d^{-v}$, so (20) applies at exponent $v$.
For $u\le\gamma$, its powers agree with those in Theorem 3.1.

Suppose $u>\gamma$ and write $m=\lfloor\gamma\rfloor$ and $r=\gamma-m$.
If $0<r<1$, the three common powers are
\[
d^{r-1},\qquad d^{-r},\qquad d^{-(m+1)s}.
\]
For finite qualification, the exponent of the additional population
regularization term at $v=\gamma$ is
\[
\begin{aligned}
-2\tau\gamma+(2\tau-\min\{s,2\tau\})m
=-2\tau r-\min\{s,2\tau\}m
\le-2r\le-r.
\end{aligned}
\]
It is therefore dominated by the second variance term. The same holds
for the smaller regularization term at any $u>\gamma$. There is no
infinite-qualification integer-point term because $\gamma$ is noninteger.
If $r=0$, the second variance term has constant order and dominates both
the finite-qualification term and the possible additional population term
$d^{-ms}$ at $v=m$. Thus the two total risks have the same order in every
case. 
\end{proof}

\subsection{Examples of Assumption 7}\label{appendix_kernel_satisfy_assump_eigenval_general}

In this subsection, we first define the Product Laguerre kernels. Then, we will verify that Assumption 7 holds for the kernels mentioned below it.

\subsubsection{Definition of Product Laguerre kernels on Gamma measures}\label{appendix_Laguerre_polynomials}

In this subsection, we define a family of kernels based on Laguerre polynomials and Gamma measures and verify that these kernels satisfy Assumption 7 under the parameter conditions below.

Fix positive constants $\{\alpha_i >0\}_{i=1}^{\infty}$ and $r > 0$, such that we have that $\sup_{i \geq 1} \alpha_i < \infty$.
For dimensions $d \geq 2r$, let $\mathcal{X} = [0, \infty)^d$ be the positive orthant of $\mathbb{R}^d$, and let $\rho_{\mathcal{X}}$ be the product measure where each coordinate $x_{(i)}$ follows an independent Gamma distribution $\Gamma(\alpha_i + 1, 1)$, $i \leq d$. 
Denote $\alpha = (\alpha_1, \cdots, \alpha_d)^{\top}$.

For any $c>0$, let $I_{c}(z):=\sum_{m=0}^{\infty} [m!\Gamma(m+c+1)]^{-1}(z/2)^{2m+c}$ be the modified Bessel function of the first kind (at zero arguments in the Bessel formula use its continuous extension). Define
\begin{equation*}
    \begin{aligned}
        \mathscr{K}_{i}(z, z^{\prime}) =
\frac{\Gamma(\alpha_i+1)}{\left(1- r /d\right)}
\left(\frac{r z z^{\prime}}{d}\right)^{-\alpha_i/2}
\exp\left\{-\frac{ \frac{r}{d}(z + z^{\prime})}{1- r / d}\right\}
I_{\alpha_i}\left(\frac{2\sqrt{ \frac{r}{d}z z^{\prime}}}{1- r /d}\right), \quad z, z^{\prime} \in [0, \infty), \quad i \leq d.
    \end{aligned}
\end{equation*}
Finally, let's define a mapping from $\mathcal{X} \times \mathcal{X}$ to $\mathbb{R}$:
\begin{equation}\label{eqn_def_product_Laguerre_kernels}
    \widetilde{\mathscr{K}}_{\alpha, r}(x,x^\prime) := \prod_{i=1}^{d} \mathscr{K}_{i}(x_{(i)}, x_{(i)}^{\prime}), \quad x, x^{\prime} \in [0, \infty)^{d}.
\end{equation}

\subsubsection{Verifying Assumption 7}

In this part, we will show that Assumption 7 holds for several types of common kernels in large dimensions, that is, for these kernels, there exists an absolute constant $\widetilde{\mathfrak{C}}$ such that $\sum_{k=0}^{\infty} \sum_{j=1}^{\widetilde{N}(d, k)} \widetilde{\lambda}_{k, j}\leq \widetilde{\mathfrak{C}}$; moreover, for any $k  = 0, 1, \cdots, 2\ell_{\gamma} +2$, we show that
\begin{equation}\label{eqn_appendix_eigen_gap_condition}
    \widetilde{N}(d, k) = \Theta_d(d^{k}) \quad \text{ and } \quad 
        \widetilde{\lambda}_{k, j} = \Theta_d(d^{-k}), \quad \text{ for } 1\leq j \leq \widetilde{N}(d, k).
\end{equation}

\noindent {\bf Inner product kernels defined on $\mathbb{S}^{d-1}$ with the uniform distribution.} 
Assumption 2 assumes that $\sum\nolimits_{i=0}^{\infty} a_i \leq a_{-1} < \infty$, hence $\sup_{x \in \mathcal X} \mathscr{K}(x, x) 
    = \sup_{x \in \mathcal X} \Phi(\left\langle x, x \right\rangle) 
    \leq \sum\nolimits_{i=0}^{\infty} a_i \leq a_{-1}$. Therefore, the integral operator defined in (3) is bounded:
    $$
        \sum_{k=0}^{\infty} \sum_{j=1}^{\widetilde{N}(d, k)} \widetilde{\lambda}_{k, j}  = \mathrm{tr}(T) :=\int \mathscr{K}(x, x)  ~\mathsf{d} \rho_{\calX}(x) \leq a_{-1}.
    $$
    Moreover, using Proposition 2.1 with the coefficient positivity requirement strengthened to $a_k>0$ for $0\leq k\leq2\ell_{\gamma}+2$, we can show (\ref{eqn_appendix_eigen_gap_condition}); hence Assumption 7 holds for inner product kernels.

\noindent {\bf Inner product kernels defined on the hypercube $\{-1, 1\}^{d}$ with the uniform distribution.}
The Walsh functions $\chi_S(x)=\prod_{i\in S}x_i$ diagonalize this kernel. With independent Rademacher variables $R_i$, the eigenvalue associated with $|S|=k$ is
\[
 \mu_{d,k}=\E\left[\Phi\left(d^{-1}\sum_{i=1}^dR_i\right)
                           \prod_{i=1}^kR_i\right].
\]
For fixed $k$, repeated central differences give the exact identity
\[
 \mu_{d,k}=d^{-k}\E\Phi^{(k)}\left(
 \frac{W+U_1+\cdots+U_k}{d}\right),
\]
where $W$ is a sum of $d-k$ independent Rademacher variables and the $U_i$ are independent uniforms on $[-1,1]$, independent of $W$. Bounded continuity of $\Phi^{(k)}$ and convergence of the argument to zero imply
\[
 \mu_{d,k}=k!a_kd^{-k}(1+o(1)).
\]
The multiplicity is $\binom dk\asymp d^k$ for fixed $k$.

To verify ordering relative to all higher degrees, let $L=2\ell_{\gamma}+2$ and $P_{2L+1}$ be the Taylor polynomial of $\Phi$ at zero. Its remainder satisfies
\[
 |\Phi(z)-P_{2L+1}(z)|\le C|z|^{2L+2}.
\]
Consequently its Fourier coefficients are uniformly bounded by
\[
 C\E\left|d^{-1}\sum_iR_i\right|^{2L+2}=O(d^{-L-1}).
\]
For $k>2L+1$, the polynomial's coefficient is zero. For the finitely many $L<k\le2L+1$, the central-difference identity gives $\mu_{d,k}=O(d^{-k})$. Thus every higher-degree eigenvalue is $O(d^{-L-1})$, separated from the block-$L$ eigenvalue of order $d^{-L}$. The lower fixed blocks are similarly ordered for large $d$. Finally, $\tr T_d=K_d(x,x)=\Phi(1)$, so Assumption 7 holds.

\noindent {\bf Gaussian kernels defined on $\mathbb{R}^d$ with the standard Gaussian distribution.} 
Notice that from the definition of Gaussian kernels, we have $\sup_{x \in \mathcal X} \mathscr{K}(x, x) \leq \widetilde{\mathfrak{C}}$. Hence
    $$
        \sum_{k=0}^{\infty} \sum_{j=1}^{\widetilde{N}(d, k)} \widetilde{\lambda}_{k, j}  = \mathrm{tr}(T) :=\int \mathscr{K}(x, x)  ~\mathsf{d} \rho_{\calX}(x) \leq \widetilde{\mathfrak{C}}.
    $$
    Moreover, from Section 4.3 in \cite{rasmussen2003gaussian}, we know that the eigenspaces of Gaussian kernels defined on $\mathbb{R}^d$ with the standard Gaussian distribution have $\Theta(d^k)$ eigenvalues of order $\Theta(d^{-k})$.
    Therefore, Assumption 7 holds for Gaussian kernels.

\noindent {\bf Product Laguerre kernels defined in (\ref{eqn_def_product_Laguerre_kernels}). }
The following lemma shows that $\widetilde{\mathscr{K}}_{\alpha, r}$ defined in (\ref{eqn_def_product_Laguerre_kernels}) is a kernel that satisfies Assumption 7.

\begin{lemma}
    Let $\widetilde{\mathscr{K}}_{\alpha, r}$ be defined in (\ref{eqn_def_product_Laguerre_kernels}). Then it is a continuous and positive-definite kernel on $\mathcal{X} \times \mathcal{X}$ with respect to $\rho_{\mathcal{X}}$, where $\mathcal{X} = [0, \infty)^d$, and $\rho_{\mathcal{X}}$ is the product measure with each coordinate $x_{(i)}$ following an independent Gamma distribution $\Gamma(\alpha_i + 1, 1)$. Moreover, it admits a Mercer decomposition given in (22), with eigenvalues and eigenfunctions satisfying Assumption 7 with $\widetilde{\mathfrak{C}} = 4^r$.
\end{lemma}
\begin{proof}
It is obvious that $\widetilde{\mathscr{K}}_{\alpha, r}$ is a continuous and positive-definite kernel. 
We then show that
$$
\widetilde{\mathscr{K}}_{\alpha, r}(x,x^\prime) = \sum_{k=0}^{\infty} \sum_{j=1}^{\widetilde{N}(d, k)} \widetilde{\lambda}_{k, j}  \widetilde{\psi}_{k, j}(x) \widetilde{\psi}_{k, j}\left(x^\prime\right),
$$
and provide explicit expressions of $\widetilde{\lambda}_{k, j}$'s and $\widetilde{\psi}_{k, j}$'s. 
For any $c>0$, let $\{L_{k}^{(c)}(\cdot)\}_{k = 0}^{\infty}$ be the generalized Laguerre polynomials (see, e.g., Chapter 6.2 in \cite{andrews1999special}) and define $\widetilde{L}_{k}^{(c)}:=\sqrt{k! \Gamma(c+1) / \Gamma(k+c+1)}L_{k}^{c}(\cdot)$, $k \geq 0$. 

Then we have (see, e.g., Eq.~(6.2.3) in \cite{andrews1999special}):
\begin{align*}
    &~\int_{0}^{\infty} \widetilde{L}_{k}^{(c)}(x) \widetilde{L}_{m}^{(c)}(x) \frac{1}{\Gamma(c+1)} x^{c}  e^{-x} ~ \mathrm{d} x\\
    =&~
    \sqrt{\frac{k! m!}{\Gamma(k+c+1)\Gamma(m+c+1)}}\int_{0}^{\infty} L_{k}^{(c)}(x) L_{m}^{(c)}(x) x^{c}  e^{-x} ~ \mathrm{d} x\\
    =&~
    \sqrt{\frac{k! m!}{\Gamma(k+c+1)\Gamma(m+c+1)}} \frac{\Gamma(k+c+1)}{k!} \mathbf{1}\{k=m\} = \mathbf{1}\{k=m\}.
\end{align*}
Furthermore, it is known that the generalized Laguerre polynomials satisfy the Hardy–Hille formula (see, e.g., Eq. (2.13) in \cite{chatterjea1966operational}), that is, 
\begin{equation}\label{eqn_Hardy_Hille_Laguerre_poly}
\begin{aligned}
    \mathscr{K}_{i}(z, z^{\prime}) =&~  \Gamma(\alpha_i+1) \cdot \sum_{k=0}^{\infty} \frac{k! (r/d)^{k}}{\Gamma(k+\alpha_i+1)} L_{k}^{(\alpha_i)}(z) L_{k}^{(\alpha_i)}(z^{\prime})\\
    =&~ \sum_{k=0}^{\infty}  (r/d)^{k} \widetilde{L}_{k}^{(\alpha_i)}(z) \widetilde{L}_{k}^{(\alpha_i)}(z^{\prime}).
\end{aligned}
\end{equation}
Combining (\ref{eqn_def_product_Laguerre_kernels}) and (\ref{eqn_Hardy_Hille_Laguerre_poly}), we have
\begin{equation*}
    \begin{aligned}
        \widetilde{\mathscr{K}}_{\alpha, r}(x,x^\prime) =&~ \sum_{k_1=0}^{\infty} \cdots \sum_{k_d=0}^{\infty} (r/d)^{k_1 + \cdots + k_d} \prod_{i=1}^{d} \widetilde{L}_{k_i}^{(\alpha_i)}(x_{(i)}) \widetilde{L}_{k_i}^{(\alpha_i)}(x_{(i)}^{\prime})\\
        =&~
        \sum_{k=0}^{\infty} (r/d)^{k} \sum_{\substack{k_1, \cdots, k_d \geq 0 \\ \sum_{i \leq d} k_{i} = k}} \left[\prod_{i=1}^{d} \widetilde{L}_{k_i}^{(\alpha_i)}(x_{(i)}) \right]
        \left[\prod_{i=1}^{d} \widetilde{L}_{k_i}^{(\alpha_i)}(x_{(i)}^{\prime}) \right],
    \end{aligned}
\end{equation*}
implying that $\widetilde{\lambda}_{k, j} = (r/d)^{k}$ and $\widetilde{N}(d, k) = C_{k+d-1}^{d-1}$. Therefore, we have
    $$
    \sum_{k=0}^{\infty} \sum_{j=1}^{\widetilde{N}(d, k)} \widetilde{\lambda}_{k, j} = \sum_{k=0}^{\infty} (r/d)^{k} C_{k+d-1}^{d-1} \overset{\text{Generating function identity}}{=} (1-r/d)^{-(d-1+1)} \overset{d \geq 2r}{\leq} 4^r;
    $$


    Moreover, for any $k  = 0, 1, \cdots, 2\ell_{\gamma} +2$, we have 
    $$
        \widetilde{\lambda}_{k, j} = (r/d)^{k}= \Theta_d(d^{-k}), \quad j \leq \widetilde{N}(d, k) = C_{k+d-1}^{d-1} = \Theta_d(d^{k}).
    $$
\end{proof}

\subsection{Proof of Proposition~\ref{prop:inner_edr}}\label{append_proof_prop_inner_edr}

We first prove the following lemma.
\begin{lemma}\label{lemma:parity_monotone_inner_kernel}
For every $d\geq3$ and integer $k\geq0$, the degree-$k$ eigenvalues
satisfy $\mu_{k+2}\leq\mu_k$. The inequality is strict if $\mu_k>0$. Moreover, we have
\begin{equation*}
 \mu_k\geq a_k\frac{k!}{d(d+2)\cdots(d+2k-2)},
\end{equation*}
with the empty product equal to one when $k=0$.
\end{lemma}
\begin{proof}
For $q\geq0$ set
\[
 b_{k,q}(d):=
 \frac{(k+2q)!\,\Gamma(d/2)}
 {2^{k+2q}q!\,\Gamma(d/2+k+q)}.
\]
The Funk--Hecke formula, applied to the monomial $t^{k+2q}$
with respect to the uniform probability measure on $\mathbb S^{d-1}$,
gives its eigenvalue on the degree-$k$ spherical harmonics as
$b_{k,q}(d)$. Monomials of degree smaller than $k$, or of parity
opposite to $k$, have zero eigenvalue on that space. Consequently,
\begin{equation*}
 \mu_k=\sum_{q=0}^{\infty}a_{k+2q}b_{k,q}(d).
\end{equation*}
To justify the interchange of the series and the integral, note that
$\sum_{i\geq0}a_i\leq a_{-1}<\infty$, so the power series converges
uniformly on $[-1,1]$. Each monomial kernel is positive semidefinite,
and its integral operator has operator norm at most one; thus its
nonnegative eigenvalues $b_{k,q}(d)$ are bounded by one, and the
series $\sum_{q=0}^{\infty}a_{k+2q}b_{k,q}(d)$ converges absolutely.
For $q\geq1$, the gamma recurrence gives the exact identity
\[
 \frac{b_{k+2,q-1}(d)}{b_{k,q}(d)}
 =\frac{q}{d/2+k+q}<1.
\]
Reindexing the nonnegative series therefore yields
\[
 \mu_{k+2}
 =\sum_{q=1}^{\infty}a_{k+2q}b_{k,q}(d)
   \frac{q}{d/2+k+q}
 \leq\sum_{q=0}^{\infty}a_{k+2q}b_{k,q}(d)=\mu_k.
\]
If $\mu_k>0$, at least one summand in the last series is positive;
its contribution is strictly decreased or omitted, so the inequality
is strict.

Finally, the $s=0$ term of
$\mu_k = \sum_{q=0}^{\infty}a_{k+2q}b_{k,q}(d)$ gives
\begin{equation*}
 \mu_k\geq a_k b_{k,0}(d)
 =a_k\frac{k!}{d(d+2)\cdots(d+2k-2)},
\end{equation*}
with the empty product equal to one when $k=0$.
\end{proof}

\begin{proof}[Proof of Proposition~\ref{prop:inner_edr}]
The dimension of the degree-$k$ spherical-harmonic space is
\[
 N(d,0)=1,\qquad
 N(d,k)=\binom{d+k-1}{k}-\binom{d+k-3}{k-2}
 \quad(k\geq1),
\]
where the second binomial coefficient is zero for $k=1$.
For each fixed $k$, this implies
\begin{equation*}
 N(d,k)=\frac{d^k}{k!}\bigl(1+O_d(d^{-1})\bigr).
\end{equation*}

We next establish upper and lower bounds on $\mu_k$. Since
$\mathscr K$ is positive semidefinite, its integral operator is
positive. Its trace is
$\int\mathscr K(x,x)\,\mathrm d\rho_{\calX}(x)
=\Phi(1)\leq a_{-1}$.
The trace identity (or its finite-rank projection inequality) gives
\begin{equation*}
 N(d,k)\mu_k\leq\Phi(1)\leq a_{-1}.
\end{equation*}
On the other hand, Lemma \ref{lemma:parity_monotone_inner_kernel} gives
\begin{equation*}
 \mu_k\geq a_k\frac{k!}{d(d+2)\cdots(d+2k-2)},
\end{equation*}
with the empty product equal to one when $k=0$.
For each fixed $0\leq k\leq\ell_\gamma+2$, the coefficient $a_k$
is positive by assumption. Combining all the above proves $\mu_k=\Theta_d(d^{-k})$ and
$N(d,k)=\Theta_d(d^k)$. Summation over the finitely many degrees
$0,\ldots,k$ proves $N_k=\Theta_d(d^k)$.

Put $L=\ell_\gamma$. The same bounds imply, for all sufficiently
large $d$ and every $0\leq k\leq L+1$,
\begin{equation}\label{eq:inner_kernel_adjacent_separation}
 \mu_k>\mu_{k+1}.
\end{equation}
Indeed, the lower bound for $\mu_k$ has order $d^{-k}$,
whereas the trace upper bound for $\mu_{k+1}$ has order
$d^{-(k+1)}$. The required threshold depends only on
$L,a_{-1},a_0,\ldots,a_{L+2}$.

To exclude all higher-degree eigenvalues from the leading blocks,
observe from Lemma~\ref{lemma:parity_monotone_inner_kernel} that
for every $j\geq L+2$,
\[
 \mu_j\leq\max\{\mu_{L+2},\mu_{L+3}\}.
\]
By the trace bound and the dimension formula,
\[
 \mu_{L+2}\leq\frac{a_{-1}}{N(d,L+2)}=O(d^{-(L+2)}),
\qquad
 \mu_{L+3}\leq\frac{a_{-1}}{N(d,L+3)}=O(d^{-(L+3)}).
\]
But Lemma \ref{lemma:parity_monotone_inner_kernel} implies
$\mu_{L+1}\geq c\,a_{L+1}d^{-(L+1)}$ for some
$c>0$ depending only on $L$. Hence
\[
 \sup_{j\geq L+2}\mu_j<\mu_{L+1}
\]
for all sufficiently large $d$. Together with
\eqref{eq:inner_kernel_adjacent_separation}, this establishes
\[
 \mu_0>\mu_1>\cdots>\mu_{L+1}>
 \sup_{j\geq L+2}\mu_j.
\]
The eigenspace for degree $k$ has dimension $N(d,k)$, so the
claimed ordering of $\lambda_j$ and the choice of $\phi_j$ follow.
\end{proof}

\section{Properties of $\reg$}\label{appendix_filter_property}

\subsection{Additional requirements on $\reg$}\label{appendix_addtion_require_of_reg}

Set $\kappa^2=K_{\max}$. We list two additional requirements on $\reg$ that are borrowed from \cite{lu2024saturation}.

\begin{itemize}
    \item If $\tau<\infty$, then there exist positive constants $\mathfrak{C}_{7}$ and $\mathfrak{C}_{8}$ independent of $d$, $\lambda$, and the real spectral variable $z$, such that we have
    \begin{align}
        (z/\lambda)^{2\tau} \rem^2(z)  \geq \mathfrak{C}_{7},
        \quad 
        &\forall \lambda \in\Lambda, z > \lambda
        \label{eq:r11-C7}\\
        (z/\lambda)^{2\tau} \rem^2(z)  \leq \mathfrak{C}_{8} z  \reg(z),
        \quad 
        &\forall \lambda \in\Lambda, z \leq \lambda.
    \end{align}

    \item Let
  \begin{align*}
    D_{\lambda} &= \left\{ z \in \bbC : \Re z \in [-\lambda/2,\kappa^2], ~ \abs{\Im z} \leq \Re z + \lambda/2 \right\} \\
    & \quad \cup \left\{ z \in \bbC : \abs{z - \kappa^2} \leq \kappa^2 + \lambda/2,~ \Re z \geq \kappa^2  \right\};
  \end{align*}
  Then $\reg(z)$ can be extended to be an analytic function on some domain containing $D_\lambda$
  and the following conditions hold for all $\lambda \in\Lambda$:
  \begin{enumerate}
    \item[(C1)] $\left|(z+\lambda)\reg(z)\right| \leq \tilde{E}$ for all $z \in D_\lambda$;
    \item[(C2)] $\left| (z+\lambda)\rem(z)\right| \leq \tilde{F} \lambda$ for all  $z \in D_\lambda$;
  \end{enumerate}
  where  $\tilde{E}, \tilde{F}$ are positive constants.
\end{itemize}

Only \eqref{eq:r11-C7}, together with residual monotonicity and continuity, is needed for the lower bound in Lemma \ref{lem:popall}.

Set $B=K_{\max}=\kappa^2$.
We verify Definition 2.2 and the complex-domain bounds for the step size and
index set in Example 4. Put $c_\eta=1-\eta B>0$. For $0<z\le B$,
\[
\begin{aligned}
0&<1-\eta z<1,\\
\exp\{-z/(\lambda c_\eta)\}
&\le (1-\eta z)^t\le \exp\{-z/\lambda\}.
\end{aligned}
\]
The residual is non-increasing in $z$. For every fixed $q>0$,
$\psi_\lambda(z)\le(q/e)^q\lambda^qz^{-q}$; the $q=0$ bound is one.
For every fixed $\kappa>0$, $z\le\kappa\lambda$ implies
$\psi_\lambda(z)\ge\exp(-\kappa/c_\eta)$ on the real spectral interval.
Furthermore,
\[
\frac1{z+\lambda}\le\varphi_\lambda(z)
\le\min\{\lambda^{-1},z^{-1}\}\le\frac2{z+\lambda}.
\]
For the lower bound use $e^{-r}\le(1+r)^{-1}$; for the upper bound use the
geometric sum and $0\le1-\psi_\lambda(z)\le1$.

Let $z=x+iy\in D_\lambda$, with $D_\lambda$ as defined above and
$0<\lambda\le1$. If $x\ge0$, the wedge and disk defining that set imply
\[
\begin{aligned}
|z|^2&\le(2B+1)x+\lambda^2/4,\\
|1-\eta z|^2
&\le1-\eta x+\eta^2\lambda^2/4,\\
|\psi_\lambda(z)|
&\le\exp\{-x/(2\lambda)+\eta\lambda/8\}.
\end{aligned}
\]
Indeed, on the wedge $|y|\le x+\lambda/2$ and $x\le B$. On the disk,
$|z-B|\le B+\lambda/2$ and $x\ge B$, which gives the same displayed bound.
The last step uses $t\eta=\lambda^{-1}$ and $1+a\le e^a$.
If $-\lambda/2\le x<0$, then $|z|\le\lambda/\sqrt2$ and
\[
|\psi_\lambda(z)|\le(1+\eta|z|)^t\le e^{1/\sqrt2}.
\]
For $x\ge0$, both parts of $D_\lambda$ satisfy
$|z+\lambda|\le\sqrt2(x+\lambda)$. The preceding estimates, and boundedness
of $(r+1)e^{-r/2}$ for $r\ge0$, therefore give
\[
|(z+\lambda)\psi_\lambda(z)|\le C\lambda
\qquad(z\in D_\lambda).
\]
For $|z|\le\lambda$, the geometric sum gives
\[
|\varphi_\lambda(z)|
\le\eta t(1+\eta\lambda)^t\le e/\lambda.
\]
For $|z|>\lambda$, use
$\varphi_\lambda(z)=(1-\psi_\lambda(z))/z$ and
$|z+\lambda|/|z|\le2$. It follows that
\[
|(z+\lambda)\varphi_\lambda(z)|\le C
\qquad(z\in D_\lambda).
\]
Thus both complex-domain conditions hold. Polynomial filters are entire, so
no extension in the regularization index is required. All constants above
are independent of $d$, $t$, $\lambda$ and $z$, with the step size fixed as stated.

\section{Auxiliary Lemmas}\label{appendix_auxiliary}

The following proposition is known as Weyl's inequality.

\begin{proposition}[Weyl's inequality]\label{prop_weyl_ine}
Let $A_1$ and $A_2$ be symmetric matrices in $\mathbb{R}^{n \times n}$, with their eigenvalues ordered in descending order. Then, we have
\begin{equation*}
    \begin{aligned}
        \lambda_{i+j-1}(A_1 + A_2) \leq &~ \lambda_i(A_1) + \lambda_j(A_2), \quad i+j \leq n+1,\\
        \lambda_i(A_1) + \lambda_j(A_2) \leq &~  \lambda_{i+j-n}(A_1 + A_2), \quad i+j \geq n+1.
    \end{aligned}
\end{equation*}
\end{proposition}

The following proposition is a restatement of von Neumann's trace inequality; see also \cite{marshall1979inequalities}.

\begin{proposition}\label{prop_ri}
    For Hermitian $n \times n$ positive semi-definite complex matrices $A$ and $B$, where now the eigenvalues are sorted decreasingly ($a_1 \geq a_2 \geq \cdots \geq a_n \geq 0$ and $b_1 \geq b_2 \geq \cdots \geq b_n \geq 0$, respectively), we have
$$
\mathrm{tr}\left( AB \right) \geq \sum_{i=1}^{n} a_i b_{n-i+1}.
$$
\end{proposition}

The following proposition is an extension of the symmetric matrix Bernstein inequality.

\begin{proposition}[Restatement of Exercise 6.10 in \cite{wainwright2019high}]\label{prop_non_sym_matrix_bernstein}
Let $\{A_i\}_{i=1}^n$ be a collection of $n$ independent, zero-mean matrices in $\mathbb{R}^{d_1 \times d_2}$, such that $\|A_i\|_2 \leq b$ almost surely. 
Let 
\[
\sigma^2 := \max\left\{
\left\|\frac{1}{n}\sum_{i=1}^n \E[A_i A_i^\top]\right\|_2,
\left\|\frac{1}{n}\sum_{i=1}^n \E[A_i^\top A_i]\right\|_2
\right\}.
\]
Then, for any $\delta > 0$, the following tail bound holds:
\[
\mathbb{P}\left(\left\|\sum_{i=1}^n A_i\right\|_2 \geq \delta\right) \leq 2(d_1 + d_2) \exp\left(-\frac{\delta^2}{2(n\sigma^2 + b\delta)}\right).
\]
\end{proposition}

The following proposition restates the Sherman-Morrison-Woodbury formula.

\begin{proposition}[Sherman-Morrison-Woodbury formula]\label{prop_smw_formula}
    Suppose $A \in \mathbb{R}^{n \times n}$ is an invertible matrix, and $Z \in \mathbb{R}^{n \times k}$ is such that $ZZ^{\top} + A$ is invertible. Then we have
    \[
(ZZ^{\top} + A)^{-1} = A^{-1} - A^{-1} Z(\mathrm{I}_{k} + Z^{\top} A^{-1} Z)^{-1} Z^{\top} A^{-1}.
\]
\end{proposition}

We use the following specialization of Corollary~2 on p.~3 of  \cite{tikhomirov2018sample}. The cited theorem is a statement about independent \emph{sample vectors}; it does not require independent coordinates within a vector.

\begin{lemma}[Tikhomirov, specialized to a sixth moment]\label{lem:tikh}
Let $Y\in\R^D$ satisfy $\E Y=0$, $\E YY^\top=\mathrm{I}_D$, and
\[
 \sup_{\norm{v}_2=1}\E|\inner Yv|^6\le B.
\]
For $N\ge2D$ independent copies $Y_1,\ldots,Y_N$, a numerical constant $C_6$ satisfies
\begin{equation}\label{eq:tikh}
 \Pp\left\{\norm{\frac1N\sum_{i=1}^N Y_iY_i^\top-\mathrm{I}_D}_{\op}
 \le C_6\left(\frac{\max_{i\le N}\norm{Y_i}_2^2}{N}
       +B^{1/3}\sqrt{\frac DN}\right)\right\}
 \ge1-D^{-1}.
\end{equation}
\end{lemma}
The parameter names differ from the source: its ambient dimension is called $n$, whereas its sample size is called $N$. In this supplement the ambient dimension in this application is $N(d, k)$, not the original covariate dimension $d$.

For a spherical polynomial $h$ of degree at most $k$ and $q\ge2$, spherical hypercontractivity gives
\begin{equation}\label{eq:hc}
 \norm{h}_{L^q}\le(q-1)^{k/2}\norm{h}_{L^2}.
\end{equation}
An explicit statement is Lemma~19 on p.~78 of \cite{mei2022generalization}. Rescaling the radius of the sphere preserves this inequality when both norms use the corresponding uniform probability measure.

The addition formula implies
\begin{equation}\label{eq:add}
 \sum_{j=1}^{D_k}\psi_{k,j}(x)^2=D_k\quad\text{for every }x\in S^{d-1}.
\end{equation}
Consequently, for any integer $m>0$, if $Y=\psi_m(x)\in\R^{N(d, m)}$, then
\begin{equation}\label{eq:moment}
 \begin{aligned}
 \E Y&=0, & \E YY^\top&=\mathrm{I}_{N(d, m)}, & \norm{Y}_2^2&=N(d, m),\\
 \sup_{\norm{v}_2=1}\E|\inner Yv|^6&\le5^{3m}.
 \end{aligned}
\end{equation}
The first identity uses $m\ge1$, the second uses orthonormality, and the last follows by applying \eqref{eq:hc} to the harmonic polynomial $\sum_jv_j\psi_{m,j}$. The sixth-moment constant is independent of $d$ because $m$ is fixed.

\section{Further Experiments and Supporting Proofs}\label{appendix_experiments}

\subsection{Additional Experiments}\label{appendix_addtional_experiments}

Figures \ref{experiment_1_settings_3} and \ref{experiment_1_settings_4} report results of the remaining Type 1 experiments with one selected $C_{\lambda}$ value.

Figures \ref{experiment_1_settings_1_all}, \ref{experiment_1_settings_2_all}, \ref{experiment_1_settings_3_all}, and \ref{experiment_1_settings_4_all} report results of all Type 1 experiments with different $C_{\lambda}$ values.

\subsection{Verification of the regression function in (25)}\label{append_proof_true_function_in_experiments}

In this subsection, we show that $f_{\star}$ defined in (25) satisfies Assumption 3 and Assumption 5. Without loss of generality, we assume that $d \geq \mathfrak{C}$ such that the results in Proposition 2.1 hold.

In the following proof, we will frequently use the addition formula (see, e.g., Proposition 1.18 in \cite{gallier2009notes}, page 56 in \cite{ghorbani2021linearized}, or (39) in \cite{lu2024pinsker}):
\begin{equation}\label{eqn_addition_formula}
    P_{k, d}(\langle x, x^{\prime}\rangle) = \frac{1}{N(d, k)} \sum_{j=1}^{N(d, k)} \psi_{k,j}(x) \psi_{k,j}(x^{\prime}),\quad k=0, 1, \cdots
\end{equation}
where $P_{k, d}(\cdot)$ is the $k$-th normalized Gegenbauer polynomial (such that $P_{k, d}(1)=1$) with parameter $(d-2)/2$, and we have
\begin{equation}\label{eqn_former_gegenbauer}
    \begin{aligned}
    P_{0, d}(t)=1, \quad P_{1, d}(t)=t, \quad P_{2, d}(t)=\frac{dt^2-1}{d-1}.
    \end{aligned}
\end{equation}

\noindent {\bf Case $s=1$. } In this case we have
\begin{equation*}
    \begin{aligned}
        f_{\star}(x) = &~ \sum_{i=1}^{3} \mathscr{K}(\xi_{i},x)
        \overset{(4)}{=}
        \sum_{j = 1}^{\infty} \left[\sum_{i=1}^{3} \lambda_j \phi_{j}(\xi_{i}) \right]\phi_{j}(x)\\
        (\because \text{Proposition 2.1} ) \quad = &~
        \sum_{k=0}^{\ell_{\gamma}+1} \sum_{j=1}^{N(d,k)} \left[\sum_{i=1}^{3} \mu_{k} \psi_{k,j}(\xi_{i})\right]\psi_{k,j}(x) + \sum_{j =N_{\ell_{\gamma}+1}+1}^{\infty} \left[\sum_{i=1}^{3} \lambda_j \phi_{j}(\xi_{i}) \right] \phi_{j}(x).
    \end{aligned}
\end{equation*}
We first note that 
\begin{align*}
 \|f\|_{[\calH]}^2 =&~ \sum_{j=1}^{\infty} \lambda_j^{-1} \left[\sum_{i=1}^{3} \lambda_j \phi_{j}(\xi_{i}) \right]^2 \leq 3\sum_{i=1}^{3}\sum_{j=1}^{\infty} \lambda_j \phi_{j}(\xi_{i}) \phi_{j}(\xi_{i}) = 3\sum_{i=1}^{3} \mathscr{K}(\xi_{i}, \xi_{i}) \leq 9K_{\max},
\end{align*}
thus $f_{\star}$ satisfies Assumption 3. 

Since we have set $\gamma<2$ and $d>10$ in our experiments, we have $\ell_{\gamma} \leq 1$ and $3-6/(d-1)>2$. Therefore, for any $m \leq \ell_{\gamma}+1 \leq 2$, we have

\begin{equation}\label{eqn_verify_assump_6_for_experiments}
    \begin{aligned}
    &~ \mu_{m}^{-1} \sum\limits_{j=1}^{N(d, m)}  \theta_{m,j}^{2}\\
    =&~ \mu_{m}^{-1} \sum\limits_{j=1}^{N(d, m)} \left[\sum_{i=1}^{3} \mu_{m} \psi_{m,j}(\xi_{i})\right]^2\\
    =&~ \mu_{m} \left[\sum_{i=1}^{3} \sum\limits_{j=1}^{N(d, m)} \psi_{m,j}^2(\xi_{i}) + \sum_{i \neq i^{\prime}}^{3} \sum\limits_{j=1}^{N(d, m)} \psi_{m,j}(\xi_{i}) \psi_{m,j}(\xi_{i^{\prime}}) \right]\\
    \overset{(\ref{eqn_addition_formula})}{=} &~ \mu_{m} N(d, m) \left[ \sum_{i=1}^{3} 
    P_{m, d}(\langle \xi_{i}, \xi_{i}\rangle) + \sum_{i \neq i^{\prime}}^{3} P_{m, d}(\langle \xi_{i}, \xi_{i^{\prime}}\rangle)\right]\\
    \overset{(\ref{eqn_former_gegenbauer})}{=}&~
    \left\{\begin{matrix}
\mu_{0} N(d, 0) [3+6] = 9 \mu_{0} N(d, 0), &  m=0;\\
\mu_{1} N(d, 1) [3 + \sum_{i \neq i^{\prime}}^{3} \langle \xi_{i}, \xi_{i^{\prime}}\rangle] \overset{(26)}{\geq} 0.1 \mu_{1} N(d, 1), &  m=1;\\
\mu_{2} N(d, 2) [3 + \sum_{i \neq i^{\prime}}^{3} \frac{d\langle \xi_{i}, \xi_{i^{\prime}}\rangle^2-1}{d-1}] \geq \mu_{2} N(d, 2) [3 - 6/(d-1)] > 2\mu_{2} N(d, 2), &  m=2;\\
\end{matrix}\right.
\end{aligned}
\end{equation}
Therefore, from Proposition 2.1 we have $\mu_{m}^{-1} \sum_{j=1}^{N(d, m)}  \theta_{m,j}^{2} = \Omega_d(1)$, $m=0, 1, 2$. 
Furthermore, 
\begin{align*}
    \sum\limits_{k=0}^{\tilde{\ell}} \sum\limits_{j =1}^{N(d, k)}  \theta_{k, j}^{2} \geq &~ \theta_{0, 1}^{2} = \left[\sum_{i=1}^{3} \mu_{0} \psi_{0,1}(\xi_{i})\right]^2 = 9\mu_0^2 =\Theta_d(1),
\end{align*}
where the last equation follows from Proposition~2.1. 
We conclude that $f_{\star}$ satisfies Assumption 5.

\noindent {\bf Case $s \neq 1$ and $\mathscr{K}=\mathscr{K}^{\mathrm{rbf}}$. } In this case we have
\begin{equation*}
    \begin{aligned}
        f_{\star}(x) = &~ \sum_{i=1}^{3} \sum_{k=0}^{2} d^{k(1-s)/2} P_{k, d}(\langle\xi_{i}, x\rangle)
        \overset{(\ref{eqn_addition_formula})}{=}
        \sum_{k=0}^{2} \sum_{j=1}^{N(d, k)} \left[\sum_{i=1}^{3} \frac{d^{k(1-s)/2}}{N(d, k)} \psi_{k,j}(\xi_{i}) \right]  \psi_{k,j}(x) 
    \end{aligned}
\end{equation*}
By Proposition 2.1, we have
$$
\begin{aligned}
 \|f\|_{[\calH]^{s}}^2 =&~ \sum_{k=0}^{2} \sum_{j=1}^{N(d, k)} \mu_k^{-s} \left[\sum_{i=1}^{3} \frac{d^{k(1-s)/2}}{N(d, k)} \psi_{k,j}(\xi_{i}) \right]^2 \overset{(\ref{eqn_addition_formula})}{\leq}
        9 \sum_{k=0}^{2} \frac{\mu_k^{-s}d^{k(1-s)}}{N(d, k)} = O_d(1),
\end{aligned}
$$
which implies that $f_{\star}$ satisfies Assumption 3.

Since we have set $\gamma<2$ and $d>10$ in our experiments, we have $\ell_{\gamma} \leq 1$ and $3-6/(d-1)>2$. 
Therefore, for any $m \leq \ell_{\gamma}+1 \leq 2$, as in (\ref{eqn_verify_assump_6_for_experiments}), we have
\begin{equation*}
    \begin{aligned}
        \mu_{m}^{-s} \sum\limits_{j=1}^{N(d, m)}  \theta_{m,j}^{2} =&~ \mu_{m}^{-s} \sum\limits_{j=1}^{N(d, m)}  \left[\sum_{i=1}^{3} \frac{d^{m(1-s)/2}}{N(d, m)} \psi_{m,j}(\xi_{i}) \right]^2\\
        =&~  \frac{\mu_{m}^{-s} d^{m(1-s)}}{N(d, m)^2} \left[\sum_{i=1}^{3} \sum\limits_{j=1}^{N(d, m)} \psi_{m,j}^2(\xi_{i}) + \sum_{i \neq i^{\prime}}^{3} \sum\limits_{j=1}^{N(d, m)} \psi_{m,j}(\xi_{i}) \psi_{m,j}(\xi_{i^{\prime}}) \right]\\
        =&~ \Omega_d\left(
        \frac{\mu_{m}^{-s} d^{m(1-s)}}{N(d, m)}
        \right). 
    \end{aligned}
\end{equation*}
 By Proposition 2.1, the right-hand side of the last display is $\Omega_d(1)$. Furthermore, 
$$
\begin{aligned}
\sum\limits_{k=0}^{\tilde{\ell}} \sum\limits_{j =1}^{N(d, k)}  \theta_{k, j}^{2} \geq &~ \theta_{0, 1}^{2} = \left[\sum_{i=1}^{3} \frac{d^{0(1-s)/2}}{N(d, 0)} \psi_{0,1}(\xi_{i}) \right]^2 = \frac{9}{N(d, 0)^2} = \Theta_d(1),
\end{aligned}
$$
where the last equation follows from Proposition 2.1. 
Therefore, $f_{\star}$ satisfies Assumption 5.

        

\clearpage

\begin{figure}[!htbp]
\centering

\subfigure{\includegraphics[width=0.9\columnwidth]{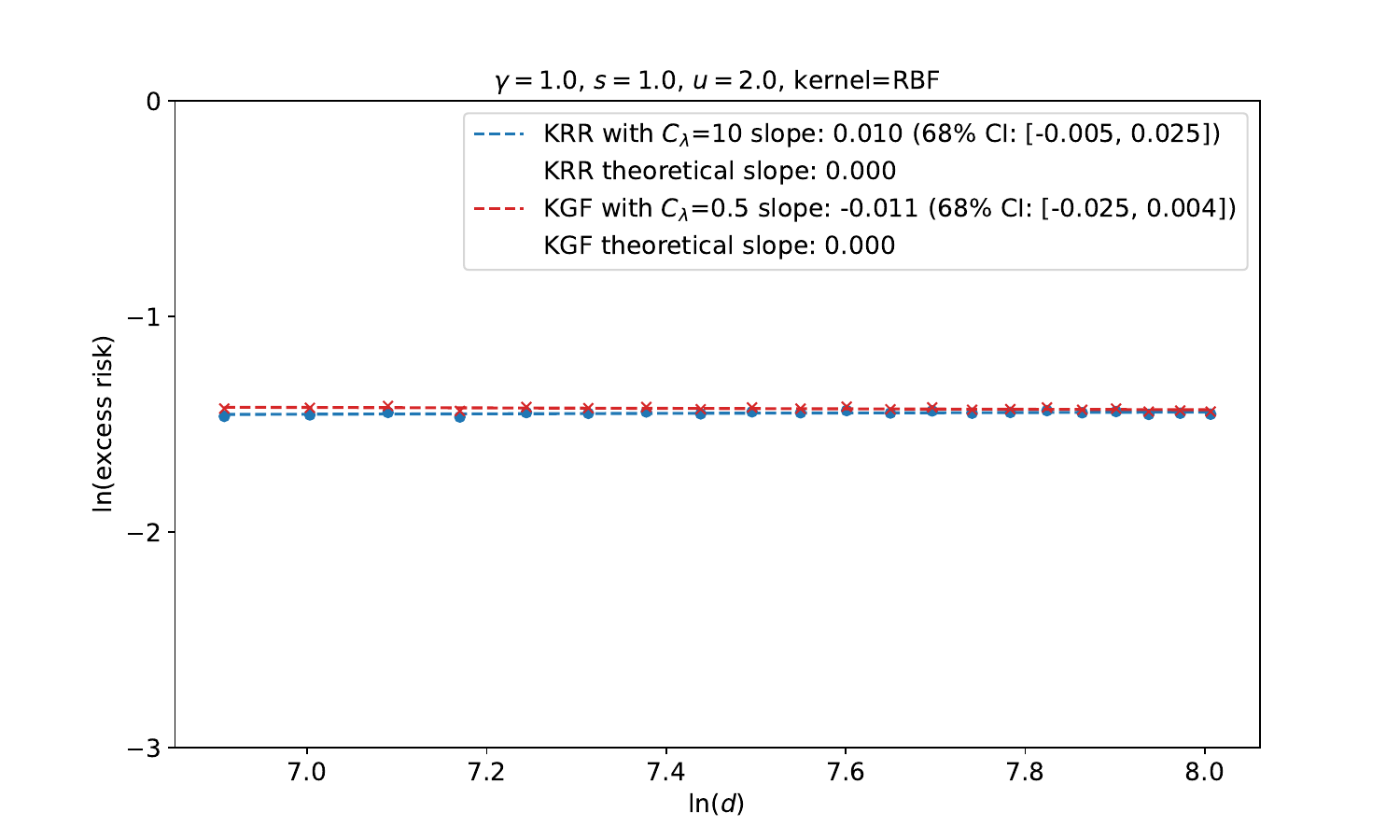}}

\caption{
Type 1 experiments with $(\gamma, s, u) = (1.0, 1.0, 2.0)$. The setup and analysis are the same as in Figure 3. Results for additional $C_{\lambda}$ values are shown in Figure \ref{experiment_1_settings_3_all}.
}
\label{experiment_1_settings_3}
\end{figure}

\begin{figure}[!htbp]
\centering

\subfigure{\includegraphics[width=0.9\columnwidth]{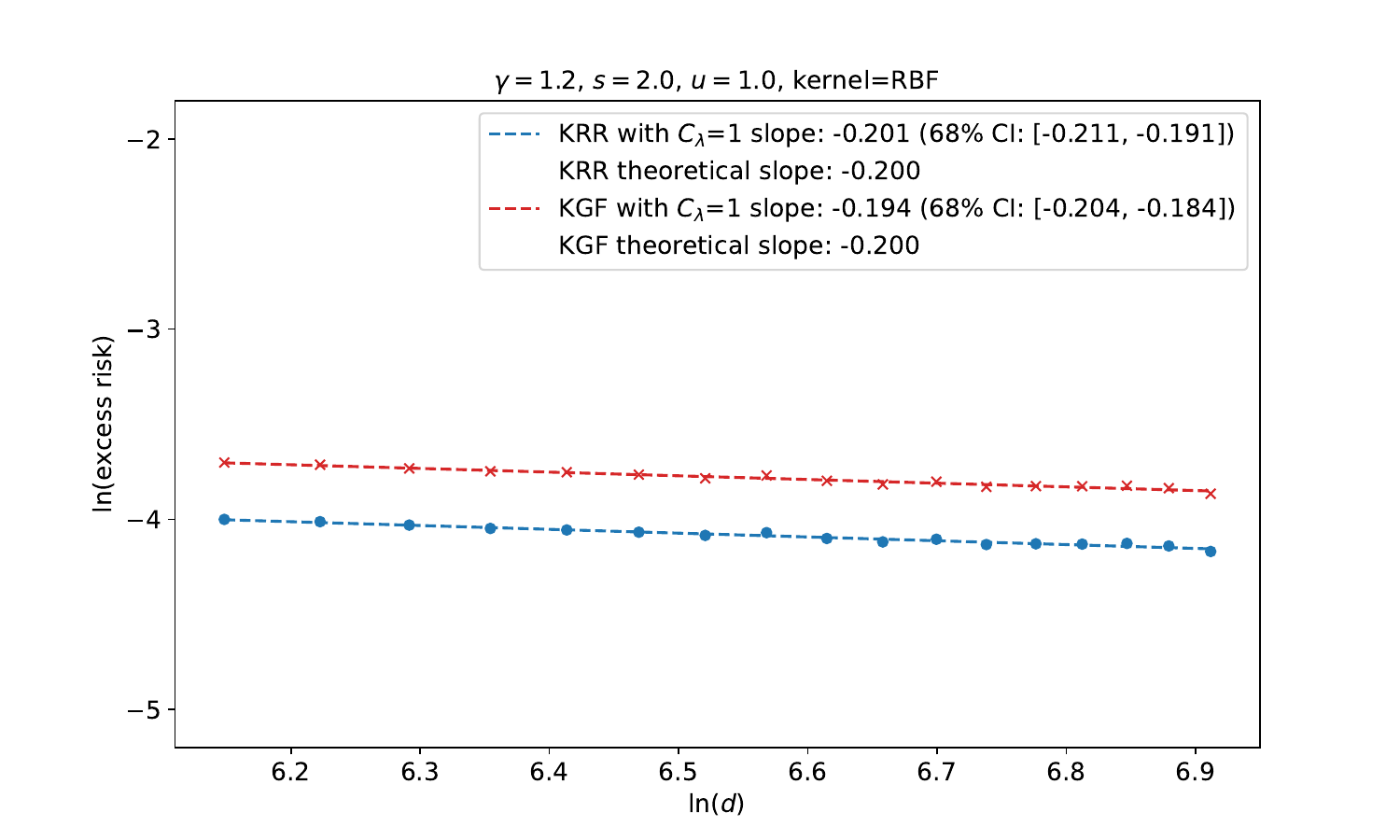}}

\caption{
Type 1 experiments with $(\gamma, s, u) = (1.2, 2.0, 1.0)$. The setup and analysis are the same as in Figure 3. Results for additional $C_{\lambda}$ values are shown in Figure \ref{experiment_1_settings_4_all}.
}
\label{experiment_1_settings_4}
\end{figure}

\begin{figure}
\centering

\subfigure{\includegraphics[width=0.9\columnwidth]{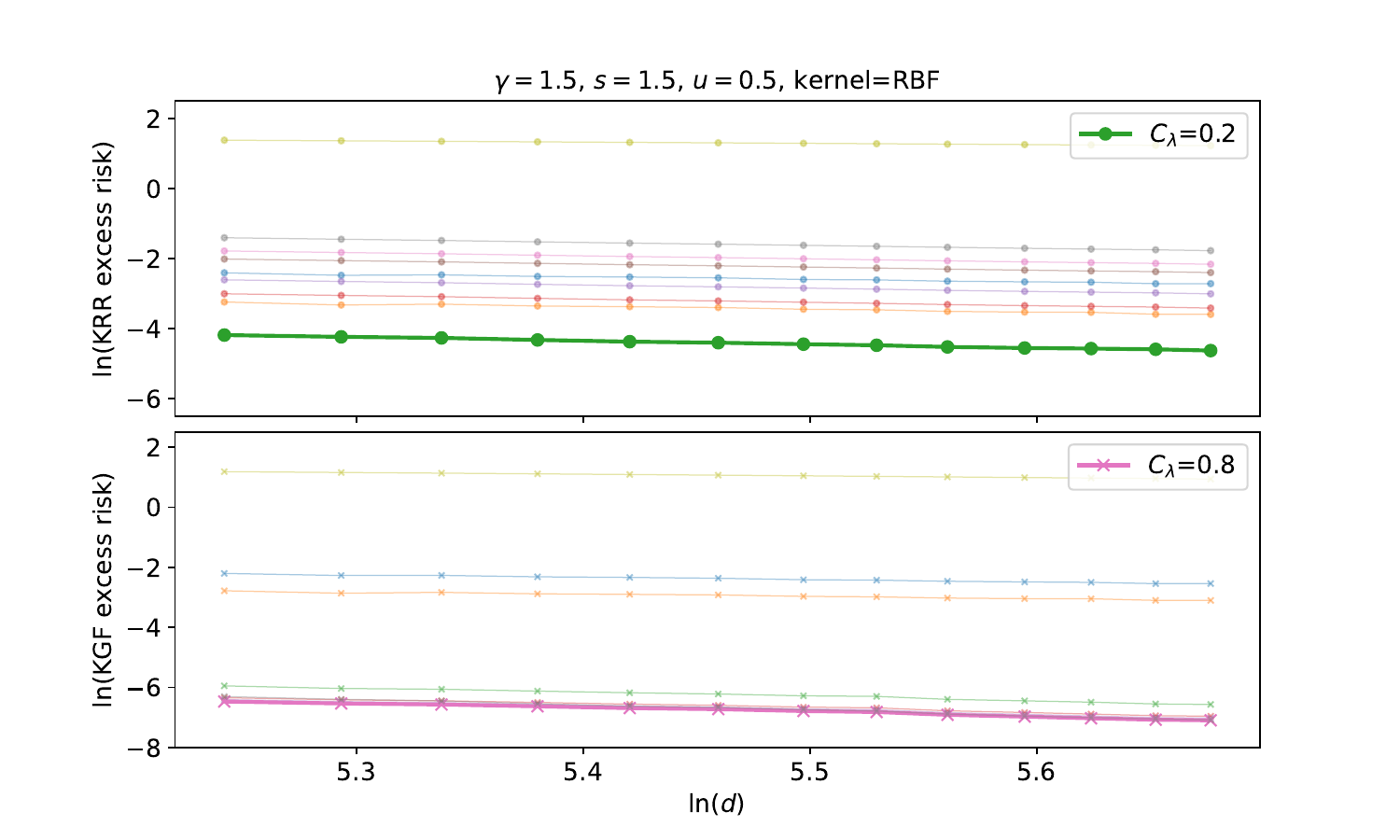}}

\caption{Complete results for Type 1 experiments with $(\gamma, s, u) = (1.5, 1.5, 0.5)$.}
\label{experiment_1_settings_1_all}
\end{figure}

\begin{figure}
\centering

\subfigure{\includegraphics[width=0.9\columnwidth]{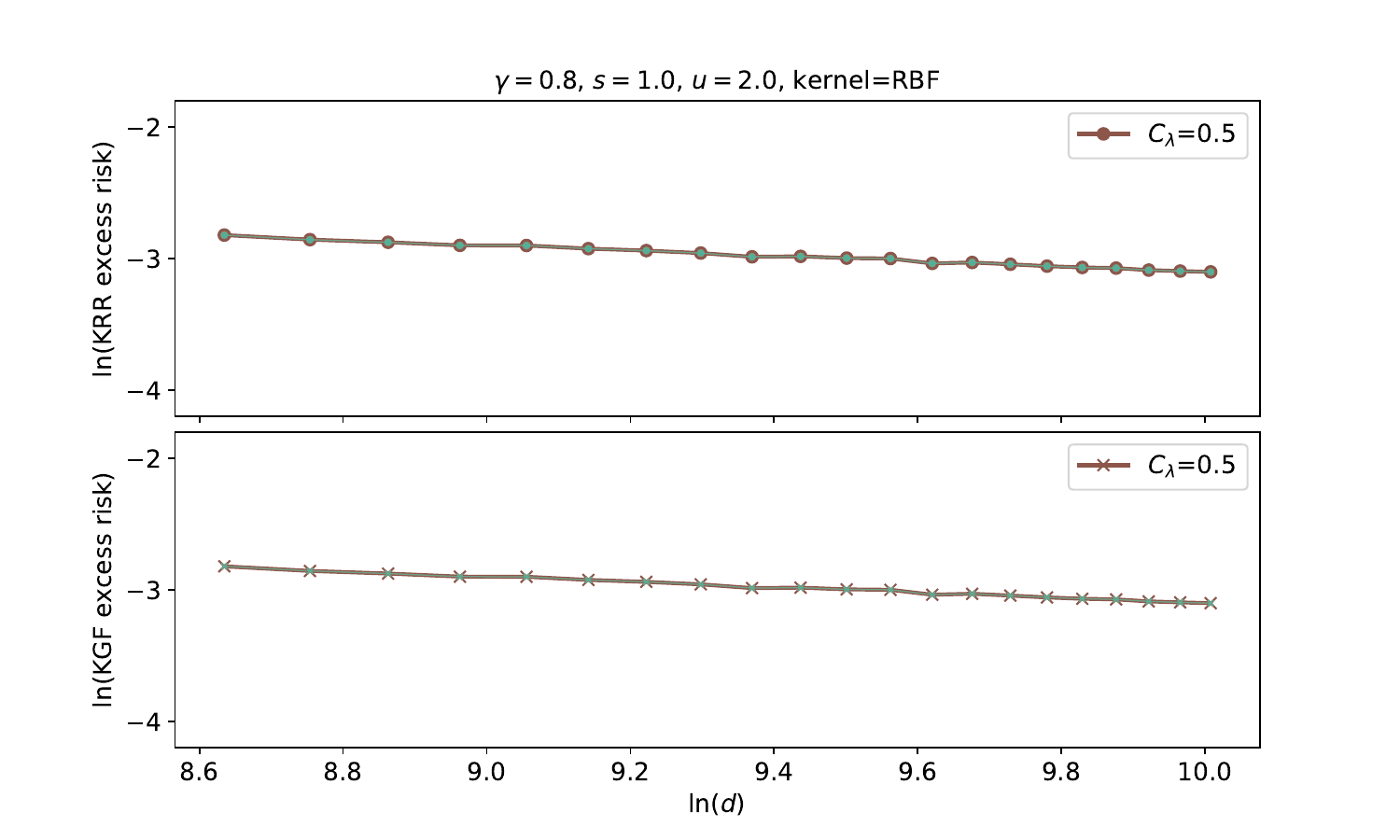}}

\caption{Complete results for Type 1 experiments with $(\gamma, s, u) = (0.8, 1.0, 2.0)$.}
\label{experiment_1_settings_2_all}
\end{figure}

\begin{figure}
\centering

\subfigure{\includegraphics[width=0.9\columnwidth]{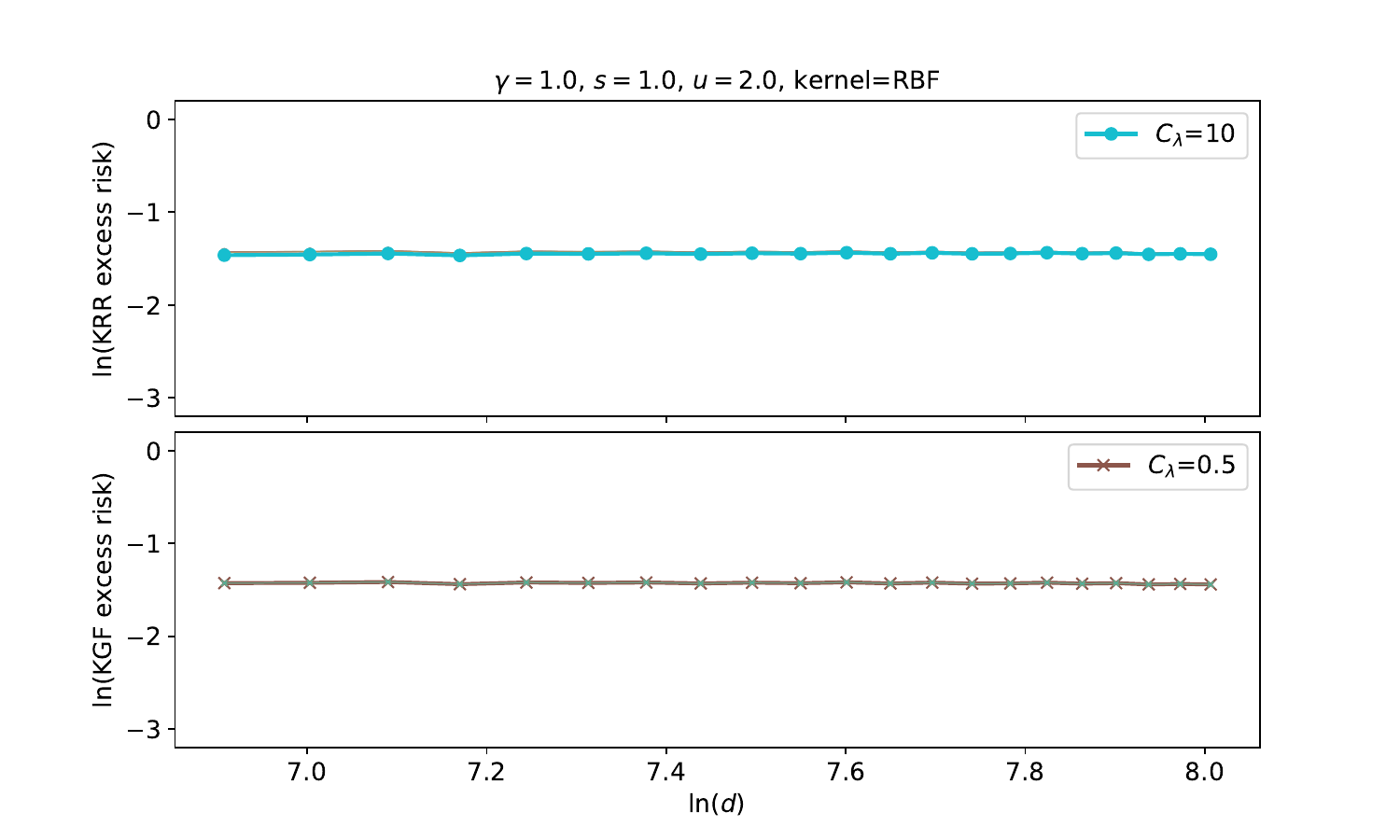}}

\caption{Complete results for Type 1 experiments with $(\gamma, s, u) = (1.0, 1.0, 2.0)$.}
\label{experiment_1_settings_3_all}
\end{figure}

\begin{figure}
\centering

\subfigure{\includegraphics[width=0.9\columnwidth]{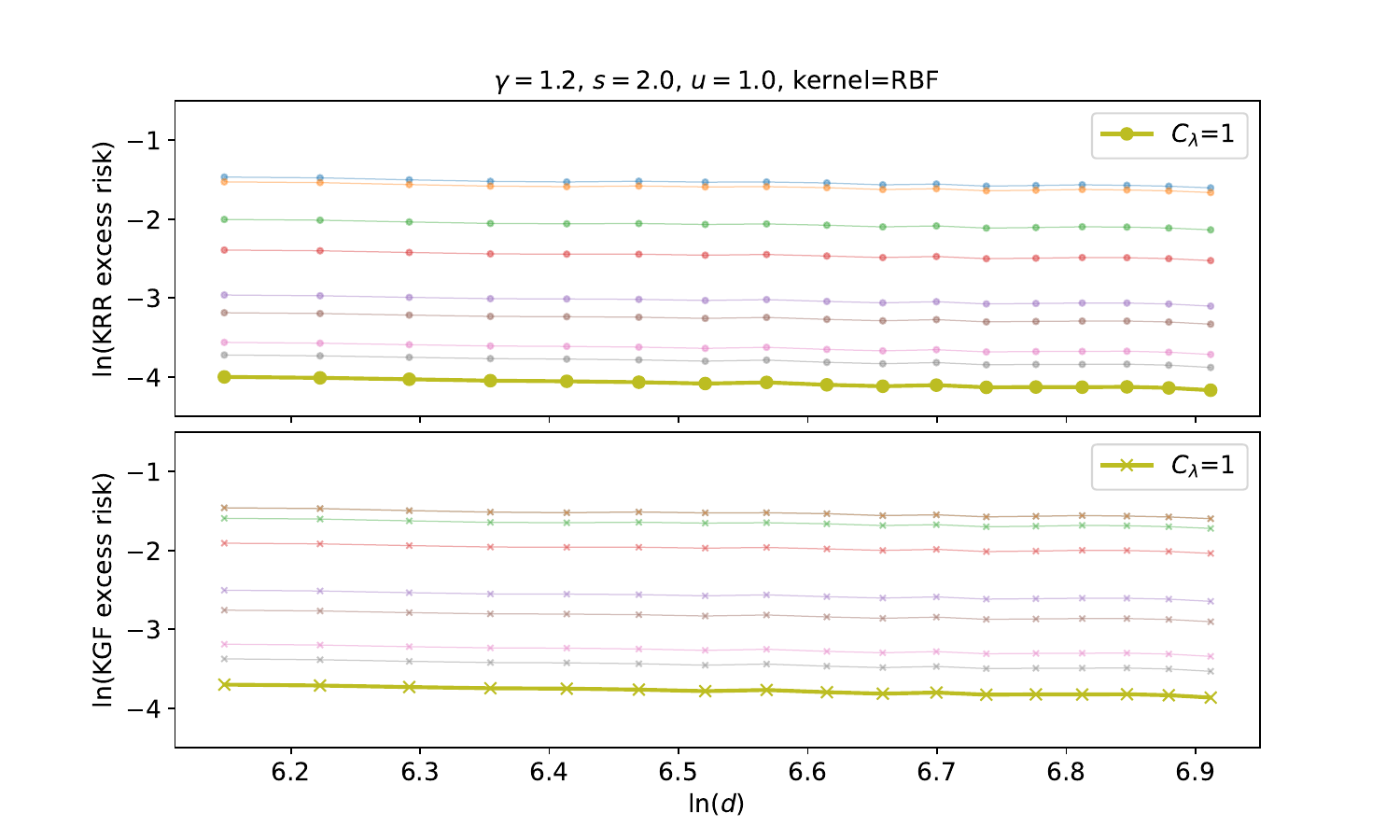}}

\caption{Complete results for Type 1 experiments with $(\gamma, s, u) = (1.2, 2.0, 1.0)$.}
\label{experiment_1_settings_4_all}
\end{figure}

\clearpage

\end{appendix}

\end{document}